\documentclass{article}
\usepackage[utf8]{inputenc}
\usepackage[margin=1in]{geometry}

\usepackage{macros}
\usepackage{url}

\usepackage[mathlines]{lineno}
\usepackage[normalem]{ulem}

\usepackage[symbol]{footmisc}

\title{Active Learning for Optimal Intervention Design
in Causal Models}
\author[1,2]{Jiaqi Zhang}
\author[2,3,4]{Louis Cammarata} 
\author[1,2,4]{Chandler Squires}
\author[1,*]{Themistoklis P. Sapsis}
\author[1,2,*]{Caroline Uhler}

\affil[1]{Massachusetts Institute of Technology}
\affil[2]{Broad Institute of MIT and Harvard}
\affil[3]{Harvard University}
\affil[4]{Equal contributions}
\affil[*]{Corresponding authors: culer@mit.edu, sapsis@mit.edu}
\date{May 13, 2023}

\begin{document}

\maketitle

\begin{abstract}
Sequential experimental design to discover interventions that achieve a desired outcome is a key problem in various domains including science, engineering and public policy. When the space of possible interventions is large, making an exhaustive search infeasible, experimental design strategies are needed. In this context, encoding the causal relationships between the variables, and thus the effect of interventions on the system, is critical for identifying desirable interventions more efficiently. Here, we develop a causal active learning strategy to identify interventions that are optimal, as measured by the discrepancy between the post-interventional mean of the distribution and a desired target mean. The approach employs a Bayesian update for the causal model and prioritizes interventions using a carefully designed, causally informed acquisition function. This acquisition function is evaluated in closed form, allowing for fast optimization. The resulting algorithms are theoretically grounded with information-theoretic bounds and provable consistency results \rev{for linear causal models with known causal graph}. We apply our approach to both synthetic data and single-cell transcriptomic data from Perturb-CITE-seq experiments to identify optimal perturbations that induce a specific cell state transition. The causally informed acquisition function generally outperforms existing criteria allowing for optimal intervention design with fewer but carefully selected samples. 
\end{abstract}

\section{Introduction}

An important problem across multiple disciplines, ranging from bioengineering to mechanical systems, operations research and environmental regulation is the discovery of interventions on a system that can produce a desired outcome. 
With little prior knowledge of the outcome before performing the intervention, the number of possible choices for the optimal design can be huge.
In particular, the interventions in many applications are combinatorial, resulting in an exponential size design space, making an exhaustive search infeasible.
%
%
Examples include experimental design of genetic perturbations, such as those for cellular reprogramming in regenerative medicine \cite{cherry2012reprogramming}, optimal feedback control in mechanical systems \cite{todorov2002optimal} as well as turbulent flows \cite{blanchard2021bayesian}, dynamic pricing strategies in customer networks \cite{sunar2019optimal} and iterative intervention research for climate change adaption~\cite{serrao2013climate}.

In this context, \emph{active learning} has been proposed as a machine learning strategy to efficiently explore the search space \cite{fu2013survey}. 
Such methods sequentially and strategically acquire new interventions, with the goal being to discover a (close to) optimal intervention using the fewest number of samples.
\rec{{While it can be of interest to identify optimal interventions for \emph{estimating} particular quantities in the model (e.g., \cite{jesson2021causal}), in this work we consider optimality of an intervention with respect to \emph{optimizing} its effect.}}
Specifically, this is done by successively (i) updating the model belief using samples acquired so far from different interventions; (ii) selecting the next intervention to obtain samples from by constructing and optimizing an acquisition function, which prioritizes interventions that are more informative for the desired outcomes; see Fig.~\ref{fig:1} for a schematic.

Standard approaches towards this problem are correlation-based.
More precisely, the idea is to estimate associational relations between intervention and outcome to update the model belief and make decisions about which samples to acquire next.
The two main approaches use either statistical theory, by minimizing the posterior variance of the outcome estimate~\cite{cohn1996active}, or information theory, by maximizing the mutual information between its samples and the quantity of interest~\cite{houlsby2011bayesian}. 
While these methods are being widely applied, correlation-based approaches are not optimal when the underlying model is causal, since they do not take into account the structural information that can reduce the number of feasible models.
 Many systems that are relevant for applications are causal, where an intervention can only have an impact on downstream variables. It is therefore of interest to develop methods that learn what is necessary about the underlying causal mechanisms in order to identify optimal interventions more quickly.

These limitations of correlation-based approaches have been noted in the related bandit setting, where the goal is to minimize the cumulative regret by selecting arms iteratively, and causal relations have been used to improve over standard regret bounds \cite{lattimore2016causal, lee2018structural}.
{Recent works in the related field of Bayesian optimization have also considered exploiting causal structure \cite{aglietti2020causal,branchini2022causal}. We provide a detailed review of these works in Supplementary Information~\ref{sec:a0}.
Our work extends these works in two ways: (1) rather than optimizing a single target
node, we optimize the entire distribution mean; (2) rather than considering discrete or finite interventions, we consider continuous-valued interventions.
This is important for various applications, such as optimizing drug dosages or product prices.}
%
%
%
%
As a concrete example, consider cellular reprogramming in genomics, a problem of great interest for regenerative medicine \cite{cahan2014cellnet,cherry2012reprogramming}. 
The aim in this field is to reprogram easily accessible cell types into a desired cell type via continuous-valued interventions such as the over-expression of particular transcription factors/genes. 
Since such interventions act on genes, which regulate each other through different pathways~\cite{kemmeren2014large}, this problem can be formulated as optimal intervention design in a causal model represented by a directed network on genes.

In particular, we model the underlying causal model in a Bayesian way using a structurally informed prior. 
We then construct an acquisition function based on this model and show how to efficiently evaluate and optimize it. 
The acquisition function enjoys both an information-theoretic bound and provably recovers the optimal intervention in the appropriate limit.
We demonstrate experimentally that our algorithms outperform baselines on both synthetic data and for the design of genetic perturbations in the context of single-cell gene expression data.
Finally, we conclude with an outlook to future research directions and discuss other potential applications.

\begin{figure}[t]
\centering
\includegraphics[width=.6\linewidth]{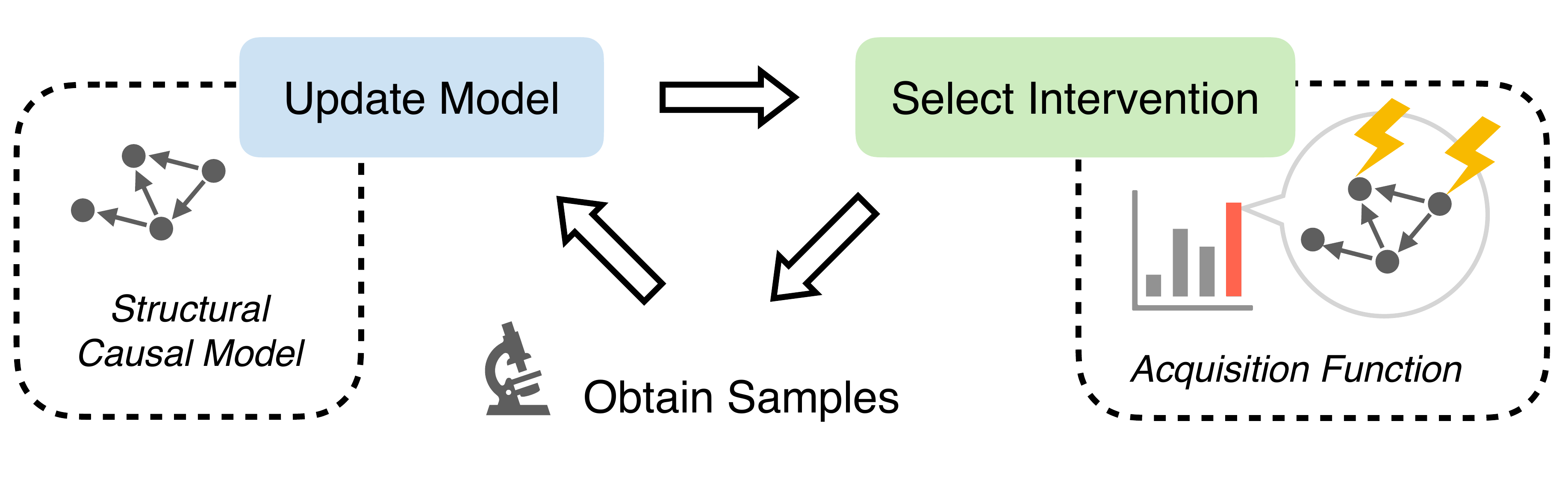}
\caption{\rec{\textbf{Overview schematic of active learning framework for optimal intervention design in causal models.}} Iterative process of active learning for intervention design in causal models, where the main design steps are to update the structural causal model using the obtained samples and to select the next intervention based on an acquisition function.}
\label{fig:1}
\end{figure}

\section{Problem Setup}

The state of the system of interest is described by a $p$-dimensional random variable $\bx=(\sfx_1,...,\sfx_p)\in \bbR^{p}$ sampled from a Structural Causal Model (SCM) \cite{spirtes2000causation, pearl2009causality}.
Precisely, the \textit{causal structure} of the system is represented by a Directed Acyclic Graph (DAG), and we assume that the joint distribution $\rmP$ of $\bx$ factorizes with respect to the DAG, i.e., $\rmP(\bx) = \prod_{i=1}^{p} \rmP(\sfx_i\mid\sfx_{\pa(i)})$, where $\pa(i)=\{j\in [p]: j\rightarrow i\}$ denotes the parents of node $i$ in the DAG.
In this formulation, the conditional distributions are the \textit{causal mechanisms} that generate the variable $\sfx_i$ from its parents $\sfx_{\pa(i)}$. 
Assuming a linear Gaussian model, then
\begin{equation}\label{eq:1}
    \sfx_i = \sum_{k=1}^{p} B_{ik} \sfx_{k} + \epsilon_i, \quad \forall i \in [p],
\end{equation}
where the real-valued coefficients $B_{ik}=0$ if $k\notin\pa(i)$ and the exogenous noise variables $\epsilon_i\sim \cN(0, \sigma_{i}^2)$ with variance $\sigma_i^2 >0$ are mutually independent.
An example of this model is given in Fig.~\ref{fig:2}. 
For simplicity, we assume that the system is centered to be mean zero, but intercepts can be easily added to Eq.~\eqref{eq:1} and the following approaches still apply.

An intervention, denoted by a vector $\ba \in \bbR^p$, modifies the conditional distribution $\rmP(\sfx_i \mid \sfx_{\pa(i)})$ into a new conditional distribution $\rmP^{\ba}(\sfx_i\mid\sfx_{\pa(i)})$. Every $i\in[p]$ for which $a_i\neq 0$ is called an \textit{intervention target}. 
In this work, we consider shift interventions \cite{rothenhausler2015backshift,zhang2021matching} {(which are a special class of soft interventions \cite{eberhardt2007interventions}),} where \rev{the interventional distribution $\rmP^\ba(\bx)$ of the modified system under intervention} is given by
\begin{equation}\label{eq:2}
    \sfx_i = \sum_{k=1}^{p} B_{ik} \sfx_{k} + a_i+ \epsilon_i, \quad \forall i \in [p],
\end{equation}
where $a_i=0$ for every $i$ that is not an intervention target. This can be written in matrix form as $\bx = \bB\bx+\ba+\bepsilon$, or $\bx=(\bI-\bB)^{-1}(\ba+\bepsilon)$, where $\bI$ is the $p$-th order identity matrix and $\bepsilon\sim\cN(\bzero,\bSigma)$ with $\bSigma$ being the diagonal matrix with $\sigma_1^2,...,\sigma_p^2$ on its diagonal. \rev{These interventions can be used to model a broad class of genetic perturbations including CRISPR interference/activation \cite{shalem2015high} and transcription factor overexpression \cite{joung2023transcription}.} \rec{{While we consider the case where all variables can be intervened on, as e.g.~in genetic perturbation experiments~\cite{replogle2022mapping}, for other applications it may be of interest to consider extensions where only a subset of variables can be intervened on. We also note that an extension to hard interventions can be obtained by setting $B_{ik}$ to $0$ for every $i$ that is an intervention target and subsequently defining the post-interventional mean.}}

\begin{figure}[t]
\centering
\includegraphics[width=.3\linewidth]{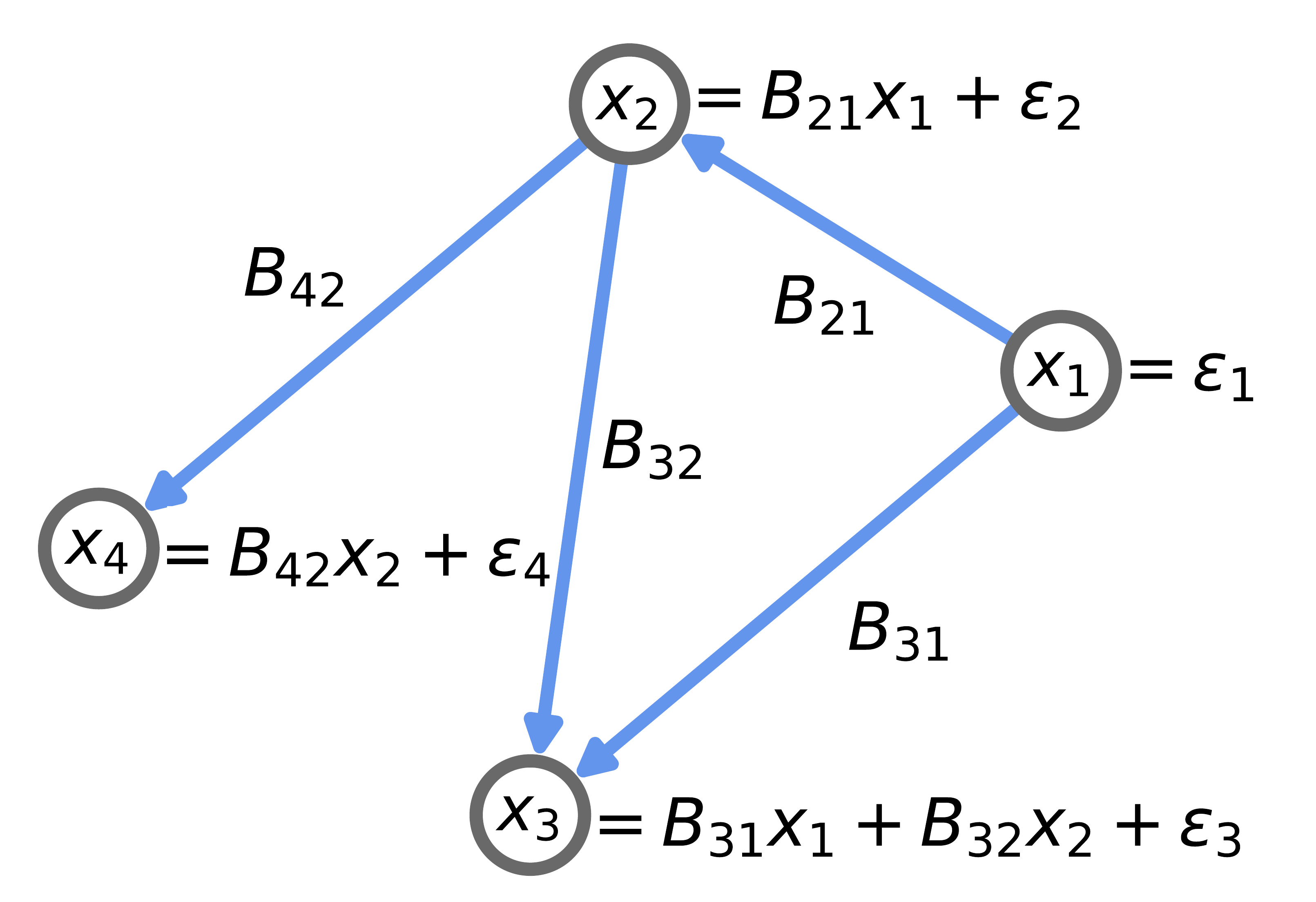}
\caption{\rec{\textbf{Example causal model.}} An example of a linear structural causal model (SCM) with Gaussian noise on a 4-node DAG, where nodes and edges are labeled with variables and coefficients, respectively.}
\label{fig:2}
\end{figure}

Whether or not an intervention induces the desired outcome is decided through samples obtained in the intervened environment, i.e., $\bx\sim\rmP^\ba$. 
Since $\bx$ is random, it is intuitive to use the average over multiple samples \cite{sen2017identifying,zhang2021matching}, i.e., the empirical estimate of the distribution mean, and compare this to the \textit{target mean} $\bmu^*$, which is user-specified and describes the desired outcome. 
Formally, we seek an intervention $\ba\in\mathbb{R}^p$ such that, after obtaining $n$ samples $\bx^{(1)}, \dots , \bx^{(n)}\in \mathbb{R}^p$, which we denote by $\bx_{[n]}$, the squared distance between the empirical mean and the target mean, i.e., the quantity $\|\frac1n \sum_{m=1}^n \bx^{(m)} -\bmu^*\|_2^2$, is minimized. 
Denoting by $\bepsilon^{(1)}, \dots , \bepsilon^{(n)}$ (or short $\bepsilon_{[n]}$) the exogenous noise vectors independently sampled from $\cN(\bzero,\bSigma)$, then the squared distance between the empirical and target mean can be written as
\begin{equation}\label{eq:3}
    \Big\|(\bI-\bB)^{-1}\big(\ba+\frac1n\sum_{m=1}^n \bepsilon^{(m)}\big)-\bmu^*\Big\|_2^2.
\end{equation}
using the matrix form. Here $\frac1n \sum_{m=1}^n\bepsilon^{(m)}$ accounts for
the finite number $n$ of interventional samples, where $n$ is user-specified based on the available budget. 
However, since $\bepsilon_{[n]}$ does not depend on the choice of $\ba$, we can discard this term for the minimization of Eq.~\eqref{eq:3} with respect to $\ba$ by considering its infinite-sample version. 
When the sample size goes to infinity, by the law of large numbers, $\frac1n\sum_{m=1}^n \bepsilon^{(m)}=0$ almost surely, 
and thus the optimal intervention $\ba^*$ achieves the minimum value of zero in Eq.~\eqref{eq:3} and has an explicit form $\ba^*=(\bI-\bB)\bmu^*$, which as expected depends on the unknown parameter $\bB$.

In what follows, we assume that the DAG structure is given,
i.e., the sparsity pattern of $\bB$ is known but not the edge weights $B_{ik}$ for $B_{ik}\neq 0$.
This assumption is natural in many applications including fluid mechanics~\cite{koumoutsakos1995high} and optimal pricing~\cite{sunar2019optimal}, where there are 
pre-specified networks given by either known laws or prior information. 
For applications where the network cannot be assumed to be known such as problems in biology, the common approach is to use existing data from various sources to learn the DAG first \cite{cahan2014cellnet,rackham2016predictive}.
Further discussions on the implications of unknown DAG structure are provided in Supplementary Information~\ref{sec:d}.

Let $\cD_{t} = \{ (\bx_{[n_1]}, \ba^{(1)}), \dots , (\bx_{[n_t]}, \ba^{(t)})\}$ be the ``current'' dataset, consisting of all samples obtained so far by performing interventions $\ba^{(1)}, \dots , \ba^{(t)}$;
here, $n_1,...,n_{t}$ denotes the number of samples obtained for each of the $t$ interventions.
To simplify notation, we will assume $n_1=\cdots=n_t=n$ in the following, but all results still hold when using different sample sizes for each intervention.
Given this dataset, the goal is to select the next intervention $\ba^{(t+1)}\in\mathbb{R}^p$ such that the resulting dataset $\cD_{t+1}=\cD_t \cup (\bx_{[n]}, \ba^{(t+1)})$ contains as much information as possible about the underlying optimal intervention $\ba^*$.
The overall aim of iteratively picking interventions is to find the optimal intervention $\ba^*$ with a minimum number of samples.

\section{Design of the Acquisition Function}

To sequentially select the next best intervention, there are two important steps (see Fig.~\ref{fig:1} for a schematic): (i) updating the posterior of the edge weights in the causal model based on the samples in $\cD_t$ collected so far; (ii) constructing an \textit{acquisition function} $h(\ba,\cD_t)$ such that the new dataset $\cD_{t+1}$ after adding samples from $\ba^{(t+1)}=\argmin_{\ba}h(\ba,\cD_t)$ is most informative of the optimal intervention $\ba^*$ and that can be evaluated and optimized efficiently. 

{For the first step, we generalize the \emph{DAG-Wishart} distribution \cite{geiger2002parameter,kuipers2022interventional} to define a Bayesian model on the parameters $\bB$, which can be updated efficiently given a dataset $\cD_t$; see \emph{Methods}.} 
%
%
{For the second step}, we first characterize the uncertainty in estimating the optimality of an arbitrary intervention $\ba$. 
Recall that the optimality of $\ba$ is given by the square distance in Eq.~\eqref{eq:3}, which measures how close the intervention is to achieving the target mean. 
Since $\bB$ is unknown, we can only estimate this square distance based on the current collected samples in $\cD_t$. 
The uncertainty of the estimation can be characterized by its variance
$\Var(\|(\bI-\bB)^{-1}(\ba+\frac1n \sum_{m=1}^n \bepsilon_m) - \bmu^*\|_2^2| \cD_t)$.
However, this quantity is typically hard to evaluate as it involves $(\bI - \bB)^{-1}$, whose posterior does not have a closed form.
We instead multiply the term inside the variance by $\bI-\bB$ and characterize the following variance:
\begin{equation}\label{eq:6}
    \sigma_{g(\ba)|\cD_t}^2 := \Var\big(g(\ba)\big| \cD_t\big),
\end{equation}
where we define
\begin{equation}
g(\ba):=\Big\|\big( \ba+\frac1n \sum_{m=1}^n \bepsilon^{(m)} \big) - (\bI-\bB) \bmu^*\Big\|_2^2.
\end{equation}
Interestingly, $g(\ba)$ can be interpreted as a noisy version of the \textit{optimality gap} by noting that the gap between an arbitrary intervention $\ba$ and the optimal intervention $\ba^*=(\bI-\bB)\bmu^*$ can be written as $\|\ba-\ba^*\|_2^2=\|\ba- (\bI-\bB) \bmu^*\|_2^2$, which is a version of $g(\ba)$ without noise terms.

Building upon Eq.~\eqref{eq:6}, the next acquired intervention should be such that after adding its samples to $\cD_{t}$ to obtain $\cD_{t+1}$, the uncertainty $\sigma_{g(\ba)|\cD_{t+1}}^2$ conditioned on $\cD_{t+1}$ is minimized. 
However, since the samples are unobserved before performing the intervention, we do not have access to $\cD_{t+1}$ yet. 
Therefore, when deciding which intervention to perform, we consider the $\ba'$-augmented dataset $\cD_{t}(\ba') = \cD_t \cup (\bar{\bx}'_{[n]}, \ba')$ with \textit{hypothetical samples} $\bar{\bx}'_{[n]}$ of $\ba'$, which are $n$ repetitions of the plug-in estimator defined as $\bar{\bx}':=(\bI-\bbE(\bB|\cD_t))^{-1}\ba'$. \rec{This estimator is obtained through $\bx'=(\bI-\bB)^{-1}(\ba'+\bepsilon)$ by using the Maximum A Posteriori estimate $\bbE(\bB|\cD_t)$, which concentrates to $\bB$ as $\cD_t$ grows \cite{kleijn2012bernstein}, and replacing $\bepsilon$ with its mean $\bzero$.} 

By denoting the feasible set of interventions by $\cA$, a reasonable choice is to select $\ba^{(t+1)}$ based on integrating the uncertainty $\sigma_{g(\ba)|\cD_{t}(\ba_t)}^2$ over all $\ba\in\cA$.
Formally, we let 
\begin{equation}\label{eq:7}
\ba^{(t+1)}=\argmin_{\ba'\in\cA}h(\ba',\cD_t),    
\end{equation}
and define the \textit{causal integrated variance} (CIV) acquisition function $h$ as follows.

\begin{definition}\label{def:2}
The CIV acquisition function evaluated at $\ba'$ with current dataset $\cD_t$ is
\begin{equation}\label{eq:8}
\begin{aligned}
    h(\ba', \cD_t) = \int_{\cA} \sigma_{g(\ba)|\cD_t(\ba')}^2~d\nu(\ba),
\end{aligned}
\end{equation}
where $\nu$ is a non-negative measure on $\cA$.
\end{definition}

Intuitively, this acquisition function provides a one-step look-ahead of the overall uncertainty after acquiring intervention $\ba'$. 
Minimizing it will prioritize interventions that are most informative towards estimating the optimal intervention.
This acquisition function also automatically accounts for the causal model by using a posterior on $\bB$.

Note that in this formulation, we can choose the measure $\nu$.
For example, a uniform measure treats each intervention $\ba\in\cA$ equally and the resulting uncertainty captures how well we can estimate the entire landscape of the optimality gap $g$.
%
In most cases, an overly concentrated measure (e.g., a Dirac measure at a single point) is not preferred, as it can lead to the erroneous estimation of $g(\ba)$ for most $\ba\in\cA$, which makes minimizing $g$ hard.
%
{Inspired by a recent line of work on output-weighted acquisition functions~\cite{sapsis2020output, mohamad2018sequential}, we describe how to choose a non-uniform measure $\nu$ in Section~\ref{sec_nu}.}

{We discuss how to optimize CIV to solve for $\ba^{(t+1)}$ in \emph{Methods}, where we show that the variance $\sigma_{g(\ba)|\cD_t(\ba')}^2$ can be computed in closed form. Considering $\cA$ to be the unit hypersphere and $\nu$ to be the uniform measure on it, this then leads to an explicit formula for CIV, which enables fast gradient-based optimizers to be used.
}

\subsection{Making the Acquisition Function Output-Weighted}
\label{sec_nu}

While the use of a uniform measure $\nu$ places an equal weight on reducing the variance of estimating the optimality gap for all $\ba$ in $\cA$, since our goal is to identify the optimal intervention (i.e., which minimizes the optimality gap), it is desirable to place more weight on interventions $\ba$ in $\cA$ with smaller optimality gap. Note that as the ambient dimension grows, the volume (and thus the probability) of interventions with optimality gap under a certain threshold shrinks (Supplementary Information~\ref{sec:c}). This motivates the following measure, which uses the inverse of the optimality gap probability to up-weight interventions $\ba$ in $\cA$ that are closer to the optimal intervention:
\begin{equation}\label{eq:12}
    d\nu(\ba) = \frac{f_{\ba}(\ba)}{f_{\|\ba-\bb\|_2^2}(\|\ba-\bb\|_2^2)}d\ba,
\end{equation}
where $\bb=(\bI - \bbE(\bB|\cD_t))\bmu^*$ is the estimated optimal intervention, and the probability density function (pdf) $f_{\|\ba-\bb\|_2^2}$ is the distribution on the optimality gap $\|\ba-\bb\|_2^2$ induced by the uniform distribution $f_{\ba}$ on $\cA$. 

\begin{figure}[ht]
\centering
\includegraphics[width=.6\linewidth]{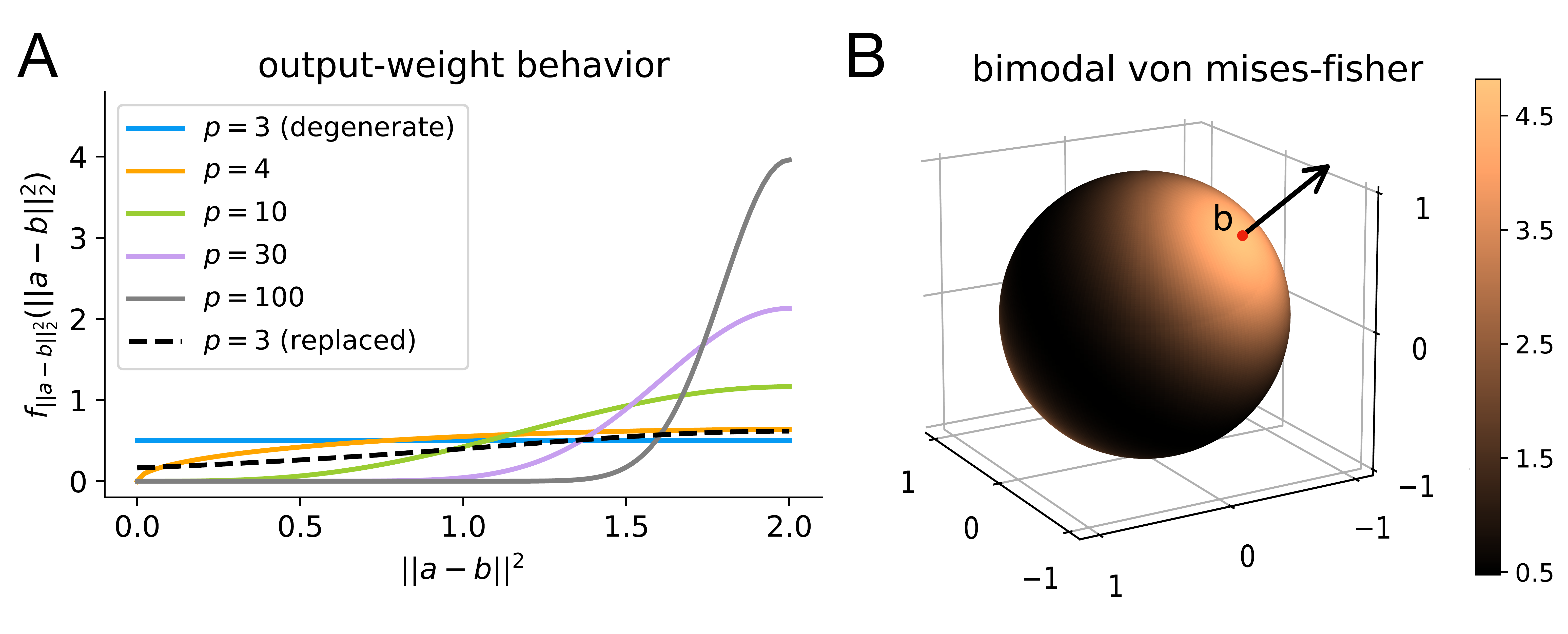}
\caption{\textbf{Illustration of output-weighted non-negative measure $\nu(\ba)$ on the space of all possible interventions $\cA$.} (A) Values of $\nicefrac{d\ba}{d\nu(\ba)}$ (up to constant multipliers) plotted against values of $\|\ba-\bb\|_2^2$ on the half sphere with $\|\ba-\bb\|_2^2\leq 2$. Solid lines correspond to $f_{\|\ba-\bb\|_2^2}$ in different dimensions; the dashed line is the non-degenerate replacement for $p=3$ using the bimodal von Mises-Fisher distribution.  (B) Visualization of $\nu(\ba)$ in 3 dimensions using the bimodal von Mises-Fisher distribution. Here, $\bb\in\bbS^2$ was randomly generated and the points on the sphere are colored corresponding to the value of $\nicefrac{d\nu(\ba)}{d\ba}$, which is larger for directions that are more aligned with $\bb$.}
\label{fig:3}
\end{figure}

Fig.~\ref{fig:3}A shows that, as desired, the proposed weighting puts more mass on interventions that are closer to $\bb$ as dimension increases. 
When $p=3$, however, $\nu$ degenerates to a uniform weighting (Supplementary Information~\ref{sec:c}). 
In this case, we can use the bimodal von Mises-Fisher distribution, which behaves similar to Eq.~\eqref{eq:12} in higher dimensions (Fig.~\ref{fig:3}B, Supplementary Information~\ref{sec:c}). 
Note that the weighting proposed here is symmetric and also puts more mass on interventions that are closer to $-\bb$. The reason for this is that the optimal intervention can also be recovered by maximizing the optimality gap $\|\ba-\ba^*\|_2$ which gives $-\ba^*$ (Supplementary Information~\ref{sec:c}).

The corresponding \textit{causal output-weighted integrated variance} (CIV-OW) acquisition function is given by
\begin{equation}\label{eq:civ-ow}
    \begin{aligned}
    h_{\textrm{OW}}(\ba',\cD_t) = \int_{\cA} \sigma^2_{g(\ba)|\cD_t(\ba')}\cdot\frac{f_\ba(\ba)}{f_{\|\ba-\bb\|_2^2}(\|\ba-\bb\|_2^2)}~d\ba.
    \end{aligned}   
\end{equation}
Methods for evaluating and optimizing CIV-OW are given in Supplementary Information~\ref{sec:c}, Supplementary Fig.~\ref{fig:s1}-\ref{fig:s2}.

\subsection{Theoretical Results}

We provide two interpretations of the introduced CIV acquisition function. First, taking an information-theoretic perspective, we show that the proposed uncertainty measure can be lower bounded by the \emph{negative} mutual information between the variables of interest and the newly acquired samples.
Thus, minimizing the uncertainty corresponds to maximizing a lower bound to the mutual information, which means that CIV {\emph{approximately}} prioritizes the most informative interventions.
Second taking a graphical perspective, we illustrate how the causal structure is utilized by the CIV acquisition function to identify an intervention that is asymptotically consistent with the optimal intervention. 
{While the details underlying these theoretical results are provided in \emph{Methods}, Fig.~\ref{fig:4} illustrates our consistency results experimentally, showing that a gradient-based optimizer with initialization close to the optimal intervention $\ba^*$ converges to $\ba^*$ as $t$ increases.}

\begin{figure}[!hb]
\centering
\includegraphics[width=.6\linewidth]{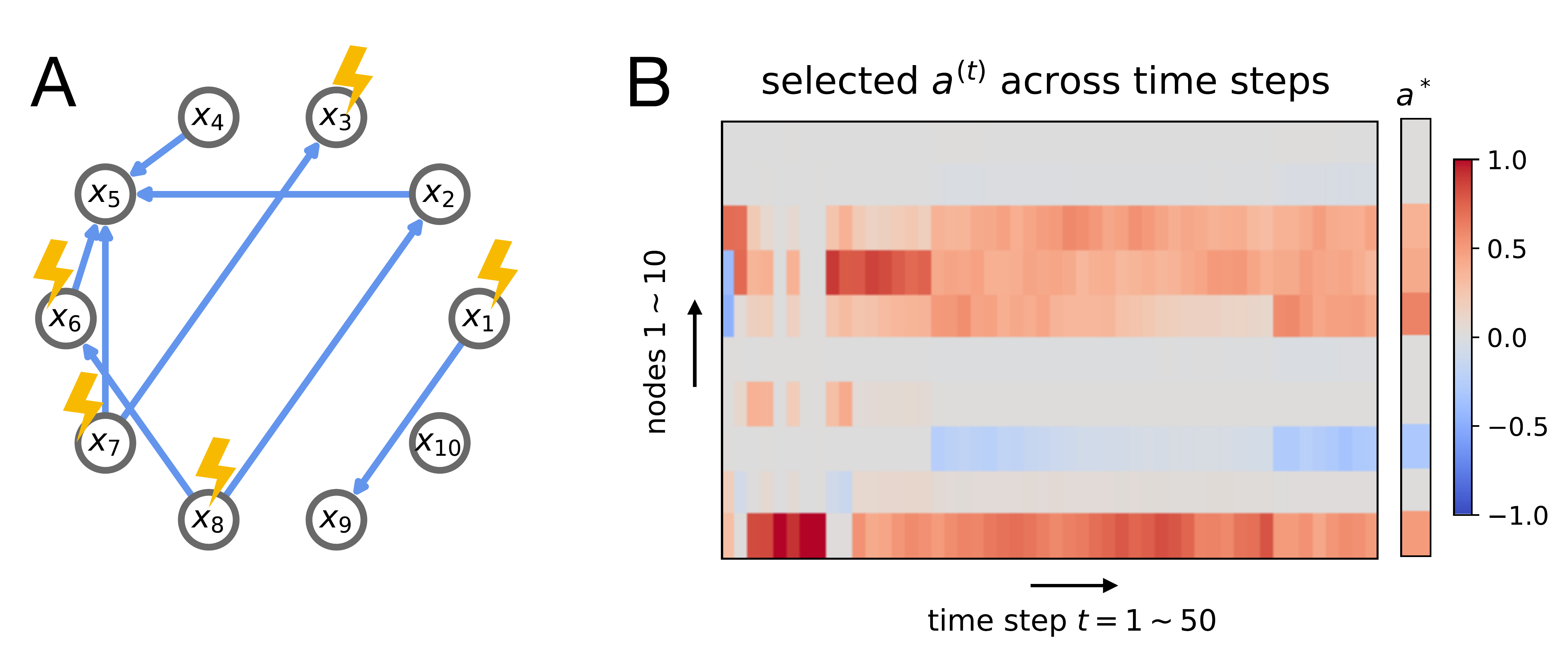}
\caption{\textbf{Convergence of the selected intervention $\ba^{(t)}$ to the optimal intervention $\ba^*$.} (A) DAG showing the targets of the optimal intervention $\ba^*$. (B) Minimizing the CIV acquisition function using a sample size of $n=1$ at each time step yields interventions $\ba^{(t)}$ for $t=1,\dots,50$ that converge to $\ba^*$.}
\label{fig:4}
\end{figure}

\section{Applications}


\subsection{Experiments on Synthetic Dataset}

To create a generative model following the linear SCM in Eq.~\eqref{eq:1}, we first generate a DAG with $p$ nodes (see Supplementary Information~\ref{sec:g}, Supplementary Fig.~\ref{fig:s3}).
Then, we randomly draw edge weights, namely, $B_{ik}$ with $k\in\pa(i)$, from a uniform distribution bounded away from zero. 
Next, we generate a sparse set of intervention targets and a randomly sampled optimal intervention $\ba^*$ over these targets.
Finally, we calculate the target mean $\bmu^*$ using $\ba^*$ and the ground-truth causal model. 
These steps construct a synthetic instance of a causal system and optimal intervention $\ba^*$. 
%
%
A more detailed description of the above procedure is given in \textit{Methods} and Supplementary Information~\ref{sec:g}.
%
%

We compare our two acquisition functions, CIV in Eq.~\eqref{eq:8} and CIV-OW in Eq.~\eqref{eq:civ-ow}, against four relevant baselines. 
The \emph{random} baseline correspond to a passive setting, where each intervention is selected at random and no information from the collected samples is used. 
We also compare against three other active methods.
The \emph{greedy} baseline selects the next intervention $\ba^{(t)}=(\bI-\bbE(\bB|\cD_t))\bmu^*$ purely based on the current estimate of $\ba^*$ (which is given by $\bbE(\bB|\cD_t)$, where $\bB$ is estimated from $\cD_t$). 
The \emph{MaxV} baseline seeks to select interventions that minimize the posterior variance of the estimate of the model parameters $\bB$. It uses as acquisition function a scalar version of this variance, namely $h_{\textrm{MaxV}}(\ba',\cD_t)=\max_{i\in[p]}\|\Var(\bB_{i,\pa(i)}|\ba',\cD_t)\|_2$, where $\|\cdot\|_2$ denotes the spectral norm of the covariance matrix of $\bB_{i,\pa(i)}$.
Since different rows of $\bB$ (e.g., $\bB_{i,\pa(i)}$ and $\bB_{j,\pa(j)}$ for $i\neq j$) are independent (Definition \ref{def:1}), we use the maximum of the spectral norms over $i\in[p]$. While our acquisition functions concentrate on estimating the model parameters that are relevant for $\ba^*$, this baseline estimates the entire model.
Finally, the \emph{CV} baseline seeks to minimize the posterior variance of estimating $\ba^*$ and uses the spectral norm of its covariance matrix as acquisition function, i.e., $h_{\textrm{CV}}(\ba',\cD_t)=\|\Var((\bI-\bB)\bmu^*|\ba',\cD_t)\|_2$. This is in contrast to our approach that integrates the estimation uncertainty over the entire feasible set of interventions.

To reduce evaluation noise effects, we run each method $20$ times over $10$ instances of a fixed DAG and optimal intervention $\ba^*$. Fig.~\ref{fig:5} shows our results for $30$-node DAGs with $10$ intervention targets and sample size $n=1$ per time step.
\rec{Denoting the obtained samples at time step $t$ by $\cD_t$, we use the unbiased estimate $\ba_t^*=(\bI-\bbE(\bB|\cD_t))\bmu^*$ of $\ba^*$ to obtain the current estimate of the target mean using the true model parameters, i.e., $\bmu^*_t=(\bI-\bB)^{-1}\ba_t^*$.}
Fig.~\ref{fig:5}A shows the decline in the {relative distance} between current and target mean {$\nicefrac{\|\bmu^*_t-\bmu^*\|_2}{\|\bmu^*\|_2}$} across time steps. 
Fig.~\ref{fig:5}B highlights the statistics of the last time step. As is apparent from these experiments, our proposed acquisition functions consistently outperform all baselines, with CIV-OW improving upon CIV by using an output-weighted measure. {In Supplementary Information~\ref{sec:g}, we examine the effect of varying different parameters including graph size, number of intervention targets, and graph structure, demonstrating the robustness of our results (Extended Data Fig.~\ref{fig:s4}-\ref{fig:s7}). We also include three additional baselines from prior works \cite{aglietti2020causal,astudillo2021bayesian,houlsby2011bayesian,bubeck2012regret} (discussed in detail in Supplementary Information~\ref{sec:g} and \ref{sec:a0}) on a $10$-node DAG (Extended Data Fig.~\ref{fig:rebut-s7}), and an experiment to examine the effect of misspecifying the underlying causal structure (Supplementary Fig.~\ref{fig:rebut-figs9} and Extended Data Fig.~\ref{fig:rebut-figs10}). These results further demonstrate the effectiveness of our proposed method in terms of both accuracy and efficiency.}

\begin{figure}[ht]
\centering
\includegraphics[width=.6\linewidth]{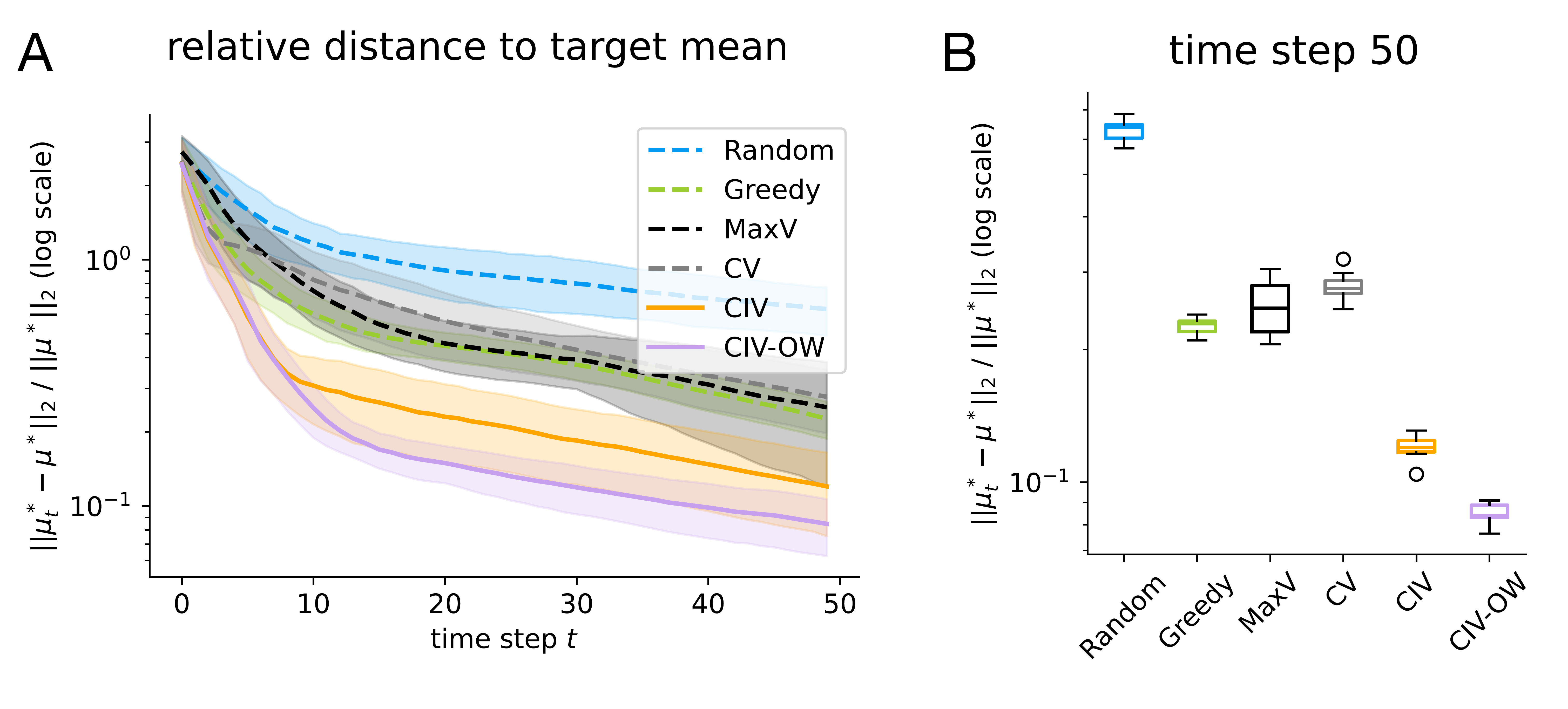}
\caption{\rec{\textbf{Comparison of different acquisition functions (random baseline, greedy baseline, baselines based on minimizing the posterior variance of estimating the model parameters (MaxV) or the optimal intervention target (CV), as well as our proposed Causal Integrated Variance (CIV) acquisition function and its output-weighted (CIV-OW) extension)  in a simulation study.} Simulations on} $10$ instances of a $30$-node DAG where the optimal intervention $\ba^*$ has $10$ targets. Each method is run $20$ times and averaged. (A) {Relative distance} between the target mean $\bmu^*$ and the best approximation $\bmu^*_t$ across time step $t$. Lines denote the mean over the $10$ instances; shading corresponds to one standard deviation. (B) {Relative distance} statistics of each method \rec{averaged over $10$ instances} at the last time step in (A).}
\label{fig:5}
\end{figure}

\subsection{Experiments on Biological Dataset}

We next study the performance of our method to identify the optimal intervention for inducing a desired cell state change in human melanoma cells, thereby mimicking a cellular reprogramming experiment. 
For this, we use Perturb-CITE-seq data from~\cite{frangieh2021multimodal} consisting of single-cell transcriptomic readouts for a large collection of patient-derived melanoma cells.
Here, an intervention is a genetic perturbation that targets one or multiple genes and drives the expression of these genes towards zero (more precisely, these are knockout interventions).
Cell states are measured by the joint distribution of the expression of a collection of genes. 
Samples from these distributions correspond to gene expression vectors of individual melanoma cells. To avoid dealing with batch effects, we use only one of the screens from \cite{frangieh2021multimodal} (namely, the control screen with no additional treatment) which contains $5,039$ cells with no perturbation and $30,486$ cells with interventions on subsets of $248$ genes associated with immunotherapy resistance (Supplementary Fig.~\ref{fig:s8}).
Gene expression of each sampled cell is captured as a vector of log-transformed Transcripts-Per-Million (log-TPM).

\begin{figure}[!t]
\centering
\includegraphics[width=.6\linewidth]{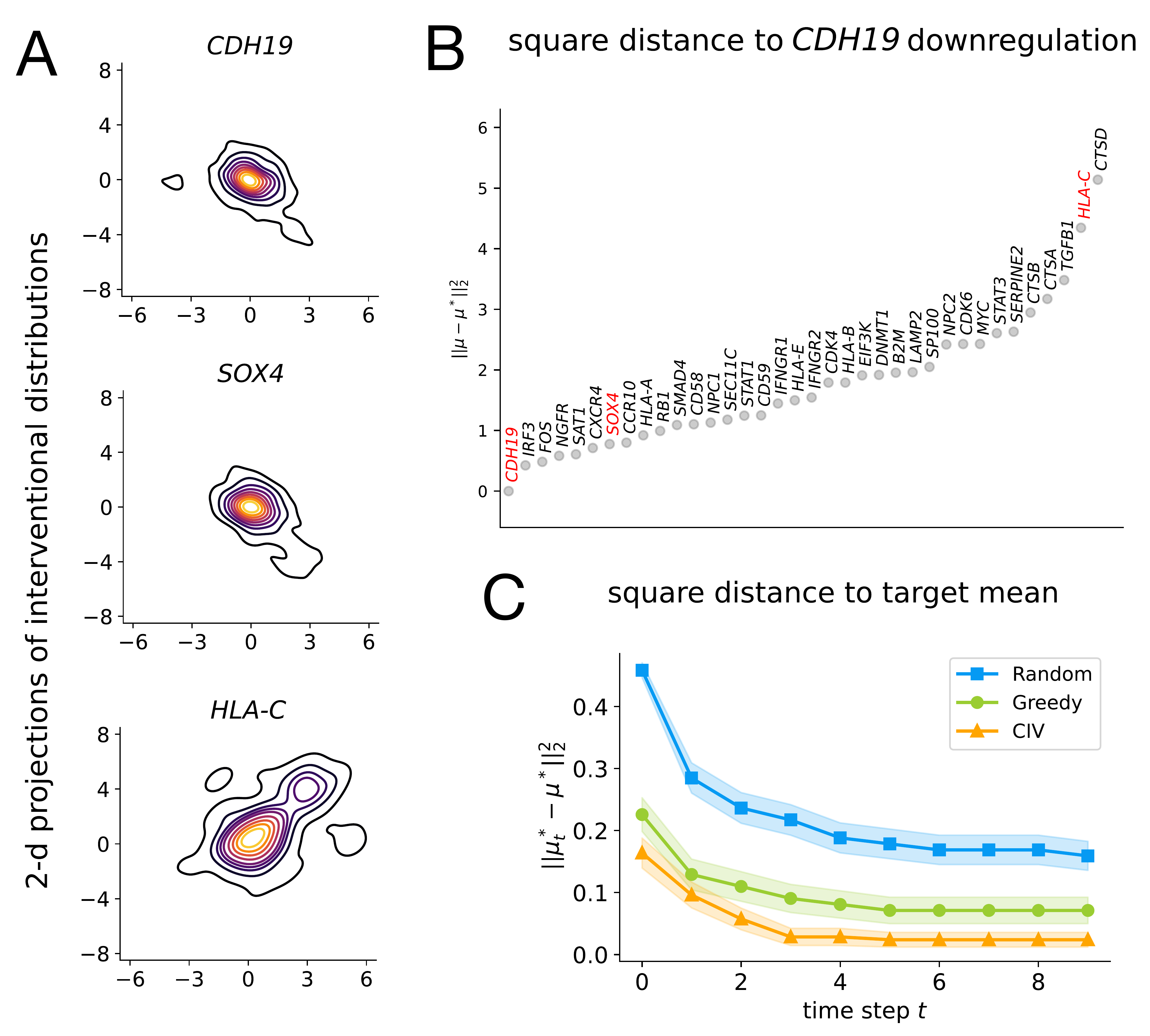}
\caption{\textbf{Results on perturbational single-cell gene expression dataset.} (A) Kernel density estimate (KDE) plot of the 2-d projection of the interventional distributions obtained by targeting 3 representative genes: \emph{CDH19}, \emph{SOX4} and \emph{HLA-C}. (B) Squared distances between the target mean of the optimal intervention targeting \emph{CDH19} and other interventions. (C) Comparison of different acquisition functions for identifying interventions that match the target mean \rec{presented as mean value +/- SEM}.}
\label{fig:6}
\end{figure}

For our optimal intervention design task, we focused on a particular functional context, namely the $p=36$ genes among the $248$ interventional targets that are involved in interferon-$\gamma$ signaling and immune response \cite{frangieh2021multimodal} (see Supplementary Information~\ref{sec:h}).
\rec{Since multi-target interventional samples are extremely scarce in this dataset, with most such interventions having no more than $1$ sample (Supplementary Fig.~\ref{fig:s9}), we only consider single-target interventions.}
We use the observational samples, i.e., the cells with no perturbation, to learn a DAG over these $36$ genes (see \textit{Methods}, Supplementary Fig.~\ref{fig:s10}, and Extended Data Fig.~\ref{fig:s11}). 
We model each of \rec{the $36$ single-target} interventions as a down-shift of the target gene by its observational mean; i.e., knocking out target gene $i\in[p]$ is modeled as a shift intervention $\ba$ with $a_i$ being the negative of the observed mean of this gene, and all other entries being zero.
This model assumes that the perturbations are effective at knocking out their target genes.
This assumption is not always met by current technologies: we observe that some interventions do not down-regulate their target genes in a statistically significant manner (\rec{Supplementary Information~\ref{sec:h}}). 
However, this is only known after performing the experiments and thus we use the idealized model. \rec{If another dataset or pre-screen (usually on a smaller scale) is available to help determine the effectiveness of an intervention to target a specific gene, then we can easily replace the idealized model with this estimate.}
As we will show, the idealized model still extracts enough signal for our method to aid in finding good interventions.

To evaluate our approach, we use the observational distribution as the source cell state and a particular single-node interventional distribution as the target cell state. 
The aim is to identity this single-node intervention or an intervention with similar effect using the least number of samples. 
For benchmarking purposes, we use a setting where the optimal intervention is contained in the feasible set of interventions,
but we note that this is not a requirement and the target cell state can be any desired distribution, the experimental design goal then being to identify an intervention that moves the distribution as close as possible to the target distribution.

Fig.~\ref{fig:6}A shows three representative examples of interventional distributions (targeting \emph{CDH19}, \emph{SOX4} and \emph{HLA-C}), visualized by a $2$-dimensional projection along the most variable directions obtained using contrastive PCA (see \textit{Methods}). 
These examples highlight that the alterations induced by single-target interventions are subtle (see also Supplementary Fig.~\ref{fig:s12}), with \emph{SOX4} being more similar to \emph{CDH19} than \emph{HLA-C} as also corroborated by the squared distance between the interventional means (see Fig.~\ref{fig:6}B). 
This observation is consistent with previous melanoma studies showing that \emph{HLA-C} is associated with positive immune response \cite{carretero2008analysis}, while \emph{CDH19} and \emph{SOX4} are associated with metastasis and immune evasion~\cite{jaeger2007gene,cheng2017sox4}.
Fig.~\ref{fig:6}C shows the performance of our results with the optimal intervention targeting gene \emph{CDH19}.
More examples using other target genes are given in Supplementary Information~\ref{sec:h}. 
Each method is run $50$ times, where each run starts with a warm-up set of $100$ observational samples and then $n=10$ interventional samples per time step. 
In Fig.~\ref{fig:6}C, we present the comparison between three acquisition functions: \emph{CIV} and the two baselines \emph{random} and \emph{greedy} baselines. More implementation details as well as the full results are given in Extended Data Fig.~\ref{fig:s13}-\ref{fig:s15}. 
Similar to what we observed in the synthetic data experiments, CIV outperforms the benchmarks in terms of distance to target mean across all time steps.
%
This suggests that the proposed approach is beneficial for identifying perturbations to induce a desired cell state change. 

\section{Discussion}
In this work, we developed an active learning framework for optimal intervention design in causal models. Our method has two main ingredients: (i) modeling and updating the edge weights in the causal model using a Bayesian approach (using the DAG-BLR distribution); and (ii) optimizing the next intervention from which to obtain samples using a class of causally aware acquisition functions (causal integrated variance acquisition (CIV) functions). 
The DAG-BLR distribution respects the underlying causal structure and allows for efficient posterior updates. 
The proposed CIV class of acquisition functions prioritizes the most informative intervention with respect to identifying the optimal intervention for moving the system towards a desired mean by minimizing an uncertainty quantity weighted based on the directions of interest. 
Importantly, the designed CIV acquisition function allows for efficient optimization by having tractable closed-form evaluations. 
In addition,  we showed that the introduced acquisition function is characterized by attractive theoretical properties, such as by mutual information bounds and consistency. 
Finally, we demonstrated the developed active learning framework on both synthetic data and a biological dataset. In both cases, the designed acquisition function  outperforms empirical analogs allowing for accurate predictions with fewer experiments.

{
We made various assumptions that may be limiting for some applications and motivate future research directions. 
First, we focused on the setting with known causal structure. An important future research direction is to consider the case where the causal structure is partially or entirely unknown. We discuss potential avenues to approach this problem in Supplementary Information~\ref{sec:d}. Second, we considered linear SCMs with additive Gaussian noise. A potential extension to the non-linear setting could be achieved using kernels with linearity over the feature space. Similar derivations as in our work could be used to evaluate or approximate the CIV acquisition functions in this setting. Third, in our model we assumed \textit{causal sufficiency} \cite{spirtes2000causation}, which excludes the existence of latent confounders as well as the possibility to perform interventions only on some nodes. This is violated when some system variables are unobserved and thus cannot be intervened on or specified in the desired state. A possible approach is to use a more agnostic model between interventions and their effects as proposed for example in~\cite{aglietti2020causal}; however, this will generally lead to weaker results and a loss of structural information. 

While we discussed our work in the context of applications to cellular reprogramming, we envision our framework to be applicable broadly for sequential design problems arising in complex systems. In Supplementary Information~\ref{sec:j}, we discuss several other applications and how they fit into our proposed framework.}

\section*{Methods}

\subsection*{Posterior Update of Edge Weights}

The current dataset $\cD_t$ induces a belief on the model parameters $\bB$. 
We consider the Bayesian setting where the belief corresponds to a distribution. 
To account for the known causal structure, we assume a generalization of the \emph{DAG-Wishart} distribution \cite{geiger2002parameter,kuipers2022interventional}, which places a prior on $\bB$ that respects the causal structure and allows for efficient posterior updates (Supplementary Information~\ref{sec:a}). 
The prior is as follows.

\begin{definition}\label{def:1}
The DAG-Bayesian linear regression (DAG-BLR) prior models $\bB$ and the noise variances $\bSigma$ jointly as $\bbP(\bB, \bSigma)=\prod_{i=1}^p\bbP(\bB_{i,\pa(i)}, \sigma_i^2)$, where each $\bbP(\bB_{i,\pa(i)}, \sigma_i^2)$ is
\begin{equation}\label{eq:4}
\begin{aligned}
    \sigma_i^2 & \sim \cI\cG(\alpha_i,\beta_i),\\
    \bB_{i,\pa(i)}|\sigma_i^2 & \sim \cN(\bmm_i, \sigma_i^2 \bM_i).
\end{aligned}
\end{equation}
Here $\cI\cG$ denotes the inverse-Gamma distribution and $\{\alpha_i,\beta_i,\bmm_i,\bM_i\}_{i=1}^p$ are hyperparameters satisfying certain constraints specified in Supplementary Information~\ref{sec:a}. 
\end{definition}

This prior is consistent with the DAG structure since it sets $B_{ik}=0$ for all $k\notin\pa(i)$. 
While it was developed in the observational setting \cite{cao2019posterior}, the following lemma shows that it can be extended to the interventional setting and is a conjugate prior for the model in Eq.~\eqref{eq:2}. 
Thus the posterior lies in the same family of distributions.

\begin{lemma}\label{lm:1}
The posterior corresponding to the DAG-BLR prior satisfies 
$
\bbP(\bB, \bSigma | \cD_t)
=
\prod_{i=1}^p \bbP(\bB_{i,\pa(i)}, \sigma_i^2 | \cD_t)
$, where $\bbP(\bB_{i,\pa(i)}, \sigma_i^2 | \cD_t)$ is
\begin{equation}\label{eq:5}
\begin{aligned}
    \sigma_i^2 | \cD_t & \sim \cI\cG\big(\alpha_i(\cD_t),\beta_i(\cD_t)\big),\\
    \bB_{i,\pa(i)}|\sigma_i^2, \cD_t & \sim \cN\big(\bmm_i(\cD_t), \sigma_i^2 \bM_i(\cD_t)\big).
\end{aligned}
\end{equation}
The hyperparameters $\{\alpha_i(\cD_t),\beta_i(\cD_t),\bmm_i(\cD_t),\bM_i(\cD_t)\}_{i=1}^p$ are specified in Supplementary Information~\ref{sec:a}.
\end{lemma}

The hyperparameters $\{\alpha_i(\cD_t),\beta_i(\cD_t),\bmm_i(\cD_t),\bM_i(\cD_t)\}_{i=1}^p$ can be updated easily using the new samples obtained at step $t$, and thus this choice of a \textit{conjugate} prior allows for efficient posterior updates.

We make one comment about the DAG-Wishart distribution. 
In previous literature \cite{geiger2002parameter,kuipers2022interventional}, this distribution has been extended to model the posterior beliefs of both the edge weights \textit{and} the DAG structure.
The posterior for the DAG structure is known as the Bayesian Gaussian equivalence (BGe) score. 
Thus while we here assume that the DAG is known or prefixed and we only update the posterior on the non-zero entries of $\bB$, our framework can be extended to the unknown DAG setting by placing a probability on the DAG structure using the BGe score, as discussed in Supplementary Information~\ref{sec:d}.

\subsection*{Evalution of CIV in Closed Form}

We discuss how to optimize CIV in order to solve for $\ba^{(t+1)}$. We start by providing a closed form of the variance $\sigma^2_{g(\ba)|\cD_t(\ba')}$. For ease of reading, we here provide the formula for the special case where the exogenous noise variances $\bSigma$ are known. The general formula, which is similar in flavor but more complicated, including the proof are given in Supplementary Information~\ref{sec:b}.

\begin{proposition}\label{prop:1}
Conditioning on $\bSigma$, we have
\begin{equation}\label{eq:10}
    \begin{aligned}
    \sigma^2_{g(\ba)|\cD_t(\ba')}=2\sum_{i=1}^p \big(v_i^2+\frac2nv_i\sigma_i^2+2v_i(a_i-b_i)^2\big)+ c,
    \end{aligned}
\end{equation}
where only $v_i := \sigma_i^2\bmu^{*\top}_{\pa(i)}\bM_i(\cD_t(\ba'))\bmu_{\pa(i)}^*$ depends on the augmented dataset $\cD_t(\ba')$, while $b_i:=\mu_i^*-\bmm_i(\cD_t)^\top\bmu^*_{\pa(i)}$ and the constant $c$ do not depend on $\ba'$.
\end{proposition}

We note that the resulting expression is a function of $b_i$ and $v_i$, which can be interpreted as follows: Lemm~ \ref{lm:1} can be used to rewrite $b_i=[(\bI-\bbE(\bB|\cD_t))\bmu^*]_i$ as the estimated shift of the optimal intervention at node $i$, and $v_i=\bmu_{\pa(i)}^{*\top}\Var(\bB_{i,\pa(i)}\big|\cD_t(\ba'))\bmu^*_{\pa(i)}$ as the covariance matrix of $\bB_{i,\pa(i)}$ scaled in the direction of the target mean $\bmu^*$.

Next, we discuss how to integrate Eq.~\eqref{eq:10} over the space of possible interventions $\cA$ to evaluate $h(\ba', \cD_t)$.
Here we consider $\cA$ to be the hypersphere $\bbS^{p-1}=\{\ba\in\bbR^p: \|\ba\|_2=1\}$ and $\nu$ to be the uniform measure on the hypersphere
\rec{(other types of feasible sets $\cA$, e.g. with sparsity constraints, can be considered, and similar derivations can be used to identify if a closed-form integration exists)}.
This corresponds to fixing the magnitude of the intervention and only optimizing over its direction. This is suitable for various applications including the biological problem considered below, since the strength of the intervention is often designed separately or prefixed manually by the experimenter. 
We also note that for many linear problems the uncertainty decreases by making the magnitude of the selected point larger \cite{sapsis2020output}. We show in Supplementary Information~\ref{sec:b} that this is the case also for the problem considered in this paper. 
Considering the hypersphere allows us to avoid this ambiguity. 
The following proposition provides the resulting formula when $\bSigma$ is known; see Supplementary Information~\ref{sec:b} for the general case including the proof.

\begin{proposition}\label{prop:2}
For $\cA=\bbS^{p-1}$ and $\nu$ being a uniform measure, the CIV acquisition function evaluated at $\ba'$, conditioned on $\bSigma$, is
\begin{equation}\label{eqn:civ-aquisition-function}
\begin{aligned}
    h(\ba',\cD_t) =c_1\cdot\sum_{i=1}^p \big(v_i^2+2v_i(\frac{\sigma_i^2}{n}+b_i^2+\frac1p)\big) + c_2,
\end{aligned}
\end{equation}
where $c_1>0$ and $c_2$ are constants that do not depend on $\ba'$.
\end{proposition}

The gradient of this objective function can be calculated explicitly also in the general case when $\bSigma$ is unknown, and thus we can use gradient-based optimization methods with a ball constraint to solve for $\ba^{(t+1)}=\argmin_{\ba'\in\bbS^{p-1}}\sum_{i=1}^p(v_i^2+2v_i(\frac{\sigma_i^2}{n}+b_i^2+\frac1p))$ {(see the section on \emph{Implementations} below for details)}. \rec{{For other types of feasible sets $\cA$, e.g.~with sparsity constraints or consisting of hard interventions, the appropriate optimization method needs to be adjusted accordingly.}}
We also note that the objective function is not necessarily convex (Supplementary Information~\ref{sec:b}), and gradient-based optimizers may therefore  only find a local minimum.

%

\subsection*{Mutual Information Bound}\label{sec:42}

To provide an information-theoretic interpretation of the CIV acquisition function, we use the \textit{relative decay} of the uncertainty $\sigma^2_{g(\ba)|\cD_t}$ measured by
\begin{equation}
\label{rel_decay}
    \frac{\sigma^2_{g(\ba)|\cD_t}-\bbE_{\bx'}\big(\sigma^2_{g(\ba)|\cD_t\cup(\bx',\ba')}\big)}{\sigma^2_{g(\ba)|\cD_t}},
\end{equation}
where the expectation $\bbE_{\bx'}$ is taken with respect to a new sample $\bx'$ from $\ba'$ whose distribution is given by $\bbP(\bx'|\cD_t,\ba')$. \rec{To simplify the notation, we restrict the discussion in this section to the case of $n=1$. Similar results can easily be derived for sample sizes $n>1$.} The following theorem shows that the mutual information between $g(\ba)$ and a sample $\bx'$ from $\ba'$ can be lower bounded by this relative decay up to a multiplicative factor. {The proof can be found in Supplementary Information~\ref{sec:e}.}
 
\begin{theorem}\label{thm:1}
Conditioning on $\cD_t$, the following inequality holds for any $\ba\in\cA$ for the mutual information between $g(\ba)$ and a sample $\bx'$ from $\rmP^{\ba'}$:
\begin{equation}\label{eq:14}
    I(g(\ba);\bx'|\ba',\cD_t)\geq \rho^2\cdot \frac{\sigma^2_{g(\ba)|\cD_t}-\bbE_{\bx'}\big(\sigma^2_{g(\ba)|\cD_t\cup(\bx',\ba')}\big)}{\sigma^2_{g(\ba)|\cD_t}},
\end{equation}
where $\rho^2$ is a constant that does not depend on $\ba'$.
\end{theorem}

This result holds in general for any function $g(\ba)$ and its posterior modeling on $\cD_t$; but the constant $\rho$ in the bound depends on the specific choice of $g$ and its posterior modeling. For example, for the prominent Bayesian linear regression task, Theorem \ref{thm:1} holds with $\rho^2=\nicefrac12$ (see Supplementary Information~\ref{sec:e}).

Note that since $\sigma^2_{g(\ba)|\cD_t}$ does not depend on $\ba'\in\cA$, minimizing the uncertainty $\bbE_{\bx'}(\sigma^2_{g(\ba)|\cD_t\cup(\bx',\ba')})$ is equivalent to maximizing its relative decay in Eq.~\eqref{rel_decay}. The uncertainty $\sigma^2_{g(\ba)|\cD_t(\ba')}$ used by the CIV acquisition function (Eq.~\eqref{eq:8}) is a computationally efficient estimator of $\bbE_{\bx'}(\sigma^2_{g(\ba)|\cD_t\cup(\bx',\ba')})$ in the relative decay (Eq.~\eqref{rel_decay}), where the expectation with respect to $\bx'\sim\bbP(\bx'|\cD_t,\ba')$ is replaced by a plug-in estimator $\bar{\bx}'=\bbE(\bx'|\cD_t,\ba')$. It follows from Theorem~\ref{thm:1} that minimizing the integrated variance in Eq.~\eqref{eq:8} corresponds to selecting an intervention that maximizes a lower bound on the mutual information between the resulting samples and $g(\ba)$.
This perspective shows how the CIV acquisition function connects to a prominent line of previous works, known as Bayesian Active Learning by Disagreement (BALD) \cite{houlsby2011bayesian, kirsch2019batchbald, jesson2021causal}, where the idea is to select interventions that maximize the information gain at each step. 
For example, if the quantity of interest is $g(\ba)$, then BALD seeks to find $\ba'$ such that the information gain measured by the entropy decay 
\begin{equation}
    H\big(g(\ba)|\cD_t\big)-H\big(g(\ba)|\cD_t\cup(\bx',\ba')\big)
\end{equation}
is maximized. 
Note that this term equals $I(g(\ba);\bx'|\ba',\cD_t)$. 
As a consequence, BALD directly operates on the mutual information, which involves high-dimensional moments of the random variable $g(\ba)$. 
This highlights an important drawback of BALD, namely that it is generally computationally difficult to evaluate and often requires approximation techniques such as MCMC \cite{houlsby2011bayesian}.
In contrast, our CIV acquisition function only depends on the first and the second moments of $g(\ba)$, thereby allowing for direct computation and efficient optimization.

\subsection*{Consistency of CIV for Identifying the Optimal Intervention}

We now provide a graphical interpretation of the proposed CIV acquisition function.
For simplicity, we restrict our discussion to the expression developed in Proposition~\ref{prop:2} for the case of uniform measure and known variances. The general setting is discussed in Supplementary Information~\ref{sec:f}.
Since $c_1, c_2$ in Proposition \ref{prop:2} are constants that do not depend on $\ba'$, the new intervention is obtained by minimizing $\sum_{i=1}^p(v_i^2+2v_i(\frac{\sigma_i^2}{n}+b_i^2+\frac1p))$; with a slight abuse of notation let 
\begin{equation}\label{eq:16}
    h(\ba',\cD_t)=\sum_{i=1}^p \big(v_i^2+2v_i(\frac{\sigma_i^2}{n}+b_i^2+\frac1p)\big).
\end{equation}

The term $v_i = \bmu_{\pa(i)}^{*\top}\Var(\bB_{i,\pa(i)}\big|\cD_t(\ba'))\bmu^*_{\pa(i)}$ can be understood as the epistemic uncertainty of node $i$ in the direction of interest;
namely, since $\sfx_i=\bB_{i,\pa(i)}^\top\sfx_{\pa(i)}+\epsilon_{i}$, the variance of $\bB_{i,\pa(i)}$ characterizes the epistemic uncertainty of estimating $\sfx_i$ using its parent nodes;
in $v_i$, this uncertainty, presented as a covariance matrix, is transformed into a scalar value by scaling it in the direction of the target mean $\bmu^*$. 
As the number of time steps $t$ increases, one can show (see Supplementary Information~\ref{sec:e}) using the Bernstein-von-Mises theorem~\cite{kleijn2012bernstein} that $v_i=O(\nicefrac1t)$, which implies that the second-order terms $v_i^2$ are dominated by the first-order terms in Eq.~\eqref{eq:16}. %
Note that their coefficients $\sigma_i^2/n+b_i^2+1/p$ are larger for nodes $i$ with larger noise variances and larger estimated optimal shift values $b_i=[(\bI-\bbE(\bB|\cD_t))\bmu^*]_i$. 
This scaling is intuitive, since it puts emphasis on reducing the variance of nodes where there is still a high uncertainty (given by $\sigma_i^2$) or that require high shift values (estimated by $b_i^2$). As we show in the following theorem, the combined effect of these terms on acquiring new interventions is that $\ba^*$ will become an approximate local minimizer of Eq.~\eqref{eq:16} as $t$ increases as shown in the following theorem, \rec{where we use the conventional asymptotic notation, $o, O$ and $\Theta$, with respect to the time step $t$.} 

\begin{theorem}\label{thm:2}
For all $\ba'\in\cA$, the CIV acquisition function either decays linearly with respect to $t$, i.e.,  $h(\ba',\cD_t)=\Theta(\nicefrac1t)$, or degenerates to the constant zero, i.e., $h(\ba',\cD_t)\equiv 0$.
The optimal intervention $\ba^*$ satisfies $\|\nabla h(\ba^*,\cD_t)\|_2=O(\nicefrac{1}{t^2})$ and $\nabla^2 h(\ba^*,\cD_t)\succeq -O(\nicefrac1t^2)\bI$.
\end{theorem}

This theorem shows that as $t$ increases the gradient at $\ba^*$ vanishes to zero faster than the acquisition function and the Hessian becomes almost positive semi-definite. This suggests that gradient-based optimizers with an initialization close enough to $\ba^*$ converge to the optimal intervention $\ba^*$ as $t\to\infty$. In experiments, we observe this to hold even for a moderate number of time steps $t$ as illustrated for a 10-node DAG example in Fig.~\ref{fig:4}. {The proof of Theorem~\ref{thm:2} can be found in Supplementary Information~\ref{sec:f}.}

\subsection*{Implementations}
Optimization of the acquisition functions was performed using a gradient-based solver, more precisely, Sequential Least Squares Programming (SLSQP) with the nonlinear constraint $\|\ba^{(t)}\|_2\leq 1$. 
As proven in Supplementary Information~\ref{sec:b}, optimizing the proposed acquisition functions with this constraint always outputs a feasible $\ba^{(t)}$ such that $\|\ba^{(t)}\|_2=1$. 
For other types of feasible sets, e.g.~$\ba^{(t)}\in [0,1]^p$, one can use other solvers such as Limited-memory BFGS with box constraints (L-BFGS-B).
SLSQP was implemented using the SciPy package~\cite{virtanen2020scipy} and initialized in two ways:
1) using the intervention $\ba^{(t-1)}$ that was selected in the previous time step; 2) using the current estimate $(\bI-\bbE(\bB|\cD_t))\bmu^*$ of the optimal intervention. 
Among the two resulting solutions, we used the one with the better acquisition function value as $\ba^{(t)}$.

For the experiments on synthetic data, we generated the DAGs using the NetworkX package \cite{hagberg2008exploring} and the CausalDAG package \cite{squires2018causaldag}. 
We used different graph types and sizes, as detailed in Supplementary Information~\ref{sec:g}. 
The edge weights were uniformly sampled from $[-1,-0.25]\cup [0.25,1]$ to ensure that the parameters were bounded away from zero. The exogenous noise levels were set to $1$ and the resulting linear Gaussian SCM was then rescaled with our implementation of the standardized model as described in Appendix~F of \cite{reisach2021beware}. %
For the DAG-BLR prior (Eq.~(\ref{eq:4})), we set the edge-weight-related hyperparamters as follows: $\bmm_i=0$ and $\bM_i$ equal to the identity matrix for all $i\in[p]$. {We assumed the variances to be known and used the known-variance formula.}


The transcriptomic dataset analyzed in the biological application was processed using the Scanpy package \cite{wolf2018scanpy}. 
The kernel density estimate (KDE) plots {shown in Fig.~\ref{fig:6}A share the same coordinate axes and were produced} by projecting the high-dimensional dataset along the first two principal components obtained using contrastive PCA~\cite{abid2018exploring} with the unperturbed samples as {\emph{background}} and all interventional samples as {\emph{foreground}}. {The purpose of constrative PCA is to identify low-dimensional structures in a foreground dataset relative to a background dataset.}
The Greedy Sparsest Permutation (GSP) algorithm~\cite{solus2021consistency} was used to learn the DAG structure from the unperturbed data.
The parameters used in GSP are given in Supplementary Information~\ref{sec:h}, Supplementary Fig.~\ref{fig:s11}. 
For the DAG-BLR prior, we set the hyperparamters as follows: $\bmm_i=0$, $\bM_i$ equal to the identity matrix, $\alpha_i=2$, and $\beta_i=0$. These hyperparameters were then updated using the posterior formula (\ref{lm:1}) based on $100$ samples from the unperturbed data. 
{Similar warm-up steps are used in the active learning literature (e.g., \cite{sapsis2020output,jesson2021causal}). Since we usually have access to some observational samples in applications of optimal intervention design, we recommend using these to obtain dataset-informed hyperparameters of the DAG-BLR prior.
%
After this warm-up step using unperturbed data,} interventions were acquired based on the known-variance formula, where we used the estimated mean of the noise variances as a proxy.

The codebase for updating the model posterior and optimizing the causally-aware acquisition functions proposed in this work can be obtained via \cite{jiaqi_zhang_2023_8170179}. We also provide a notebook to extract and process the Perturb-CITE-seq data from~\cite{frangieh2021multimodal}.
Our codebase can be used to replicate all the main results and figures as well as for other user-defined applications. 
All methods can be run efficiently on a CPU for a moderate number of variables.

\section*{Data Availability}
The Perturb-CITE-seq \cite{frangieh2021multimodal} data can be obtained from \url{https://doi.org/10.1038/s41588-021-00779-1}.

{\section*{Code Availability}
All code has been deposited at \cite{jiaqi_zhang_2023_8170179}.
}

\section*{Acknowledgements}

J.Z., C.S. and C.U.~were partially supported by NCCIH/NIH, ONR (N00014-22-1-2116), NSF (DMS-1651995), the MIT-IBM Watson AI Lab, MIT J-Clinic for Machine Learning and Health, the Eric and Wendy Schmidt Center at the Broad Institute, and a Simons Investigator Award to C.U. T.P.S. acknowledges support by ONR (N00014-21-1-2357) and AFOSR (MURI FA9550-21-1-0058). C.S. was partially supported by an NSF Graduate Fellowship.

{
\subsection*{Author Contributions Statement}
J.Z., L.C., T.P.S., and C.U. conceived the research and designed the method. J.Z. derived the theoretical results and performed the numerical experiments. L.C. and J.Z. processed the biological data. C.S. and J.Z. derived the extension of the DAG-Wishart distribution. All authors interpreted the results and wrote the manuscript.
}

{
\subsection*{Competing Interests Statement}
The authors declare no competing interests.}
\clearpage

\renewcommand{\figurename}{Extended Data Figure}
\renewcommand{\thefigure}{\arabic{figure}}
\setcounter{figure}{0}

\bibliographystyle{abbrvurl}
\bibliography{references}

\clearpage

\renewcommand{\figurename}{Supplementary Figure}
\renewcommand{\thefigure}{\arabic{figure}}
\setcounter{figure}{0}
\nolinenumbers

\appendix





\vspace{0.5cm}

\section{Conjugacy of the DAG-BLR prior}\label{sec:a}

We begin by remarking on the relationship between the DAG-BLR prior and the well-studied DAG-Wishart prior.

\begin{remark}
The DAG-BLR prior is a generalization of the (multi-shape) DAG-Wishart prior introduced in \cite{ben2011high} and used in later works including \cite{cao2019posterior}.
In particular, as mentioned in Section 3.1 of \cite{cao2019posterior}, the DAG-Wishart prior with parameters $U$ and $\tilde{\alpha}$, for $U \in \bbR^{p \times p}$ positive definite, is a special case of the prior used in this work with
\begin{align*}
    \alpha_i &= \frac{\tilde{\alpha}_i}{2} - \frac{|\pa(i)|}{2} - 1,
    \\
    \beta_i &= \frac{1}{2} \left( U_{ii} - U_{i,\pa(i)} U_{\pa(i),\pa(i)}^\inv U_{\pa(i),i} \right),
    \\
    \bmm_i &= U_{\pa(i),\pa(i)}^\inv U_{\pa(i),i},
    \\
    \bM_i &= U_{\pa(i),\pa(i)}^\inv
\end{align*}
Note that, in general, the DAG-Wishart prior imposes more constraints than the DAG-BLR prior.
For example, given two nodes $i$ and $j$ such that $\pa(i) = \pa(j)$, the DAG-Wishart prior requires that $\bM_i = \bM_j$, i.e., the covariance over the edge weights into $i$ must be equal to the covariance over the edge weights into $j$.
These constraints ensure that the DAG-Wishart prior satisfies desirable properties for model selection such as \emph{score equivalence} \cite{peluso2020compatible}.
However, these constraints are not necessary to ensure that the DAG-BLR prior is a conjugate prior.
In this paper, we do not use our prior over $\bB, \bSigma$ for model selection - we only require that the prior is a conjugate prior.
Thus, we present our results for the more general case.
\end{remark}

\subsection{Conjugacy in the observational setting}

Here, we prove that the DAG-BLR prior is a conjugate prior to the Gaussian likelihood, and provide equations for the posterior hyperparameters. This section applies in the observational setting without interventions, where all samples acquired follow Eq.~\eqref{eq:1} in the main text.

\begin{claim}\label{claim:hyperparameter-update}
Let $\bbP(\bB, \bSigma)$ be a DAG-BLR distribution with hyperparameters $\{ \alpha_i, \beta_i, \bmm_i, \bM_i \}_{i=1}^p$.
Let $\cD_t \in \bbR^{n \times p}$ be sampled from a linear Gaussian DAG model with parameters $(\bB, \bSigma)$.
Then $\bbP(\bB, \bSigma|\cD_t)$ is a DAG-BLR distribution with hyperparameters $\{ \alpha_i(\cD_t), \beta_i(\cD_t), \bmm_i(\cD_t), \bM_i(\cD_t) \}_{i=1}^p$, where
\begin{equation}\label{eq:hyperparameter-update-dag-blr}
    \begin{aligned}
    \alpha_i(\cD_t) & = \alpha_i + \frac{n}{2},
    \\
    \beta_i(\cD_t) & = \beta_i + \frac12
    \left(
    \sum_{m=1}^n
    (\sfx^{(m)}_i)^2 + \bmm_i^\top \bM_i^{-1}\bmm_i - \bmm_i(\cD_t)^\top \bM_i(\cD_t)^{-1} \bmm_i(\cD_t)
    \right),
    \\
    \bmm_i(\cD_t) &= \bM_i(\cD_t)
    \left(\bM_i^{-1}\bmm_{i} + \sum_{m=1}^n \sfx_i^{(m)} \sfx_{\pa(i)}^{(m)}
    \right),
    \\
    \bM_i(\cD_t) & = \left(
    \bM_i^{-1} + \sum_{m=1}^n \sfx_{\pa(i)}^{(m)} \sfx_{\pa(i)}^{(m)~\top}
    \right)^{-1}.
    \end{aligned}
\end{equation}
\end{claim}

\begin{proof}
We begin by proving that $\bbP(\bx) = \prod_{i=1}^p \bbP(\sfx_i \mid \sfx_{\pa(i)})$.
In particular, it holds by the DAG factorization structure on $\bbP(\bx \mid \bB, \bSigma)$ and the product structure on $\bbP(\bB, \bSigma)$ that
\begin{equation}\label{eq:dag-blr-step1}
\begin{aligned}
    \bbP(\bx)
    &=
    \bbE_{\bB,\bSigma} (\bbP(\bx \mid \bB, \bSigma)) 
    \\
    &=
    \bbE_{\bB,\bSigma} \left( \prod_{i=1}^p \bbP(\sfx_i \mid \bB_{i,\pa(i)}, \sigma_i^2, \sfx_{\pa(i)}) \right)
    \\
    &=
    \prod_{i=1}^p \bbE_{\bB, \bSigma} \left( \bbP(\sfx_i \mid \bB_{i,\pa(i)}, \sigma_i^2, \sfx_{\pa(i)}) \right)
    \\
    &=
    \prod_{i=1}^p \bbP(\sfx_i \mid \sfx_{\pa(i)}).
\end{aligned}
\end{equation}

Next, we prove that
$
\bbP(\bB, \bSigma, \sfx) 
=
\prod_{i=1}^p 
\left( 
\bbP(\sfx_i \mid \sfx_{\pa(i)})
\cdot
\bbP(\bB_{i,\pa(i)}, \sigma_i^2 \mid \sfx_{\pa(i)}, \sfx_i)
\right)
$. Namely:
\begin{equation}\label{eq:dag-blr-step2}
\begin{aligned}
    \bbP(\bB, \bSigma, \sfx)
    &=
    \bbP(\sfx \mid \bB, \bSigma) \cdot \bbP(\bB, \bSigma)
    \\
    &=
    \left( \prod_{i=1}^p \bbP(\sfx_i \mid \sfx_{\pa(i)}, \bB_{i,\pa(i)}, \sigma_i^2) \right)
    \cdot
    \left( \prod_{i=1}^p \bbP(\bB_{i,\pa(i)}, \sigma_i^2) \right)
    \\
    &= 
    \prod_{i=1}^p 
    \bbP(\sfx_i \mid \sfx_{\pa(i)}, \bB_{i,\pa(i)}, \sigma_i^2) 
    \cdot 
    \bbP(\bB_{i,\pa(i)}, \sigma_i^2)
    \\
    &=
    \prod_{i=1}^p 
    \bbP(\sfx_i, \bB_{i,\pa(i)}, \sigma^2_i \mid \sfx_{\pa(i)})
    \\
    &=
    \prod_{i=1}^p 
    \left(
    \bbP(\sfx_i \mid \sfx_{\pa(i)})
    \cdot
    \bbP(\bB_{i,\pa(i)}, \sigma^2_i \mid \sfx_{\pa(i)}, \sfx_i)
    \right)
\end{aligned}
\end{equation}
In the fourth line, we used the chain rule together with the independence of $\bB_i, \sigma_i$ from $\sfx_{\pa(i)}$, and in the fifth line we used the chain rule again.
Combining Eq.~\eqref{eq:dag-blr-step1} and Eq.~\eqref{eq:dag-blr-step2}, we obtain that
\begin{align*}
    \bbP(\bB, \bSigma \mid \bx)
    =
    \frac{\bbP(\bB, \bSigma, \bx)}{\bbP(\bx)}
    =
    \prod_{i=1}^p 
    \bbP(\bB_{i,\pa(i)}, \sigma^2_i \mid \sfx_{\pa(i)}, \sfx_i).
\end{align*}

Finally, we note that, for $\cD_t = \{ \bx \}$, it holds that, for $\alpha_i(\cD_t), \beta_i(\cD_t), \bm_i(\cD_t),$ and $\bM_i(\cD_t)$ defined in Eq.~\eqref{eq:hyperparameter-update-dag-blr},
\[
\bbP(\bB_{i,\pa(i)}, \sigma_i^2 \mid \sfx_{\pa(i)}, \sfx_i) = \cN(\bB_{i,\pa(i)}; \bmm_i(\cD_t), \sigma_i^2 \bM_i(\cD_t)) \cdot \cI\cG(\sigma_i^2 ; \alpha_i(\cD_t), \beta_i(\cD_t))
\]
by standard results on the conjugacy of the normal-inverse-gamma distribution for Bayesian linear regression (e.g., Proposition 2.11 of \cite{jackman2009bayesian}).
The result for the case of general $n$ follows by induction, which completes the proof.
\end{proof}

\subsection{Conjugacy in the setting of shift interventions with known values}
Next, we extend the previous section by considering the interventional setting with known shift interventions. In particular, we present the proof of Lemma~\ref{lm:1} in the main text, which we restate here for ease of reading and also specify the hyperparameter updates.

\begin{namedenv}[Lemma 1]\label{lm:1-a}
The posterior corresponding to the DAG-BLR prior satisfies $\bbP(\bB, \bSigma|\cD_t)=\prod_{i=1}^p\bbP(\bB_{i,\pa(i)},\sigma_i^2|\cD_t)$, where $\bbP(\bB_{i,\pa(i)}, \sigma_i^2|\cD_t)$ is given by
\begin{equation*}\label{eq:5-a}
\begin{aligned}
    \sigma_i^2|\cD_t & \sim \cI\cG\big(\alpha_i(\cD_t),\beta_i(\cD_t)\big),\\
    \bB_{i,\pa(i)}|\sigma_i^2, \cD_t & \sim \cN\big(\bmm_i(\cD_t), \sigma_i^2 \bM_i(\cD_t)\big),
\end{aligned}
\end{equation*}
where
\begin{equation}\label{eq:hyperparameter-update-dag-blr-shift}
\begin{aligned}
    \alpha_i(\cD_t) & = \alpha_i + \frac{n}{2},
    \\
    \beta_i(\cD_t) & = \beta_i + \frac12\left(
    \sum_{m=1}^n(\sfx^{(m)}_i-a_i^{(m)})^2 + \bmm_i^\top \bM_i^{-1}\bmm_i - \bmm_i(\cD_t)^\top \bM_i(\cD_t)^{-1} \bmm_i(\cD_t)
    \right),
    \\
    \bmm_i(\cD_t) &= \bM_i(\cD_t)
    \left(\bM_i^{-1}\bmm_{i} + \sum_{m=1}^n (\sfx_i^{(m)}-a_i^{(m)})\sfx_{\pa(i)}^{(m)}
    \right),
    \\
    \bM_i(\cD_t) & = \left(
    \bM_i^{-1} + \sum_{m=1}^n \sfx_{\pa(i)}^{(m)} \sfx_{\pa(i)}^{(m)~\top}
    \right)^{-1}.
\end{aligned}
\end{equation}
\end{namedenv}

\begin{proof}
Note that, by the definition of a shift intervention,
\begin{equation*}
    \bbP(\bx \mid \bB, \bSigma; \ba)
    =
    \prod_{i=1}^p \bbP \left( \sfx_i^{(m)} - a_i^{(m)} \mid \sfx_{\pa(i)}, \bB, \bSigma \right).
\end{equation*}
Thus, we have
\begin{equation}\label{eq:dag-blr-shift-step1}
\begin{aligned}
    \bbP(\bx; \ba)
    &=
    \bbE_{\bB,\bSigma} (\bbP(\bx \mid \bB, \bSigma ; \ba)) 
    \\
    &=
    \bbE_{\bB,\bSigma} \left( \prod_{i=1}^p \bbP(\sfx_i - a_i \mid \bB_{i,\pa(i)}, \sigma_i^2, \sfx_{\pa(i)}) \right)
    \\
    &=
    \prod_{i=1}^p \bbE_{\bB, \bSigma} \left( \bbP(\sfx_i - a_i \mid \bB_{i,\pa(i)}, \sigma_i^2, \sfx_{\pa(i)}) \right)
    \\
    &=
    \prod_{i=1}^p \bbP(\sfx_i - a_i \mid \bx_{\pa(i)}).
\end{aligned}
\end{equation}
Similarly, we have
\begin{equation}\label{eq:dag-blr-shift-step2}
\begin{aligned}
    \bbP(\bB, \bSigma, \bx ; \ba)
    &=
    \bbP(\bx \mid \bB, \bSigma ; \ba) \cdot \bbP(\bB, \bSigma)
    \\
    &=
    \left( \prod_{i=1}^p \bbP(\sfx_i - a_i \mid \sfx_{\pa(i)}, \bB_{i,\pa(i)}, \sigma_i^2) \right)
    \cdot
    \left( \prod_{i=1}^p \bbP(\bB_{i,\pa(i)}, \sigma_i^2) \right)
    \\
    &= 
    \prod_{i=1}^p 
    \bbP(\sfx_i - a_i \mid \sfx_{\pa(i)}, \bB_{i,\pa(i)}, \sigma_i^2) 
    \cdot 
    \bbP(\bB_{i,\pa(i)}, \sigma_i^2)
    \\
    &=
    \prod_{i=1}^p 
    \bbP(\sfx_i - a_i, \bB_{i,\pa(i)}, \sigma^2_i \mid \sfx_{\pa(i)})
    \\
    &=
    \prod_{i=1}^p 
    \left(
    \bbP(\sfx_i - a_i \mid \sfx_{\pa(i)})
    \cdot
    \bbP(\bB_{i,\pa(i)}, \sigma^2_i \mid \sfx_{\pa(i)}, \sfx_i - a_i)
    \right).
\end{aligned}
\end{equation}
Then by combining Eq.~\eqref{eq:dag-blr-shift-step1} and Eq.~\eqref{eq:dag-blr-shift-step2}, we obtain that
\begin{align*}
    \bbP(\bB, \bSigma \mid \bx)
    =
    \frac{\bbP(\bB, \bSigma, \bx)}{\bbP(\bx)}
    =
    \prod_{i=1}^p 
    \bbP(\bB_{i,\pa(i)}, \sigma^2_i \mid \sfx_{\pa(i)}, \sfx_i - a_i).
\end{align*}
Once again, by conjugacy of the normal-inverse Gamma distribution for Bayesian linear regression, we have
\[
\bbP(\bB_{i,\pa(i)}, \sigma_i^2 \mid \sfx_{\pa(i)}, \sfx_i - a_i) = \cN(\bB_{i,\pa(i)}; \bmm_i(\cD_t), \sigma_i^2 \bM_i(\cD_t)) \cdot \cI\cG(\sigma_i^2 ; \alpha_i(\cD_t), \beta_i(\cD_t))
\]
for $\alpha_i(\cD_t), \beta_i(\cD_t), \bm_i(\cD_t),$ and $\bM_i(\cD_t)$ defined in Eq.~\eqref{eq:hyperparameter-update-dag-blr-shift},
which concludes the proof.
\end{proof}

{
\begin{remark}
    Note that Lemma~1 can easily be extended to hard interventions by updating only the hyperparameters indexed by nodes $i$ that are not an interventional target.
\end{remark}
}

\section{Details for Causal Integrated Variance}\label{sec:b}

In this section, we show how to evaluate the CIV acquisition function in closed form. 
This then allows for the use of gradient-based optimization methods.
Recall that the causal integrated variance (CIV) acquisition function is defined as
\begin{equation*}
    h(\ba',\cD_t) = \int_{\cA} \sigma_{g(\ba)|\cD_t(\ba')}^2 d\nu(\ba),
\end{equation*}
where $g(\ba)=\|(\ba+\frac1n\sum_{m=1}^n\bepsilon_n)-(\bI-\bB)\bmu^*\|_2^2$ and $\cD_t(\ba')=\cD_t\cup(\bar{\bx}'_{[n]}, \ba')$ with $\bar{\bx}'_{[n]}$ being $n$ repetitions of $\bar{\bx}'=(\bI-\bbE(\bB|\cD_t))^{-1}\ba'$.

{
We remark here that if only a subset of variables $S$ are of interest in the target mean $\bmu^*$, we can replace $g(\ba)$ with a projected version $g(\ba)=\min_{[\bmu]_S=[\bmu^*]_S}\|( \ba+\frac1n \sum_{m=1}^n \bepsilon^{(m)}) - (\bI-\bB) \bmu\|_2^2$ that measures the discrepancy only over the variables in $S$. This minimization problem can be solved exactly by writing $\bmu=\bmu^*+\mathbf{P}_S\bv$ where $\mathbf{P}_S$ is a projection matrix that sets the $S$-entries of $\bv$ to zero. Solving for $\bv$ results in $g(\ba)=\| (\ba+\frac1n \sum_{m=1}^n \bepsilon^{(m)}) - (\bI-\bB) \bmu\|_{\mathbf{Q}_S}^2$, where $\mathbf{Q}_S=\bI-(\bI-\bB)\mathbf{P}_S((\bI-\bB)\mathbf{P}_S)^\dagger$. To evaluate and optimize CIV, we could for example use a plug-in estimand for $\mathbf{Q}_S$ and follow similar derivations as below.

For minimizing (or maximizing) the values of a subset of nodes $S$, we can modify the optimality gap to be a scalar objective function that describes how well the subset of nodes is optimized (e.g., $g(\ba)=\max_{i\in S} \bmu^*_i = \max_{i\in S} [(\bI-\bB)^{-1}\ba]_i$ for minimization). However, since the variance of this function may not be computable in closed form, approximation methods may need to be developed to evaluate and optimize CIV in this case.}


\subsection{Evaluations}

We first show how to evaluate the term inside the integral, $\sigma_{g(\ba)|\cD_t(\ba')}^2$, which characterizes the post-$\ba'$ uncertainty of estimating the optimality of intervention $\ba$. 
To this end, we will need the following lemma on the hyperparameters of the DAG-BLR distribution.

\begin{lemma}\label{lm:2}
Given the current dataset $\cD_t$ and the augmented dataset $\cD_t(\ba')$, the posterior updates on the hyperparameters $\alpha_i(\cD_t(\ba')), \beta_i(\cD_t(\ba'))$ and $\bmm_i(\cD_t(\ba'))$ do not depend on $\ba'$. In particular, $\alpha_i(\cD_t(\ba'))=\alpha_i(\cD_t)+\frac{n}{2}$, whereas $\beta_i(\cD_t(\ba'))=\beta_i(\cD_t)$ and $\bmm_i(\cD_t(\ba'))=\bmm_i(\cD_i)$.
\end{lemma}

\begin{proof}
First, by \nameref{lm:1-a}, it holds that $\alpha_i(\cD_t(\ba'))=\alpha_i(\cD_t)+\frac{n}{2}$. Second, it holds that
\begin{align*}
    \bmm_i\big(\cD_t(\ba')\big) & = \left(\bM_i(\cD_t)^{-1}+n\bar\sfx'_{\pa(i)}\bar\sfx_{\pa(i)}^{\prime\top}\right)^{-1}\left(\bM_i^{-1}(\cD_t)\bmm_{i}(\cD_t) +n(\bar{\sfx}_i-a'_i)\bar\sfx_{\pa(i)} \right)\\
    & = \left(\bM_i(\cD_t)^{-1}+n\bar\sfx'_{\pa(i)}\bar\sfx_{\pa(i)}^{\prime\top}\right)^{-1}\left(\bM_i^{-1}(\cD_t)\bmm_{i}(\cD_t) +n\big(\bar\sfx_{\pa(i)}^{\prime\top}\bmm_{i}(\cD_t)\big)\bar\sfx'_{\pa(i)} \right)\\
    & = \left(\bM_i(\cD_t)^{-1}+n\bar\sfx'_{\pa(i)}\bar\sfx_{\pa(i)}^{\prime\top}\right)^{-1}\left(\bM_i^{-1}(\cD_t)\bmm_{i}(\cD_t) +n\big(\bar\sfx'_{\pa(i)}\bar\sfx_{\pa(i)}^{\prime\top}\big)\bmm_{i}(\cD_t) \right)\\
    & = \bmm_i(\cD_t),
\end{align*}
where we used $\bar{\sfx}'_i-a'_i=\bar\sfx_{\pa(i)}^{\prime\top}\bmm_{i}(\cD_t)$ in the second equation. Lastly, it follows that,
\begin{align*}
    \beta_i\big(\cD_t(\ba')\big) 
    &= 
    \beta_i(\cD_t) + \frac12\Big[n(\bar{\sfx}'_i-a'_i)^2 + \bmm_i^\top(\cD_t) \bM_i^{-1}(\cD_t)\bmm_i(\cD_t) - \bmm_i^\top\big(\cD_t(\ba')\big) \bM_i^{-1}\big(\cD_t(\ba')\big)\bmm_i\big(\cD_t(\ba')\big)\Big] 
    \\
    &= 
    \beta_i(\cD_t) + \frac12\Big[n(\bar{\sfx}'_i-a'_i)^2 + \bmm_i^\top(\cD_t) \Big(\bM_i^{-1}(\cD_t)-\bM_i^{-1}\big(\cD_t(\ba')\big)\Big)\bmm_i(\cD_t)\Big] 
    \\
    &= \beta_i(\cD_t) + \frac12\big[n(\bar{\sfx}'_i-a'_i)^2 - \bmm_i^\top(\cD_t) (n\bar\sfx_{\pa(i)}'\bar\sfx_{\pa(i)}^{\prime\top})\bmm_i(\cD_t)\big] 
    \\
    &= \beta_i(\cD_t).
\end{align*}
Therefore, the updates for these three variables do not depend on $\ba'$, which completes the proof.
\end{proof}

With this result, we can now calculate $\sigma_{g(\ba)|\cD_t(\ba')}^2$ by proving a generalized version of Proposition~\ref{prop:1} from the main text in which we do not condition on a known $\bSigma$. 
The major techniques we will employ are the law of total expectation and variance. These allow us to separate various sources of randomness and obtain a closed-form expression.

\begin{remark}
For convenience, we record the mean and variance (or moments) for a number of distributions which will be used in the following proof.
First, suppose $W \sim ( \cN(\mu, \sigma^2))^2$.
Then
\begin{equation}\label{eq:moment-chi-square}
    \begin{aligned}
        \bbE (W) &= \sigma^2 + \mu^2,
        \\
        \Var (W) &= 2 \sigma^4 + 4\mu^2 \sigma^2.
    \end{aligned}
\end{equation}

Next, suppose $W \sim \cI\cG(\alpha, \beta)$.
Then for $k\in\mathbb{N}_+$,
\begin{equation}\label{eq:moment-inverse-gamma}
    \begin{aligned}
        \bbE (W^k) &= \frac{\beta^k\Gamma(\alpha-k)}{\Gamma(\alpha)} = \frac{\beta^k}{(\alpha - 1)...(\alpha-k)}.
    \end{aligned}
\end{equation}

Finally, suppose $W \sim \bar{\chi}^2(\bw, \bk, \blambda, 0, 0)$, where $\bar{\chi}^2$ denotes a generalized chi-square distribution \cite{davies1980algorithm}.
Then
\begin{equation}\label{eq:moment-generalized-chi-square}
    \begin{aligned}
        \bbE (W) &= \sum_j w_j (k_j + \lambda_j),
        \\
        \Var (W) &= 2 \sum_j w_j^2 (k_j + 2 \lambda_j).
    \end{aligned}
\end{equation}
\end{remark}

\begin{namedenv}[Proposition 1]\label{prop:1-a}
The variance $\sigma_{g(\ba)|\cD_t(\ba')}^2$ satisfies
\begin{equation*}
    \sigma_{g(\ba)|\cD_t(\ba')}^2 
    = 
    \sum_{i=1}^p
    \frac{\beta_i'}{\alpha_i'-1} \left[ \frac{\beta_i'(2\alpha'_i-1)}{(\alpha'_i-1)(\alpha'_i-2)} \left( u_i^2 + \frac2n u_i \right) 
    + 
    4 u_i (a_i - b_i)^2
    \right] + c,
\end{equation*}
where only $u_i:=\bmu_{\pa(i)}^{*\top} \bM_i(\cD_t(\ba'))\bmu_{\pa(i)}^*$ depends on $\ba'$. The variables $\alpha_i':=\alpha_i(\cD_t)+\frac{n}{2}$,  $\beta_i':=\beta_i(\cD_t)$ and $b_i:=\bmu_i^*-\bmm_i(\cD_t)^\top \bmu_{\pa(i)}^*$ only depend on $\cD_t$. The scalar $c$ is a constant that does not depend on $\ba'$.
\end{namedenv}

\begin{proof}
First, note that the uncertainty in $\sigma_{g(\ba)|\cD_t(\ba')}^2$ comes from two sources: $\bB$ and $\bepsilon_{[n]}$, which jointly depend on $\bSigma$. By the law of total variance, we condition out $\bSigma$ as follows:
\begin{equation}\label{prop1:eq1}
    \begin{aligned}
    \sigma_{g(\ba)|\cD_t(\ba')}^2 
    &= 
    \Var_{\bepsilon_{[n]},\bB} \big(g(\ba)|\cD_t(\ba')\big) 
    \\
    &= 
    \Var_{\bSigma} \Big(\bbE_{\bepsilon_{[n]},\bB}\big(g(\ba)\big|\bSigma,\cD_t(\ba')\big)\Big|\cD_t(\ba')\Big) + \bbE_{\bSigma} \Big(\Var_{\bepsilon_{[n]},\bB}\big(g(\ba)\big|\bSigma,\cD_t(\ba')\big)\Big|\cD_t(\ba')\Big).
    \end{aligned}
\end{equation}

Now, we evaluate the inside terms conditioning on $\bSigma$ using another layer of the law of total expectation and variance:
\begin{equation}\label{prop1:eq2}
    \begin{aligned}
    &\bbE_{\bepsilon_{[n]},\bB}
    \big(
    g(\ba)\big|\bSigma,\cD_t(\ba')\big) \\
    ={}& 
    \bbE_{\bB} \big(
    \bbE_{\bepsilon_{[n]}}
    (g(\ba) | \bB, \bSigma, \cD_t(\ba') )
    \big|
    \bSigma,\cD_t(\ba')
    \big)
    \\
    ={}& \bbE_{\bB}\big(\bbE_{\bepsilon_{[n]}}(g(\ba)|\bB, \bSigma)\big|\bSigma,\cD_t(\ba')\big),\\
    & \Var_{\bepsilon_{[n]},\bB}\big(g(\ba)\big|\bSigma,\cD_t(\ba')\big)\\
    ={}&  \Var_{\bB}\big(\bbE_{\bepsilon_{[n]}}(g(\ba)|\bB,\bSigma,\cD_t(\ba'))\big|\bSigma,\cD_t(\ba')\big) + \bbE_{\bB}\big(\Var_{\bepsilon_{[n]}}(g(\ba)|\bB,\bSigma,\cD_t(\ba'))\big|\bSigma,\cD_t(\ba')\big)\\
    ={}& \Var_{\bB}\big(\bbE_{\bepsilon_{[n]}}(g(\ba)|\bB,\bSigma)\big|\bSigma,\cD_t(\ba')\big) + \bbE_{\bB}\big(\Var_{\bepsilon_{[n]}}(g(\ba)|\bB,\bSigma)\big|\bSigma,\cD_t(\ba')\big),
    \end{aligned}
\end{equation}
where we used the fact that $g(\ba)$ is independent of $\cD_t(\ba')$ conditioning on $\bB$ and $\bSigma$. Conditioning on $\bB$ and $\bSigma$, we can see that the random variable
\begin{equation*}
    g(\ba)
    =
    \sum_{i=1}^p 
    \left( 
    \left[\frac1n\sum_{m=1}^n \bepsilon_m \right]_i
    +
    [\ba-(\bI-\bB)\bmu^*]_i
    \right)^2
\end{equation*}
is a sum of squares of $p$ independent Gaussian variables $[\frac1n\sum_{m=1}^n \bepsilon_m]_i+[\ba-(\bI-\bB)\bmu^*]_i\sim \cN([\ba-(\bI-\bB)\bmu^*]_i,\frac{1}{n}\sigma_i^2)$, and therefore follows a generalized chi-square distribution
\begin{equation*}
    g(\ba) | \bB, \bSigma
    \sim \bar{\chi}^2
    \left(
    \frac1n\diag(\bSigma), 
    \mathbf{1},
    \blambda,
    0, 
    0
    \right),
\end{equation*}
where $\blambda = n\diag(\bSigma^{-1}[\ba-(\bI-\bB)\bmu^*][\ba-(\bI-\bB)\bmu^*]^\top)$ \rec{and notation $\diag(\cdot)$ denotes the diagonal vector of a matrix.}
Thus, by Eq.~\eqref{eq:moment-generalized-chi-square}, we obtain
\begin{equation}\label{prop1:eq3}
    \begin{aligned}
    \bbE_{\bepsilon_{[n]}}\big(g(\ba)|\bB,\bSigma\big) 
    &= 
    \sum_{i=1}^p \left(
    \frac{\sigma_i^2}{n} +[\ba-(\bI-\bB)\bmu^*]_i^2
    \right), 
    \\
    \Var_{\bepsilon_{[n]}}\big(g(\ba)|\bB,\bSigma\big) 
    &= 
    2\sum_{i=1}^p \left(
    \frac{\sigma_i^4}{n^2} + \frac{2\sigma_i^2}{n} [\ba-(\bI-\bB)\bmu^*]_i^2
    \right).
    \end{aligned}
\end{equation}
Then, conditioning the right-hand sides on $\bSigma$ and $\cD_t(\ba')$, we obtain by Lemma~\ref{lm:1} that the only variable terms are $[\ba-(\bI-\bB)\bmu^*]_i^2=(a_i-\mu_i^*+\bmu^{*\top}_{\pa(i)}\bB_{i,\pa(i)})^2\sim \cN(a_i-b_i,\sigma_i^2 u_i)^2$, which are independent among different $i\in[p]$. 
Therefore, by Eq.~\eqref{prop1:eq2}, Eq.~\eqref{prop1:eq3}, and Eq.~\eqref{eq:moment-chi-square}, we have
\begin{equation}\label{prop1:eq4}
    \begin{aligned}
     &~~~~\bbE_{\bepsilon_{[n]},\bB}\big(g(\ba)\big|\bSigma,\cD_t(\ba')\big) \\
     &= 
     \bbE_\bB \left( 
     \sum_{i=1}^p \left(
     \frac{\sigma_i^2}{n} +[\ba-(\bI-\bB)\bmu^*]_i^2
     \right)
     \bigg|
     \bSigma,\cD_t(\ba')
     \right)
     \\
     & =
     \sum_{i=1}^p \left(
     \frac{\sigma_i^2}{n}+\sigma_i^2u_i+(a_i-b_i)^2
     \right), 
     \\
    &~~~~ \Var_{\bepsilon_{[n]},\bB}
     \big(
     g(\ba)\big|\bSigma,\cD_t(\ba')\big)\\
     & = \Var_\bB\left(  \sum_{i=1}^p \left(
    \frac{\sigma_i^2}{n} +[\ba-(\bI-\bB)\bmu^*]_i^2
    \right)\bigg|\bSigma,\cD_t(\ba')\right) 
    \\
    &\qquad\qquad\qquad + \bbE_\bB \left(
    2\sum_{i=1}^p \left(
    \frac{\sigma_i^4}{n^2} + \frac{2\sigma_i^2}{n} [\ba-(\bI-\bB)\bmu^*]_i^2
    \right)
    \bigg|
    \bSigma,\cD_t(\ba')
    \right)
     \\
     &= 
     \sum_{i=1}^p 
     \left[ 
     \Var_\bB \big([\ba-(\bI-\bB)\bmu^*]_i^2\big|\bSigma,\cD_t(\ba')\big) 
     + 
     2\bbE_\bB
     \left(
     \frac{\sigma_i^4}{n^2} + \frac{2\sigma_i^2}{n}[\ba-(\bI-\bB)\bmu^*]_i^2
     \Big|
     \bSigma,\cD_t(\ba')
     \right) 
     \right]
     \\
     & = \sum_{i=1}^p \left[ \big(2\sigma_i^4u_i^2 +4
     \sigma_i^2u_i(a_i-b_i)^2\big)+2\Big(\frac{\sigma_i^4}{n^2}+\frac{2\sigma_i^2}{n}\big(\sigma_i^2u_i+(a_i-b_i)^2\big)\Big)
     \right]
     \\
     &= 
     2\sum_{i=1}^p 
     \left[
     \sigma_i^4 \left( u_i^2+\frac2n u_i \right) + 2\sigma_i^2 u_i (a_i-b_i)^2
     \right] + c_1,
\end{aligned}
\end{equation}
where $c_1=2\sum_{i=1}^p(\nicefrac{\sigma_i^4}{n^2}+2\nicefrac{\sigma_i^2(a_i-b_i)^2}{n})$ is a scalar that only depends on $\bSigma, \ba$ and $\bb$.

Next, we calculate $\sigma_{g(\ba)|\cD_t(\ba')}^2$. 
Using Lemma~\ref{lm:1} and Lemma~\ref{lm:2}, we know that conditioning on $\cD_t(\ba')$, the random variable $\sigma_i^2$ follows an inverse Gamma distribution with hyperparameters $\alpha'_i$ and $\beta'_i$.
Since $\alpha'_i$ and $\beta'_i$ do not depend on $\ba'$, we then know that the distribution of $\bSigma$ conditioned on $\cD_t(\ba')$ does not depend on $\ba'$.
In turn, conditioning on $\cD_t(\ba')$, $c_1$ in Eq.~\eqref{prop1:eq4}, which only depends on $\bSigma$, $\ba$ and $\bb$, does not depend on $\ba'$. Thus, using Eq.~\eqref{prop1:eq4} in Eq.~\eqref{prop1:eq1} and then Eq.~\eqref{eq:moment-inverse-gamma}, we obtain
\begin{align*}
    &~~~~\sigma_{g(\ba)|\cD_t(\ba')}^2 \\
    &= \Var_{\bSigma} \Big(\sum_{i=1}^p \big(
     \frac{\sigma_i^2}{n}+\sigma_i^2u_i+(a_i-b_i)^2
     \big)\Big|\cD_t(\ba')\Big) + \bbE_{\bSigma} \Big(2\sum_{i=1}^p 
     \big[
     \sigma_i^4 \left( u_i^2+\frac2n u_i \right) + 2\sigma_i^2 u_i (a_i-b_i)^2
     \big] + c_1\Big|\cD_t(\ba')\Big)
    \\
    &= \sum_{i=1}^p\Var_{\bSigma|\cD_t(\ba')}( \sigma_i^2)\left( u_i+\frac1n \right)^2  + 2\sum_{i=1}^p\left[ \bbE_{\bSigma|\cD_t(\ba')}(\sigma_i^4)\left( u_i^2+\frac2n u_i \right)+ 2\bbE_{\bSigma|\cD_t(\ba')}(\sigma_i^2) u_i(a_i-b_i)^2\right] + c_1
    \\
    &= 
    \sum_{i=1}^p
    \frac{\beta_i'}{\alpha_i'-1}\left[\frac{\beta_i'}{(\alpha_i'-1)(\alpha'_i-2)} \left( u_i+\frac1n \right)^2 
    +
    \frac{2\beta_i'}{\alpha'_i-2} \left( u_i^2+\frac2nu_i \right) 
    +
    4u_i(a_i-b_i)^2
    \right] + c_1 
    \\
    &= 
    \sum_{i=1}^p
    \frac{\beta_i'}{\alpha_i'-1}\left[ \frac{\beta_i'(2\alpha'_i-1)}{(\alpha'_i-1)(\alpha'_i-2)}
    \left( u_i^2+\frac2nu_i \right) 
    + 
    4u_i(a_i-b_i)^2
    \right] + c,
\end{align*}
where $c$ is constant that does not depend on $\ba'$, thereby completing the proof.
\end{proof}

Conditioning on $\bSigma$ in the above result gives rise to the following corollary, which is  Proposition~\ref{prop:1} in the main text. The calculation is similar to the case of unknown $\bSigma$, but simplifies since we do not need to consider the randomness in $\bSigma$ anymore.
\begin{corollary}\label{cor:1}
Conditioning on $\bSigma$, the variance $\sigma^2_{g(\ba)|\cD_t(\ba')}$ satisfies
\begin{equation*}
    \begin{aligned}
    \sigma^2_{g(\ba)|\cD_t(\ba')}=2\sum_{i=1}^p \big(v_i^2+\frac2nv_i\sigma_i^2+2v_i(a_i-b_i)^2\big)+ c,
    \end{aligned}
\end{equation*}
where only $v_i := \sigma_i^2\bmu^{*\top}_{\pa(i)}\bM_i\big(\cD_t(\ba')\big)\bmu_{\pa(i)}^*$ depends on the augmented dataset $\cD_t(\ba')$; the variable $b_i:=\mu_i^*-\bmm_i(\cD_t)^\top\bmu^*_{\pa(i)}$ and constant $c$ do not depend on $\ba'$.
\end{corollary}

\begin{proof}
From Eq.~\eqref{prop1:eq1} and Eq.~\eqref{prop1:eq4}, we know that conditioning on $\bSigma$ results in
\begin{align*}
    \sigma^2_{g(\ba)|\cD_t(\ba')} 
    &=  
    2\sum_{i=1}^p 
    \left[\sigma_i^4 \left(u_i^2 + \frac2n u_i\right) + 2\sigma_i^2 u_i (a_i-b_i)^2 \right] + c_1 
    \\
    & = 2\sum_{i=1}^p \left( v_i^2 + \frac2n v_i \sigma_i^2 + 2v_i (a_i-b_i)^2 \right) + c_1,
\end{align*}
where $c=c_1$ does not depend on $\ba'$, thereby completing the proof.
\end{proof}

Next, we show in the following lemma that, for any $\ba\in\cA$, we can trivially make the uncertainty $\sigma^2_{g(\ba)|\cD_t(\ba')}$ smaller by increasing the magnitude of $\ba'$. 
This has previously been observed in the Bayesian linear regression setting~\cite{sapsis2020output}. 
To show this result formally for our problem, we need the above closed-form results and the posterior of the DAG-BLR distribution.

\begin{lemma}\label{lm:3}
For any $r>1$, there is 
\begin{equation*}
    \sigma^2_{g(\ba)|\cD_t(r\ba')}\leq \sigma^2_{g(\ba)|\cD_t(\ba')}.
\end{equation*}
\end{lemma}

\begin{proof}
Using \nameref{prop:1-a}, since only $u_i$ depends on the augmented dataset $\cD_t(\ba')$ and $\cD_t(r\ba')$, we can show that $ \sigma^2_{g(\ba)|\cD_t(r\ba')}\leq \sigma^2_{g(\ba)|\cD_t(\ba')}$ by showing $u_i\big(\cD_t(r\ba')\big)\leq u_i\big(\cD_t(\ba')\big)$. This follows directly from
\begin{align*}
    \bM_i\big(\cD_t(r\ba')\big) = \left(\bM_i(\cD_t)^{-1}+nr^2\bar\sfx_{\pa(i)}\bar\sfx_{\pa(i)}^\top \right)^{-1}\preceq  \left(\bM_i(\cD_t)^{-1}+n\bar\sfx_{\pa(i)}\bar\sfx_{\pa(i)}^\top \right)^{-1} = \bM_i\big(\cD_t(\ba')\big),
\end{align*}
which completes the proof.
\end{proof}

Therefore, we set $\cA$ to be the unit hypersphere $\bbS^{p-1}=\{\ba\in\bbR^p: \|\ba\|_2=1\}$ and only optimize over the direction of an intervention and not its magnitude. When $\nu$ is the uniform measure on $\cA$, we can calculate $h(\ba',\cD_t)$ analytically. The following proposition provides the explicit form, which is a generalized version of Proposition~\ref{prop:2} in the main text.

\begin{namedenv}[Proposition 2]\label{prop:2-a}
For $\cA=\bbS^{p-1}$ and $\nu$ being the uniform measure on $\cA$, the CIV acquisition function evaluated at $\ba'$ is given by
\begin{equation*}
\begin{aligned}
    h(\ba',\cD_t) = 
    \sum_{i=1}^p
    \frac{\beta_i'}{\alpha_i'-1}
    \left[ 
    \frac{\beta_i'(2\alpha'_i-1)}{(\alpha'_i-1)(\alpha'_i-2)}
    \left( u_i^2 + \frac2nu_i \right) + 4u_i \left( b_i^2+\frac1p \right)
    \right] + C,
\end{aligned}
\end{equation*}
where $C$ is a constant that does not depend on $\ba'$.
\end{namedenv}

\begin{proof}
Using \nameref{prop:1-a}, we obtain that
\begin{equation}\label{eq_3}
\begin{aligned}
    h(\ba',\cD_t) 
    &=
    \int_{\bbS^{p-1}} \sigma^2_{g(\ba)|\cD_t(\ba')} d\nu(\ba) 
    \\
    &= 
    \int_{\bbS^{p-1}} \left\{\sum_{i=1}^p\frac{\beta_i'}{\alpha_i'-1}\left[ \frac{\beta_i'(2\alpha'_i-1)}{(\alpha'_i-1)(\alpha'_i-2)} \left( u_i^2+\frac2nu_i \right) + 4u_i(a_i-b_i)^2
    \right] + c\right\} d\nu(\ba) 
    \\
    &= 
    \sum_{i=1}^p
    \frac{\beta_i'}{\alpha_i'-1} \left[ \frac{\beta_i'(2\alpha'_i-1)}{(\alpha'_i-1)(\alpha'_i-2)} \left( u_i^2+\frac2nu_i \right) + 4u_i\int_{\bbS^{p-1}}(a_i-b_i)^2d\nu(\ba)
    \right] + C 
    \\
    &= 
    \sum_{i=1}^p
    \frac{\beta_i'}{\alpha_i'-1} \left[ \frac{\beta_i'(2\alpha'_i-1)}{(\alpha'_i-1)(\alpha'_i-2)} \left( u_i^2 + \frac2nu_i \right) 
    + 4u_i
    \left(
    b_i^2 + \int_{\bbS^{p-1}} (a_i^2 - 2 a_i b_i) d\nu(\ba)
    \right)
    \right] + C,
\end{aligned}
\end{equation}
where $C$ is a constant that does not depend on $\ba'$. Since $\bbS^{p-1}$ is symmetric and isometric, we have
\begin{align}\label{eq_4}
    \int_{\bbS^{p-1}} a_ib_id\nu(\ba)  = b_i\int_{\bbS^{p-1}} a_id\nu(\ba) = 0,
\end{align}
and
\begin{align}\label{eq_5}
    \int_{\bbS^{p-1}} a_i^2 d\nu(\ba) & = \frac1p \sum_{i=1}^p\int_{\bbS^{p-1}} a_i^2 d\nu(\ba) =  \frac1p \int_{\bbS^{p-1}}\|\ba\|_2^2 d\nu(\ba) = \frac1p. 
\end{align}
Thus plugging Eq.~\eqref{eq_4} and Eq.~\eqref{eq_5} into Eq.~\eqref{eq_3}, we obtain that
\begin{align*}
    h(\ba',\cD_t) = 
    \sum_{i=1}^p\frac{\beta_i'}{\alpha_i'-1}\left[ \frac{\beta_i'(2\alpha'_i-1)}{(\alpha'_i-1)(\alpha'_i-2)} \left( u_i^2+\frac2nu_i \right) + 4u_i \left( b_i^2+\frac1p \right)
    \right] + C,
\end{align*}
which completes the proof.
\end{proof}

To obtain the simpler version provided in the main text as Proposition~\ref{prop:2} where $\bSigma$ is fixed, we can use a similar calculation following the same simplifications as in Corollary~\ref{cor:1}.

These closed-form expressions also provide a way to save computation computational costs when evaluating the acquisition function for multiple $\ba'$. Since $\alpha_i(\cD_t(\ba')), \beta_i(\cD_t(\ba'))$, and $\bmm_i(\cD_t(\ba'))$ do not depend on the value of $\ba'$, they can be calculated once and saved for evaluating different $\ba'$. The only variable in $h(\ba', \cD_t)$ that depends on $\ba'$, which needs to be calculated for each point, is $u_i$ (which depends on $\bM_i(\cD_t(\ba'))$).
This is because
\begin{align*}
    \bM_i\big(\cD_t(\ba')\big) = \left(\bM_i(\cD_t)^{-1}+n\bar\sfx'_{\pa(i)}\bar\sfx_{\pa(i)}^{\prime\top} \right)^{-1},
\end{align*}
where $\bar\bx'=(\bI-\bmm(\cD_t))^{-1}\ba'$ and $\bmm(\cD_t):=\bbE(\bB|\cD_t)$ is a $p$-by-$p$ matrix consisting of the current estimates of the edge weights of the underlying graph, i.e., its entries are $[\bmm(\cD_t)]_{i,\pa(i)}=\bmm_i(\cD_t)^\top$ for all $i\in [p]$ and  $[\bmm(\cD_t)]_{i,k}=0$ for all $i\in[p],~k\notin \pa(i)$.
Thus, we can evaluate $h(\ba',\cD_t)$ for each $\ba'\in\cA$ efficiently using the updated hyperparameters $\{\alpha(\cD_t),\beta(\cD_t),\bmm_i(\cD_t),\bM_i(\cD_t)\}_{i=1}^p$. {The  computational complexity of evaluating CIV for a single $\ba'$ is $O(p^3)$.}

\subsection{Optimization}\label{sec:b2}

When optimizing $h(\ba',\cD_t)$ on $\bbS^{p-1}$, we can relax the constraint to the unit ball $\mathbb{B}^p=\{\ba\in\bbR^{p}:\|\ba\|_2\leq1\}$ (which is convex). 
This is a consequence of Lemma~\ref{lm:3}, which shows that for all $\ba'\in \mathbb{B}^p$, it holds that $h(\ba',\cD_t)\geq h(\frac{\ba'}{\|\ba'\|_2},\cD_t)$. 
Therefore minimizing $h(\ba',\cD_t)$ over $\mathbb{B}^p$ will always return points on $\bbS^{p-1}$. 
In other words, we solve for $\ba^{(t+1)}$ by
\begin{equation*}
    \min_{\ba'\in\mathbb{B}^{p}} h(\ba',\cD_t),\quad s.t.~ \|\ba'\|_2\leq1.
\end{equation*}
The particular optimizers used for solving this problem are given in \emph{Methods} in the main text. 
We summarize the entire procedure in Algorithm~\ref{alg:1}.

\begin{algorithm}[ht]
\begin{algorithmic}
\Require causal graph $\cG$, samples $\cD$, target mean $\bmu^*$\;
\State initialize DAG-Wishart prior $\bbP(\bB,\bSigma)$;
\If{$\cD$ is not empty}
\State update $\bbP(\bB,\bSigma|\cD)$ as in \nameref{lm:1-a}\;
\EndIf
\While{not exceed time step limit}
\State select $\ba=\argmin_{\ba'} h(\ba',\cD)$ as in \nameref{prop:2-a}\;
\State obtain samples from $\ba$ and append to $\cD$\;
\State update $\bbP(\bB,\bSigma|\cD)$ as in \nameref{lm:1-a}\; 
\EndWhile
\Ensure estimated optimal intervention $(\bI-\bbE(\bB|\cD))\bmu^*$\;
\caption{Active Learning for Intervention Design}
\label{alg:1}
\end{algorithmic}
\end{algorithm}

\section{Details for Causal Output-weighted Integrated Variance}\label{sec:c}

\subsection{Background on Output-weighted Integrated Variance}

When studying sequential sampling for the purpose of estimating linear mappings between inputs and outputs, \cite{cohn1996active} proposed to use the variance of estimation integrated over the entire input space, as a criterion for sampling new data points. 
However, such an integration treats each input point equally and does not account for their different output values. 
This is far from ideal for applications such as extreme event estimation \cite{mohamad2018sequential}, where one would like to estimate the mapping accurately for inputs whose outputs occur with small probabilities. 
Therefore, \cite{sapsis2020output} proposed to weight each variance in the integration by the inverse of its output probability. 
The resulting acquisition function can also be interpreted as an approximation to the $l_1$-difference between the true mapping and the estimated mapping (see Section 4 in their paper).

Translating this to the optimal intervention design problem considered in our work, the output can be taken to be the optimality gap of the input interventions. 
Our goal is to optimize this output value, which depends non-linearly on the input. 
Towards this end, a stronger objective is to estimate the output value accurately over the entire input space. 
This results in the CIV acquisition function described in the main text, which uses a uniform measure $\nu$ on $\cA$. %
We can further improve this by putting more importance on inputs that result in better output values. To achieve this, we make use of the output~weighting scheme in~\cite{sapsis2020output}.

\subsection{Output-weight Calculations}

Let $f_{\ba}$ be the probability density function (pdf) of a uniform distribution on the input space $\cA=\bbS^{p-1}$. This induces a distribution on the optimality gap with pdf $f_{\|\ba-\bb\|_2^2}$. Recall that $\bb=(\bI-\bbE(\bB|\cD_t))\bmu^*$ is the estimated optimal intervention based on $\cD_t$. The output-weighted method uses the measure
\begin{equation*}
    d\nu(\ba) = \frac{f_\ba(\ba)}{f_{\|\ba-\bb\|_2^2}(\|\ba-\bb\|_2^2)}d\ba.
\end{equation*}
Using the formula for the area of the hyperspherical cap, we can calculate this ratio explicitly as follows.

\begin{lemma}\label{lm:4}
For any $\ba\in\bbS^{p-1}$, it holds that
\begin{align*}
    \frac{f_\ba(\ba)}{f_{\|\ba-\bb\|_2^2}(\|\ba-\bb\|_2^2)} \propto \left(\frac{1}{\|\bb\|_2^2-(\ba^\top\bb)^2}\right)^\frac{p-3}{2}.
\end{align*}
\end{lemma}

\begin{proof}
Since $f_\ba$ is the pdf of the uniform distribution on $\bbS^{p-1}$, it holds that 
\begin{align*}
    f_\ba(\ba) \propto 1.
\end{align*}
As illustrated in Supplementary Fig.~\ref{fig:s1}, the induced probability $\bbP(\|\ba-\bb\|_2^2\leq k)$ corresponds to the probability of $\ba$ being on the hyperspherical cap with angle $\phi$ satisfying 
\[\cos\phi =\frac{1+\|\bb\|_2^2-k}{2\|\bb\|_2}.\]
\begin{figure}[ht]
    \centering
    \includegraphics[width=.25\textwidth]{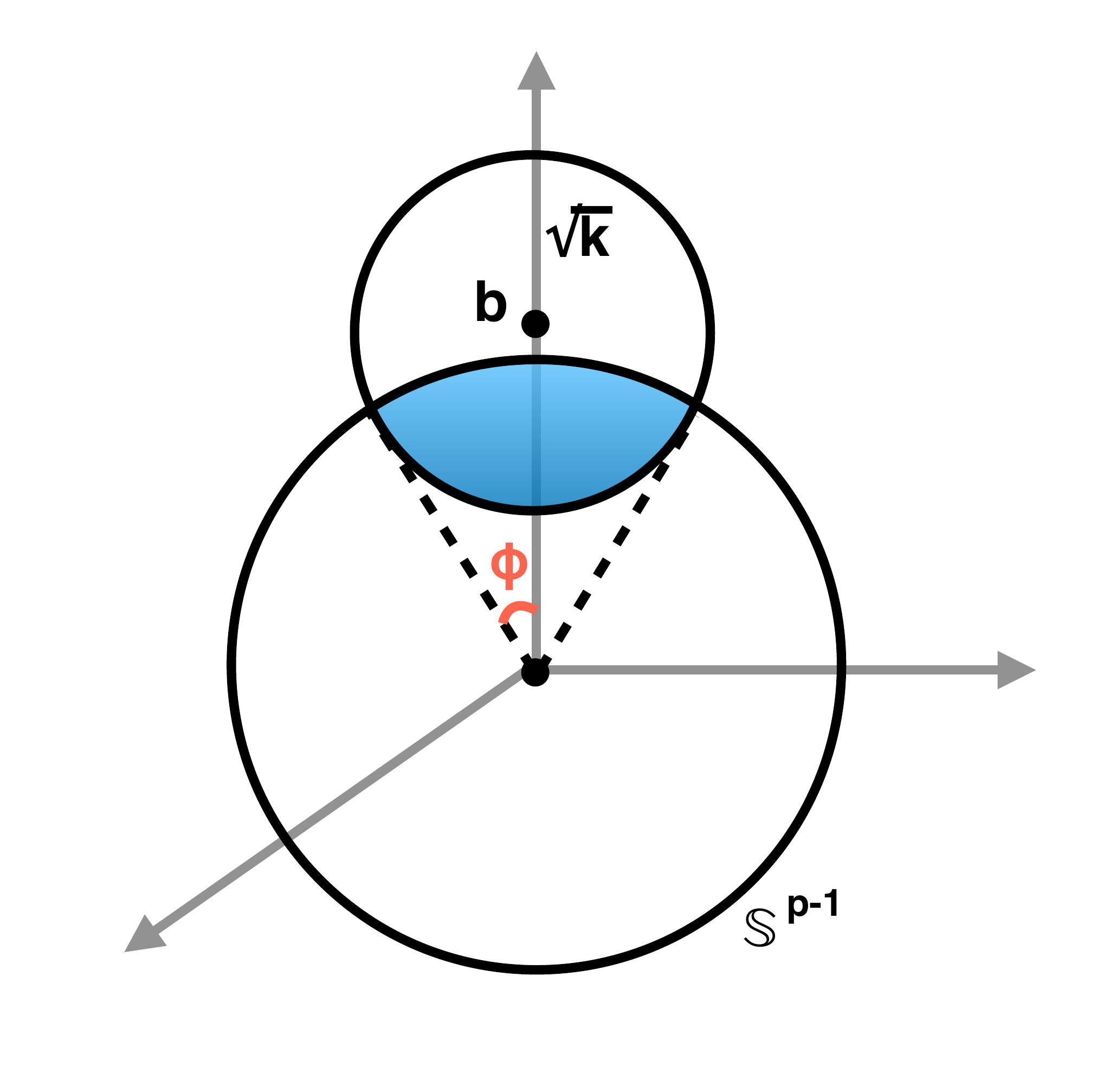}
    \caption{\rec{\textbf{Schematic illustrating the calculation of output-weighting.}} Hyperspherical cap corresponding to $\|\ba-\bb\|_2^2\leq k$ with angle $\phi$.}
    \label{fig:s1}
\end{figure}

\noindent Using the formula for the area of the hyperspherical cap \cite{li2011concise}, we obtain
\begin{align*}
    \bbP(\|\ba-\bb\|_2^2\leq k) \propto \int_0^{1-\cos^2\phi} u^{\frac{p-3}{2}}(1-u)^{-\frac12} du.
\end{align*}
Thus, using the Leibniz integral rule, we obtain
\begin{align*}
    f_{\|\ba-\bb\|_2^2}(\|\ba-\bb\|_2^2=k)&\propto \frac{\partial \bbP(\|\ba-\bb\|_2^2\leq k)}{\partial k} 
    \\
    &= (1-\cos^2\phi)^{\frac{p-3}{2}}(\cos\phi)^{-1}\frac{\partial (1-\cos^2\phi)}{\partial k}
    \\
    & = (1 + \cos \phi)^{\frac{p-3}{2}} (1 - \cos \phi)^{\frac{p-3}{2}} (\cos\phi)^{-1}\frac{\partial (1-\cos^2\phi)}{\partial k}
    \\
    & = \left(\left(1+\frac{1+\|\bb\|_2^2-k}{2\|\bb\|_2}\right)\left(1-\frac{1+\|\bb\|_2^2-k}{2\|\bb\|_2}\right)\right)^{\frac{p-3}{2}} \cdot \frac{1}{\|\bb\|_2}
    \\
    & =\frac{1}{\|\bb\|_2}\left(\frac{\big((1+\|\bb\|_2)^2-k\big)\big(-(1-\|\bb\|_2)^2+k\big)}{4\|\bb\|_2^2}\right)^{\frac{p-3}{2}}.
\end{align*}
As a consequence, we obtain that
\begin{align*}
    \frac{f_\ba(\ba)}{f_{\|\ba-\bb\|_2^2}(\|\ba-\bb\|_2^2)} 
    &\propto 
    \left(\frac{1}{\big((1+\|\bb\|_2)^2-\|\ba-\bb\|_2^2\big)\big(-(1-\|\bb\|_2)^2+\|\ba-\bb\|_2^2\big)}\right)^{\frac{p-3}{2}}
    \\
    & = \left(
    \frac{1}{(2 \| \bb \|_2 + 2 \ba^\top \bb) (2 \| \bb \|_2 - 2 \ba^\top \bb)}
    \right)^{\frac{p-3}{2}}
    \\
    & \propto \left(\frac{1}{\|\bb\|_2^2-(\ba^\top\bb)^2}\right)^\frac{p-3}{2},
\end{align*}
where we used $\|\ba\|_2=1$ in the second line. This completes the proof.
\end{proof}

\begin{remark}\label{rem_1}
Note that based on this formula, the resulting weighting scheme is symmetric, in the sense that $d\nu(\ba)=d\nu(-\ba)$ for any $\ba\in\bbS^{p-1}$. 
Therefore when up-weighting $\ba$'s that are close to $\ba^*$, we also up-weight $\ba$'s that are close to $-\ba^*$. 
Since the optimal intervention can be obtained by minimizing or maximizing $\|\ba-\ba^*\|_2^2$, a reasonable choice is to up-weight $\ba$'s that are close to its extremum.
The symmetric weighting scheme adopted here achieves this.
\end{remark}

When $p>3$, this measure puts more weight on $\ba\in\bbS^{p-1}$ with larger $|\ba^\top\bb|$, i.e., directions that are more aligned with the estimated optimal intervention $\bb$. 
However, when $p=3$, this measure degenerates to the uniform measure on $\cA$.
To obtain the desired behavior when $p=3$, we instead propose using the bimodal von-Mises Fisher distribution
\begin{align*}
d\nu(\ba) = \left(e^{\kappa\ba^\top\bb}+e^{-\kappa\ba^\top\bb}\right)d\ba,
\end{align*}
with a suitable scale $\kappa>0$. Note that this measure discounts $\ba$ exponentially with the value of $|\ba^\top\bb|$, and therefore behaves similar to the output-weighted measure in high dimensions.

\subsection{Approximation of CIV-OW}
We now show how to evaluate the causal output-weighted integrated variance (CIV-OW) acquisition function
\begin{align*}
    h_\mathrm{OW}(\ba',\cD_t) =  \int_\cA \sigma^2_{g(\ba)|\cD_t(\ba')} d\nu(\ba),
\end{align*}
where using Lemma~\ref{lm:4} and Remark~\ref{rem_1},
\begin{align*}
    d\nu(\ba) \propto  \begin{cases}
    \left(\frac{1}{\|\bb\|_2^2-(\ba^\top\bb)^2}\right)^\frac{p-3}{2}d\ba, & p>3\\
    \left(e^{\kappa\ba^\top\bb}+e^{-\kappa\ba^\top\bb}\right)d\ba, & p=3.
    \end{cases}
\end{align*}
In the case when $p=2$, there are only two variables and it is easy to show that the CIV acquisition function always selects the intervention that targets just the upstream variable, no matter what the measure $\nu$ is. Therefore, we do not need to consider the case $p=2$.

To evaluate CIV-OW for a particular $\ba'$, first, recall that using \nameref{prop:1-a}, we have
\begin{equation}\label{c:1}
    h_\mathrm{OW}(\ba',\cD_t) 
    = 
    \sum_{i=1}^p
    \frac{\beta_i'}{\alpha_i'-1}\left[ \frac{\beta_i'(2\alpha'_i-1)}{(\alpha'_i-1)(\alpha'_i-2)} 
    \left( u_i^2 + \frac2n u_i \right) 
    + 
    4 u_i \int_\cA (a_i-b_i)^2 d\nu(\ba)
    \right] + C,
\end{equation}
where $C$ is a constant that does not depend on $\ba'$. For $p=3$, we use the first and second moments of the bimodal von-Mises Fisher distribution (see Corollary 2 in \cite{hillen2017moments}) to derive that
\begin{equation}\label{c:2}
    \int_\cA (a_i-b_i)^2d\nu(\ba) = \left(
    \frac{\coth{\kappa}}{\kappa}-\frac{1}{\kappa^2}
    \right) 
    + 
    \left(
    2-\frac{3\coth{\kappa}}{\kappa}+\frac{3}{\kappa^2}
    \right) b_i^2.
\end{equation}
For $p>3$, we use the following approximation for the integral:
\begin{equation}\label{c:3}
    \int_\cA (a_i-b_i)^2d\nu(\ba) \approx (1+\|\bb\|_2^2)\frac{(\|\bb\|_2^2-b_i^2)^{\frac{4-p}{2}}b_i^2}{\sum_{j=1}^p(\|\bb\|_2^2-b_j^2)^{\frac{4-p}{2}}b_j^2}.
\end{equation}
This approximation is obtained as follows: First, note that 
\begin{equation}\label{eq_6}
\begin{aligned}
    \sum_{i=1}^p \int_\cA (a_i-b_i)^2 d\nu(\ba) & = \int_\cA (\|\ba\|_2^2 + \|\bb\|_2^2 -2\ba^\top\bb) d\nu(\ba) \\
    & = 1+\|\bb\|_2^2,
\end{aligned}
\end{equation}
where we used the symmetry of $\cA$ and $\nu(\ba)=\nu(-\ba)$ to derive $\int_\cA \ba d\nu(\ba) = \mathbf{0}$ in the second equation. To obtain only one of the terms in the summation instead of the sum over all possible $i\in[p]$, it is sufficient to identify the ratio of the integration between two summands $i,j\in[p]$. We obtain this ratio using the following approximation:
\begin{equation}\label{eq_7}
\begin{aligned}
    \frac{\int_\cA (a_i-b_i)^2 d\nu(\ba)}{\int_\cA (a_j-b_j)^2 d\nu(\ba)}\approx \frac{c_i}{c_j} :& = \frac{b_i^2\int_{\bbS^{p-2}(\sqrt{1-b_i^2/\|\bb\|^2})}\left(\|\bb_{-i}\|_2^2-(\ba_{-i}^\top\bb_{-i})^2\right)^\frac{3-p}{2}d\ba_{-i}
    }{b_j^2\int_{\bbS^{p-2}(\sqrt{1-b_j^2/\|\bb\|^2_2})}\left(\|\bb_{-j}\|_2^2-(\ba_{-j}^\top\bb_{-j})^2\right)^\frac{3-p}{2}d\ba_{-j}} \\
    & = \frac{b_i^2\|\bb_{-i}\|_2^{4-p}\int_{\bbS^{p-2}(1)}\left(\|\bb\|_2^2-\|\bb_{-i}\|^2a_1^2\right)^\frac{3-p}{2}d\ba
    }{b_j^2\|\bb_{-j}\|_2^{4-p}\int_{\bbS^{p-2}(1)}\left(\|\bb\|_2^2-\|\bb_{-j}\|^2a_1^2\right)^\frac{3-p}{2}d\ba} \\
    & = \frac{b_i^2\|\bb_{-i}\|_2^{4-p}\int_{-1}^1\left(\frac{1-a_1^2}{b_i^2+\|\bb_{-i}\|^2(1-a_1^2)}\right)^{\frac{p-3}{2}}da_1
    }{b_j^2\|\bb_{-j}\|_2^{4-p}\int_{-1}^1\left(\frac{1-a_1^2}{b_j^2+\|\bb_{-j}\|^2(1-a_1^2)}\right)^{\frac{p-3}{2}}da_1} \approx \frac{b_i^2 \big(\|\bb\|_2^2-b_i^2\big)^{\frac{4-p}{2}}}{b_j^2\big(\|\bb\|_2^2-b_j^2\big)^{\frac{4-p}{2}}},
\end{aligned}
\end{equation}
where in the third line we used a change of variables using the area element. The intuition behind the first approximation is that $\nu(\ba)$ is exponentially smaller when $|a_i|\neq|b_i|$ in high dimensions with large $p$ and therefore we can approximate the integration around $|a_i|=|b_i|$. Since $(a_i-b_i)^2=0$, this is then contracted to $a_i=-b_i$ for which we calculate the ratio $\frac{c_i}{c_j}$. In the second approximation, we approximated the integration with a constant regardless of $i$. The intuition for this is that when $p$ grows, the integration will be dominated by $a_1=0$, where the integrated terms for all $i$ are the same.

Combining the sum and ratio formulas in Eq.~\eqref{eq_6} and Eq.~\eqref{eq_7} gives rise to the approximation in Eq.~\eqref{c:3}. We also verified this approximation numerically; see Supplementary Fig.~\ref{fig:s2}. Finally, using Eq.~\eqref{c:2} and Eq.~\eqref{c:3} in Eq.~\eqref{c:1}, we can evaluate (approximately when $p>3$) CIV-OW for any given $\ba'$ in closed form as follows: 
\begin{equation}
\begin{aligned}
    &h_\mathrm{OW}(\ba',\cD_t)\\ \approx & \begin{cases}
    \sum_{i=1}^p\frac{\beta_i'}{\alpha_i'-1}\left[ \frac{\beta_i'(2\alpha'_i-1)}{(\alpha'_i-1)(\alpha'_i-2)}(u_i^2+\frac2nu_i) + \frac{4u_i(1+\|\bb\|_2^2)(\|\bb\|_2^2-b_i^2)^{\frac{4-p}{2}}b_i^2}{\sum_{j=1}^p(\|\bb\|_2^2-b_j^2)^{\frac{4-p}{2}}b_j^2}
    \right] + C, & p>3\\
    \sum_{i=1}^p\frac{\beta_i'}{\alpha_i'-1}\left[ \frac{\beta_i'(2\alpha'_i-1)}{(\alpha'_i-1)(\alpha'_i-2)}(u_i^2+\frac2nu_i) + 4u_i\big((\frac{\coth{\kappa}}{\kappa}-\frac{1}{\kappa^2}) + (2-\frac{3\coth{\kappa}}{\kappa}+\frac{3}{\kappa^2})b_i^2\big)
    \right] + C, & p=3.
    \end{cases}.
\end{aligned}
\end{equation}
This then allows for similar optimization methods as described in Sec.~\ref{sec:b2} to be used for selecting $\ba^{(t+1)}$. {The computational complexity of evaluating CIV-OW for a single $\ba'$ is $O(p^3)$.}

\begin{figure}[ht]
    \centering
     \begin{subfigure}[b]{0.23\textwidth}
         \centering
         \includegraphics[width=.9\textwidth]{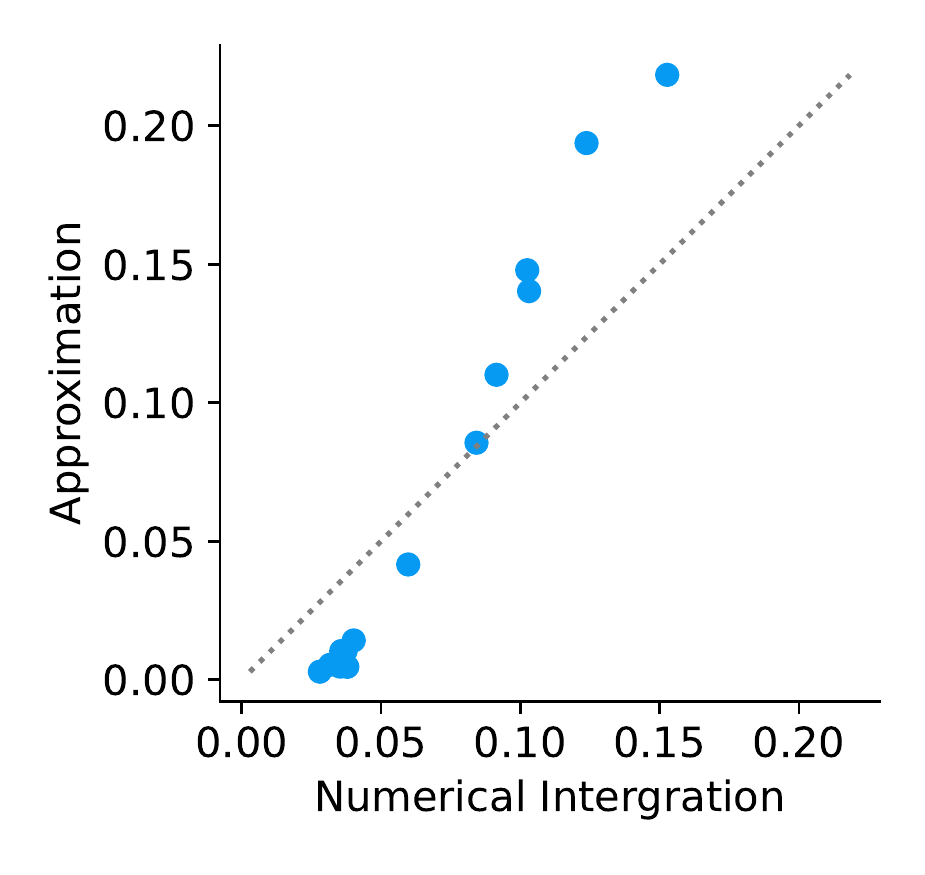}
     \end{subfigure}
     \begin{subfigure}[b]{0.23\textwidth}
         \centering
         \includegraphics[width=.9\textwidth]{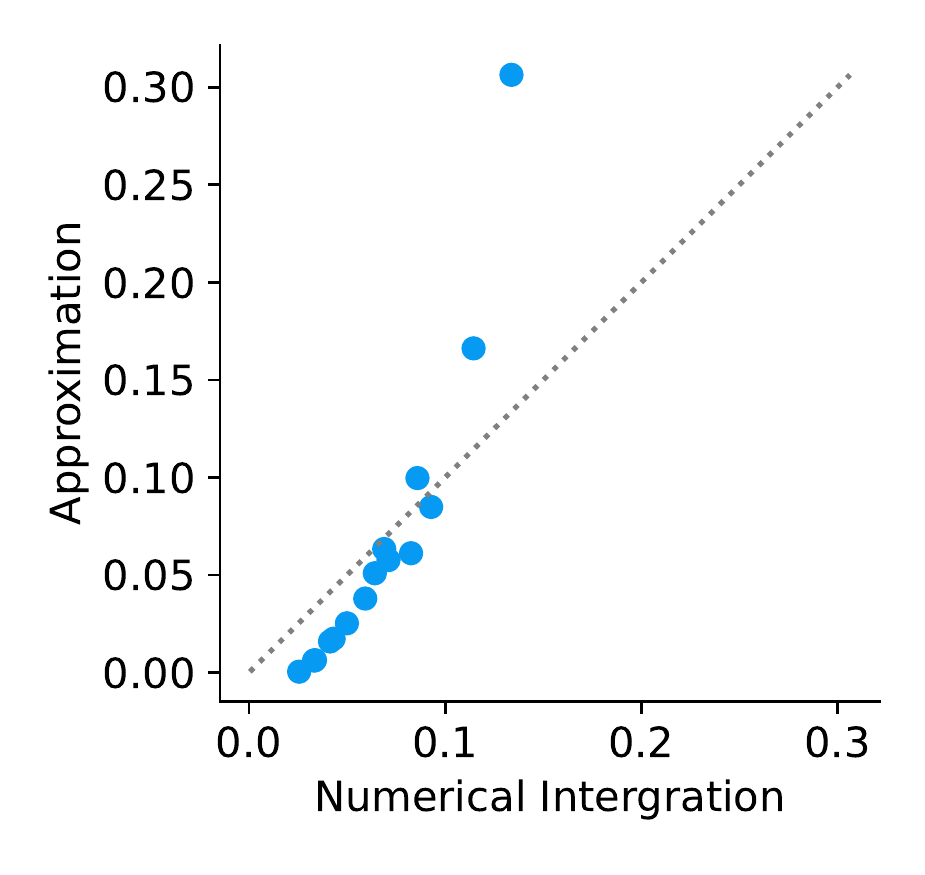}
     \end{subfigure}
     \begin{subfigure}[b]{0.23\textwidth}
         \centering
         \includegraphics[width=.9\textwidth]{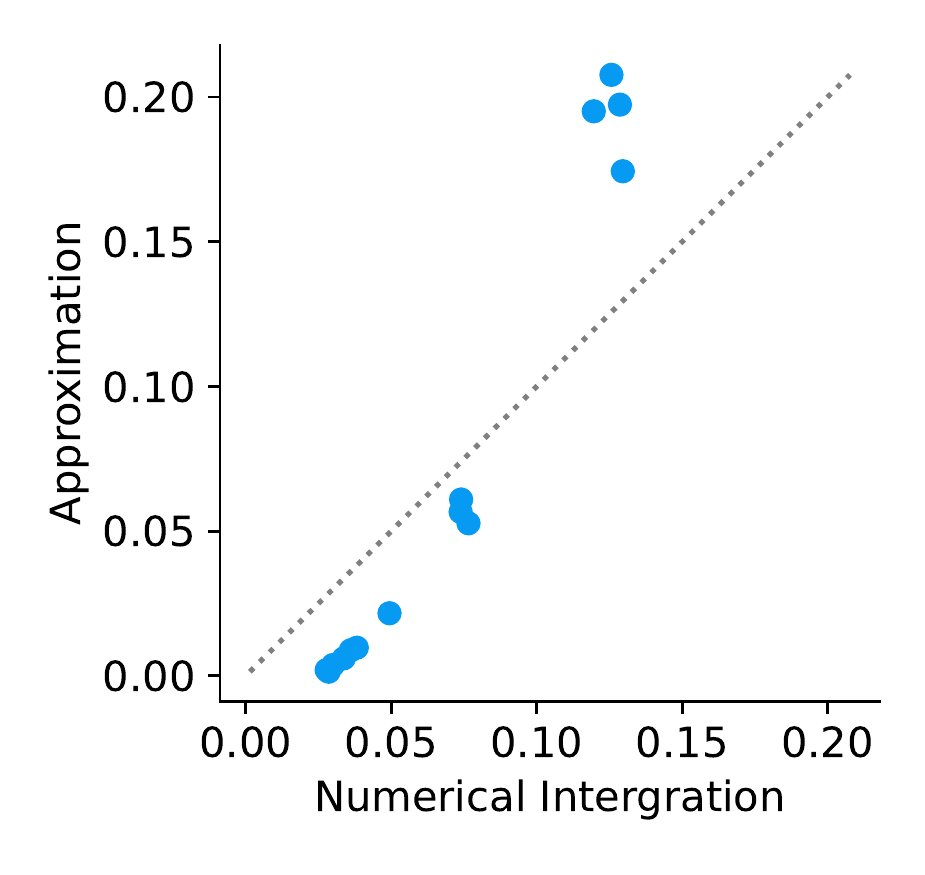}
     \end{subfigure}
    \caption{\rec{\textbf{Verification of the proposed approximation for calculating the output-weighting.} Simulations on} $p=15$ with three randomly generated $\bb$ (with $\|\bb\|_2=1$). Each dot corresponds to an $i\in[p]$, where the x-axis is the numerical integration of $\int_\cA (a_i-b_i)^2d\nu(\ba)$ by sampling $1000$ points from $\cA$ and the y-axis is the approximation $b_i^2(\|\bb\|_2^2-b_i^2)^{\frac{4-p}{2}}$.}\label{fig:s2}
\end{figure}

\section{Extensions to the Unknown DAG Setting}\label{sec:d}

So far, we have assumed that the underlying causal structure is known and represented by a DAG. In the case of an unknown DAG, we also need to take into account the uncertainty over the DAG structure. We can model this uncertainty over the DAG $\cG$ when computing the conditional variance $\sigma^2_{g(\ba) \mid \cD_t}$ that is used to compute the causal integrated variance $h(\ba', \cD_t)$.
In particular, we can add to our generative model a prior $\bbP(\cG)$ over DAGs, and use the DAG-BLR prior \textit{conditional} on the DAG $\cG$.

To update the posterior on the causal structure along with the causal mechanisms (i.e., the edge weights), we can use the interventional Bayesian equivalance (iBGe) score \cite{geiger2002parameter,kuipers2022interventional} as follows. Denoting the collection of samples by $\cD$ for simplicity, then the joint posterior of the unknown DAG $\cG$, the edge weights $\bB$ and noise variances $\bSigma$ satisfy 
\[\bbP(\cG,\bB, \bSigma|\cD) = \bbP(\bB, \bSigma|\cG, \cD) \bbP(\cG|\cD).\]
To select the next intervention based on the CIV acquisition functions, we need to integrate the uncertainty $\sigma^2_{g(\ba)|\cD}$, where $g(\ba)$ only depends on the joint posterior of $\bB$ and $\bSigma$. Marginalizing over $\cG$, we obtain that $\bbP(\bB, \bSigma|\cD) = \sum_{\cG} \bbP(\bB, \bSigma|\cG, \cD) \bbP(\cG|\cD)$. Following \cite{agrawal2019abcd}, we can approximate this posterior by sampling several DAGs $\cG_1,\dots,\cG_k$ from the iBGe score $\bbP(\cG|\cD)$: 
\begin{align*}
    \bbP(\bB,\Sigma|\cD)\approx\sum_{l=1}^k \bbP(\bB, \bSigma|\cG_l, \cD)\cdot \omega(\cG_l),
\end{align*}
where $\omega(\cG_l) = \frac{\bbP(\cG_l|\cD)}{\sum_{l'=1}^k \bbP(\cG_{l'}|\cD)}$ is the importance weight for the $l$-th sampled DAG $\cG_l$. We can approximate the uncertainty using an unbiased estimator of its lower bound:
\begin{align*}
    \sigma^2_{g(\ba)|\cD} & = \Var_{\cG}\big( \bbE(g(\ba)|\cG, \cD)|\cD\big) + \bbE_{\cG}\big( \Var(g(\ba)|\cG, \cD)|\cD\big) \\
    & \geq \bbE_{\cG}\big( \Var(g(\ba)|\cG,\cD)|\cD\big) \\
    & \approx \sum_{l=1}^k \Var(g(\ba)|\cG_l,\cD) \cdot \omega(\cG_l).
\end{align*}
For this, we used the law of total variance in the first equation and then dropped the variance over $\cG$ since it is usually intractable. 
Based on this approximation, we can then approximate the CIV acquisition function as follows:
\begin{align*}
    h(\ba',\cD) & \approx \sum_{l=1}^k \int_{\ba}  \Var\big(g(\ba)|\cG_l,\cD(\ba')\big) \cdot \omega(\cG_l) d\ba\\
    & = \sum_{l=1}^k h(\ba,\cD_t|\cG_l)\cdot\omega(\cG_l),
\end{align*}
where $h(\ba,\cD_t|\cG_l)$ can be evaluated as described in Sec.~\ref{sec:b} for any given DAG $\cG_l$. 
This procedure allows the CIV acquisition function to be optimized even in the unknown DAG case.
Of note, since the number of DAGs grows super-exponentially with the graph size $p$, the computational budget is often not sufficient to obtain a large enough sample size $k$ to guarantee accurate approximation.
Therefore in practice, one often uses just one DAG, for example the one obtained by maximizing the likelihood of or sampling from $\bbP(\cG|\cD)$; see~\cite{agrawal2019abcd}.
{We note that sampling DAGs from $\bbP(\cG|\cD)$ is a difficult problem and is considered an open research direction on its own (c.f. \cite{wienobst2021polynomial}).}

\section{Proofs for the Mutual Information Bound}\label{sec:e}

To prove Theorem~\ref{thm:1} in the main text, we need several lemmas. The following lemma shows that the mutual information between a Gaussian and a non-Gaussian variable can be lower bounded by transforming the non-Gaussian variable into a Gaussian variable. For completeness we provide our proof here and note that alternative proofs exist, for example by using information geometric techniques as in \cite{cardoso2003dependence}.

\begin{lemma}\label{lm:5}
Let $\sfx_g$ be a multivariate Gaussian variable following $\cN(\mu_{\sfx_g},\Sigma_{\sfx_g})$, and let $\sfz$ be an arbitrary univariate random variable with mean $\mu_{\sfz}$ and variance $\sigma_{\sfz}^2$.
Denoting by $\rho=\sigma_{\sfz}^{-1}\Sigma_{\sfx_g}^{-\frac12}\bbE_{\sfx_g,\sfz}((x_g-\mu_{\sfx_g})(z-\mu_{\sfz}))$ and $\sfz_g = \sigma_\sfz \rho^\top\Sigma_{\sfx_g}^{-\frac12}(\sfx_g - \mu_{\sfx_g})+ \cN\big(\mu_{\sfz}, (1-\|\rho\|_2^2)\sigma_{\sfz}^2\big)$, then
\begin{equation*}
    I(\sfz;\sfx_g)\geq I(\sfz_g;\sfx_g).
\end{equation*}
\end{lemma}

\begin{remark}
Note that here $\sfz_g$ can be understood as a transformation of $\sfz$ to a linear combination of $\sfx_g$ and some independent Gaussian noise. The choices of the linear coefficients are to make sure that i) the correlation between $\sfz_g$ and $\sfx$ is kept the same as that between $\sfz$ and $\sfx$; ii) $\sfz_g$ has the same mean and variance as $\sfz$. The reason for these two requirements will become evident in the proof.
\end{remark}

\begin{proof}[Proof of Lemma~\ref{lm:5}]
First, note that
\begin{align*}
    I(\sfz;\sfx_g) & =
    \bbE_{\sfx_g,\sfz} \left(\log\frac{\bbP_{\sfx_g,\sfz}(x_g,z)}{\bbP_{\sfx_g}(x_g)\bbP_\sfz(z)}\right) \\
    %
    & = 
    \bbE_{\sfx_g,\sfz}\left(\log\frac{\bbP_{\sfx_g|\sfz}(x_g|z)}{\bbP_{\sfx_g|\sfz_g}(x_g|z)}\right)
    + 
    \bbE_{\sfx_g,\sfz} \left( \log\frac{\bbP_{\sfx_g|\sfz_g}(x_g|z)}{\bbP_{\sfx_g}(x_g)} 
    \right)
    \\
    & = \bbE_{\sfz}\bbE_{\sfx_g|\sfz}\left(\log\frac{\bbP_{\sfx_g|\sfz}(x_g|z)}{\bbP_{\sfx_g|\sfz_g}(x_g|z)}\right)
    + 
    \bbE_{\sfx_g,\sfz} \left( \log\frac{\bbP_{\sfx_g|\sfz_g}(x_g|z)}{\bbP_{\sfx_g}(x_g)} 
    \right)
    \\
    & \geq \bbE_{\sfx_g,\sfz} \left( \log\frac{\bbP_{\sfx_g|\sfz_g}(x_g|z)}{\bbP_{\sfx_g}(x_g)} \right) = \bbE_{\sfx_g,\sfz} \left( \log\frac{\bbP_{\sfz_g|\sfx_g}(z|x_g)}{
    \bbP_{\sfz_g}(z)} \right),
\end{align*}   
where $\log = \log_2$ and we used the Gibb's inequality $\bbE_{\sfx_g|\sfz}\left(\log\frac{\bbP_{\sfx_g|\sfz}(x_g|z)}{\bbP_{\sfx_g|\sfz_g}(x_g|z)}\right)\geq0$ in the fourth line.
Since $\sfz_g$ is a Gaussian variable with mean and variance
\begin{equation}\label{lm5:eq1}
    \begin{aligned}
        \mu_{\sfz_g} & =\mu_\sfz, \\
        \sigma^2_{\sfz_g} & =\sigma^2_{\sfz},
    \end{aligned}
\end{equation}
conditioning on $\sfx_g=x_g$ results in a Gaussian variable with mean and variance
\begin{equation}\label{lm5:eq2}
    \begin{aligned}
    \mu_{\sfz_g|x_g} &= \mu_{\sfz} + \sigma_\sfz\rho^\top\Sigma_{\sfx_g}^{-\frac12}(x_g - \mu_{\sfx_g}), \\
    \sigma^2_{\sfz_g|x_g} &= (1-\|\rho\|^2_2)\sigma^2_{\sfz}.
    \end{aligned}
\end{equation}
Therefore, we obtain
\begin{align*}
    I(\sfz;\sfx_g) 
    &\geq 
    \bbE_{\sfx_g,\sfz} \left(
    \log\frac{\bbP_{\sfz_g|\sfx_g}(z|x_g)}{\bbP_{\sfz_g}(z)} 
    \right)
    \\
    &= 
    \bbE_{\sfx_g,\sfz} \left(
    \log \frac{\sigma_{\sfz_g|x_g}^{-1}\exp\left(-\frac{1}{2\sigma_{\sfz_g|x_g}^{2}}(z-\mu_{\sfz_g|x_g})^2\right)}{\sigma_{\sfz_g}^{-1}\exp\big(-\frac{1}{2\sigma_{\sfz_g}^{2}}(z-\mu_{\sfz_g})^2\big)} 
    \right)
    \\
    &= 
    \frac12\log\frac{1}{1-\|\rho\|^{2}_2} 
    + 
    \frac{\log e}{2} \bbE_{\sfx_g,\sfz} \left(
    \frac{(z-\mu_{\sfz_g})^2}{\sigma_{\sfz_g}^{2}}-\frac{(z-\mu_{\sfz_g|x_g})^2}{\sigma_{\sfz_g|x_g}^{2}}\right)
 \\
    &= 
    \frac12\log\frac{1}{1-\|\rho\|^{2}_2} 
    + 
    \frac{\log e}{2} \left(
    \frac{\bbE_{\sfz}(z-\mu_{\sfz_g})^2}{\sigma_{\sfz_g}^{2}}-\frac{\bbE_{\sfx_g,\sfz}(z-\mu_{\sfz_g|x_g})^2}{\sigma_{\sfz_g|x_g}^{2}}
    \right)
    \end{align*}
  where we used Eq.~\eqref{lm5:eq1} and Eq.~\eqref{lm5:eq2} for the last equation.  As a consequence,
\begin{align*}
    I(\sfz;\sfx_g) 
    &\geq 
    \frac12\log\frac{1}{1-\|\rho\|^{2}_2} 
    +
    \frac{\log e}{2} \left(
    \frac{\bbE_{\sfz}(z-\mu_{\sfz})^2}{\sigma_{\sfz}^{2}}
    -
    \frac{\bbE_{\sfx_g,\sfz}\big((z-\mu_{\sfz})-\sigma_\sfz \rho^\top\Sigma_{\sfx_g}^{-\frac12}(x_g-\mu_{\sfx_g})\big)^2}{(1-\|\rho\|^2_2)\sigma_{\sfz}^{2}}
    \right) 
    \\
    &= 
    \frac12\log\frac{1}{1-\|\rho\|^{2}_2} + \frac{\log e}{2} \left( 1 - \frac{\sigma_{\sfz}^2+\|\rho\|^2_2\sigma_{\sfz}^2-2\|\rho\|^2_2\sigma_{\sfz}^2}{(1-\|\rho\|^2_2)\sigma_{\sfz}^{2}} \right) 
    \\
    &= 
    \frac12 \log \frac{1}{1-\|\rho\|^{2}_2} = I(\sfz_g;\sfx_g),
\end{align*}
which completes the proof.
\end{proof}

Next, we derive an equation for the correlation between two random variables. Since the mutual information between two Gaussian random variables is a function of their correlation, we can then use the following lemma to get another expression of the mutual information. Note that in the following, we use $\sigma^2_{\cdot}$ to denote the variance of a scalar random variable, e.g., $\sigma_{\sfz|\sfx=x}^2=\Var(\sfz|\sfx=x)$ denotes the variance of the scalar variable $\sfz$ following the conditional distribution $\bbP(\sfz|\sfx=x)$.

\begin{lemma}\label{lm:6}
Let $\sfx$ and $\sfz$ be random variables such that $\sfx$ is multivariate and $\sfz$ is univariate. For any multivariate $\sfx'$ and univariate $\sfz'$, we define $\rho_{\sfx',\sfz'}:=\sigma_{\sfz'}^{-1}\bSigma_{\sfx'}^{-\frac12}\bbE_{\sfx',\sfz'}((x'-\mu_{\sfx'})(z'-\mu_{\sfz'}))$ . It holds that
\begin{equation}\label{eq_8}
    \|\rho_{\sfx,\sfz}\|_2^2 = \left( 1-\frac{\bbE_{\sfx} (\sigma_{\sfz|\sfx}^2)}{\sigma^2_{\sfz}} \right)
    \|\rho_{\sfx,\bbE(\sfz|\sfx)}\|_2^2.
\end{equation}
\end{lemma}

\begin{proof}
Multiplying both sides of the desired equality \eqref{eq_8} by $\sigma_{\sfz}^2$ and expanding $\| \rho_{\sfx,\sfz} \|_2^2$, we obtain that
\begin{align*}
     {}& \|\rho_{\sfx,\sfz}\|_2^2 
     = 
     \left( 1-\frac{\bbE_{\sfx} (\sigma_{\sfz|\sfx}^2)}{\sigma^2_{\sfz}} \right)
    \|\rho_{\sfx,\bbE(\sfz|\sfx)}\|_2^2
   \quad  \Longleftrightarrow \quad \Cov(\sfx,\sfz)^\top\bSigma_\sfx^{-1}\Cov(\sfx,\sfz) 
    = 
    (\sigma^2_{\sfz}-\bbE_{\sfx} (\sigma_{\sfz|\sfx}^2))
    \cdot \|\rho_{\sfx,\bbE(\sfz|\sfx)}\|_2^2.
\end{align*}
Thus, it suffices to show that $\Cov(\sfx,\sfz)^\top\bSigma_\sfx^{-1}\Cov(\sfx,\sfz) = (\sigma^2_{\sfz}-\bbE_{\sfx} (\sigma_{\sfz|\sfx}^2)) \cdot \|\rho_{\sfx,\bbE(\sfz|\sfx)}\|_2^2$. 
By the law of total expectation,
\begin{equation}\label{lm6:eq1}
    \begin{aligned}
        \Cov(\sfx,\sfz) = \bbE(\sfx\sfz)-\bbE\sfx\cdot\bbE\sfz = \bbE\big(\sfx\bbE(\sfz|\sfx)\big)-\bbE\sfx\cdot\bbE\big(\bbE(\sfz|\sfx)\big)=\Cov\big(\sfx,\bbE(\sfz|\sfx)\big),
    \end{aligned}
\end{equation}
and by the law of total variance,
\begin{equation}\label{lm6:eq2}
    \begin{aligned}
    \sigma^2_{\sfz}-\bbE_{\sfx} (\sigma_{\sfz|\sfx}^2)
    &=
    \Var(\sfz) - \bbE_{\sfx}\big(\Var(\sfz|\sfx=x)\big)
    \\
    &=
    \Var_{\sfx}\big(\bbE(\sfz|\sfx=x)\big)
    = 
    \sigma^2_{\bbE(\sfz|\sfx)}.    
    \end{aligned}
\end{equation}
Therefore, using Eq.~\eqref{lm6:eq1} and Eq.~\eqref{lm6:eq2} together with the definition of $\rho_{\sfx,\bbE(\sfz|\sfx)}$ results in
\begin{align*}
    \Cov(\sfx,\sfz)^\top\bSigma_\sfx^{-1}\Cov(\sfx,\sfz) 
    &= 
    \Cov\big(\sfx,\bbE(\sfz|\sfx)\big)^\top\bSigma_\sfx^{-1}\Cov\big(\sfx,\bbE(\sfz|\sfx)\big), 
    \\
    (\sigma^2_{\sfz}-\bbE_{\sfx} (\sigma_{\sfz|\sfx}^2))\|\rho_{\sfx,\bbE(\sfz|\sfx)}\|_2^2
    &= 
    \Cov\big(\sfx,\bbE(\sfz|\sfx)\big)^\top\bSigma_\sfx^{-1}\Cov\big(\sfx,\bbE(\sfz|\sfx)\big),
\end{align*}
which proves that $\Cov(\sfx,\sfz)^\top\bSigma_\sfx^{-1}\Cov(\sfx,\sfz) = (\sigma^2_{\sfz}-\bbE_{\sfx} (\sigma_{\sfz|\sfx}^2))\|\rho_{\sfx,\bbE(\sfz|\sfx)}\|_2^2$, as desired.
\end{proof}

Next, we show that we can transform a non-Gaussian variable into a Gaussian variable using deterministic functions. This is an extension of Theorem 2 in \cite{painsky2017gaussian}, whose scalar version is proven in \cite{shayevitz2011optimal}. We restate the extended version in the following lemma and provide a proof here.

\begin{lemma}\label{lm:7}
Let $\sfx\in \bbR^p$  be a multivariate random variable following a distribution with continuous and strictly increasing conditional cumulative distribution functions (cdf) $F_{\sfx_1},F_{\sfx_2|\sfx_1},..., F_{\sfx_p|\sfx_{1:p-1}}$, where $\sfx_{1:p-1}$ denotes the collection of $\sfx_1,\sfx_2,...,\sfx_{p-1}$. Let $\Phi_{N}^{-1}$ be the inverse cdf of $\cN(0,1)$. Then the transformation $\sfu=\Phi(\sfx)$ defined by
\begin{align*}
    \sfu_1 & =\Phi^{-1}_N\cdot F_{\sfx_1}(\sfx_1), \\
    \sfu_2|u_1 & = \Phi^{-1}_N\cdot F_{\sfx_2|\sfu_1}(\sfx_2|\sfu_1=u_1), \\
    & ...\\
    \sfu_2|u_{1:p-1} & =\Phi^{-1}_N\cdot F_{\sfx_p|\sfu_{1:p}}(\sfx_p|\sfu_{1:p-1}=u_{1:p-1})
\end{align*}
is invertible and $\sfu\sim \cN(\bzero,\bI)$.
\end{lemma}

\begin{proof}
First note that the cdf of  $\sfu_1$ satisfies $F_{\sfu_1}(u)=\bbP(\Phi_N^{-1}\cdot F_{\sfx_1} (\sfx_1)\leq u)=\bbP(F_{\sfx_1} (\sfx_1)\leq \Phi_N(u))=\Phi_N(u)$, and therefore $\sfu_1$ is $\cN(0,1)$-distributed. Similar arguments can be used to show that each of $\sfu_2|u_1,\dots,\sfu_p|u_{1:p-1}$ is $\cN(0,1)$-distributed. Therefore, the joint distribution $\bbP(\sfu_{1:p})=\bbP(\sfu_1)\bbP(\sfu_2|\sfu_1)\dots$ $\bbP(\sfu_p|\sfu_{1:p-1})=\cN(0,1)^p$ is a multivariate Gaussian $\cN(\bzero,\bI)$.

Second, since $F_{\sfx_1}$ is continuous and strictly increasing, the transformation $\sfx_1\rightarrow \sfu_1$ is invertible and continuous. This, together with the fact that $F_{\sfx_2|\sfx_1}$ is continuous and strictly increasing, means that $F_{\sfx_2|\sfu_1}$ is continuous and strictly increasing. Therefore $\sfx_1,\sfx_2\rightarrow \sfu_1,\sfu_2$ is invertible and continuous. By iterating this argument it follows that the transformation $\sfx_1,...,\sfx_p\rightarrow \sfu_1,...,\sfu_p$, i.e., $\Phi$, is invertible, which completes the proof.
\end{proof}

With these results, we can now prove Theorem~\ref{thm:1} from the main text. We first use Lemma~\ref{lm:7} to transform the sample distribution into a Gaussian distribution. Then we derive a lower bound on the mutual information using Lemma~\ref{lm:5}. This lower bound is simplified to be a scalar multiplied by the relative decay, using Lemma~\ref{lm:6}.

\begin{proof}[Proof of Theorem 1]
 Let $\Phi$ be the transformation in Lemma~\ref{lm:7} that turns the random variable $\sfx'\sim\bbP(\sfx'|\cD_t,\ba')$ into a Gaussian variable $\Phi(\sfx')\sim\cN(\bzero,\bI)$. Here, the distribution before the transformation is $\bbP(\sfx'|\cD_t,\ba')=\bbP((\bI-\bB)^{-1}(\ba'+\bepsilon')|\cD_t)$, the posterior of  sample $\bx'$ from $\ba'$ given data $\cD_t$. \rec{This posterior can be viewed as an induced distribution of $(\bI-\bB)^{-1}(\ba'+\bepsilon')$, where $\bB$ and $\bepsilon'$ follow the posteriors given $\cD_t$ obtained in Lemma \ref{lm:1}.} Let $\cE = \{\bbP(\sfx'|\cD_t,\ba'):\ba'\in \cA\}$ denote the class of all such posterior distributions.

Using the data processing inequality it holds that
\begin{equation*}
    I(g(\ba);\bx'|\ba',\cD_t) \geq I(g(\ba); \Phi(\bx')|\ba',\cD_t).
\end{equation*}
Now since $\Phi(\bx')$ is multivariate Gaussian, we can use Lemma~\ref{lm:5} to derive that
\begin{equation*}
    \begin{aligned}
    I(g(\ba);\bx'|\ba',\cD_t) & \geq I(g(\ba); \Phi(\bx')|\ba',\cD_t) \\
    & \geq \frac12 \log\frac{1}{1-\|\rho_{ \Phi(\bx'),g(\ba)|\ba',\cD_t}\|_2^2}\geq \frac12\|\rho_{\Phi(\bx'),g(\ba)|\ba',\cD_t}\|_2^2,
\end{aligned}
\end{equation*}
where we denote $\rho_{\Phi(\bx'),g(\ba)|\ba',\cD_t} = \sigma_{g(\ba)|\cD_t}^{-1}\Var(\Phi(\bx')|\ba',\cD_t)^{-\frac12}\Cov(\Phi(\bx'),g(\ba)|\ba',\cD_t)$. Next, we apply Lemma~\ref{lm:6} to obtain
\begin{equation}\label{thm1:eq0}
    \begin{aligned}
    I(g(\ba);\bx'|\ba',\cD_t) 
    &\geq 
    \frac12\|\rho_{\Phi(\bx'),g(\ba)|\ba',\cD_t}\|_2^2 
    \\
    &= 
    \frac12 \left( 
    1-\frac{\bbE_{\Phi(\bx')}\big(\sigma_{g(\ba)|\cD_t, \ba',\Phi(\bx')}^2\big)}{\sigma_{g(\ba)|\cD_t}^2}
    \right)
    \|\rho(g(\ba);\bx'|\ba',\cD_t)\|_2^2,
\end{aligned}
\end{equation}
where we denote
\begin{equation}\label{thm1:eq0.5}
\begin{aligned}
& \rho(g(\ba);\bx'|\ba',\cD_t) \\
={} & \sigma_{\bbE(g(\ba)|\cD_t,\ba',\Phi(\bx'))|\cD_t,\ba'}^{-1} \Var\left(\Phi(\bx')\big|\cD_t,\ba'\right)^{-\frac12}\Cov\Big(\Phi(\bx'),\bbE\big(g(\ba)|\cD_t,\ba',\Phi(\bx')\big)\Big|\cD_t,\ba'\Big) \\
={} & \sigma_{\bbE(g(\ba)|\cD_t,\ba',\Phi(\bx'))|\cD_t,\ba'}^{-1}\Cov\Big(\Phi(\bx'),\bbE\big(g(\ba)|\cD_t,\ba',\Phi(\bx')\big)\Big|\cD_t,\ba'\Big) \quad\quad (since~\Phi(\bx')|\cD_t,\ba'\sim \cN(\bzero,\bI)).
\end{aligned}
\end{equation}

Note that since $\Phi$ is invertible and deterministic, we have that $\bbP(g(\ba')|\cD_t,\ba',\Phi(\bx'))= \bbP(g(\ba')|\cD_t\cup(\bx',\ba')\})$. 
Therefore
\begin{equation}\label{thm1:eq1}
    \begin{aligned}
      \bbE(g(\ba')|\cD_t,\ba',\Phi(\bx')) & =\bbE(g(\ba')|\cD_t\cup(\bx',\ba')) \\
      \Var(g(\ba')|\cD_t,\ba',\Phi(\bx')) & =\Var(g(\ba')|\cD_t\cup(\bx',\ba')).
    \end{aligned}
\end{equation}
Then to show the expected posterior variance $\bbE_{\Phi(\bx')}(\sigma_{g(\ba)|\cD_t, \ba',\Phi(\bx')}^2)$ in Eq.~\eqref{thm1:eq0} equals to the expected posterior variance $\bbE_{\bx'}(\sigma_{g(\ba)|\cD_t, \ba',\bx'}^2)$, we have
\begin{align*}
    &~~~~\bbE_{\Phi(\bx')}\left(\sigma_{g(\ba)|\cD_t, \ba',\Phi(\bx')}^2\right)\\ 
    &= 
    \bbE_{\Phi(\bx')} \left( \Var(g(\ba')|\cD_t,\ba',\Phi(\bx')) \right) \quad \quad (by~definition)
    \\
    &= 
    \bbE_{\Phi(\bx')} \left( \Var(g(\ba')|\cD_t\cup(\bx',\ba'))  \right) \quad \quad(by~Eq.~\eqref{thm1:eq1})
    \\
        &= 
    \int \Var(g(\ba')|\cD_t\cup(\bx',\ba')) f_{\Phi(\bx')}(\Phi(\bx')) d \Phi(\bx')\quad\quad(f~denotes~the~pdf~of~the~corresponding~variable)
    \end{align*}
\begin{align*}    
    &= 
    \int \Var(g(\ba')|\cD_t\cup(\bx',\ba')) f_{\bx'}(\bx') \left|\det \left( \frac{d\bx'}{d\Phi(\bx')} \right)\right| d \Phi(\bx') \quad (change~of~variable~in~probability)
    \\
    &= 
    \int \Var(g(\ba')|\cD_t\cup(\bx',\ba')) f_{\bx'}(\bx')d\bx'\quad(change~of~variable~in~integration)
    \\
    &= 
    \bbE_{\bx'} \left( \Var(g(\ba')|\cD_t\cup(\bx',\ba')) \right) 
    = \bbE_{\bx'}\big(\sigma^2_{g(\ba)|\cD_t\cup(\bx',\ba')}\big).
\end{align*}
It follows by using this in Eq.~\eqref{thm1:eq0} that
\begin{align*}
    I(g(\ba);\bx'|\ba',\cD_t) 
    &\geq 
    \frac12 \left(
    1-\frac{\bbE_{\Phi(\bx')}\big(\sigma_{g(\ba)|\cD_t, \ba',\Phi(\bx')}^2\big)}{\sigma_{g(\ba)|\cD_t}^2}
    \right)
    \|\rho(g(\ba);\bx'|\ba',\cD_t)\|_2^2 
    \\
    &= 
    \frac12 \left(
    1-\frac{\bbE_{\bx'}\big(\sigma_{g(\ba)|\cD_t\cup(\bx',\ba')}^2\big)}{\sigma_{g(\ba)|\cD_t}^2}
    \right)
    \|\rho(g(\ba);\bx'|\ba',\cD_t)\|_2^2.
\end{align*}
Now we lower bound $\|\rho(g(\ba);\bx'|\ba',\cD_t)\|_2^2$. Using Eq.~\eqref{thm1:eq1} in Eq.~\eqref{thm1:eq0.5}, we have
\begin{equation}\label{thm1:eq2}
   \rho(g(\ba);\bx'|\ba',\cD_t) = \sigma_{\bbE(g(\ba)|\cD_t,\ba',\bx')|\cD_t,\ba'}^{-1}\Cov\Big(\Phi(\bx'),\bbE\big(g(\ba)|\cD_t,\ba',\bx'\big)\Big|\cD_t,\ba'\Big). 
\end{equation}
Note that both $\bbE(g(\ba)|\cD_t,\ba',\bx')$ and $\Phi(\bx')$ are deterministic functions of $\bx'$ who follows a posterior distribution $\bbP(\bx'|\cD_t,\ba')$. Therefore, the quantity $\rho(g(\ba);\bx'|\ba',\cD_t)$ only depends on the posterior distribution $\bbP(\bx'|\cD_t,\ba')\in\cE$. Defining $\rho^2 := \frac12\min_{\bbP(\bx'|\cD_t,\ba')\in \cE} \|\rho(g(\ba);\bx'|\ba',\cD_t)\|_2^2$, then for all $\ba'\in\cA$ it holds that
\begin{equation*}
    I(g(\ba);\bx'|\ba',\cD_t) 
    \geq 
    \rho^2\left(1-\frac{\bbE_{\bx'} \left( \sigma_{g(\ba)|\cD_t\cup(\bx',\ba')}^2 \right)}{\sigma_{g(\ba)|\cD_t}^2}\right),
\end{equation*}
where $\rho$ only depends on $\cD_t$, the family of posterior distributions $\cE$, and the function of interest $g$. This completes the proof.
\end{proof}

\begin{remark}
[Evaluation of $\rho$]
We note that since $\|\rho_{\sfx,\sfz}\|_2^2\leq 1$ for any multivariate $\sfx$ and univariate $\sfz$, it holds that $\rho^2\leq \frac12$ in the proof of Theorem~1 above. While Theorem~1 applies to general posterior modeling of an intervention effect $\bbP(\bx'|\cD_t, \ba')$ and variable of interest $\bbP(g(\ba)|\cD_t\cup(\bx',\ba'))$, for certain special cases we can compute the constant $\rho^2$ exactly. In particular, in the following we show that $\rho^2=\frac12$ for the prominent special case of Bayesian linear regression, also considered in~\cite{sapsis2020output}.

In the Bayesian linear regression setting, the ``intervention effect'' $\bbP(\sfy|\cD, \sfx)$ corresponds to the output distribution of $\sfy$ given input $\sfx$ and previous data $\cD$, which is modeled as $\sfy=\theta^\top \sfx$ with $\theta$ following a Gaussian posterior.
The ``variable of interest'' is any linear transformation of $\theta$, e.g., $c^\top\theta$ for some constant vector $c$.
Then $\Phi$ is a linear transformation of $\sfy$, and $\bbE(c^\top\theta|\cD\cup(\bx',\ba'))$ is linear in $\sfy$.
Therefore $\rho^2(c^\top\theta;\sfy|\sfx,\cD)$, written out similarly as in Eq.~\eqref{thm1:eq2}, equals $\nicefrac{c}{\|c\|_2}$ for all input $\sfx$.
Thus $\rho^2$ in Theorem 1 equals exactly $\frac12$, and as a consequence, 
\[
I(c^\top\theta;\sfy|\sfx,\cD)
\geq 
\frac12 \left( 
1-\frac{\bbE_{\sfy} \left(\sigma^2_{c^\top\theta|\cD\cup(\sfy,\sfx)}\right)}
{\sigma^2_{c^\top\theta|\cD}}
\right)
\]

For the case of a linear SCM with a DAG-Wishart prior, however, direct calculation of $\rho^2$ can be difficult since (i) the posterior distribution of $\bx'$ is a composition of two normal-Wishart distributions, and (ii) $\bbE(g(\ba)|\cD_t\cup(\bx',\ba'))$ is a complicated function of $\bx'$, given in \nameref{prop:1-a}. Of note, $\rho^2$ can be made larger by considering transformations $\Phi$ beyond Lemma~\ref{lm:7}, which is related to research on Gaussianization \cite{painsky2017gaussian}.
\end{remark}

\section{Consistency Proofs}\label{sec:f}

In this section, we present the proof of Theorem~\ref{thm:2} from the main text and discuss extensions to other realizations of the CIV acquisition function.

\begin{proof}[Proof of Theorem 2]
First, we recall the CIV acquisition function from the main text, namely
\begin{equation}\label{thm2:eq1}
    h(\ba',\cD_t) 
    = 
    \sum_{i=1}^p 
    \left( v_i^2 + 2v_i \big( \frac{\sigma_i^2}{n} + b_i^2 + \frac1p \big) \right),
\end{equation}
for $\ba'\in\cA$. To clarify the dependence on $\ba'\in\cA$, in the following we use $v_i(\ba')$ instead of $v_i$ to denote the variance in the direction of interest. We first prove the following claim, which we will use various times throughout the proof.

\begin{claim}
Define
\begin{equation*}
    \begin{aligned}
        \bM_i'&=\Var(\bB_{i,\pa(i)}|\cD_t),\\
        \bA_i&=[(\bI-\bB)^{-1}]_{\pa(i),:},\\
         \bA_i'&=[(\bI-\bbE(\bB|\cD_t))^{-1}]_{\pa(i),:},
    \end{aligned}
\end{equation*}
where $[\cdot]_{\pa(i),:}$ denotes the sub-matrix corresponding to the $\pa(i)$-indexed rows of the corresponding matrix. Denote the following products by
\begin{equation}\label{thm2:def1}
    \begin{aligned}
    \bK_i&:=\bA_i^\top\bM_i'\bA_i,\\ \bK_i'&=\bA_i^{'\top}\bM_i'\bA_i,\\ \bK_i''&=\bA_i^{'\top}\bM_i'\bA_i'.
\end{aligned}
\end{equation}
We have $\|\bK_i'\ba^*\|_2= O(\nicefrac{1}{t})$ and
\begin{align*}
    \ba^{*\top}\bK_i\ba^*=\begin{cases}\Theta(\nicefrac1t)& if~\bA_i\ba^*\neq \bzero,\\
    0 & if~\bA_i\ba^*= \bzero.
    \end{cases} 
\end{align*}
In addition, for all $\ba'\in\cA$, it holds that
\begin{equation}
    \begin{aligned}
        \ba^{\prime\top}\bK_i'\ba^*&=O(\nicefrac1t)\\
        \ba^{\prime\top}\bK_i''\ba'&=O(\nicefrac1t).
    \end{aligned}
\end{equation}
\end{claim}
\begin{quote}
\begin{proof}[Proof of Claim 2.]
First note that $\cD_t$ contains $nt$ samples, where $n$ is  the batch size. By the Bernstein-von-Mises Theorem \cite{kleijn2012bernstein}, the posterior of $\bB_{i,\pa(i)}$ converges at a rate of $\nicefrac{1}{\sqrt{t}}$, and thus $\bM_i'=\Var(\bB_{i,\pa(i)}|\cD_t)=\mathbf{\Theta}(\nicefrac1t)$,  where we use the asymptotic notation $\mathbf{\Theta}$ for a matrix to denote that $\frac{C}{t}\bI\succeq\Var(\bB_{i,\pa(i)}|\cD_t)\succeq \frac{c}{t}\bI$ for constants $c,C>0$ that are independent of $t$.

Since $\bB$ and $\bbE(\bB|\cD_t)$ are lower triangular after permutation with respect to the topological order in the DAG, we have that $\bI-\bB$ and $\bI-\bbE(\bB|\cD_t)$ both only have eigenvalues equal to $1$. Therefore $(\bI-\bB)^{-1}$ and $(\bI-\bbE(\bB|\cD_t))^{-1}$ both only have eigenvalues equal to $1$. This means that $\bA_i\ba^*$ and $\bA_i'\ba'$, which correspond to vectors by taking the $\pa(i)$-entries of $(\bI-\bB)^{-1}\ba^*$ and $(\bI-\bbE(\bB|\cD_t))^{-1}\ba'$ respectively, are bounded. Thus $\|\bK_i'\ba^*\|_2=O(\nicefrac{1}{t})$, $\ba^{*\top}\bK_i\ba^*=O(\nicefrac{1}{t})$, $\ba^{\prime\top}\bK_i'\ba^*=O(\nicefrac{1}{t})$, and $\ba^{\prime\top}\bK_i''\ba'=O(\nicefrac{1}{t})$.

It remains to show that $\ba^{*\top}\bK_i\ba^*=\Omega(\nicefrac{1}{t})$ if $\bA_i\ba^*\neq \bzero$, and otherwise $\ba^{*\top}\bK_i\ba^*=\bzero$. This follows directly by noting that $\ba^{*\top}\bK_i\ba^*=(\bA_i\ba^*)^\top\bM_i'(\bA_i\ba^*)$ and $\bM_i'=\mathbf{\Theta}(\nicefrac1t)$, which completes the proof of Claim~2.
\end{proof}
\end{quote}
Now we prove the theorem in three steps. 
\begin{enumerate}[label=(\roman*)]
    \item For any $i\in[p]$, it holds that
\begin{align*}
    0&\leq v_i(\ba')\\
     &= \sigma_i^2\bmu_{\pa(i)}^{*\top} \Var(\bB_{i,\pa(i)}|\cD_t(\ba')) \bmu_{\pa(i)}^* \\
    &\leq\sigma_i^2 \bmu_{\pa(i)}^{*\top} \bM_i' \bmu_{\pa(i)}^*\quad\quad(since~\Var(\bB_{i,\pa(i)}|\cD_t(\ba'))\preceq\Var(\bB_{i,\pa(i)}|\cD_t))\\
    &=\sigma_i^2\ba^{*\top}\bK_i\ba^*\quad\quad(using~\bmu_{\pa(i)}^*=\bA_i\ba^*~and~Eq.~ \eqref{thm2:def1}).
\end{align*}
As a consequence, it follows from Claim~2 that $v_i(\ba')=O(\nicefrac{1}{t})$ if $\bA_i\ba^*\neq \bzero$ and otherwise $v_i(\ba')=0$. Now we show that $v_i(\ba')=\Omega(\nicefrac{1}{t})$ if $\bA_i\ba^*\neq \bzero$.
According to the Sherman-Morrison formula, we have that for a matrix $\Lambda$ and vectors~$u$~and~$v$,
\begin{equation*}
    (\Lambda + u v^\top)^\inv 
    = 
    \Lambda^\inv - \frac{\Lambda^\inv u v^\top \Lambda^\inv}{1 + v^\top \Lambda^\inv u}.
\end{equation*}
Using the Sherman-Morrison formula with $\Lambda = {\bM_i'}^\inv$, $u = n \bA_i' \ba'$ and $v = \bA_i' \ba'$, we obtain that
\begin{equation}\label{thm2:eq2}
    \begin{aligned}
    v_i(\ba') &= \sigma_i^2\bmu_{\pa(i)}^{*\top} \Var(\bB_{i,\pa(i)}|\cD_t(\ba')) \bmu_{\pa(i)}^* \\
    & = \sigma_i^2 \ba^{*\top}\bA_i^\top \left(\bM_i^{\prime -1}+n\bA_i'\ba'\ba^{\prime\top}\bA_i^{\prime\top}\right)^{-1}\bA_i\ba^* \quad\quad(using~Lemma~\ref{lm:1}~and~\bmu_{\pa(i)}^*=\bA_i\ba^*)\\
    & =  \sigma_i^2 \ba^{*\top}\bA_i^\top \left(\bM_i'-\frac{n\bM_i'\bA_i'\ba'\ba^{\prime\top}\bA_i^{\prime\top}\bM_i'}{1+n\ba^{\prime\top}\bA_i^{\prime\top}\bM_i'\bA_i'\ba'}\right) \bA_i\ba^*\\
    & = \sigma_i^2\left(\ba^{*\top}\bK_i\ba^*-\frac{n(\ba^{\prime\top}\bK_i'\ba^*)^2}{1+n\ba^{\prime\top}\bK_i''\ba'}\right)\quad\quad(using~Eq.~\eqref{thm2:def1}).
    \end{aligned}
\end{equation}
 By Claim~2, it follows that $\ba^{*\top}\bK_i\ba^*=\Omega(\nicefrac{1}{t})$ if $\bA_i\ba^*\neq\bzero$ and $\nicefrac{n(\ba^{\prime\top}\bK_i'\ba^*)^2}{1+n\ba^{\prime\top}\bK_i''\ba'}=O(\nicefrac{1}{t^2})$. Therefore $v_i(\ba')=O(\nicefrac{1}{t})$ if $\bA_i\ba^*\neq \bzero$. 
Thus we have for all $\ba'\in\cA$,
\begin{equation*}
    v_i(\ba')=\begin{cases}
    \Theta(\frac1t),\quad & if~\bA_i\ba^*\neq\bzero,\\
    0,\quad & if~\bA_i\ba^*=\bzero.
    \end{cases}
\end{equation*}
Since $\frac{\sigma_i^2}{n}+b_i^2+\frac1p=\Theta(1)$, then following Eq.~\eqref{thm2:eq1}, we know that for all $\ba'\in\cA$, 
\begin{equation*}
    h(\ba',\cD_t)=\begin{cases}
    \Theta(\frac1t),\quad & if~\bA_i\ba^*\neq\bzero~for~some~i\in[p],\\
    0,\quad & if~\bA_i\ba^*=\bzero~for~all~i\in[p].
    \end{cases}
\end{equation*}
\item Next, we study the gradient of $v_i(\ba')$ at $\ba'=\ba^*$. In fact, following Eq.~\eqref{thm2:eq2}, since $\bK_i$, $\bK_i'$ and $\bK_i''$ are constants with respect to $\ba'$, we can calculate according to the quotient rule that
\begin{equation}\label{thm2:eq3}
\begin{aligned}
    \nabla v_i(\ba') &= \sigma_i^2\left(-\frac{(1+n\ba^{\prime\top}\bK_i''\ba')\cdot(2n\ba^{\prime\top}\bK_i'\ba^*)\bK_i'\ba^*-n(\ba^{\prime\top}\bK_i'\ba^*)^2\cdot2n\bK_i''\ba'}{(1+n\ba^{\prime\top}\bK_i''\ba')^2}
    \right)
    \\
    &= -
    \frac{2n\sigma_i^2\cdot\ba^{\prime\top}\bK_i'\ba^*}{(1+n\ba^{\prime\top}\bK_i''\ba')^2} \left(\bK_i'\ba^*+n(\ba^{\prime\top}\bK_i''\ba'\cdot\bK_i'\ba^*-\ba^{\prime\top}\bK_i'\ba^*\cdot\bK_i''\ba')\right).
\end{aligned}
\end{equation}
Therefore, the gradient at $\ba'=\ba^*$ is
\begin{equation*}
    \nabla v_i(\ba^*) = -
    \frac{2n\sigma_i^2\cdot\ba^{*\top}\bK_i'\ba^*}{(1+n\ba^{*\top}\bK_i''\ba^*)^2} \left(\bK_i'\ba^*+n(\ba^{*\top}\bK_i''\ba^*\cdot\bK_i'\ba^*-\ba^{*\top}\bK_i'\ba^*\cdot\bK_i''\ba^*)\right).
\end{equation*}
Using Claim 2 again, we can then deduce that $\|\nabla v_i(\ba^*)\|=O(\nicefrac{1}{t^2})$, where $\|\cdot\|$ is the Euclidean norm.
Together with Eq.~\eqref{thm2:eq1}, we obtain that
\begin{align*}
    \nabla h(\ba^*,\cD_t) 
    = 
    \sum_{i=1}^p 2 \left( v_i+ \frac{\sigma_i^2}{n} + b_i^2 + \frac1p \right) \nabla v_i(\ba^*),
\end{align*}
and thus by $ v_i+ \frac{\sigma_i^2}{n} + b_i^2 + \frac1p=O(1)$, we have that $\|\nabla h(\ba^*,\cD_t)\|=O(\nicefrac{1}{t^2})$.
\item Now, we look at the Hessian.
Following Eq.~\eqref{thm2:eq3}, we have
\begin{equation*}
\begin{aligned}
    & \nabla^2 v_i(\ba')\\
    ={} & \left(\bK_i'\ba^*+n(\ba^{\prime\top}\bK_i''\ba'\cdot\bK_i'\ba^*-\ba^{\prime\top}\bK_i'\ba^*\cdot\bK_i''\ba')\right)\nabla\left(-
    \frac{2n\sigma_i^2\cdot\ba^{\prime\top}\bK_i'\ba^*}{(1+n\ba^{\prime\top}\bK_i''\ba')^2}\right)\\
    &\quad\quad\quad-
    \frac{2n\sigma_i^2\cdot\ba^{\prime\top}\bK_i'\ba^*}{(1+n\ba^{\prime\top}\bK_i''\ba')^2}\nabla\left(\bK_i'\ba^*+n(\ba^{\prime\top}\bK_i''\ba'\cdot\bK_i'\ba^*-\ba^{\prime\top}\bK_i'\ba^*\cdot\bK_i''\ba')\right)\\
    ={} &\left(\bK_i'\ba^*+n(\ba^{\prime\top}\bK_i''\ba'\cdot\bK_i'\ba^*-\ba^{\prime\top}\bK_i'\ba^*\cdot\bK_i''\ba')\right)\left(\frac{-2n\sigma_i^2\cdot\ba^{\prime\top}\bK_i'\ba^*(-4n\ba^{\prime\top}\bK_i'')}{(1+n\ba^{\prime\top}\bK_i''\ba')^3}-\frac{2n\sigma_i^2\cdot\ba^{*\top}\bK_i^{\prime\top}}{(1+n\ba^{\prime\top}\bK_i''\ba')^2}\right)\\
    &\quad\quad\quad-
    \frac{2n\sigma_i^2\cdot\ba^{\prime\top}\bK_i'\ba^*}{(1+n\ba^{\prime\top}\bK_i''\ba')^2}\left(n\big(\bK_i'\ba^*(2\ba^{\prime\top}\bK_i'')-\bK_i''\ba'\ba^{*\top}\bK_i^{\prime\top}-\ba^{\prime\top}\bK_i'\ba^*\cdot\bK_i''\big)\right).
\end{aligned}
\end{equation*}
Let $\ba'=\ba^*$, we have
\begin{equation*}
\begin{aligned}
    & \nabla^2 v_i(\ba^*) \\
    = {} &\left((1+n\ba^{*\top}\bK_i''\ba^*)\cdot\bK_i'\ba^*-n\ba^{*\top}\bK_i'\ba^*\cdot\bK_i''\ba^*\right)\left(\frac{-2n\sigma_i^2\cdot\ba^{*\top}\bK_i'\ba^*(-4n\ba^{*\top}\bK_i'')}{(1+n\ba^{*\top}\bK_i''\ba^*)^3}-\frac{2n\sigma_i^2\cdot\ba^{*\top}\bK_i^{\prime\top}}{(1+n\ba^{*\top}\bK_i''\ba^*)^2}\right)\\
    &\quad\quad\quad-
    \frac{2n\sigma_i^2\cdot\ba^{*\top}\bK_i'\ba^*}{(1+n\ba^{*\top}\bK_i''\ba^*)^2}\left(n\big(\bK_i'\ba^*(2\ba^{*\top}\bK_i'')-\bK_i''\ba^*\ba^{*\top}\bK_i^{\prime\top}-\ba^{*\top}\bK_i'\ba^*\cdot\bK_i''\big)\right)\\
\end{aligned}
\end{equation*}

\begin{equation*}
\begin{aligned}
    \\
    ={} &\frac{2n\sigma_i^2}{(1+n\ba^{*\top}\bK_i''\ba^*)^3}\bigg\{(1+n\ba^{*\top}\bK_i''\ba^*)\Big(
    -(1+n\ba^{*\top}\bK_i''\ba^*)\bK_i'\ba^*\ba^{*\top}\bK_i^{\prime\top} + n(\ba^{*\top}\bK_i'\ba^*)^2\bK_i''
    \\
    & \quad\quad\quad +2n(\ba^{*\top}\bK_i'\ba^*)\big(\bK_i'\ba^*\ba^{*\top}\bK_i''+\bK_i''\ba^*\ba^{*\top}\bK_i^{\prime\top}\big)
    \Big)
    -4n^2(\ba^{*\top}\bK_i'\ba^*)^2\bK_i''\ba^*\ba^{*\top}\bK_i''
    \bigg\} \\
    \succeq{} & \frac{2n\sigma_i^2}{(1+n\ba^{*\top}\bK_i''\ba^*)^3} \Big(-(1+n\ba^{*\top}\bK_i''\ba^*)^2\bK_i'\ba^*\ba^{*\top}\bK_i^{\prime\top}-4n^2(\ba^{*\top}\bK_i'\ba^*)^2\bK_i''\ba^*\ba^{*\top}\bK_i'' \Big) \\
    &.
\end{aligned}
\end{equation*}
Using the claim, we note $\bK_i'\ba^*\ba^{*\top}\bK_i^{\prime\top}\preceq O(\nicefrac1t^2)\bI$ and $(\ba^{*\top}\bK_i'\ba^*)^2\bK_i''\ba^*\ba^{*\top}\bK_i''\preceq O(\nicefrac1t^4)\bI$. Thus $\nabla^2 v_i(\ba^*)\succeq -O(\nicefrac1t^2)\bI$. Finally, with Eq.~\eqref{thm2:eq1}, there is
\begin{align*}
    \nabla^2 h(\ba^*,\cD_t) 
    & = 
    \sum_{i=1}^p 
    \left[
    2\nabla v_i(\ba^*)\nabla v_i(\ba^*)^\top 
    + 
    2 \left( v_i + \frac{\sigma^2_i}{n} + b_i + \frac1p \right) \nabla^2v_i(\ba^*) \right] \\
    & \succeq 
    \sum_{i=1}^p 
    2 \left( v_i+ \frac{\sigma^2}{n} + b_i + \frac1p \right) \nabla^2v_i(\ba^*).
\end{align*}
Using the results developed above, we know $\nabla^2h(\ba^*,\cD_t)\succeq -O(\nicefrac{1}{t^2})\bI$.
\end{enumerate}
These three steps conclude our proof. Note that here our discussion is based on the acquisition function with known noise levels. Similar arguments can be applied to the case with unknown noise levels, using the derived explicit form in Sec.~\ref{sec:b}.
\end{proof}

\begin{remark}\label{rmk:4}
We make a note about the case where $h(\ba',\cD_t)$ degenerates to constant $0$ for all $\ba'\in\cA$. From the proof above, we see that this is equivalent to the case where $\bmu^*_{\pa(i)}=\bA_i\ba^*=\bzero$ for all $i\in [p]$. This is also equivalent to the optimal intervention $\ba^*$ only intervenes on sink nodes, i.e., nodes with no children. For such $\ba^*$, we can directly calculate $\ba^*$ by subtracting the target and source distribution, $\ba^*=(\bI-\bB)\bmu^*=\bmu^*$, since $\bB\bmu^*=\bzero$. Therefore, in simulations, we avoid this case and generate $\ba^*$ that targets not only the sink nodes.
\end{remark}

This theorem implies that $\ba^*$ is an approximate-$O(\nicefrac{1}{t^2})$ local minimum of $h(\ba',\cD_t)$. Note that here we define the approximate local minima following \cite{agarwal2017finding, nesterov2006cubic}, where the asymptotic notation means that there exists a constant $c\geq 0$ such that
\begin{align*}
    \|\nabla h(\ba^*,\cD_t)\|\leq \frac{c^2}{t^2}\quad and\quad \nabla^2 h(\ba^*,\cD_t)\succeq -\frac{c}{t}\bI.
\end{align*}
Note, however, that approximate local minima are not guaranteed to be closed to any local minimum unless additional constraints are satisfied by the function \cite{agarwal2017finding}.

\section{Extended Simulation Results}\label{sec:g}

In this section, we provide a detailed description of the simulated dataset and present additional empirical results.

To construct a synthetic instance of a causal system and an optimal intervention, we use the following procedure:
\begin{enumerate}[label=(\roman*)]
    \item Generate a DAG $\cG$ with $p$ nodes. 
    Specifically, in our experiments $p$ ranges from $10$ to $30$. 
    To analyze the effect of edge density on our results, we consider different types of graphs including complete graphs, Erdös-Rényi graphs~\cite{erdos1960evolution}, and path graphs. The edges in the graph are oriented to obtain a DAG by uniformly sampling a permutation of the nodes and orienting the edges according to the order given by the permutation.
    Examples of these three types of graphs on $p=5$ nodes are shown in Supplementary Fig.~\ref{fig:s3}.
    \item For each edge $k\rightarrow i$ in $\cG$, its corresponding edge weight $\bB_{ik}$ is generated by sampling from a uniform
    distribution bounded away from zero, namely $\bB_{ik}\sim \mathrm{Unif}[-1,-0.25]\cup[0.25,1]$.
    \item The noise levels $\bSigma$ are chosen so as to standardize the causal model, i.e., ensure that the variance of each variable $\sfx_i$ for $i\in[p]$ is of the same scale. Such standardization has been shown to be important for benchmarking causal inference algorithms~\cite{reisach2021beware}, although for the problem considered here we found it to not change the results. We used the standardized model proposed in  \cite{mooij2020joint} with $\epsilon_i\sim\mathcal{N}(0,1)$ and then re-scaling the causal mechanism $\sfx_i=\sum_{k=1}^p\bB_{ik}\sfx_k+\epsilon_i$ by $\sqrt{\Var(\sfx_i)}$.
    We still denote the re-scaled model as $\bx=\bB\bx+\bepsilon, \bepsilon\sim\cN(\bzero,\bSigma)$ in the following.
    
    \item Generate the optimal intervention $\ba^*\in\cA$ with a fixed-size support in $[p]$. 
    Specifically, in our experiments we use the unit hypersphere $\bbS^{p-1}$ as the feasible set $\cA$ and $\|\ba^*\|_0$ ranges between $1$ and $\lfloor\nicefrac{p}{2}\rfloor$. 
    The target mean $\bmu^*=(\bI-\bB)^{-1}\ba^*$ is calculated using the ground truth $\bB$.
\end{enumerate}
Only the DAG $\cG$ and the target mean $\bmu^*$ are revealed to the learner. Then the learner proceeds as indicated in Fig.~\ref{fig:1} in the main text, i.e., the learner selects an intervention based on the acquisition function, samples from this intervention are generated using Eq.~\eqref{eq:2} in the main text, and the model is updated based on \nameref{lm:1-a}.

\begin{figure}[!t]
    \centering
     \begin{subfigure}[b]{0.2\textwidth}
         \centering
         \includegraphics[width=.7\textwidth]{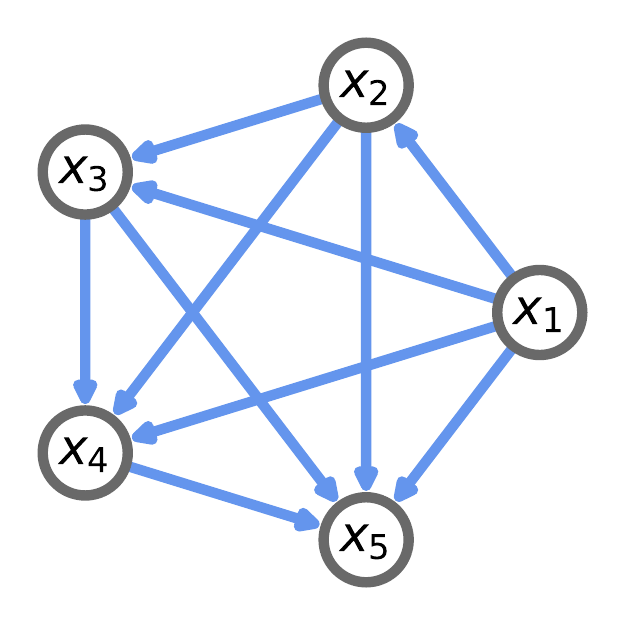}
     \end{subfigure}
     \begin{subfigure}[b]{0.2\textwidth}
         \centering
         \includegraphics[width=.7\textwidth]{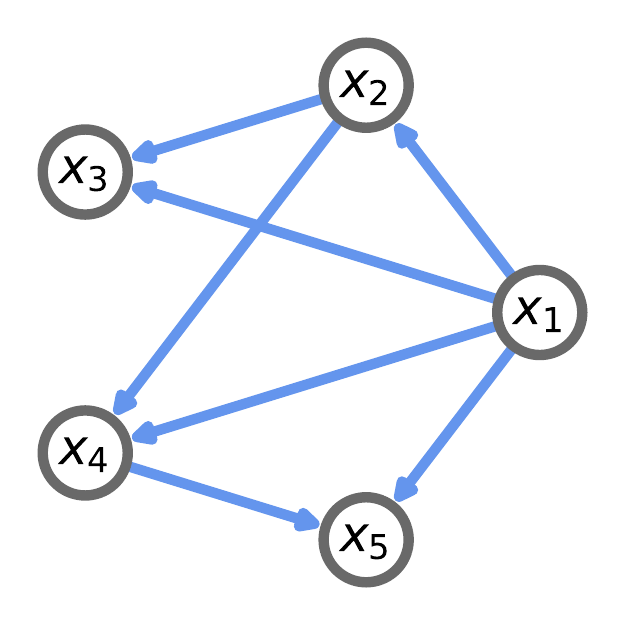}
     \end{subfigure}
     \begin{subfigure}[b]{0.2\textwidth}
         \centering
         \includegraphics[width=.7\textwidth]{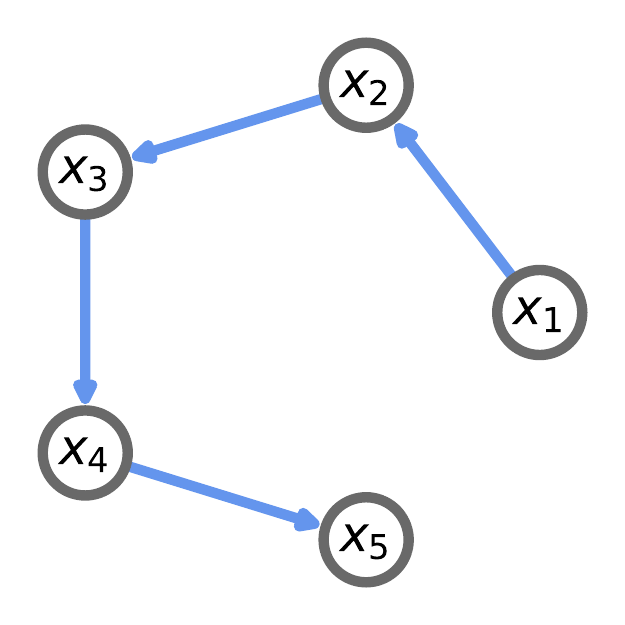}
     \end{subfigure}
    \caption{\rec{\textbf{Example DAGs on five nodes.}} Complete (Left), random (Middle), path (Right) graph on 5 nodes.}
    \label{fig:s3}
\end{figure}

To account for the randomness in the active learning process, each method is run $20$ times and the average is reported for each metric.
For a fixed DAG $\cG$ and a fixed number of intervention targets $\|\ba^*\|_0$, the realization of edge weights $\bB$ and the optimal intervention $\ba^*$ can substantially affect the resulting metric. 
To obtain more representative results, we generate $10$ instances for each fixed DAG $\cG$ and fixed number of intervention targets $\|\ba^*\|_0$. 
In all figures, we report for each metric that we consider both the mean and the standard deviation across the $10$ instances.
In summary, each curve reported below corresponds to running one of the methods $200$ times.

In our experiments, we carefully examine the effect of varying different parameters:
1) we examine the effect of the DAG size $p$  by varying  $p$ between $10$ and $30$, while fixing the ratio of intervention targets $\nicefrac{\|\ba^*\|_0}{p}$;
2) for $p$ fixed, we vary the number of intervention targets  $\|\ba^*\|_0$ between $1$ and $\lfloor\nicefrac{p}{2}\rfloor$;
3) for fixed $p$ and $\|\ba^*\|_0$, we compare the different methods across different classes of DAGs with varying edge density (see Supplementary Fig.~\ref{fig:s3}).
{In addition, we also include three additional baselines discussed below as well as in Supplementary Information~\ref{sec:a0}. Since some of these baselines are highly computationally expensive, we run the comparison over $10$-node DAGs. Finally, we perform an experiment to examine the importance of specifying the correct causal structure.
}

\begin{figure}[!t]
     \centering
     \begin{subfigure}[b]{0.27\textwidth}
         \centering
         \includegraphics[width=\textwidth]{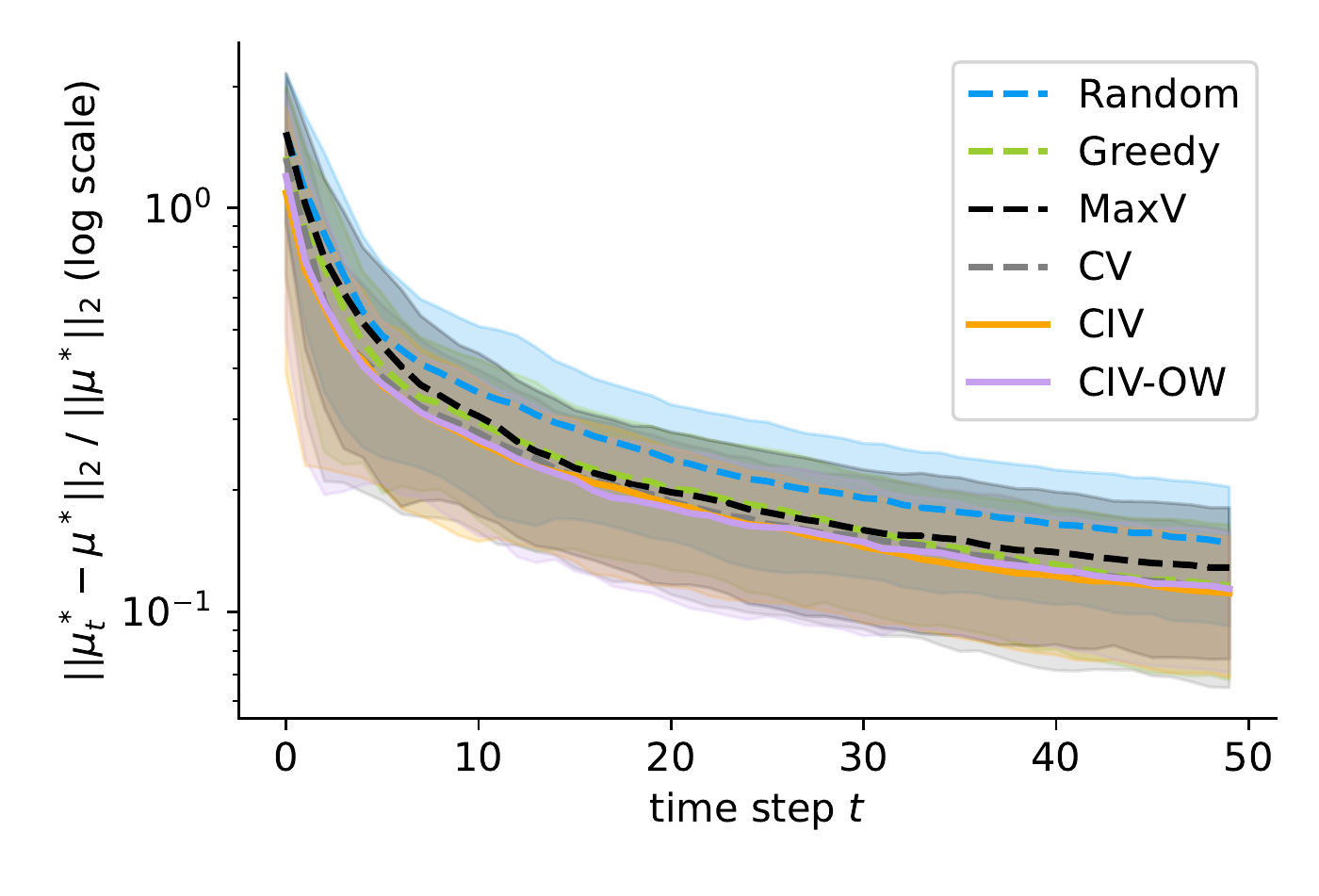}
         \caption{$p=10$}
     \end{subfigure}
     \begin{subfigure}[b]{0.27\textwidth}
         \centering
         \includegraphics[width=\textwidth]{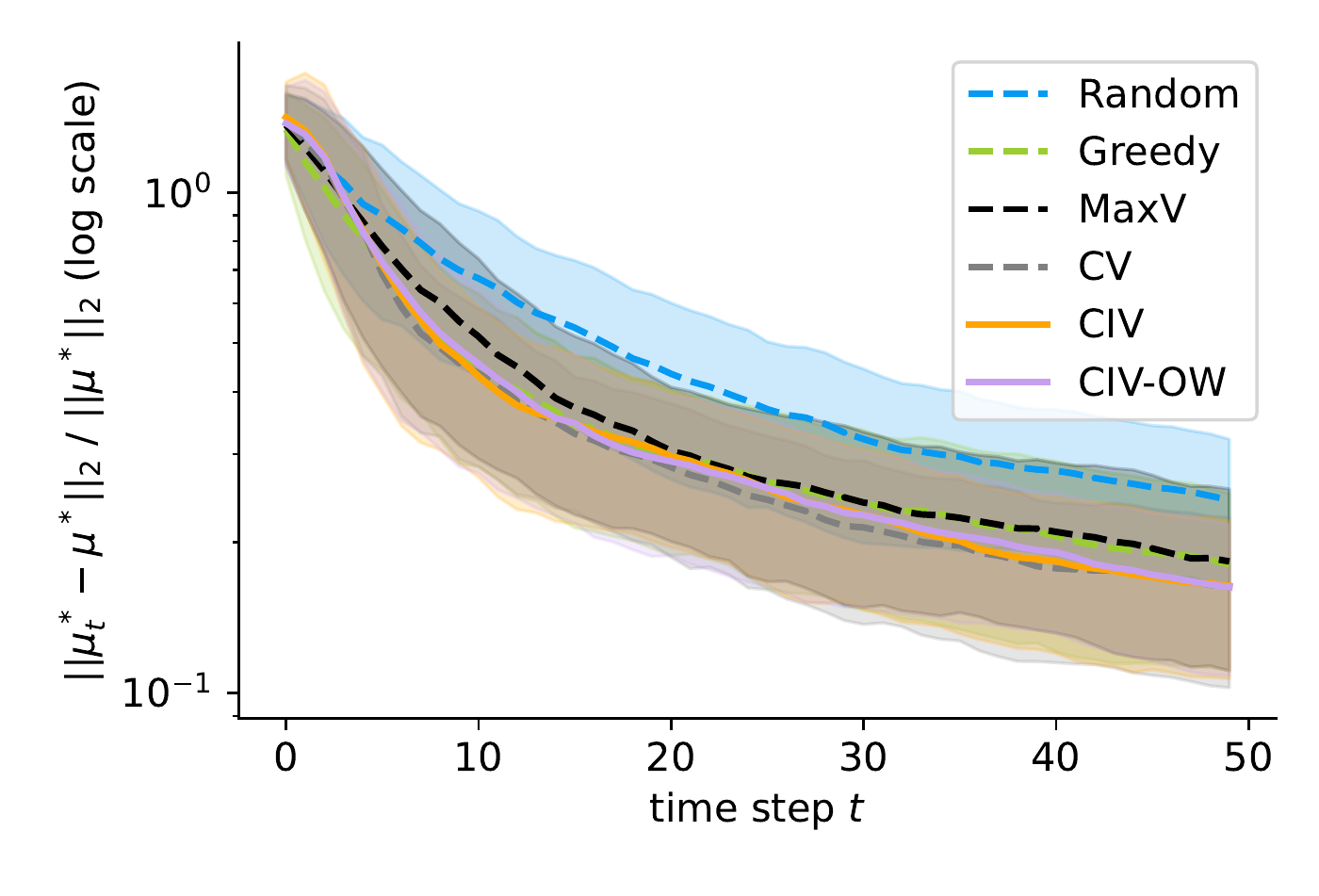}
         \caption{$p=20$}
     \end{subfigure}
    \begin{subfigure}[b]{0.27\textwidth}
         \centering
         \includegraphics[width=\textwidth]{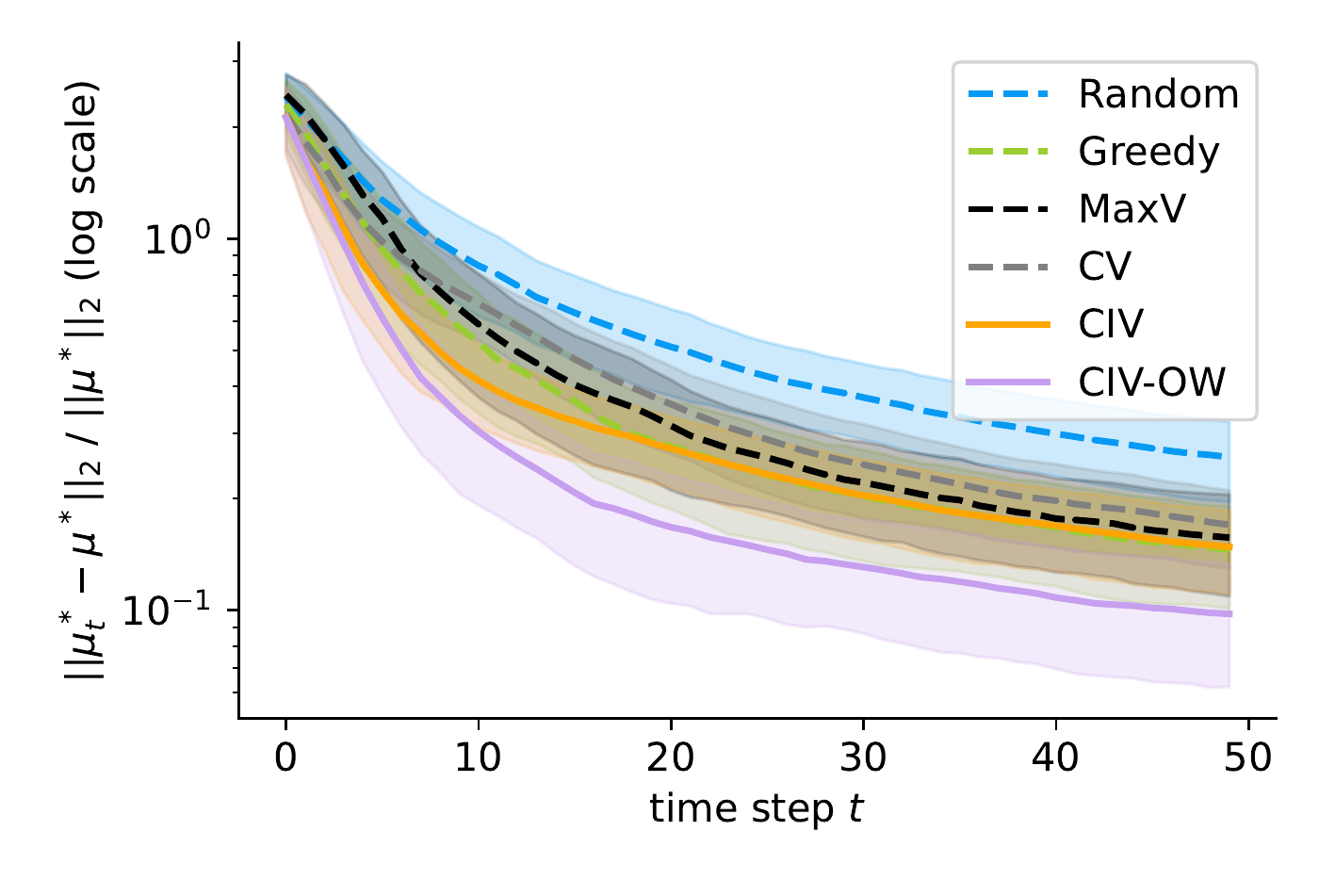}
         \caption{$p=30$}
     \end{subfigure}
     \\
     \begin{subfigure}[b]{0.27\textwidth}
         \centering
         \includegraphics[width=.8\textwidth]{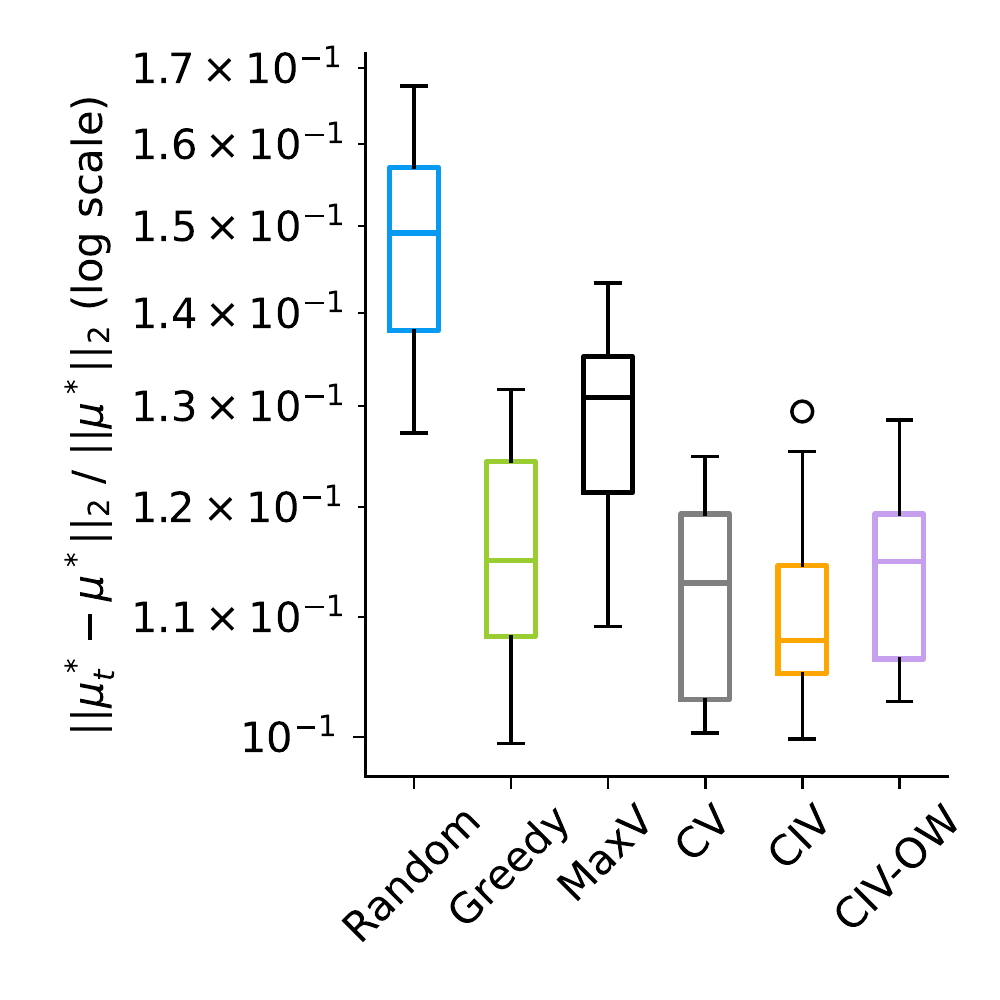}
         \caption{$p=10$}
     \end{subfigure}
     \begin{subfigure}[b]{0.27\textwidth}
         \centering
         \includegraphics[width=.8\textwidth]{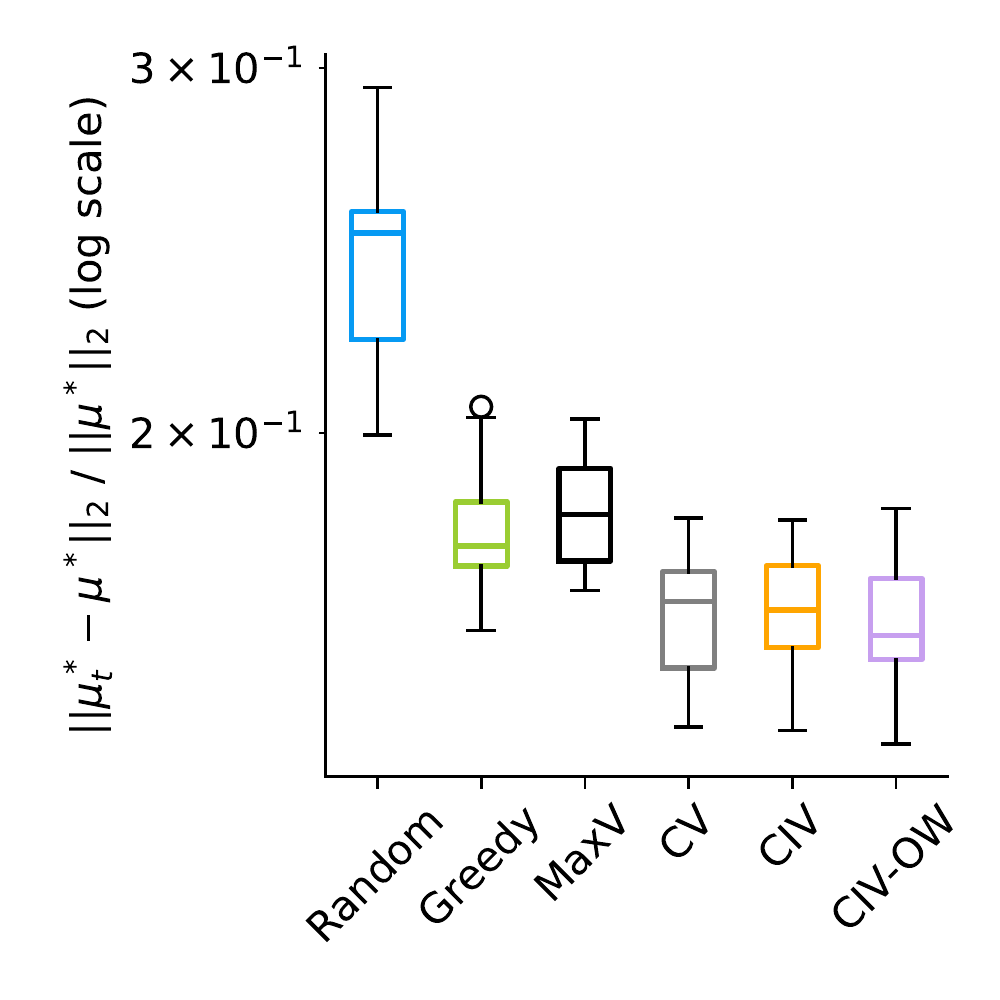}
         \caption{$p=20$}
     \end{subfigure}
    \begin{subfigure}[b]{0.27\textwidth}
         \centering
         \includegraphics[width=.8\textwidth]{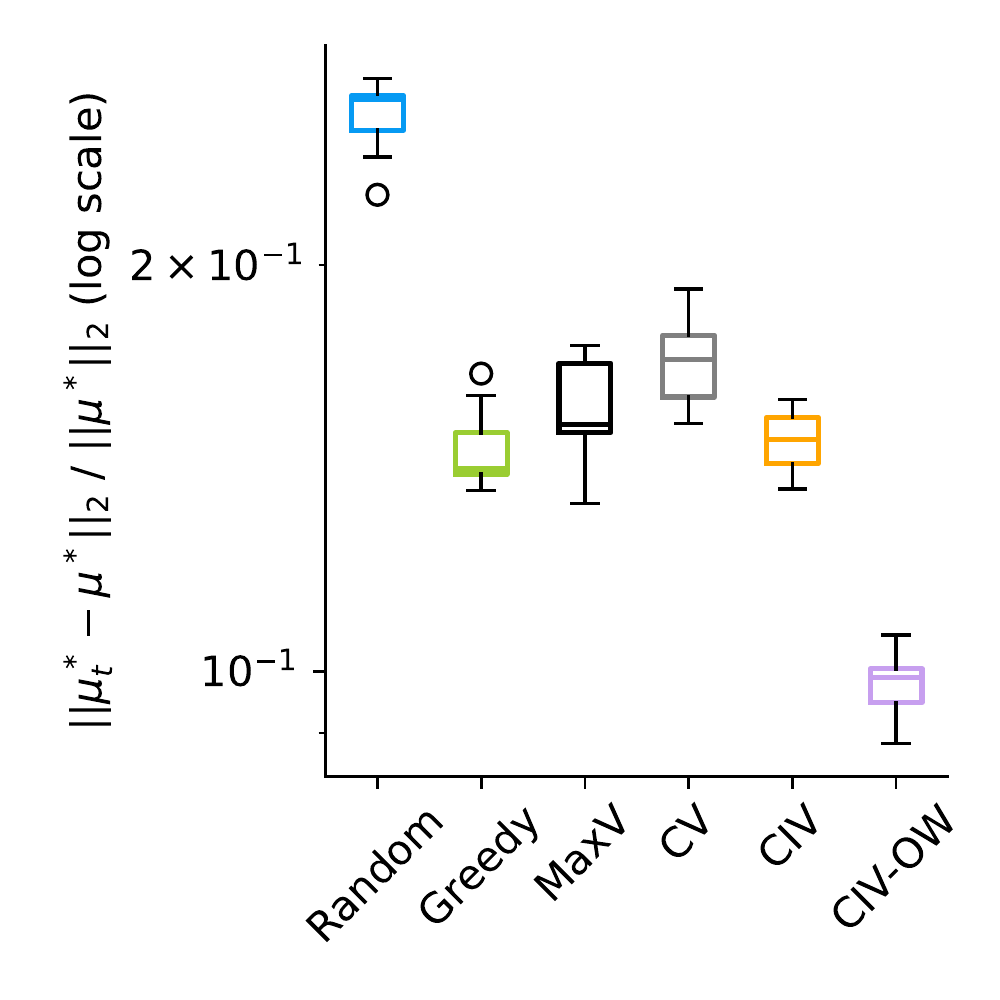}
         \caption{$p=30$}
     \end{subfigure}
     \\
     \begin{subfigure}[b]{0.27\textwidth}
         \centering
         \includegraphics[width=\textwidth]{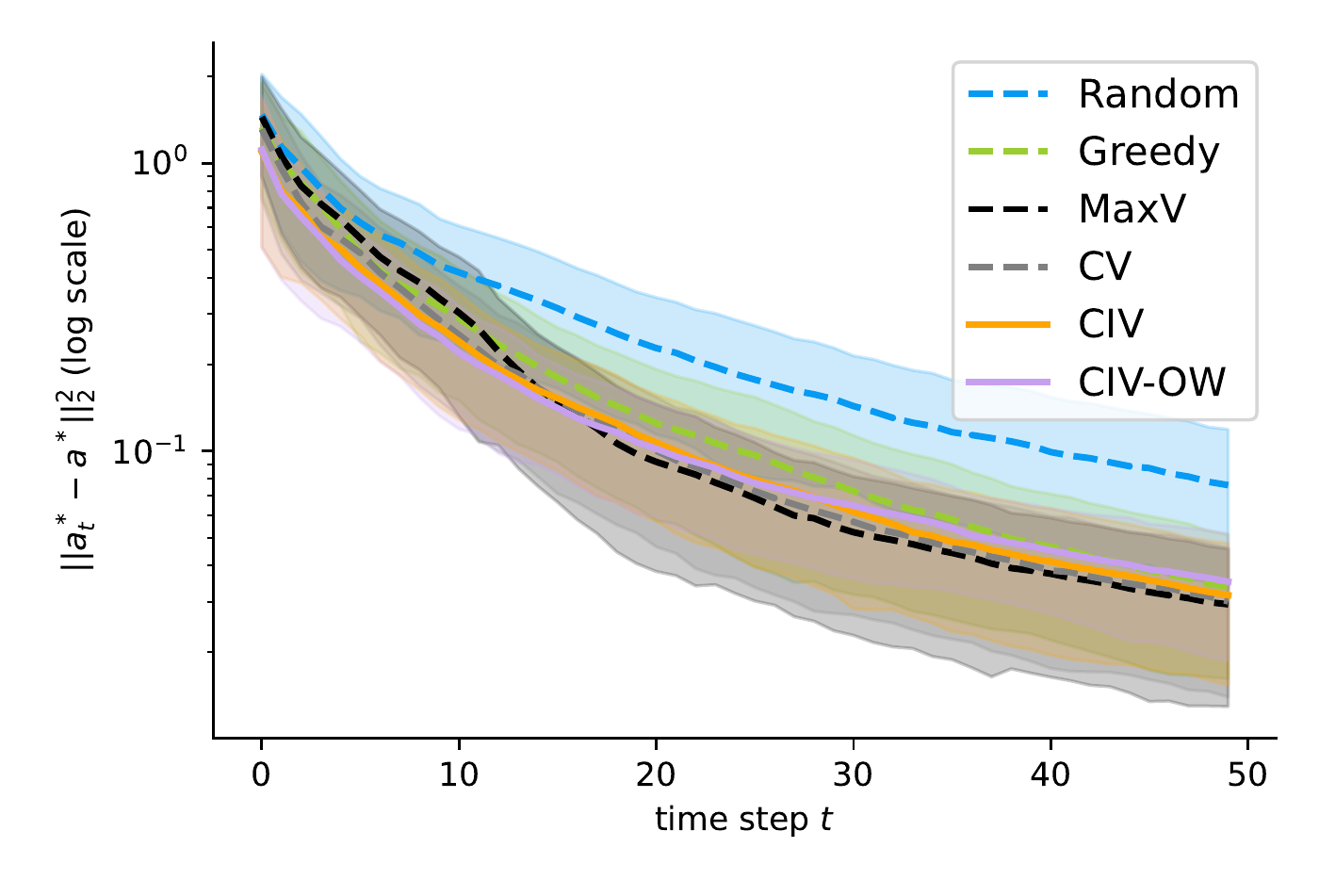}
         \caption{$p=10$}
     \end{subfigure}
     \begin{subfigure}[b]{0.27\textwidth}
         \centering
         \includegraphics[width=\textwidth]{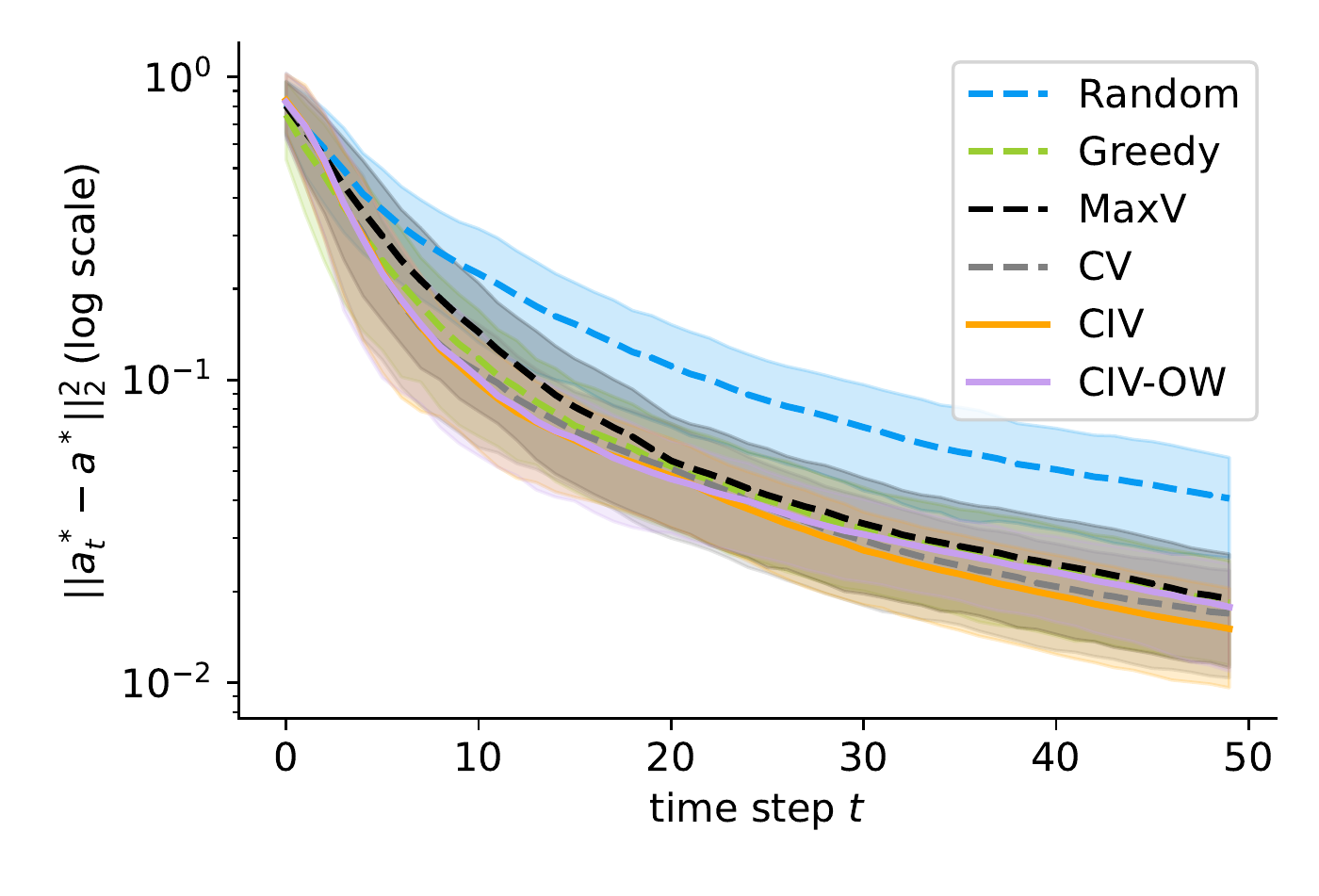}
         \caption{$p=20$}
     \end{subfigure}
    \begin{subfigure}[b]{0.27\textwidth}
         \centering
         \includegraphics[width=\textwidth]{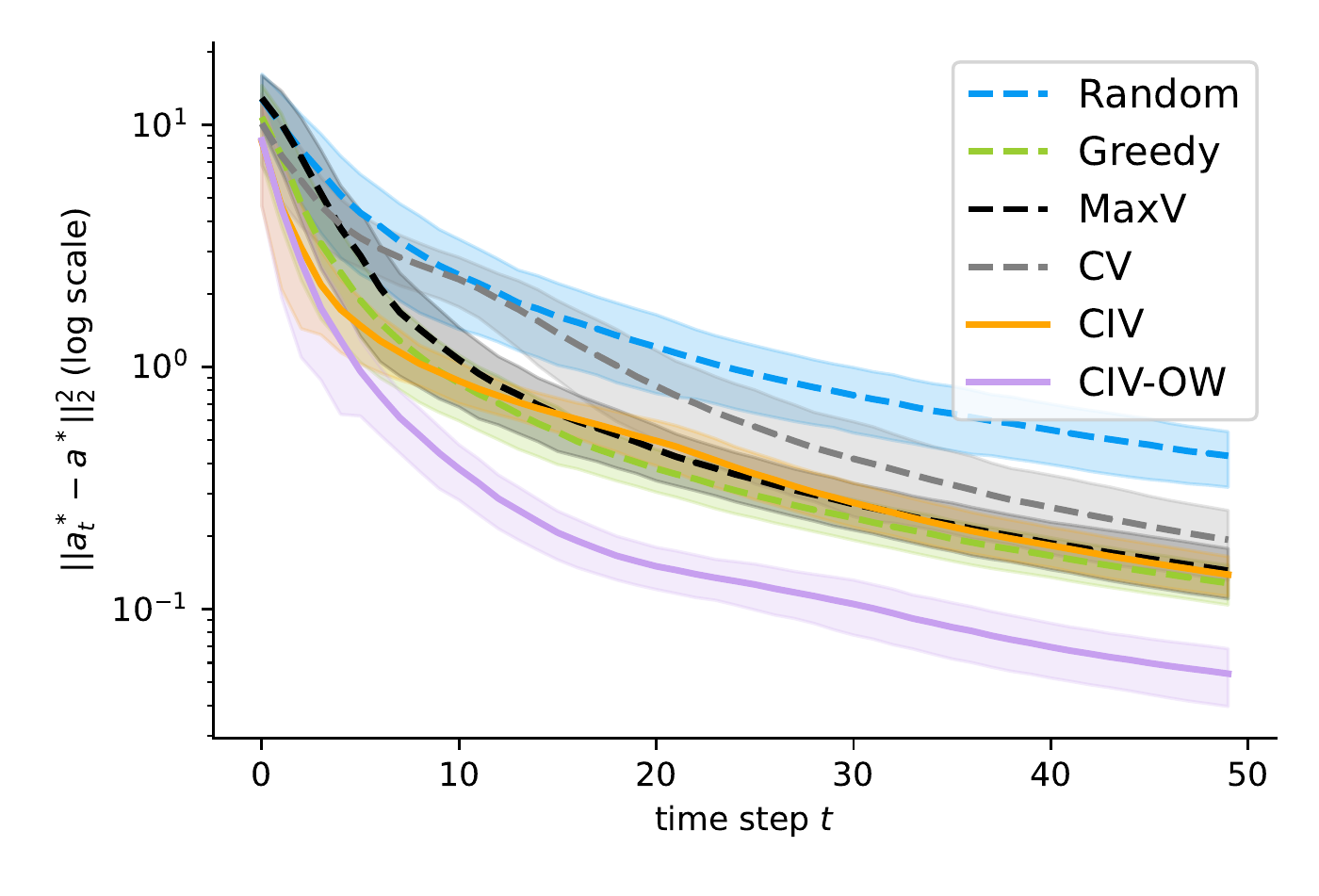}
         \caption{$p=30$}
     \end{subfigure}
    \caption{\rec{\textbf{Comparison of different acquisition functions in a simulation study where the underlying causal graph is the complete graph, half of the nodes are selected at random as intervention targets, and we vary the number of nodes $p$.}} Each plot corresponds to an average of $10$ instances and each method is run 20 times and averaged. 
    {Top row: Relative distance} between the target mean $\bmu^*$ and the best approximation $\bmu^*_t$ (defined in Fig.~\ref{fig:5}A in the main text) up to time step $t$. Lines denote the mean over 10 instances; the shading corresponds to one standard deviation. {Middle row: Relative distance} statistics of each method \rec{averaged over $10$ instances} at the last time step ($t=50$). {Bottom row: Squared distance \rec{presented as mean value +/- SEM} between the optimal intervention $\ba^*$ and the best approximation $\ba_t^*$ that is used to obtain $\bmu^*_t$ up to time step $t$.}   
    }
    \label{fig:s4}
\end{figure}

\subsection{Varying Graph Sizes}

In the following set of experiments, we consider complete graphs with a varying number of nodes $p$. 
For each DAG, we fix the number of intervention targets in $\ba^*$ to be half of all the nodes (e.g., a $20$-node DAG will be paired with an optimal intervention satisfying $\|\ba^*\|_0 = 10$). 
Such a setup ensures that the only difference between each experiment is the number of nodes.

We consider two cases in this set of experiments. 
Extended Data Fig.~\ref{fig:s4} shows the first case where the targets of $\ba^*$ are randomly sampled from the set $[p]$; Extended Data Fig.~\ref{fig:s5} removes the randomness in the selection of the intervention targets by picking the targets to be the most downstream $\nicefrac{p}{2}$ nodes.
In both cases, we observe that the five active methods outperform the passive baseline (i.e., Random), and that the three methods that consider the variance of the estimated optimal intervention (i.e., CV, CIV, and CIV-OW) generally outperform the other methods.
Importantly, the benefit of using active methods is magnified for larger graphs, and
%
%
the two proposed approaches, CIV and CIV-OW, are consistently among the best-performing methods in all experiments. 
This observation is even clearer in the setting where the randomness in selecting the intervention targets is removed; see Extended Data Fig.~\ref{fig:s5}.

\begin{figure}[!t]
     \centering
     \begin{subfigure}[b]{0.27\textwidth}
         \centering
         \includegraphics[width=\textwidth]{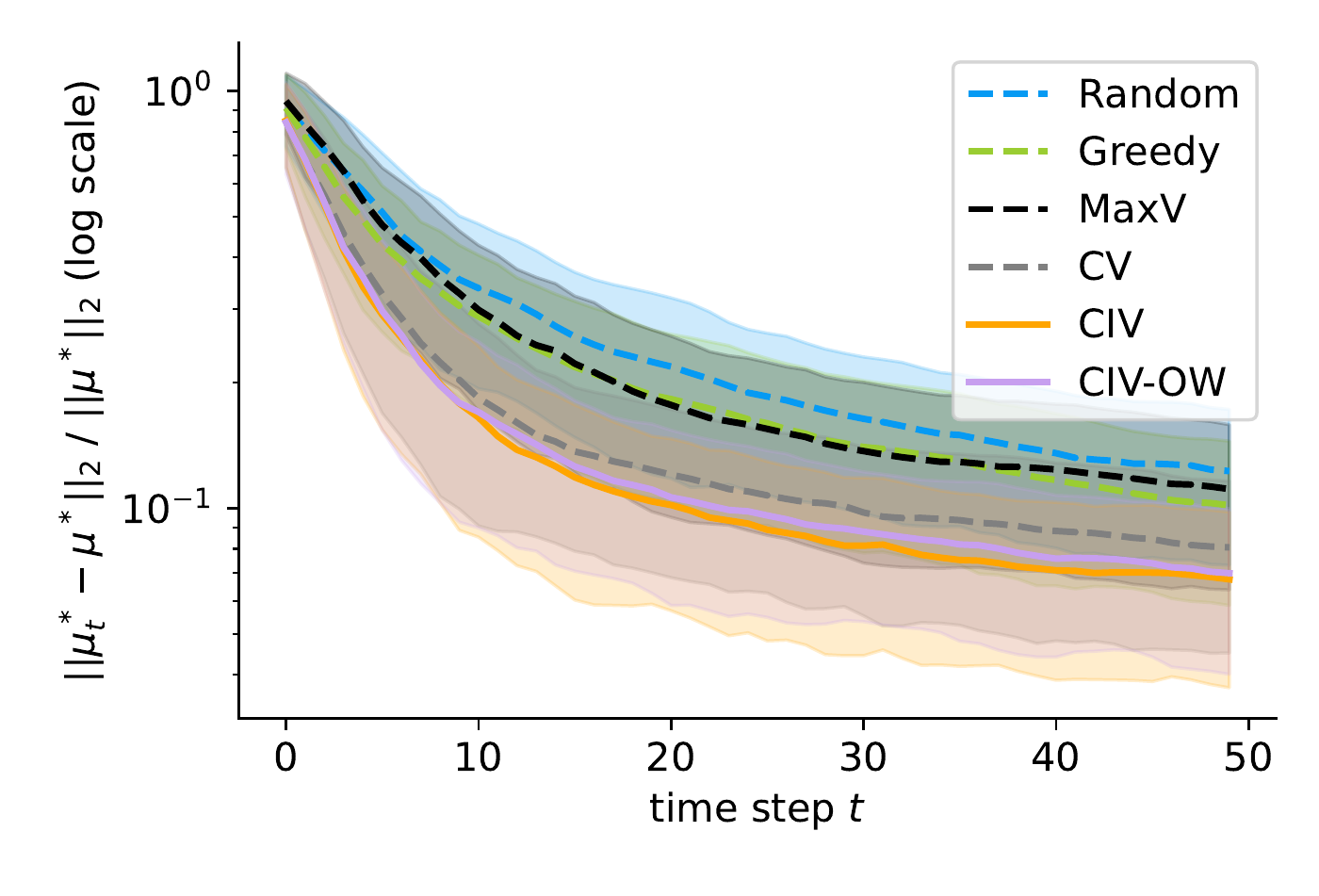}
         \caption{$p=10$}
     \end{subfigure}
     \begin{subfigure}[b]{0.27\textwidth}
         \centering
         \includegraphics[width=\textwidth]{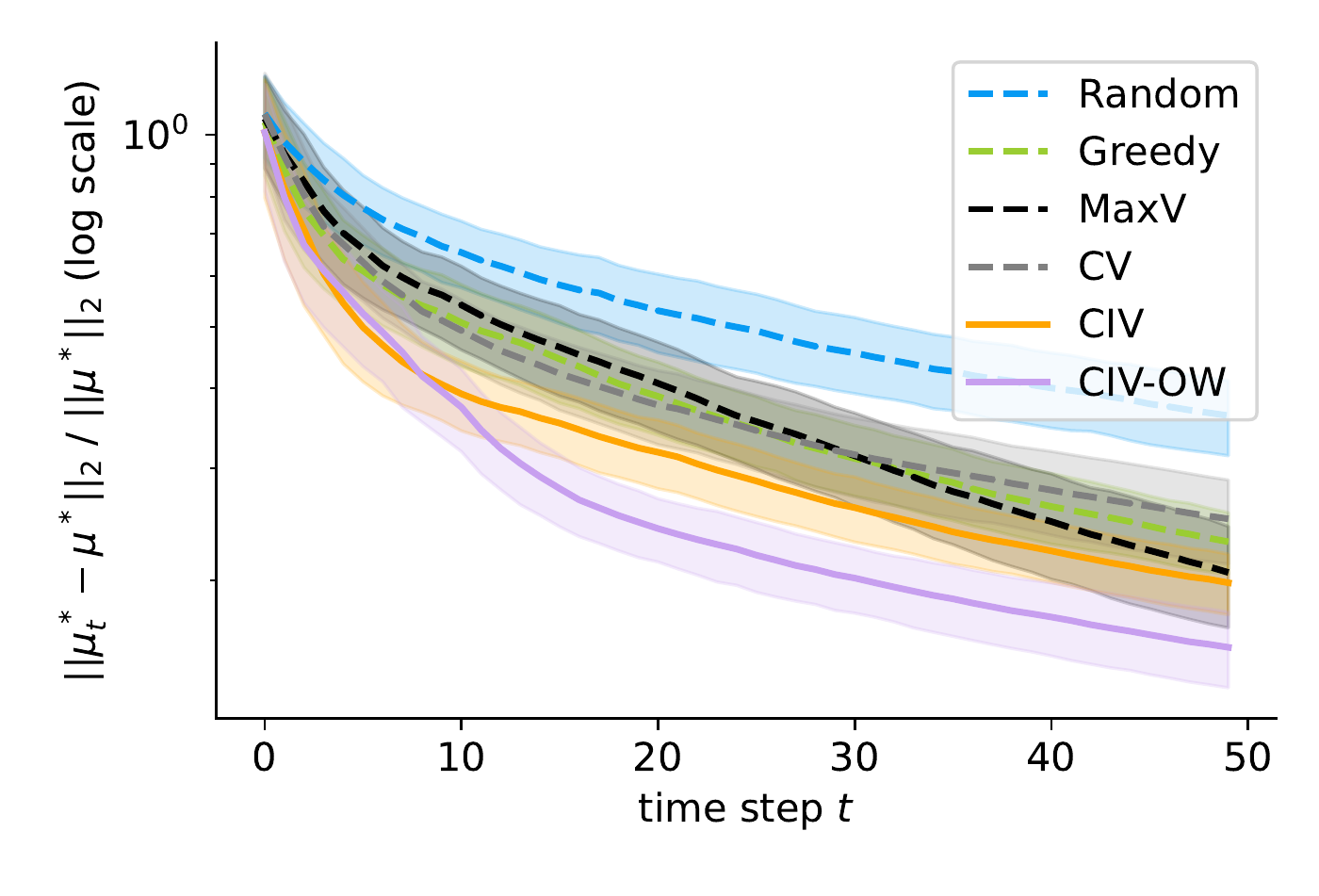}
         \caption{$p=20$}
     \end{subfigure}
    \begin{subfigure}[b]{0.27\textwidth}
         \centering
         \includegraphics[width=\textwidth]{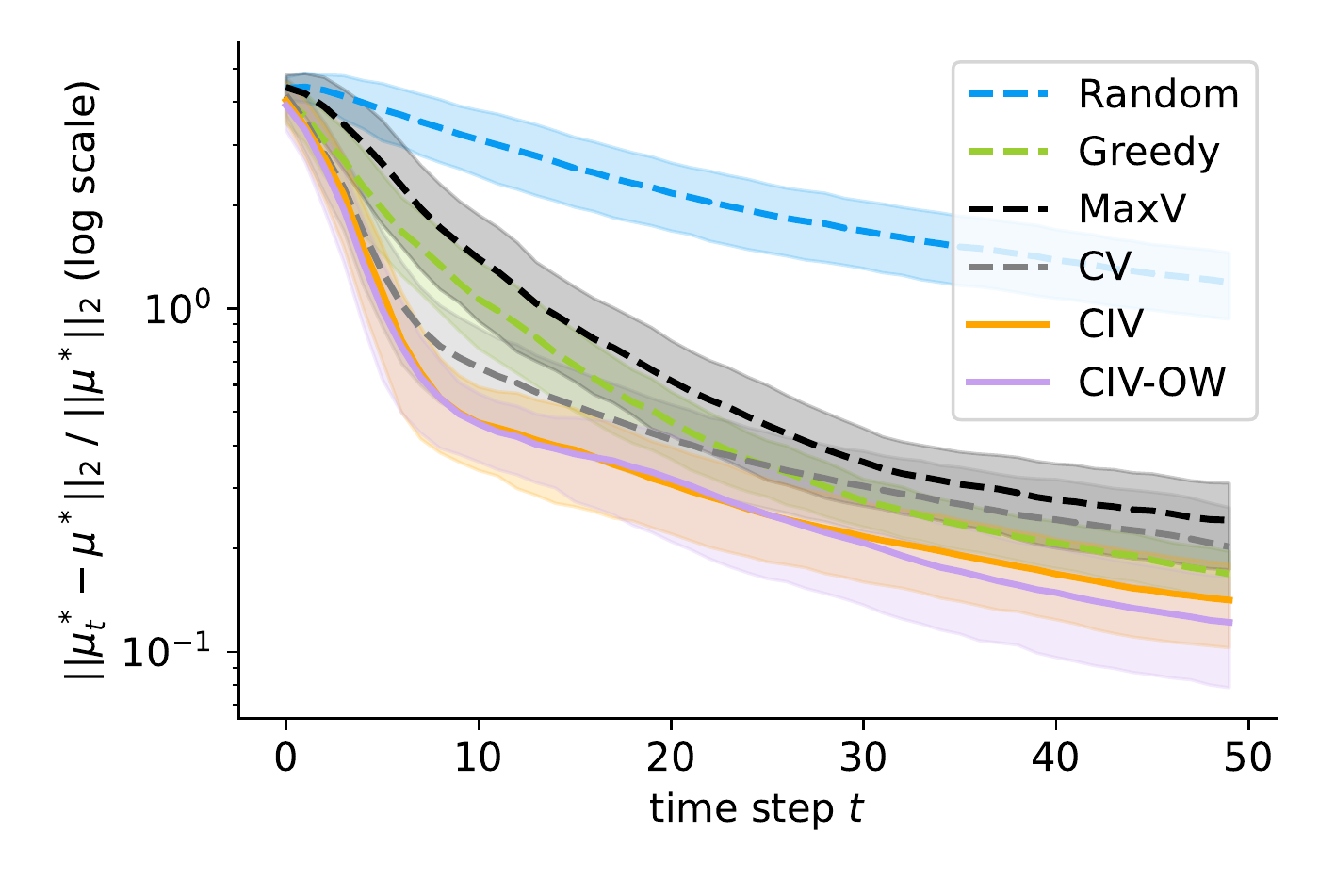}
         \caption{$p=30$}
     \end{subfigure}
     \\
     \begin{subfigure}[b]{0.27\textwidth}
         \centering
         \includegraphics[width=.8\textwidth]{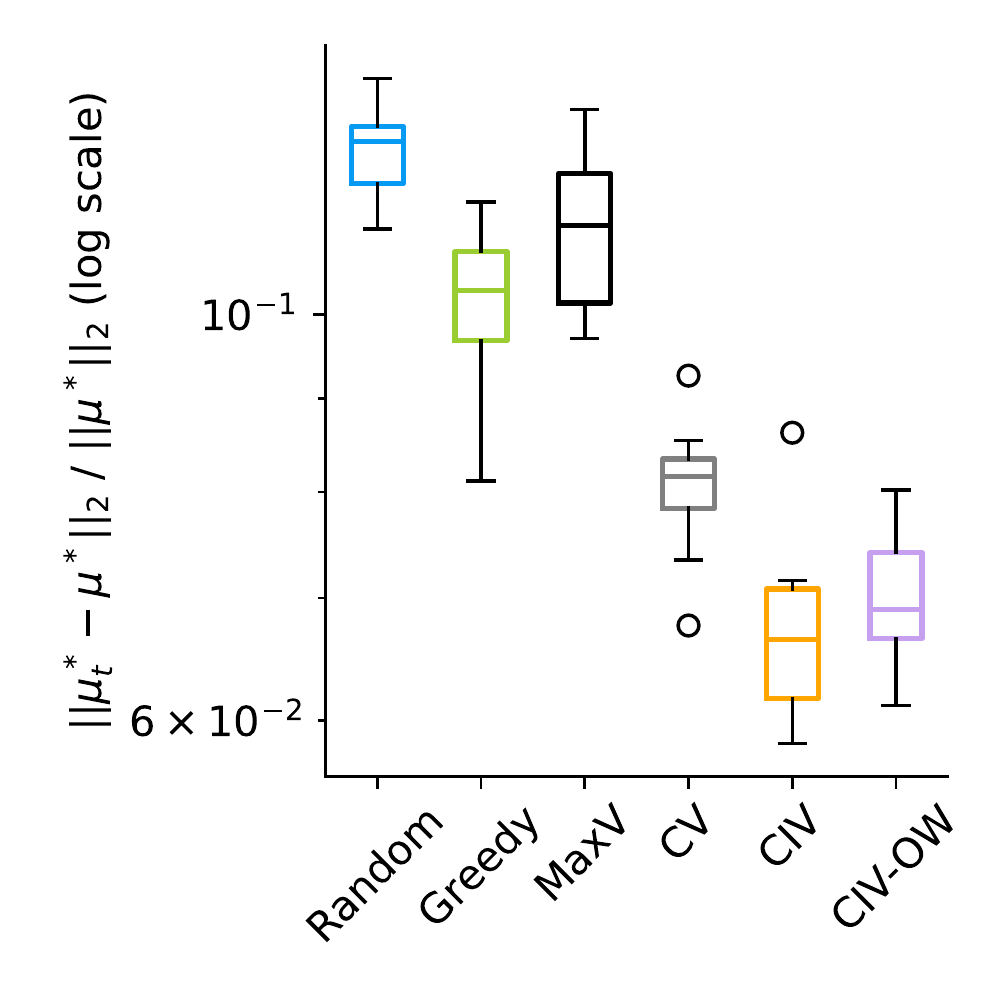}
         \caption{$p=10$}
     \end{subfigure}
     \begin{subfigure}[b]{0.27\textwidth}
         \centering
         \includegraphics[width=.8\textwidth]{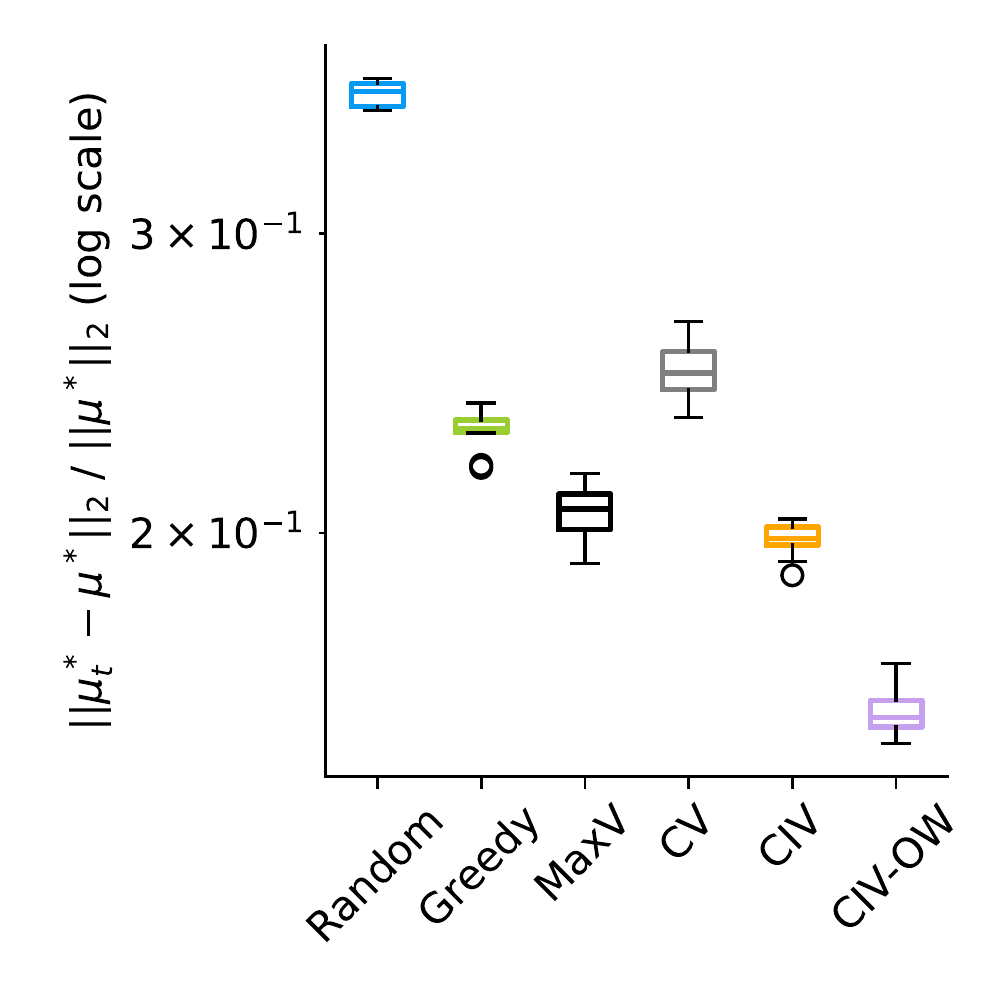}
         \caption{$p=20$}
     \end{subfigure}
    \begin{subfigure}[b]{0.27\textwidth}
         \centering
         \includegraphics[width=.8\textwidth]{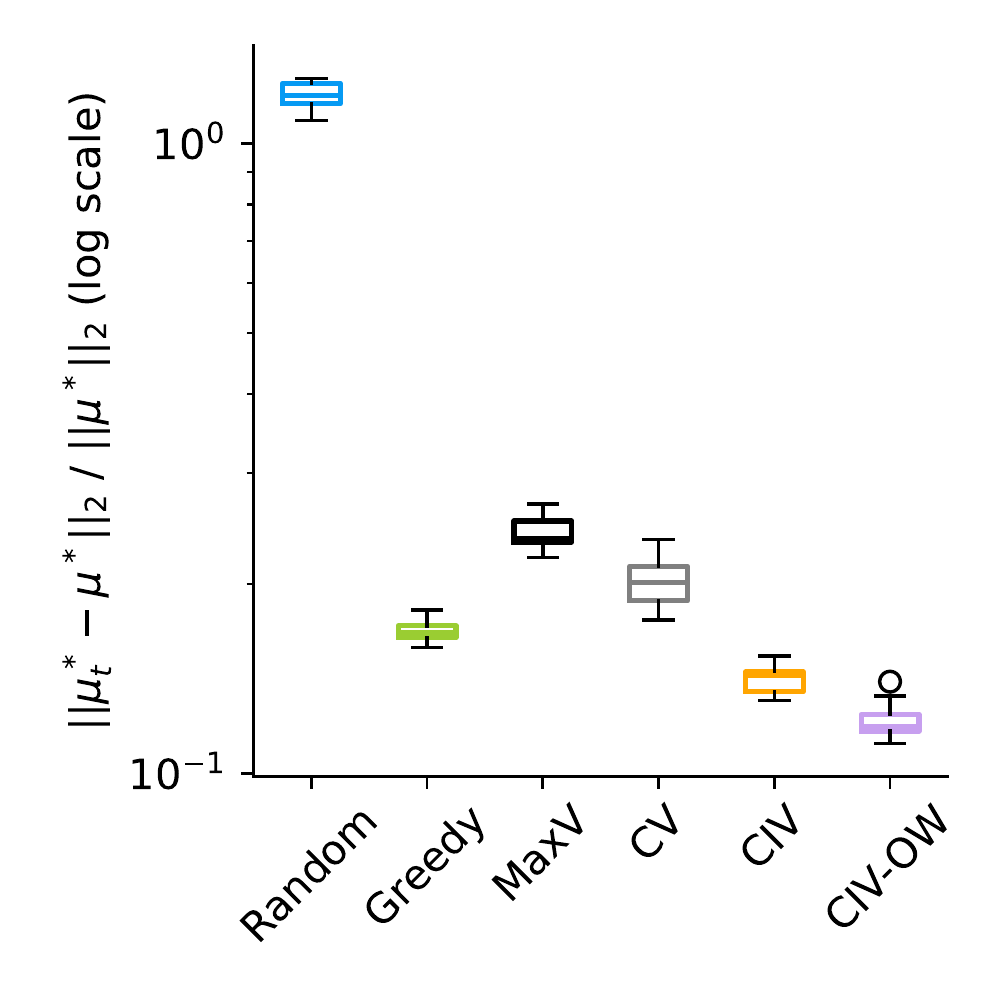}
         \caption{$p=30$}
     \end{subfigure}
    \\
    \begin{subfigure}[b]{0.27\textwidth}
         \centering
         \includegraphics[width=\textwidth]{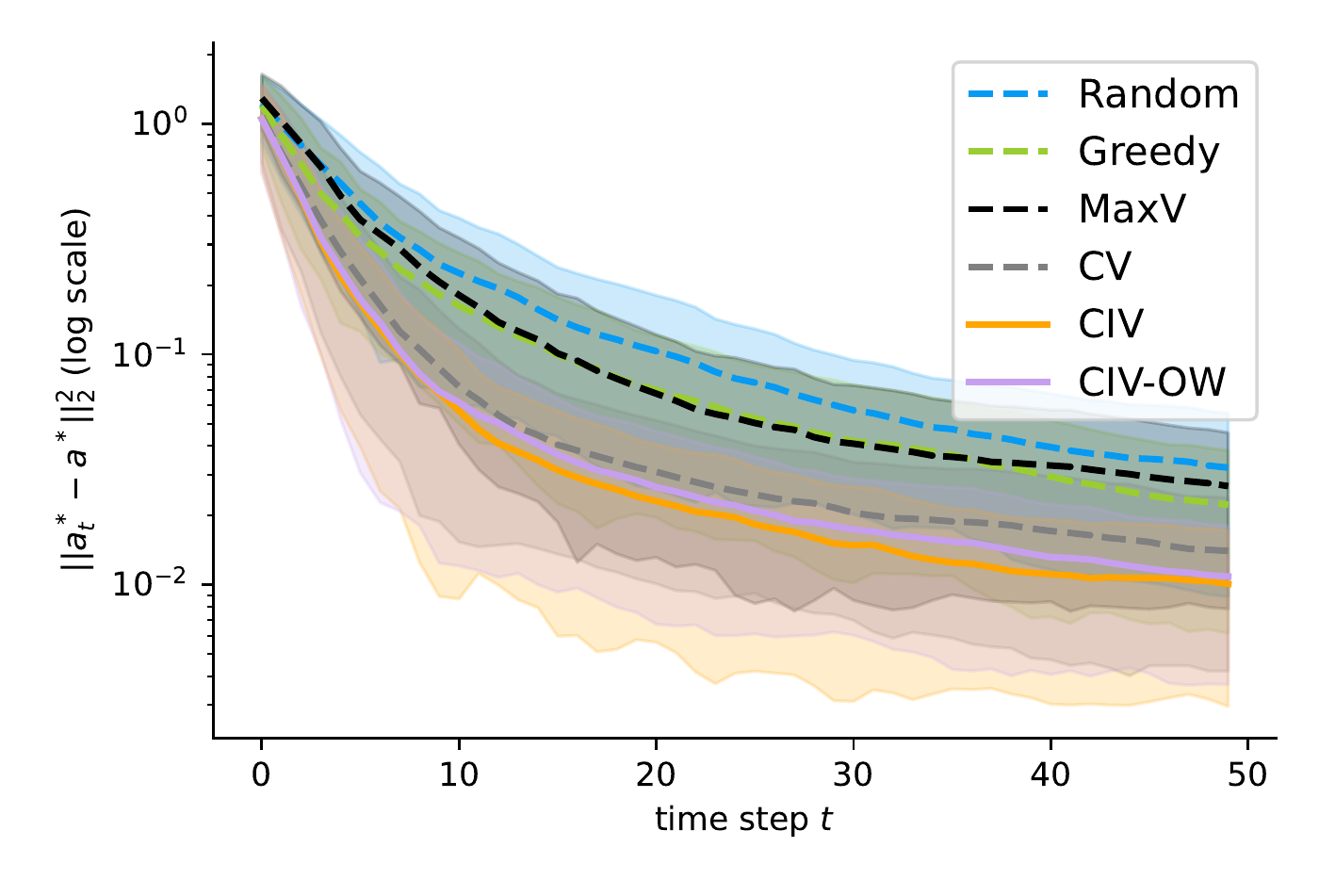}
         \caption{$p=10$}
     \end{subfigure}
     \begin{subfigure}[b]{0.27\textwidth}
         \centering
         \includegraphics[width=\textwidth]{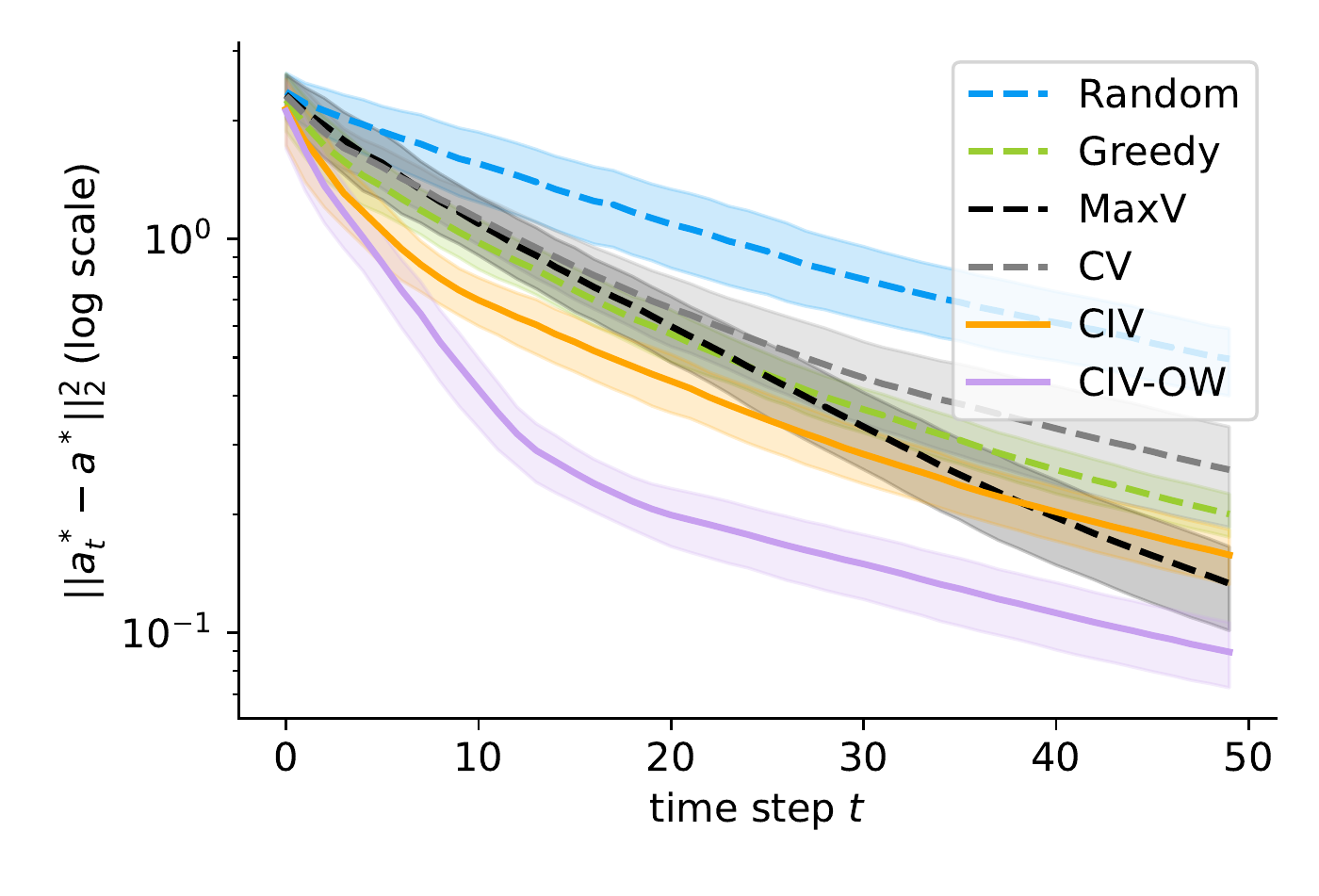}
         \caption{$p=20$}
     \end{subfigure}
    \begin{subfigure}[b]{0.27\textwidth}
         \centering
         \includegraphics[width=\textwidth]{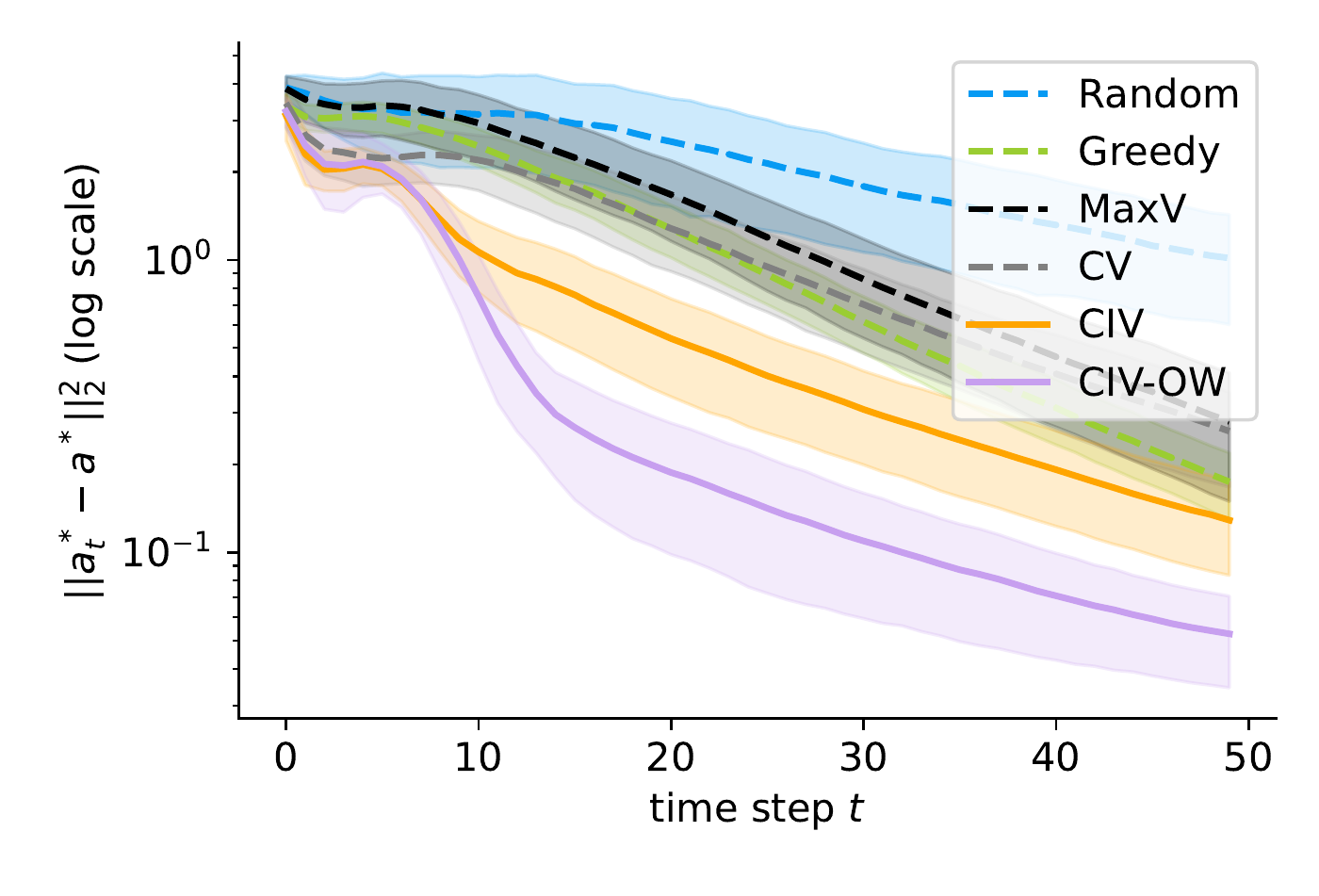}
         \caption{$p=30$}
     \end{subfigure}
    \caption{\rec{\textbf{Comparison of different acquisition functions in a simulation study where the underlying causal graph is the complete graph, the most downstream half of the nodes are fixed as intervention targets, and we vary the number of nodes $p$.}} Each plot corresponds to an average of $10$ instances and each method is run 20 times and averaged.
    {Top row: Relative distance} between the target mean $\bmu^*$ and the best approximation $\bmu^*_t$ (defined in Fig.~\ref{fig:5}A in the main text) up to time step $t$. Lines denote the mean over 10 instances; the shading corresponds to one standard deviation. {Middle row: Relative distance} statistic of each method \rec{averaged over $10$ instances} at the last time step ($t=50$).  {Bottom row: Squared distance \rec{presented as mean value +/- SEM} between the optimal intervention $\ba^*$ and the best approximation $\ba_t^*$ that is used to obtain $\bmu^*_t$ up to time step $t$.}
    }\label{fig:s5}
\end{figure}

\subsection{Varying Number of intervention targets} 

We next compare the different methods when varying the number of intervention targets in $\ba^*$. 
For these experiments, we fix the DAG to be a $30$-node complete graph. 
To examine the effect of $\|\ba^*\|_0$, we fix the intervention targets to be the most downstream $\|\ba^*\|_0$ nodes, except for the case when $\|\ba^*\|_0=1$. 
For $\|\ba^*\|_0=1$, as discussed in Remark~\ref{rmk:4}, setting the intervention target to be a sink node degenerates to a trivial case where one can easily infer $\ba^*$ from $\bmu^*$.
Therefore, in this case, we set the intervention target to be the most downstream node that is not a sink node. 
Extended Data Fig.~\ref{fig:s6} shows the experimental results for $\|\ba^*\|_0$ ranging from $1$ to $15$. 
We observe that the proposed methods, CIV and CIV-OW, consistently outperform other baselines, with CIV-OW outperforming CIV in higher dimensions. 
%

\begin{figure}[th]
     \centering
     \begin{subfigure}[b]{0.24\textwidth}
         \centering
         \includegraphics[width=\textwidth]{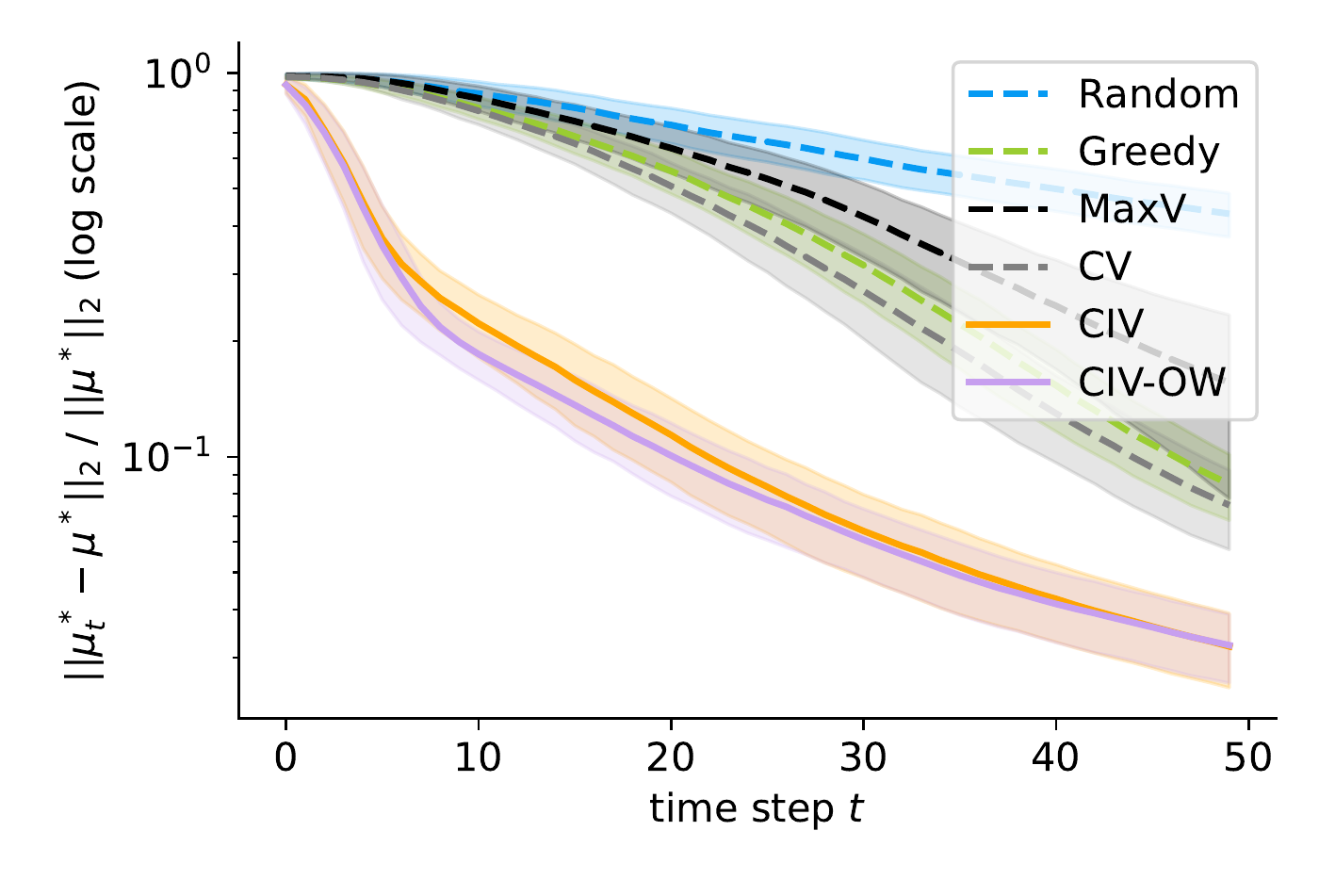}
         \caption{$\|\ba^*\|_0=1$}
     \end{subfigure}
     \begin{subfigure}[b]{0.24\textwidth}
         \centering
         \includegraphics[width=\textwidth]{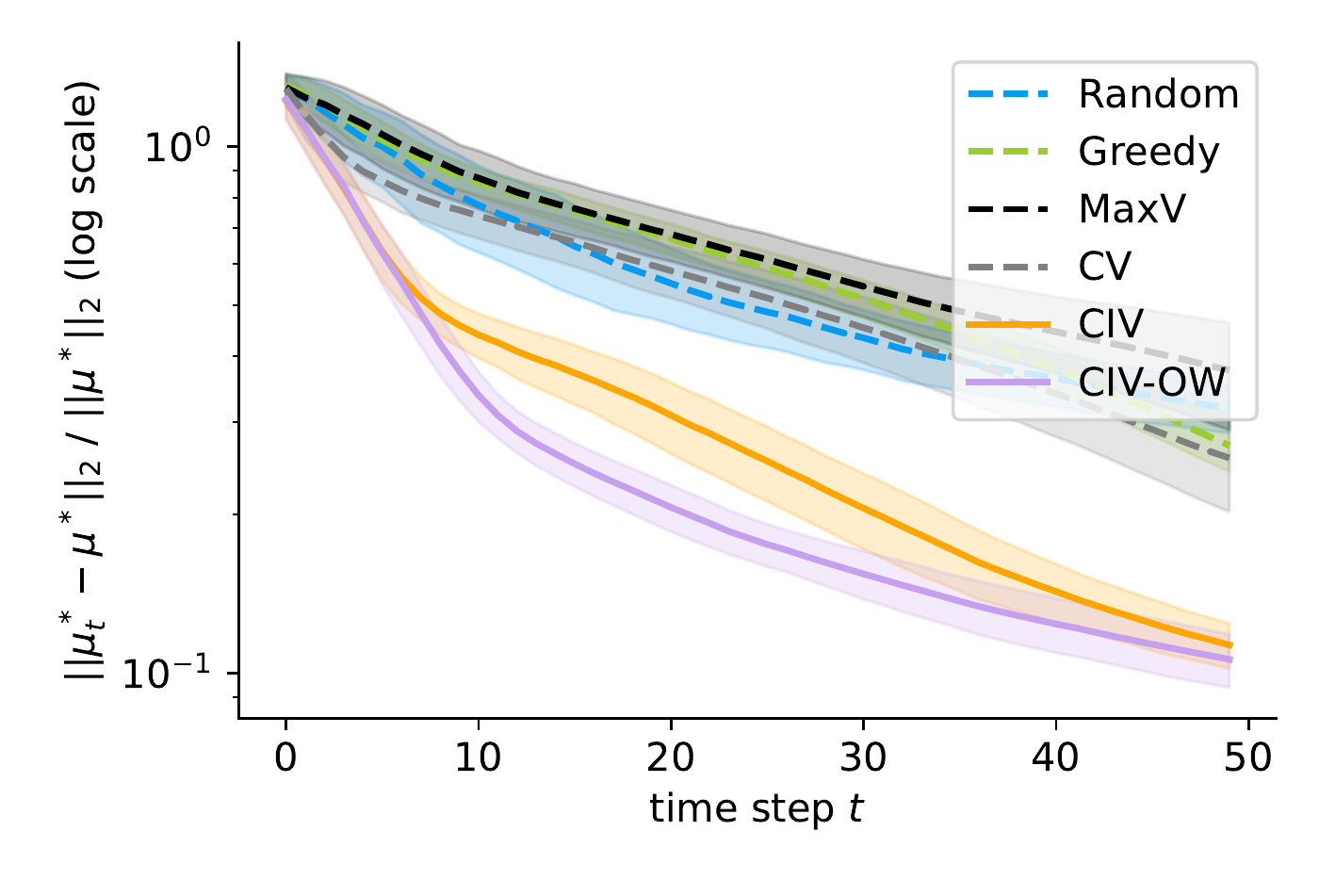}
         \caption{$\|\ba^*\|_0=5$}
     \end{subfigure}
     \begin{subfigure}[b]{0.24\textwidth}
         \centering
         \includegraphics[width=\textwidth]{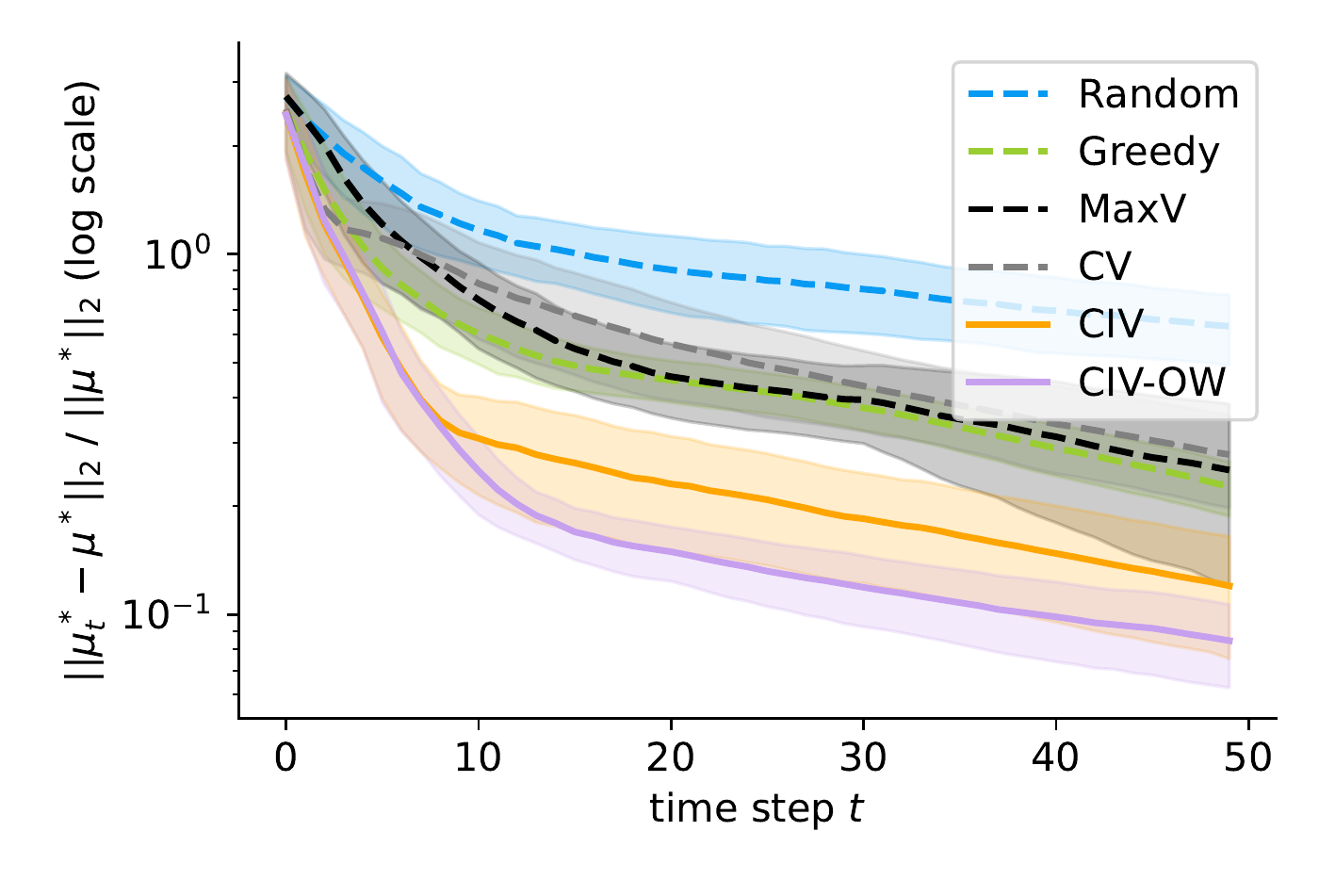}
         \caption{$\|\ba^*\|_0=10$}
     \end{subfigure}
    \begin{subfigure}[b]{0.24\textwidth}
         \centering
         \includegraphics[width=\textwidth]{final-figs/appendix/relative-rmse_complete-30-15.pdf}
         \caption{$\|\ba^*\|_0=15$}
     \end{subfigure}
     \\
     \begin{subfigure}[b]{0.24\textwidth}
         \centering
         \includegraphics[width=.8\textwidth]{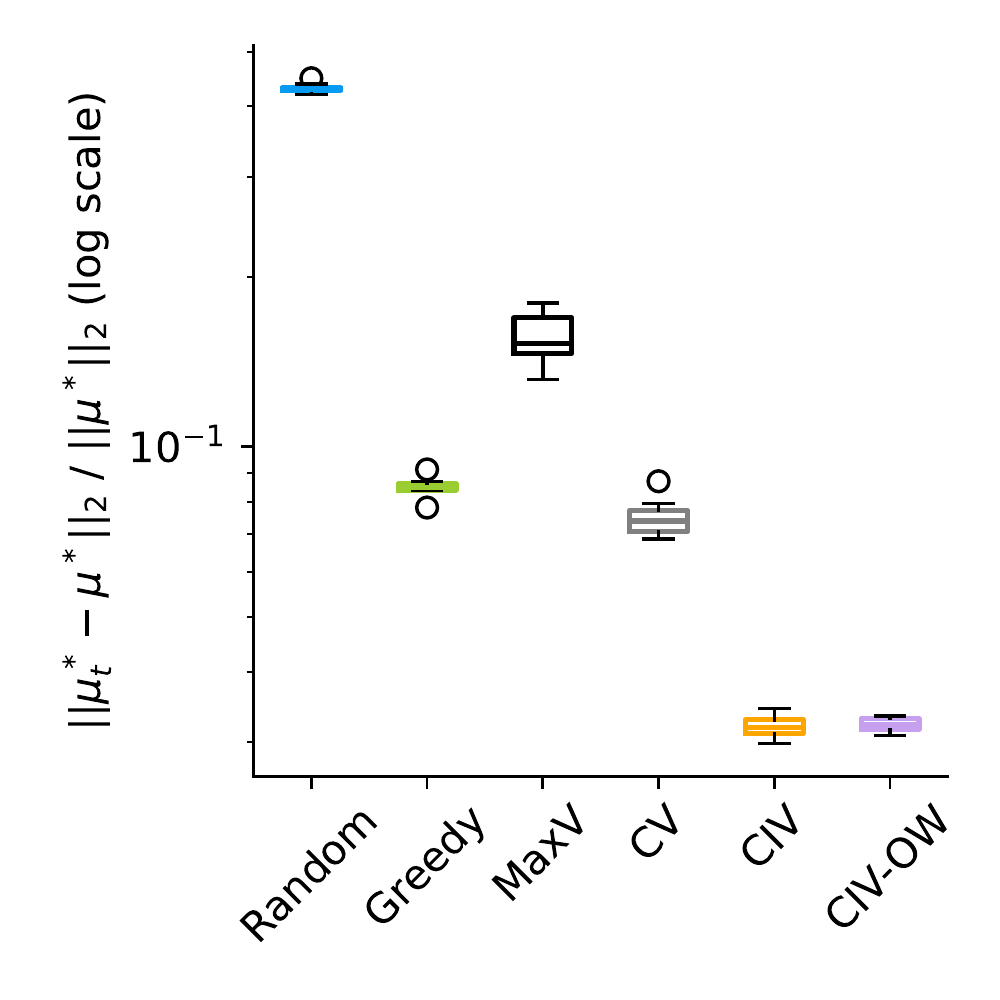}
         \caption{$\|\ba^*\|_0=1$}
     \end{subfigure}
     \begin{subfigure}[b]{0.24\textwidth}
         \centering
         \includegraphics[width=.8\textwidth]{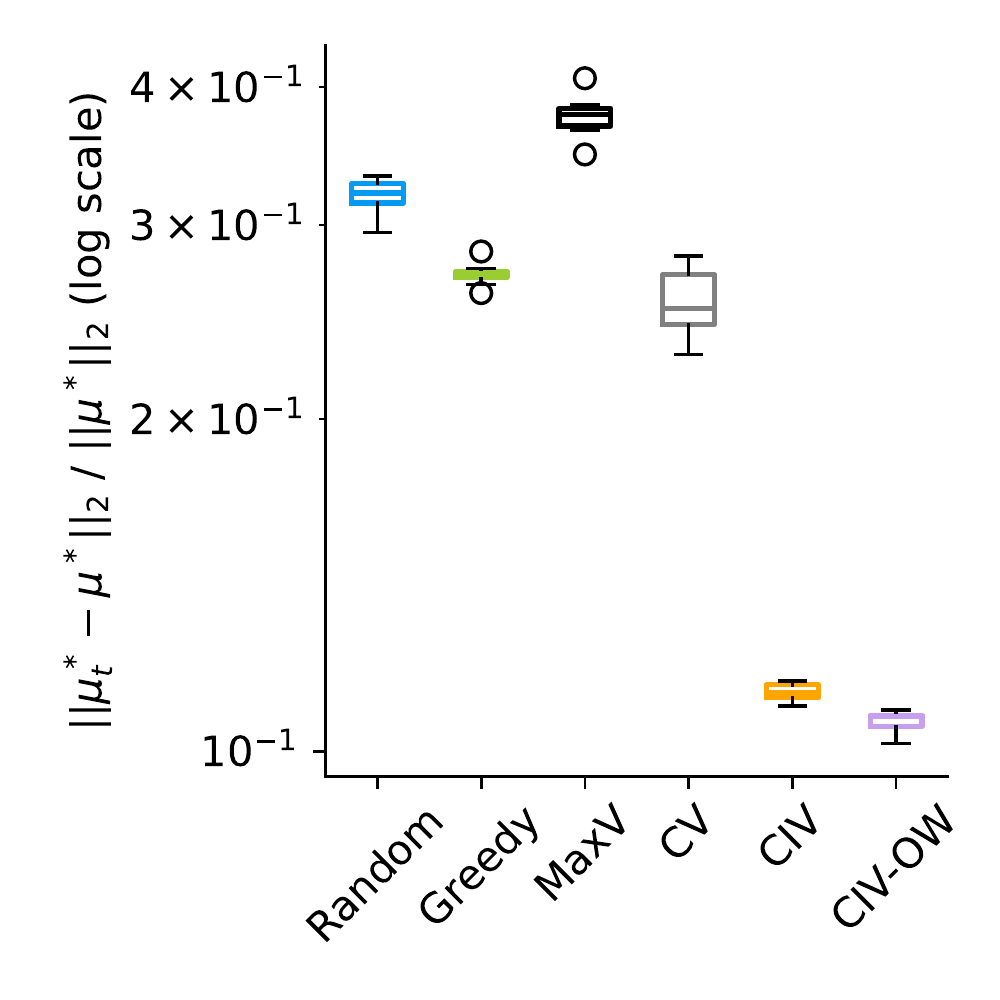}
         \caption{$\|\ba^*\|_0=5$}
     \end{subfigure}
     \begin{subfigure}[b]{0.24\textwidth}
         \centering
         \includegraphics[width=.8\textwidth]{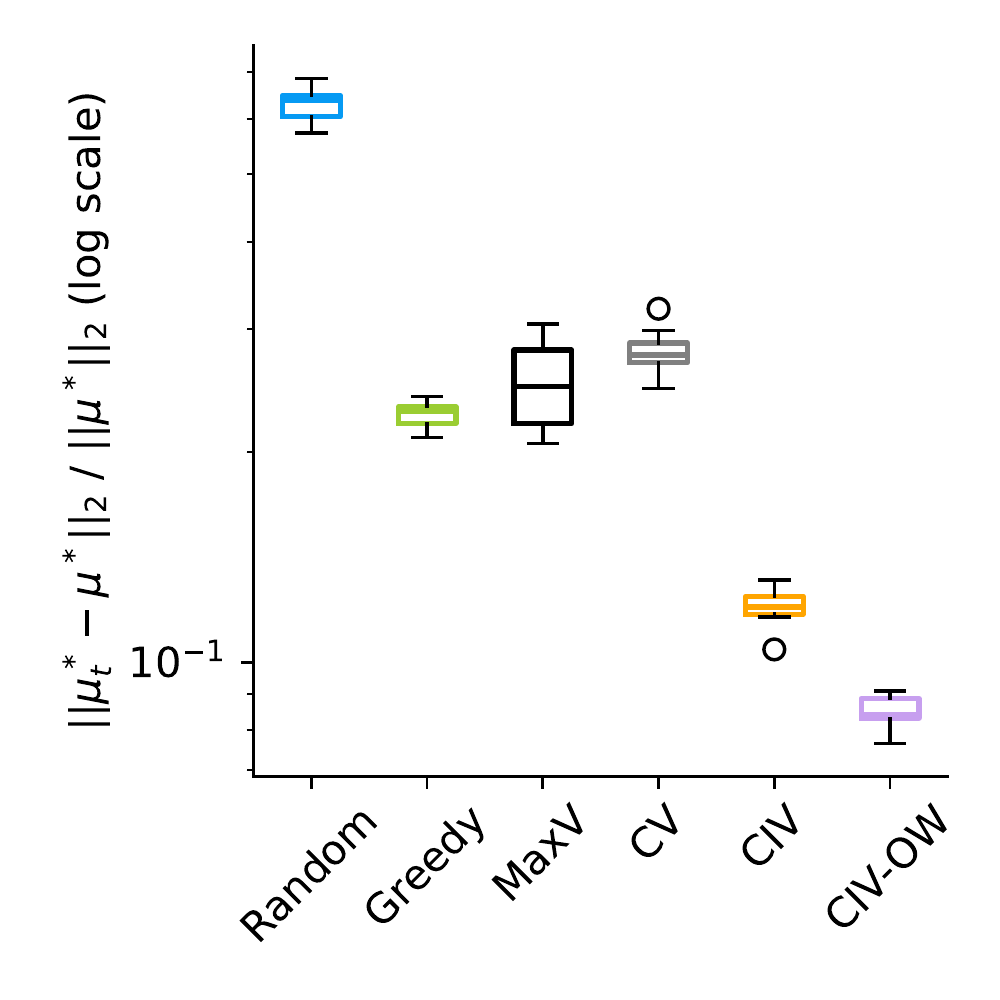}
         \caption{$\|\ba^*\|_0=10$}
     \end{subfigure}
    \begin{subfigure}[b]{0.24\textwidth}
         \centering
         \includegraphics[width=.8\textwidth]{final-figs/appendix/relative-err-lastround_complete-30-15.pdf}
         \caption{$\|\ba^*\|_0=15$}
     \end{subfigure}
    \\
     \begin{subfigure}[b]{0.24\textwidth}
         \centering
         \includegraphics[width=\textwidth]{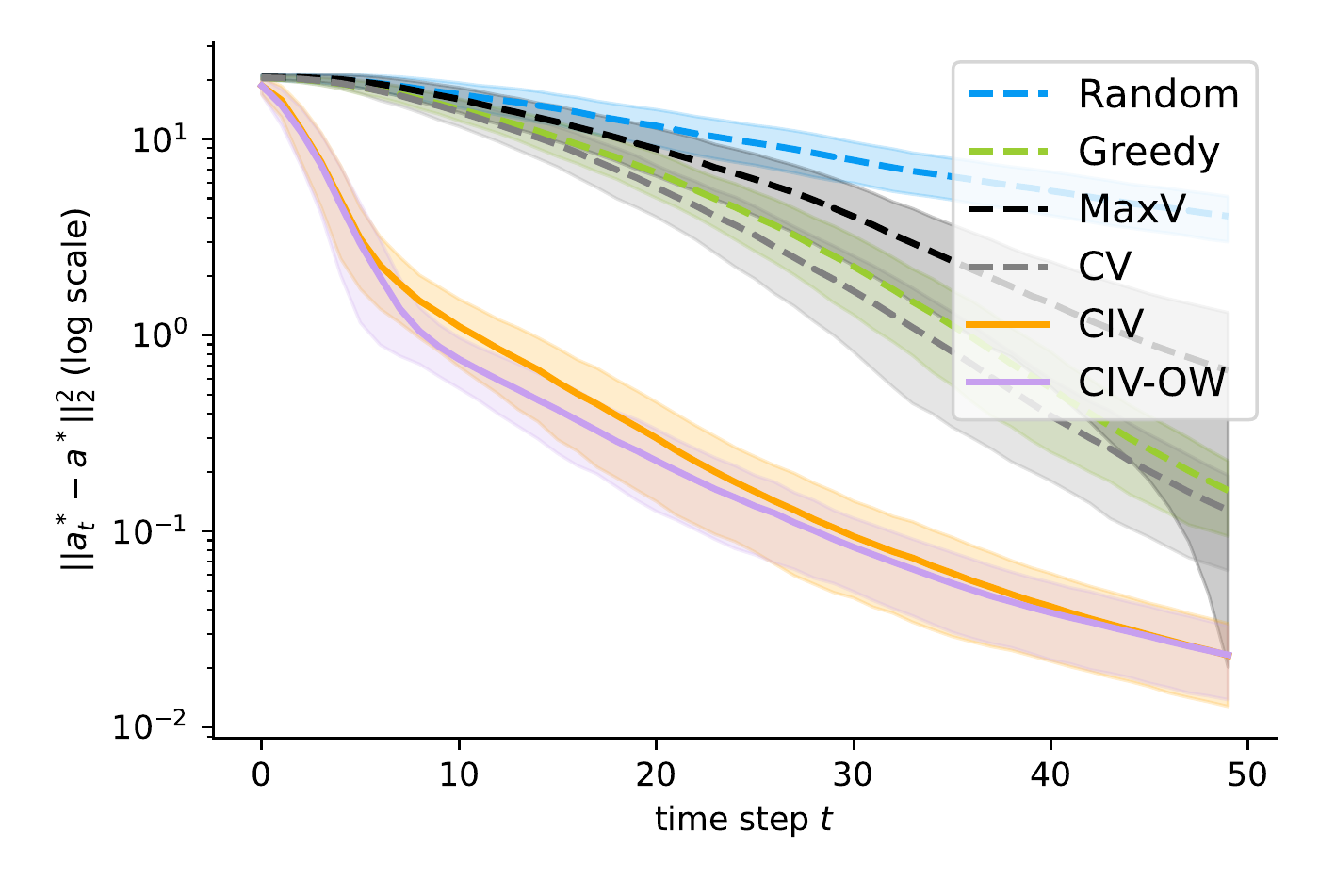}
         \caption{$\|\ba^*\|_0=1$}
     \end{subfigure}
     \begin{subfigure}[b]{0.24\textwidth}
         \centering
         \includegraphics[width=\textwidth]{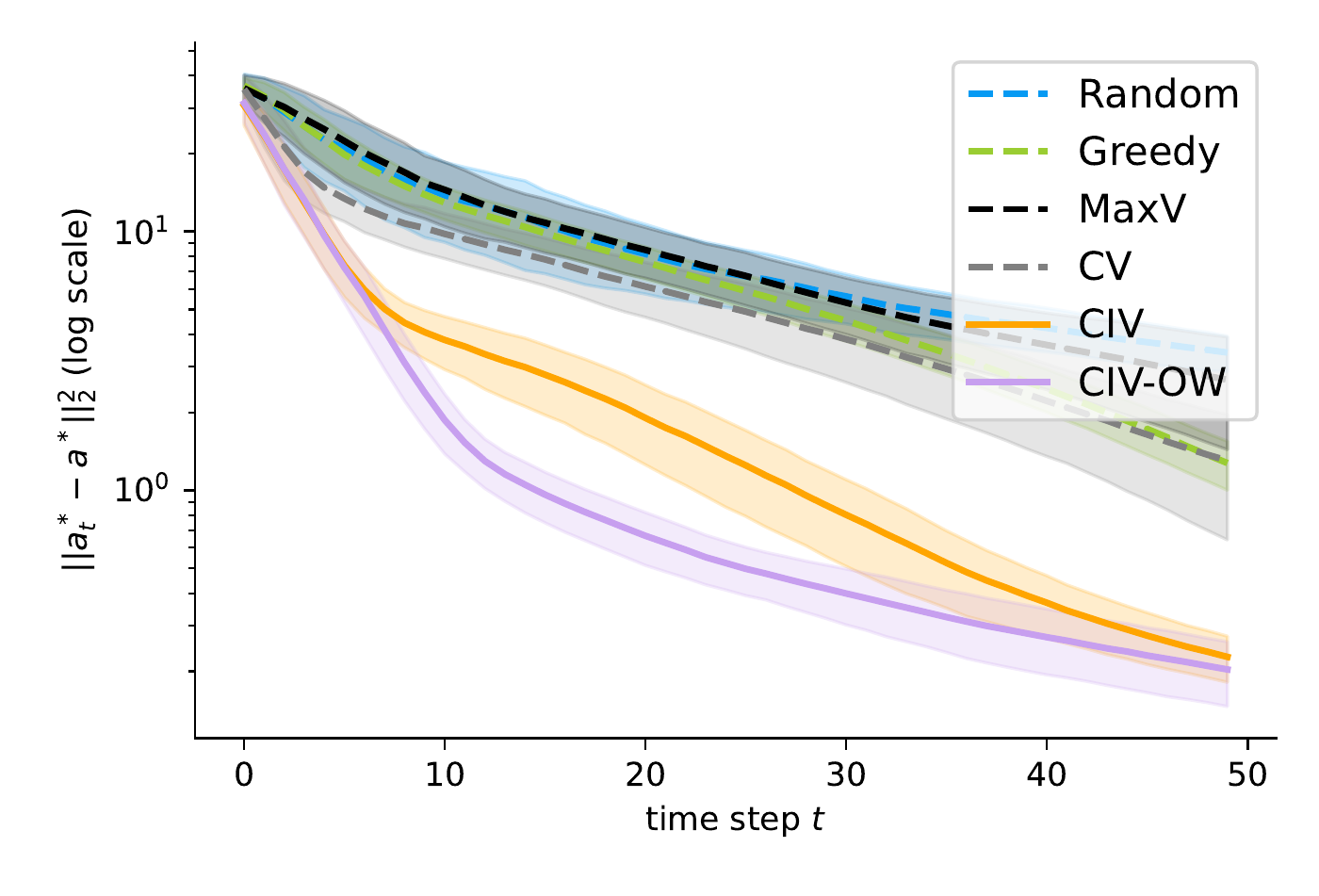}
         \caption{$\|\ba^*\|_0=5$}
     \end{subfigure}
     \begin{subfigure}[b]{0.24\textwidth}
         \centering
         \includegraphics[width=\textwidth]{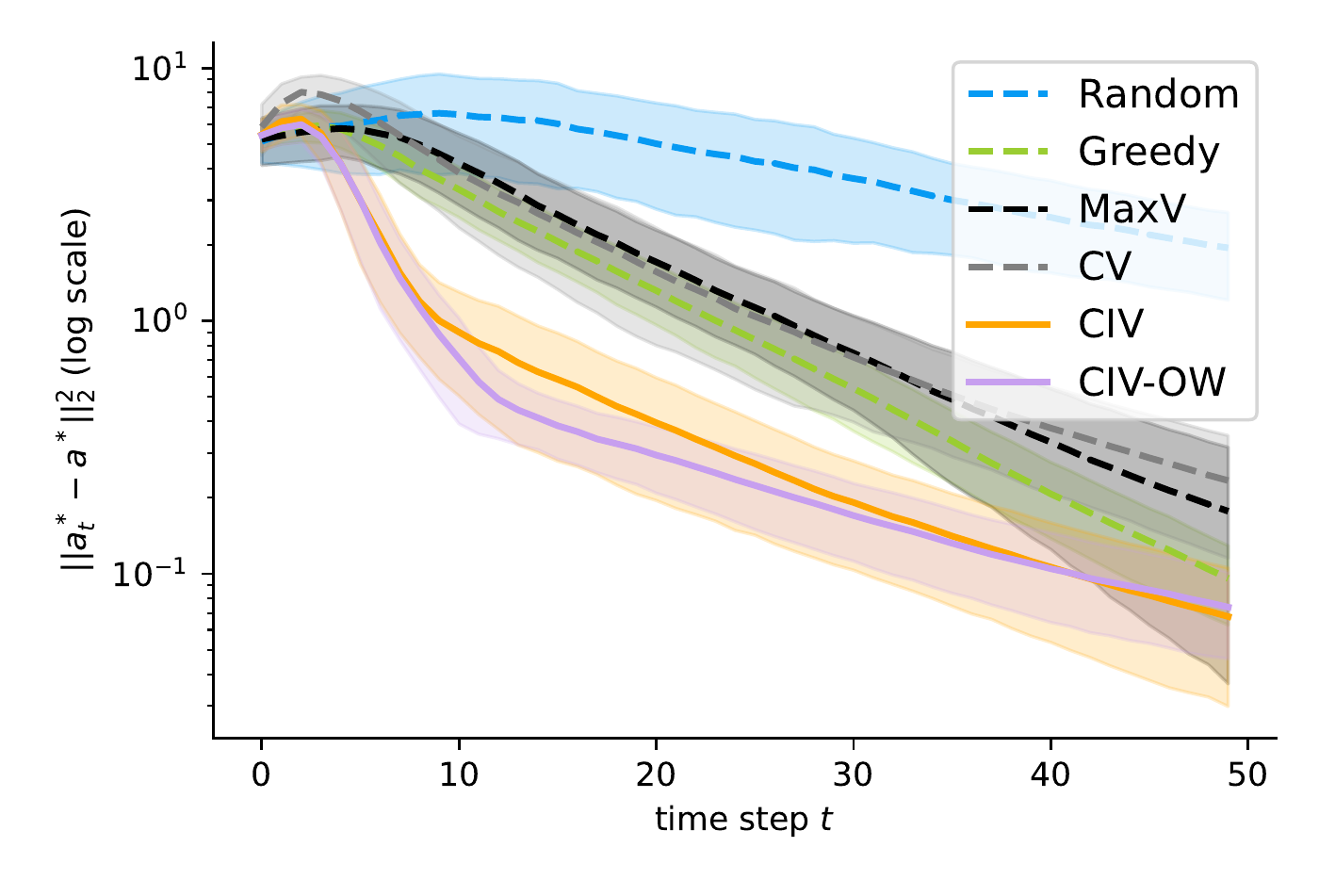}
         \caption{$\|\ba^*\|_0=10$}
     \end{subfigure}
    \begin{subfigure}[b]{0.24\textwidth}
         \centering
         \includegraphics[width=\textwidth]{final-figs/appendix/mse-int-mean_complete-30-15.pdf}
         \caption{$\|\ba^*\|_0=15$}
     \end{subfigure}
    \caption{\rec{\textbf{Comparison of different acquisition functions in a simulation study where the underlying causal graph is the complete graph on 30 nodes, the most downstream nodes are fixed as intervention targets, and we vary the number of intervention targets.}} Each plot corresponds to an average of $10$ instances and each method is run 20 times and averaged. {Top row: Relative distance} between the target mean $\bmu^*$ and the best approximation $\bmu^*_t$ (defined in Fig.~\ref{fig:5}A in the main text) up to time step $t$. Lines denote the mean over 10 instances; the shading corresponds to one standard deviation. {Middle row: Relative distance} statistic of each method \rec{averaged over $10$ instances} at the last time step ($t=50$). {Bottom row: Squared distance \rec{presented as mean value +/- SEM} between the optimal intervention $\ba^*$ and the best approximation $\ba_t^*$ that is used to obtain $\bmu^*_t$ up to time step $t$.}
    }
    \label{fig:s6}
\end{figure}

\begin{figure}[!t]
     \centering
     \begin{subfigure}[b]{0.24\textwidth}
         \centering
         \includegraphics[width=\textwidth]{final-figs/appendix/relative-rmse_complete-30-15.pdf}
         \caption{Complete graph}
     \end{subfigure}
     \begin{subfigure}[b]{0.24\textwidth}
         \centering
         \includegraphics[width=\textwidth]{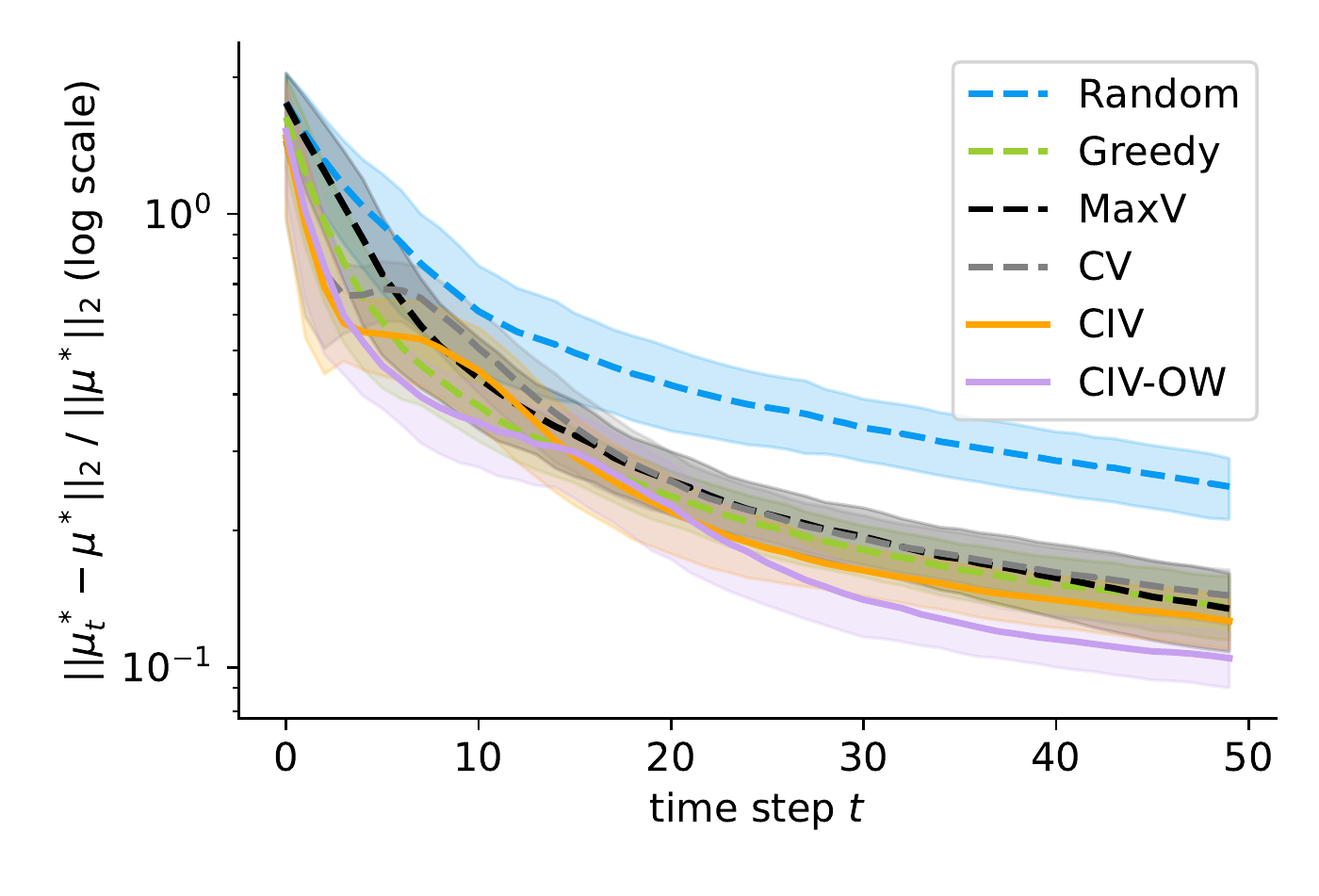}
         \caption{Erdös-Rényi graph ($0.8$)}
     \end{subfigure}
     \begin{subfigure}[b]{0.24\textwidth}
         \centering
         \includegraphics[width=\textwidth]{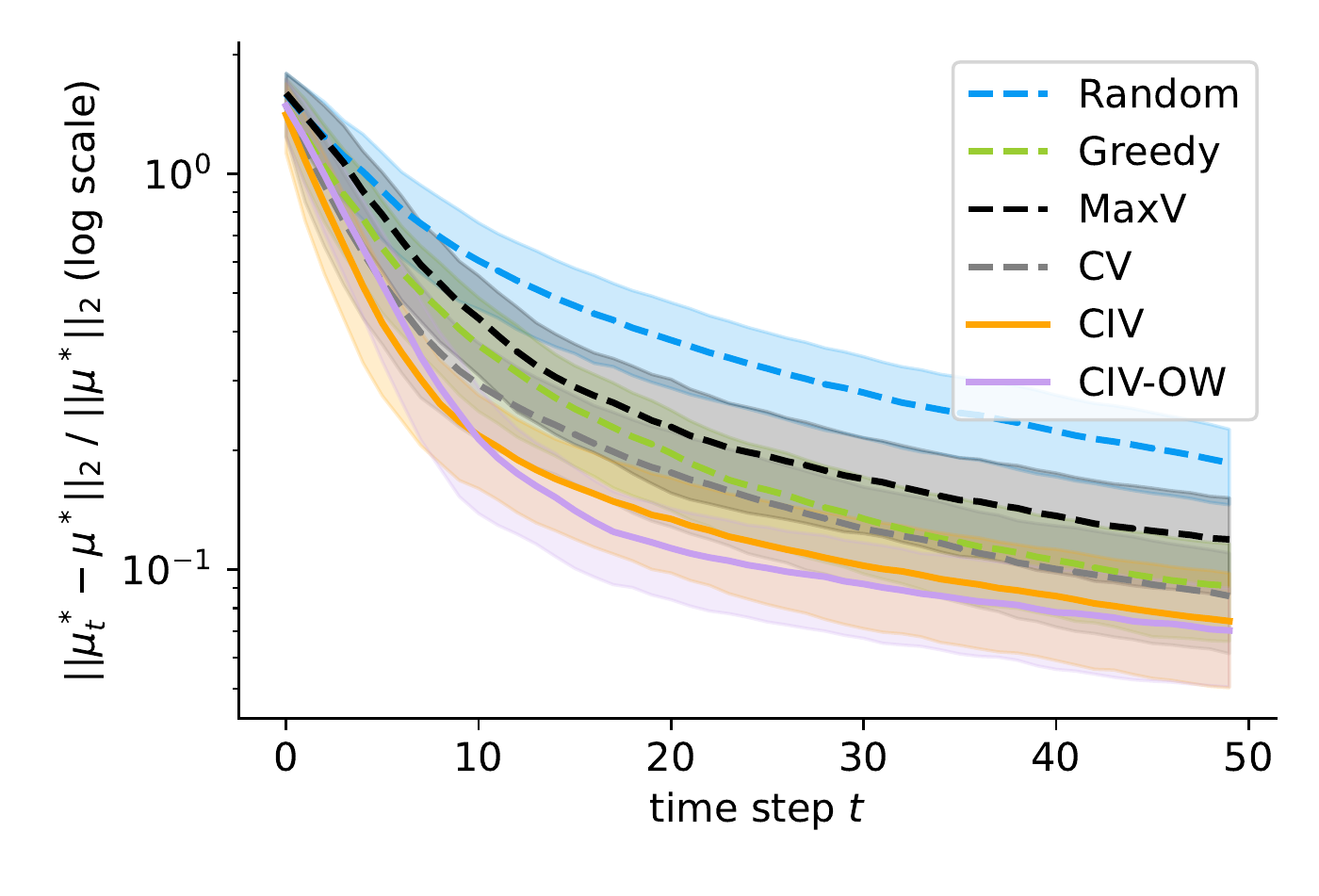}
         \caption{Erdös-Rényi graph ($0.5$)}
     \end{subfigure}
    \begin{subfigure}[b]{0.24\textwidth}
         \centering
         \includegraphics[width=\textwidth]{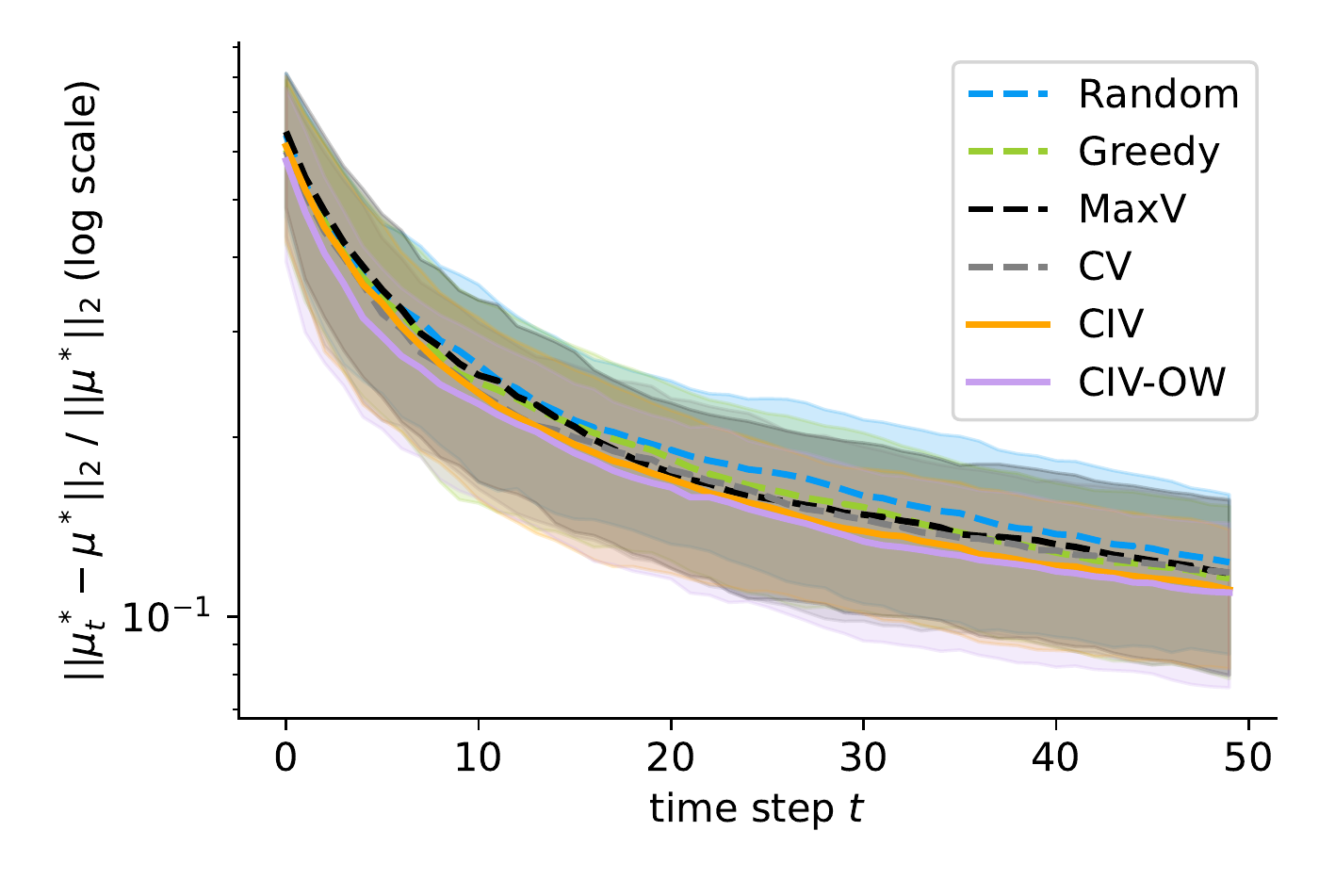}
         \caption{Path graph}
     \end{subfigure}
     \\
     \begin{subfigure}[b]{0.24\textwidth}
         \centering
         \includegraphics[width=.8\textwidth]{final-figs/appendix/relative-err-lastround_complete-30-15.pdf}
         \caption{Complete graph}
     \end{subfigure}
     \begin{subfigure}[b]{0.24\textwidth}
         \centering
         \includegraphics[width=.8\textwidth]{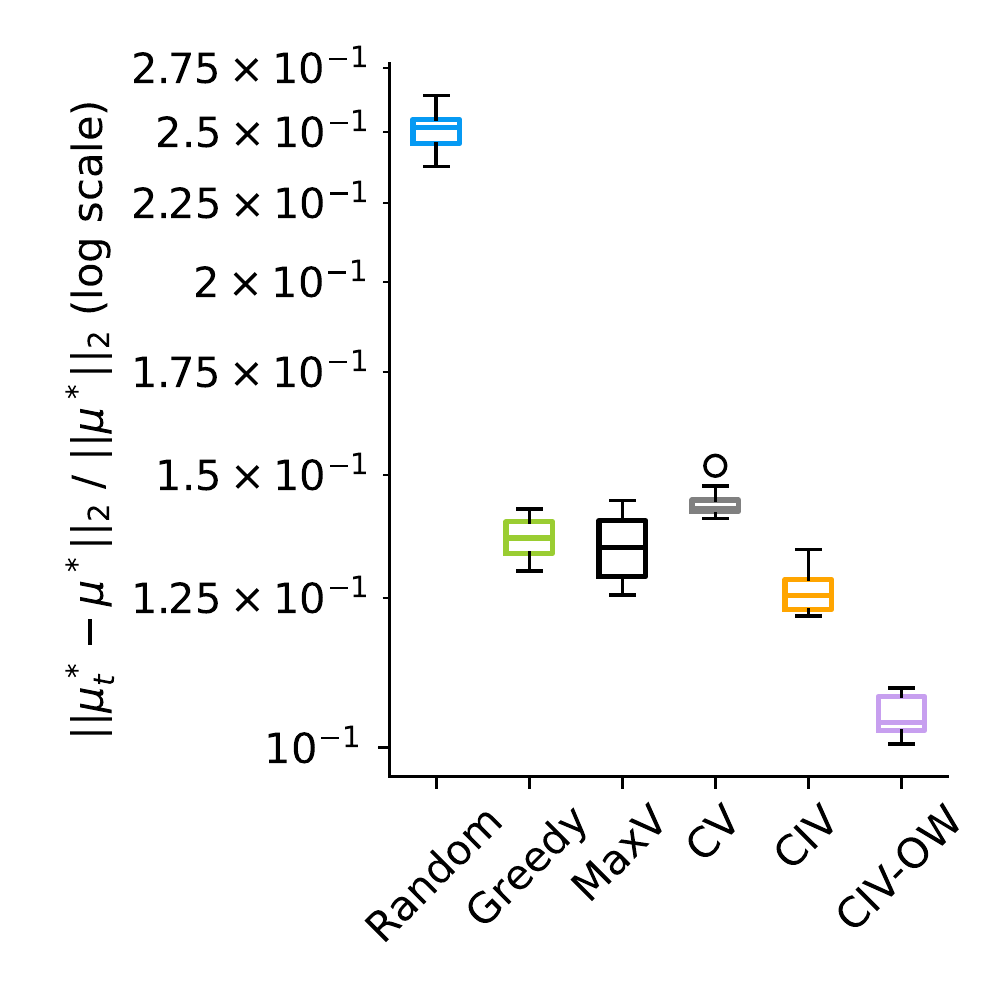}
         \caption{Erdös-Rényi graph ($0.8$)}
     \end{subfigure}
     \begin{subfigure}[b]{0.24\textwidth}
         \centering
         \includegraphics[width=.8\textwidth]{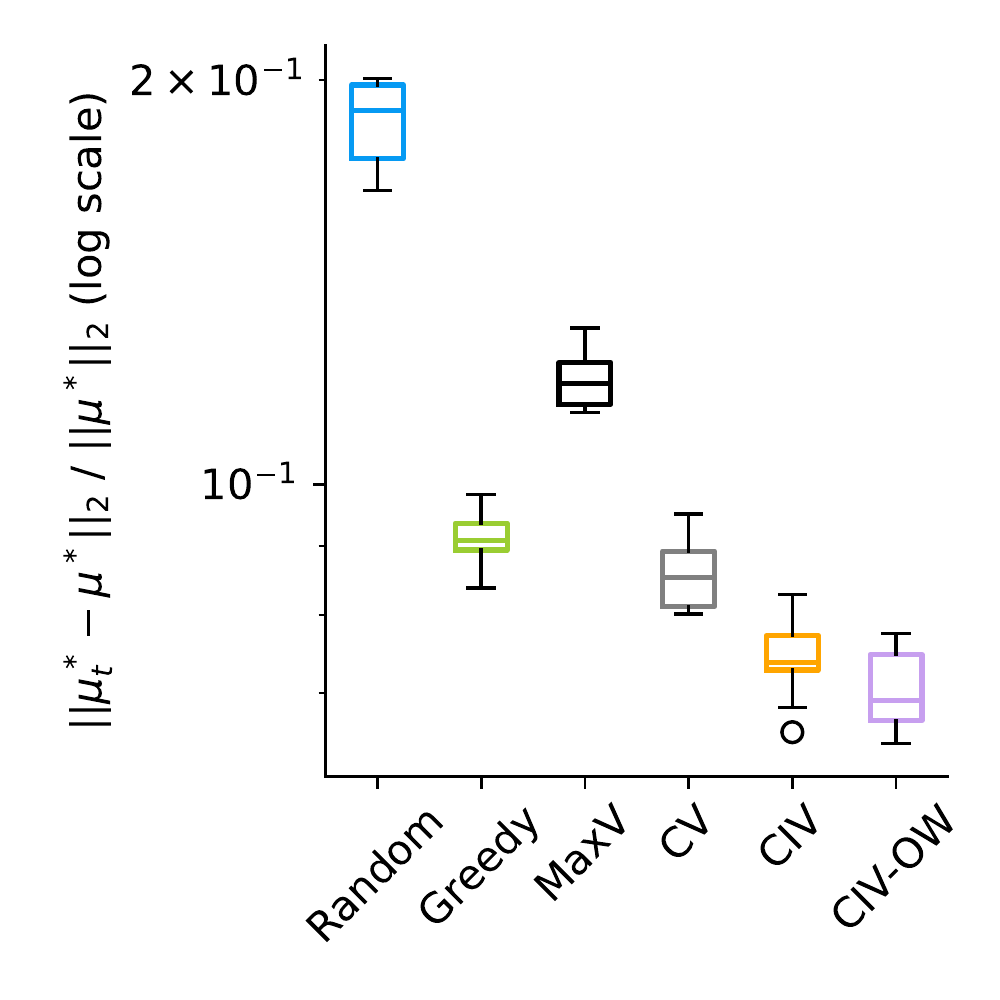}
         \caption{Erdös-Rényi graph ($0.5$)}
     \end{subfigure}
    \begin{subfigure}[b]{0.24\textwidth}
         \centering
         \includegraphics[width=.8\textwidth]{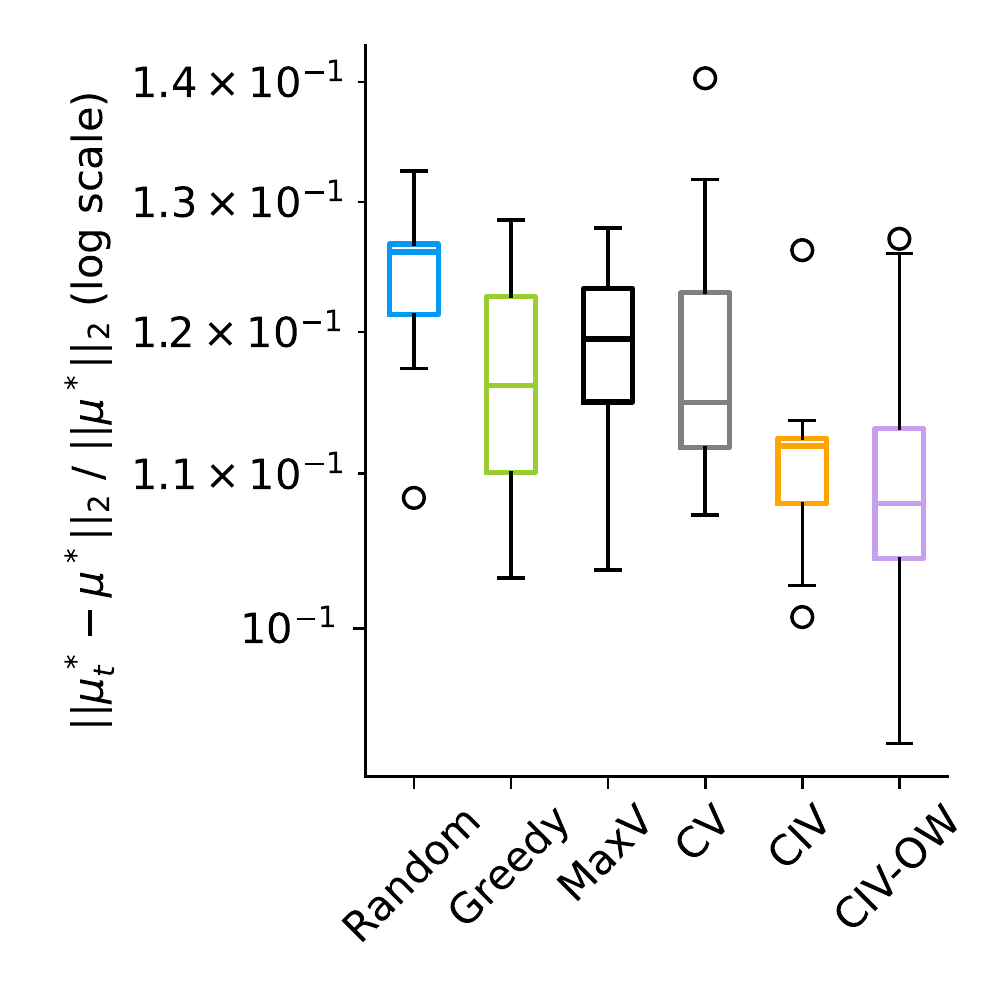}
         \caption{Path graph}
     \end{subfigure}
     \\
     \begin{subfigure}[b]{0.24\textwidth}
         \centering
         \includegraphics[width=\textwidth]{final-figs/appendix/mse-int-mean_complete-30-15.pdf}
         \caption{Complete graph}
     \end{subfigure}
     \begin{subfigure}[b]{0.24\textwidth}
         \centering
         \includegraphics[width=\textwidth]{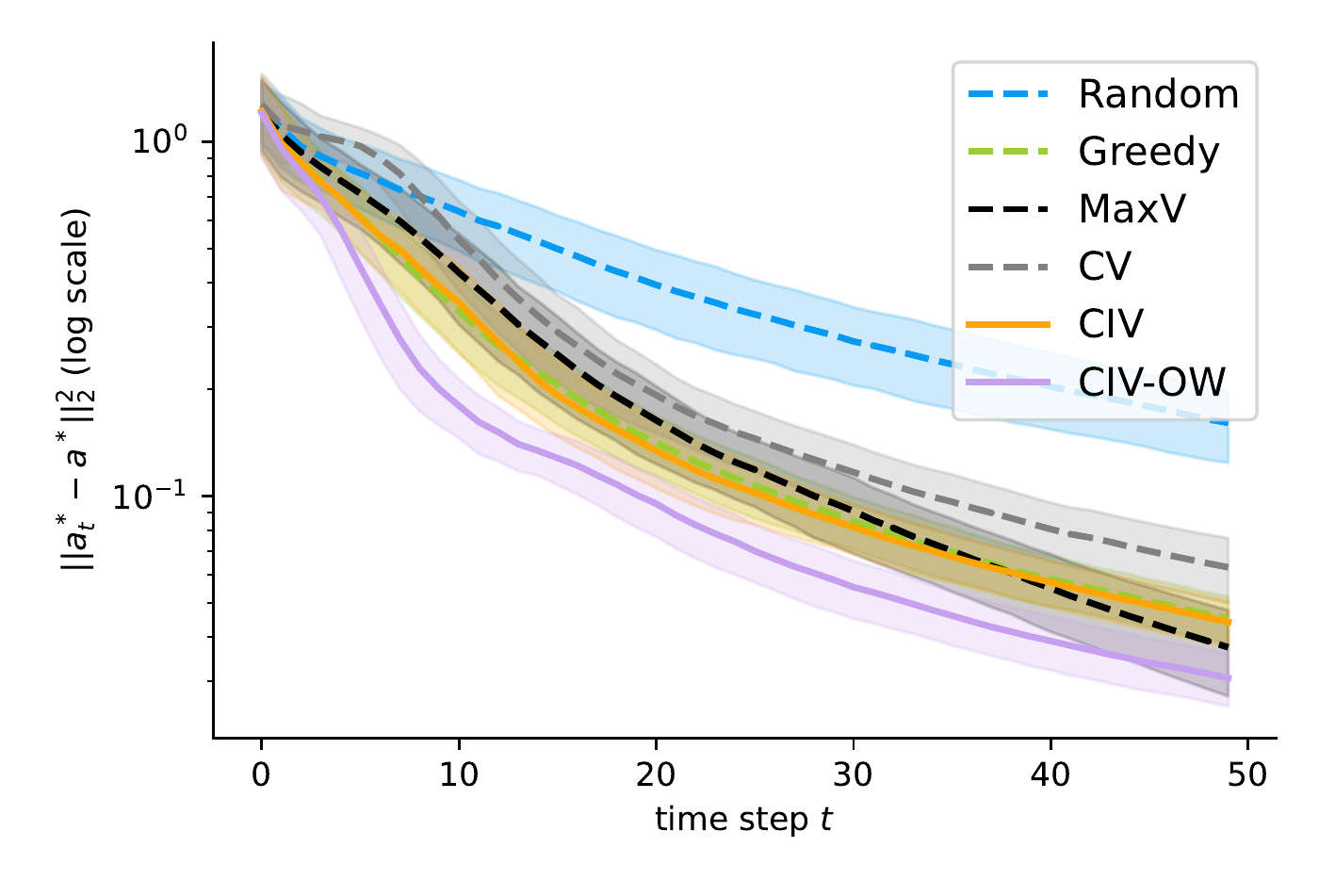}
         \caption{Erdös-Rényi graph ($0.8$)}
     \end{subfigure}
     \begin{subfigure}[b]{0.24\textwidth}
         \centering
         \includegraphics[width=\textwidth]{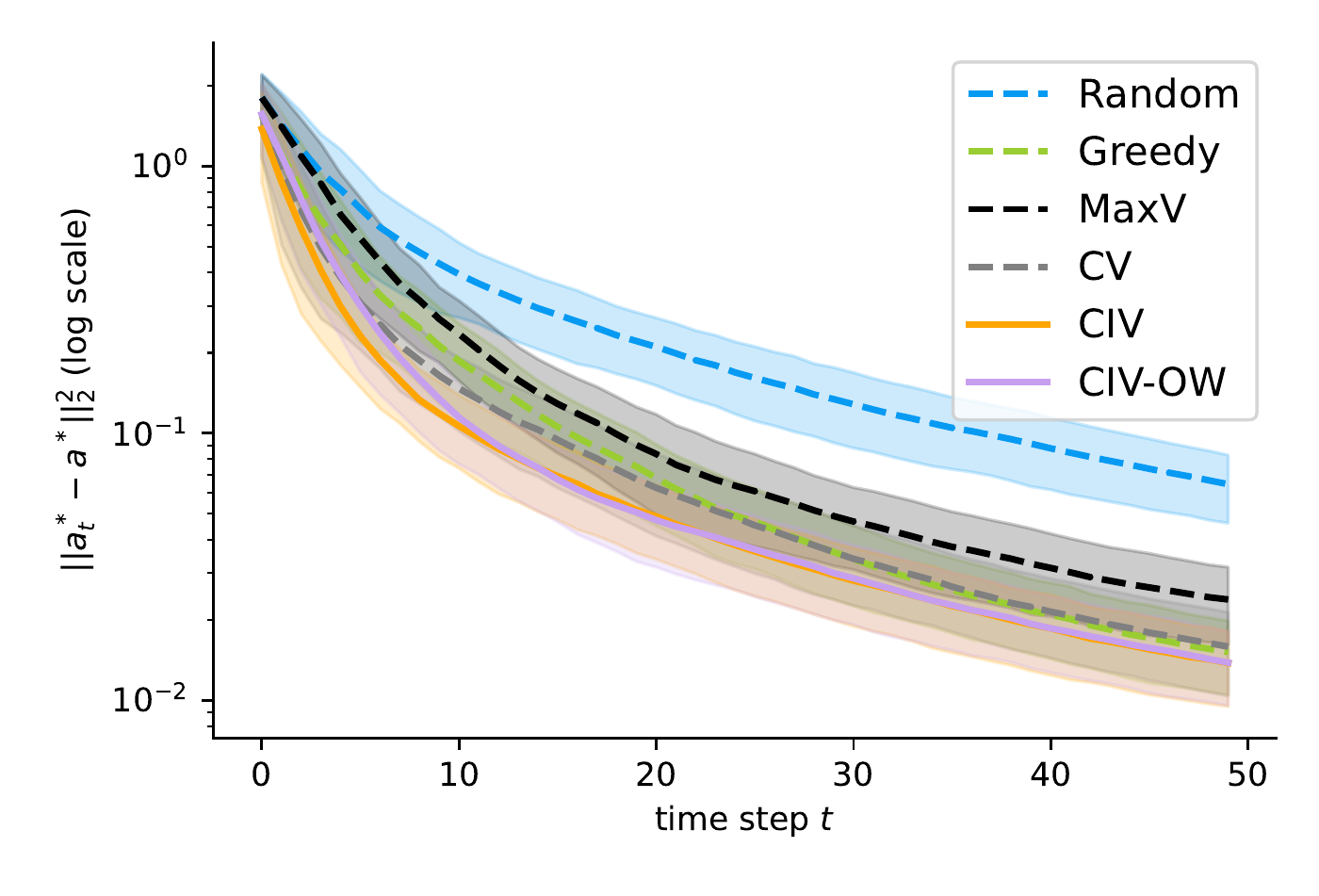}
         \caption{Erdös-Rényi graph ($0.5$)}
     \end{subfigure}
    \begin{subfigure}[b]{0.24\textwidth}
         \centering
         \includegraphics[width=\textwidth]{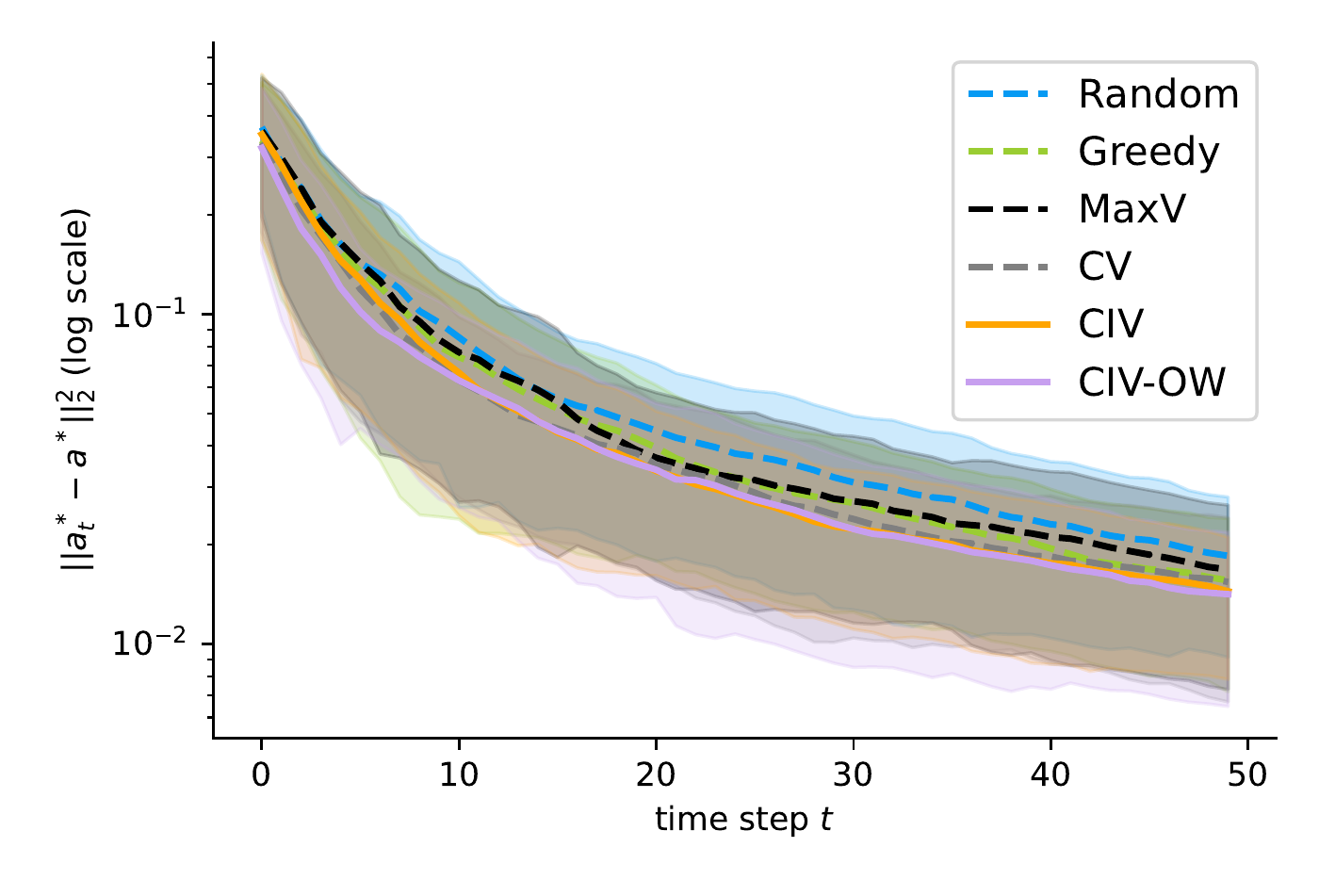}
         \caption{Path graph}
     \end{subfigure}
    \caption{\rec{\textbf{Comparison of different acquisition functions in a simulation study where we vary the underlying causal graph (complete graph, Erd\"os-R\'enyi graph with edge probability 0.8, Erd\"os-R\'enyi graph with edge probability 0.8, path graph) and the most downstream half of the nodes are fixed as intervention targets.}}  Each plot corresponds to an average of $10$ instances on a 30-node DAG with 15 perturbation targets. Each method is run 20 times and averaged. {Top row: Relative distance} between the target mean $\bmu^*$ and the best approximation $\bmu^*_t$ (defined in Fig.~\ref{fig:5}A in the main text) up to time step $t$. Lines denote the mean over 10 instances; the shading corresponds to one standard deviation. {Middle row: Relative distance} statistic of each method \rec{averaged over $10$ instances} at the last time step ($t=50$). Note that the DAGs become sparser from left to right.  {Bottom row: Squared distance \rec{presented as mean value +/- SEM} between the optimal intervention $\ba^*$ and the best approximation $\ba_t^*$ that is used to obtain $\bmu^*_t$ up to time step $t$.}}
    \label{fig:s7}
\end{figure}

\subsection{Varying Graph Types}

In this set of experiments, we consider the effect of graph types. 
While the complete graph is the densest among all DAGs of a given size, the path graph is among the sparsest connected DAGs. 
To examine graphs with density between these extremes, we also consider Erdös-Rényi graphs \cite{erdos1960evolution}, one with edge probability $0.8$ and one with edge probability $0.5$. 
Extended Data Fig.~\ref{fig:s7} shows that the proposed methods, CIV and CIV-OW, outperform the other baselines, with the difference being larger for denser graphs.

{
\subsection{Additional Baselines}
\label{add_baselines}

We perform an experimental comparison to the following three additional acquisition functions from prior works \cite{aglietti2020causal,astudillo2021bayesian,houlsby2011bayesian,bubeck2012regret} (discussed in more detail in Supplementary Information~\ref{sec:a0}), which we adapted for the task of identifying optimal interventions:
\begin{itemize}
    \item[(1)] \textbf{Expected Improvement.} In \cite{aglietti2020causal, astudillo2021bayesian}, the expected improvement (EI) in the objective function has been used as an acquisition function. To compute EI, \cite{aglietti2020causal} used a GP between input and output and \cite{astudillo2021bayesian} used a series of GPs to model a network of functions. We adapt EI to use the squared distance $\|(\bI-\bB)^{-1}\ba-\bmu^*\|_2^2$ between the interventional mean and target mean (Eq.~\eqref{eq:3} with no noise) as the objective function together with the Bayesian model for unknown $\bB$ used also by the other methods (e.g., CIV). Minimizing EI at time step $t$ can be written as \[\min_{\ba'\in\cA}\bbE\big(\min{\{\|(\bI-\bB)^{-1}\ba'-\bmu^*\|_2^2-F^*_t,0\}}|\cD_t\big),\] where $F_t^*$ denotes the minimum squared distance of the interventions selected up to time point $t$. Since this acquisition function cannot be evaluated in closed form, we use Monte Carlo samples of $\bB$ from $\bbP(\bB|\cD_t)$ to approximate it. To optimize it over $\ba'\in \cA$, we sample multiple instances of $\ba'$ by perturbing the current best guess $\ba_{t}^*$ and use the one with the smallest estimated EI value. {\rec{For the current best guess of $\ba^*$ we use its unbiased estimate $\ba_{t}^*=(\bI-\bbE(B|\cD_t))\bmu^*$.}} This baseline is denoted EI-Int in the experiments.
    \item[(2)] \textbf{Mutual Information.} In concurrent work \cite{toth2022active} as well as in \cite{houlsby2011bayesian}, the mutual information (MI) between a quantity of interest and the samples one wishes to acquire is used as the acquisition function. As discussed in Section~\ref{sec:42}, the MI acquisition function equals the entropy decay. Using the optimal intervention $\ba^*=(\bI-\bB)\bmu^*$ as the quantity of interest, then minimizing MI at time step $t$ can be written as
    \[
    \max_{\ba'\in\cA} I\big((\bI-\bB)\bmu^*;\bx'|\ba',\cD_t\big) \Longleftrightarrow \min_{\ba'\in\cA}H\big((\bI-\bB)\bmu^*|\cD_t\cup(\bx',\ba')\big),
    \]
    where $\bx'$ are interventional samples from $\ba'$. Note that although for a fixed $\bx'$, the entropy of $(\bI-\bB)\bmu^*$ can be evaluated in closed form {\rec{($\frac12\ln|\Var(\bB\bmu^*|\cD_t\cup(\bx',\ba'))|+\frac{p}{2}(1+\ln{2\pi})$, where $|\cdot|$ denotes the determinant)}} since it follows a Gaussian distribution, the conditional entropy $H((\bI-\bB)\bmu^*|\cD_t\cup(\bx',\ba'))$ cannot be computed in closed form as it requires marginalizing over $\bbP(\bx'|\ba',\cD_t)$. Therefore we use Monte Carlo samples of $\bx'$ from $\bbP(\bx'|\ba',\cD_t)$ to approximate it. To optimize MI over $\ba'$, we adopt the same strategy as described for EI above. We denote this baseline as MI-Int in the experiments.
    \item[(3)] \textbf{Upper Confidence Bound.} In concurrent work \cite{sussex2023modelbased}, the upper-confidence bound (UCB) from the bandit literature \cite{bubeck2012regret} is used as an acquisition function, assuming the noise variances are known. Adapting this to our setting, we consider the upper confidence bound of the optimality gap $\|(\bI-\bB)\bmu^*-\ba\|_2^2$ (Eq.~\eqref{eq:6} with no noise). This allows for closed-form evaluation and can thus be optimized easily. Acquiring new points according to UCB can then be written as 
    \[
    \max_{a'\in \cA} - \bbE\big(\|(\bI-\bB)\bmu^*-\ba'\|_2^2|\cD_t\big) +\beta \sqrt{\Var\big(\|(\bI-\bB)\bmu^*-\ba'\|_2^2|\cD_t\big)},
    \]
    where $\beta>0$ is a hyperparameter that controls the exploration-exploitation trade-off. Note that the first term equals $-\|(\bI-\bbE(\bB|\cD_t))\bmu^*-\ba'\|_2^2$ plus the trace of $\Var((\bI-\bB)\bmu^*-\ba'|\cD_t)$, which is a constant that does not depend on $\ba'$. The second term can be computed using similar arguments as in Corollary~1 and equals $\beta\sqrt{4\sum_{i=1}^p(\sigma_i^2\bmu_{\pa(i)}^{*\top} \bM_i(\cD_t)\bmu_{\pa(i)}^*)(a'_i-b_i)^2+\beta_1}$ with $\beta_1\geq0$ being a constant that does not depend on $\ba'$. We denote this baseline as UCB-Int in the experiments. 
\end{itemize}

Extended Data Fig.~\ref{fig:rebut-s7} compares the methods described in the previous sections with these three baselines. Since evaluating EI-Int and MI-Int requires Monte Carlo estimation, we perform the benchmark on a smaller DAG with $10$ nodes. The set-up is the same as in Extended Data Fig.~\ref{fig:s5} (A) and (D). We use $200$ Monte Carlo samples for EI-Int and MI-Int for each of the $20$ samples of $\ba'$ around $\ba_t^*$. For UCB-Int, we set $\beta=\frac12$ and $\beta_1=0$. In addition to the relative error across time steps (Extended Data Fig.~\ref{fig:rebut-s7}a) and at the last round (Extended Data Fig.~\ref{fig:rebut-s7}b), we also report the per-iteration runtime (Extended Data Fig.~\ref{fig:rebut-s7}c). {\rec{Note that a total of $10\times 20\times 50$ iterations are needed to generate Extended Data Fig.~8a, since we average across $10$ instances and repeat each algorithm for $20$ runs using a time step of $50$.}} We observe that the three additional active baselines are superior to the passive baseline (i.e., Random) and that our proposed methods (i.e., CIV and CIV-OW) outperform the three active baselines. Computation time of EI-Int and MI-Int is much higher than for the other methods, due to the required Monte Carlo estimation both when evaluating and optimizing the acquisition function.

\begin{figure}[!h]
    \centering
    \begin{subfigure}[b]{0.35\textwidth}
         \centering
         \includegraphics[width=\textwidth]{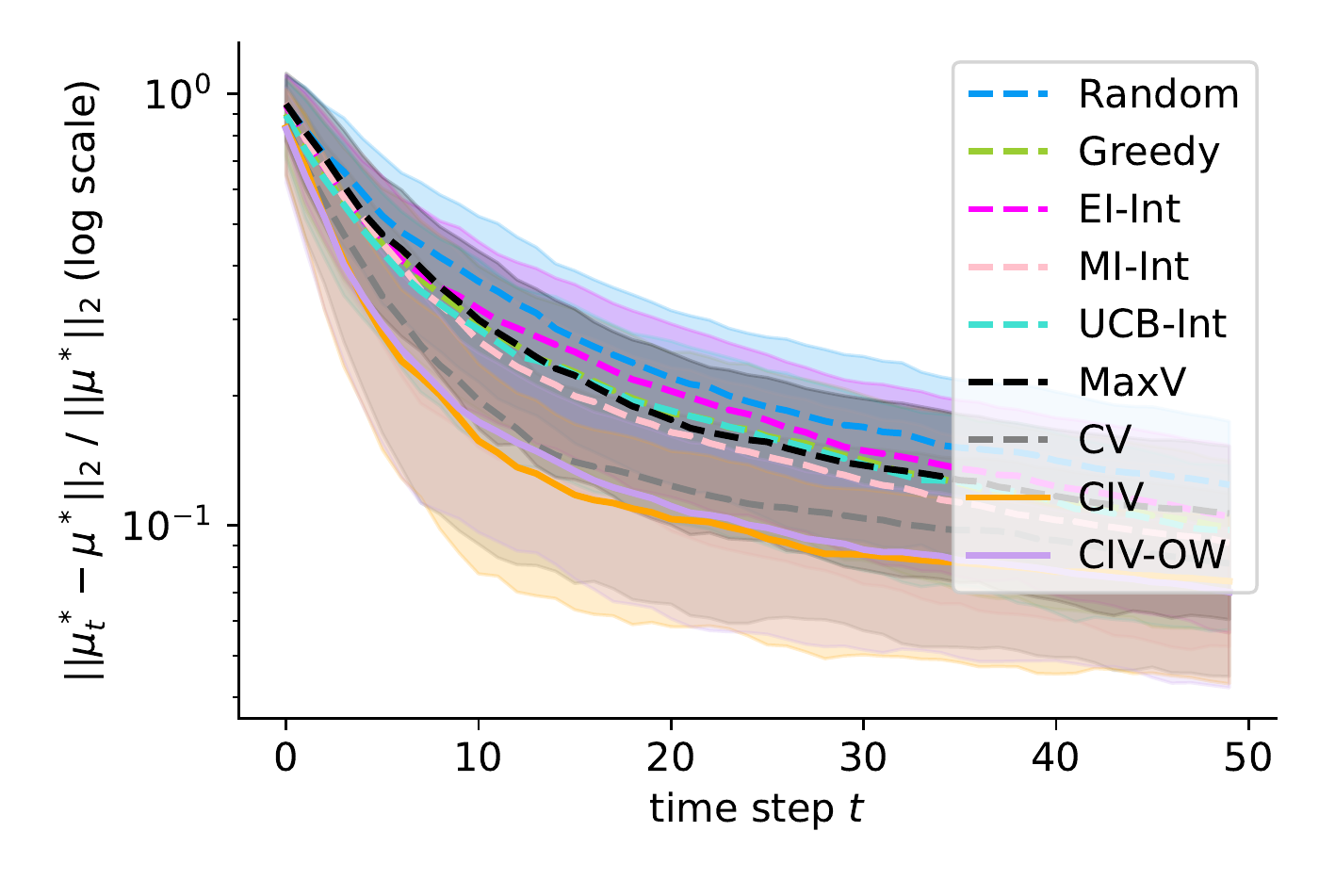}
         \caption{{Relative distance across time steps}}
    \end{subfigure}
    \begin{subfigure}[b]{0.25\textwidth}
         \centering
         \includegraphics[width=\textwidth]{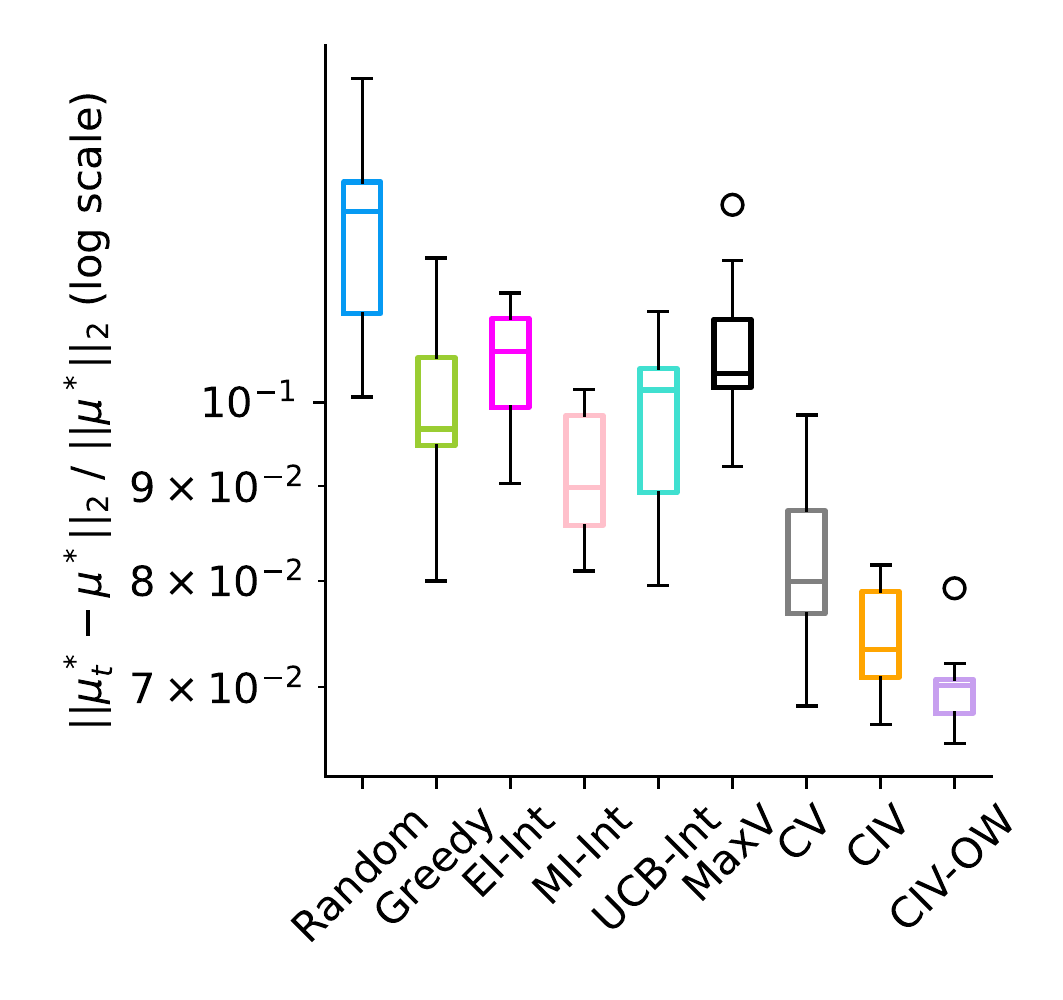}
         \caption{{Last time step}}
    \end{subfigure}
    \begin{subfigure}[b]{0.25\textwidth}
         \centering
         \includegraphics[width=\textwidth]{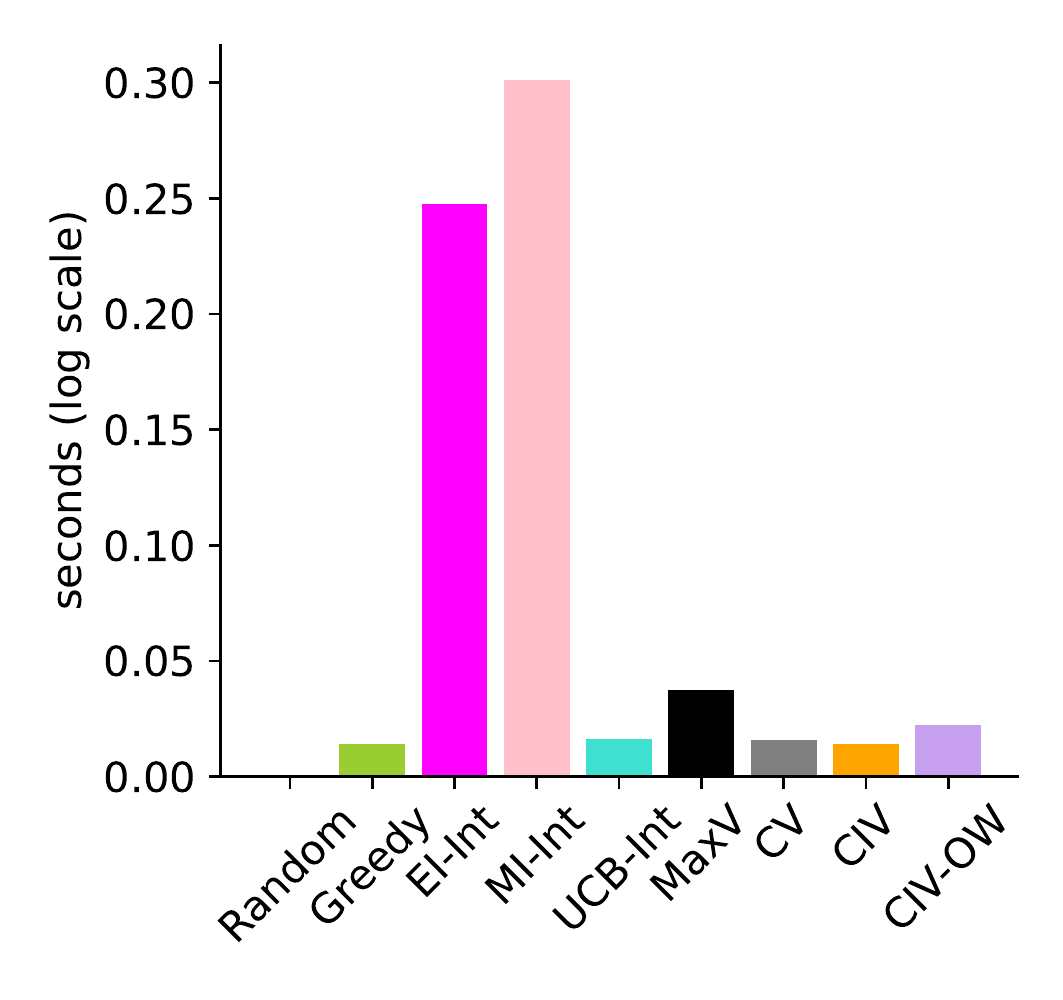}
         \caption{{Runtime per iteration}}
    \end{subfigure}
    \caption{\rec{\textbf{Comparison of our acquisition functions to baseline acquisition functions adapted from prior works (EI-Int: based on Expected Improvement, MI-Int: based on Mutual Information, and UCB-Int: based on Upper Confidence Bound) in a simulation study where the underlying causal graph is the complete graph on 10 nodes and the most downstream 5 nodes are fixed as intervention targets.}} Each plot corresponds to an average of 10 instances and each method is run 20 times and averaged. (A) Relative distance between the target mean $\bmu^*$ and the best approximation $\bmu^*_t$ (defined in Fig. \ref{fig:5}A in the main text) up to time step $t$. Lines denote the mean over 10 instances; the shading corresponds to one standard deviation. (B) Relative distance statistic of each method \rec{averaged over $10$ instances} at the last time step ($t=50$). (C) Runtime per iteration of each method in seconds.}
    \label{fig:rebut-s7}
\end{figure}

\subsection{Misspecifying the Underlying Causal Graph}
In the following, we investigate the performance of different methods when the underlying causal structure is misspecified. We consider three different types of misspecification: missing edges, excessive edges, and reversed edges; see Supplementary Fig.~\ref{fig:rebut-figs9} for an example. Note that any graph misspecification can be represented as a combination of these three cases.

For simplicity, we consider a $5$-node random Erdös-Rényi DAG with edge density $0.5$. The ground-truth optimal intervention is set to target $3$ nodes. As in the previous experiments, we generate $10$ instances with varying edge weights and optimal intervention vector and repeat each method for $10$ runs using a time step of $10$. Extended Data Fig.~\ref{fig:rebut-figs10} shows the relative distance at the last time step (averaged across runs and instances) for different methods. The structural Hamming distance (SHD) records the level of misspecification. We observe that while the ordering of the different methods is mostly preserved, performance of all methods decreases with increasing number of missing and reversed edges. In the presence of excessive edges, the degradation in performance is relatively mild. The excessive edge setting corresponds to additional non-zero entries in the prior of $\bB$. The true values of these entries can be recovered as we learn $\bB$ with more samples. Thus our consistency results still apply in this setting. This is however not the case for missing or reversed edges, where the misspecification of the causal structure introduces irreversible errors. This could explain the different effects of graph misspecification observed in our synthetic data experiments. 
\begin{figure}[!ht]
    \centering
    \begin{subfigure}[b]{0.19\textwidth}
         \centering
         \includegraphics[width=.8\textwidth]{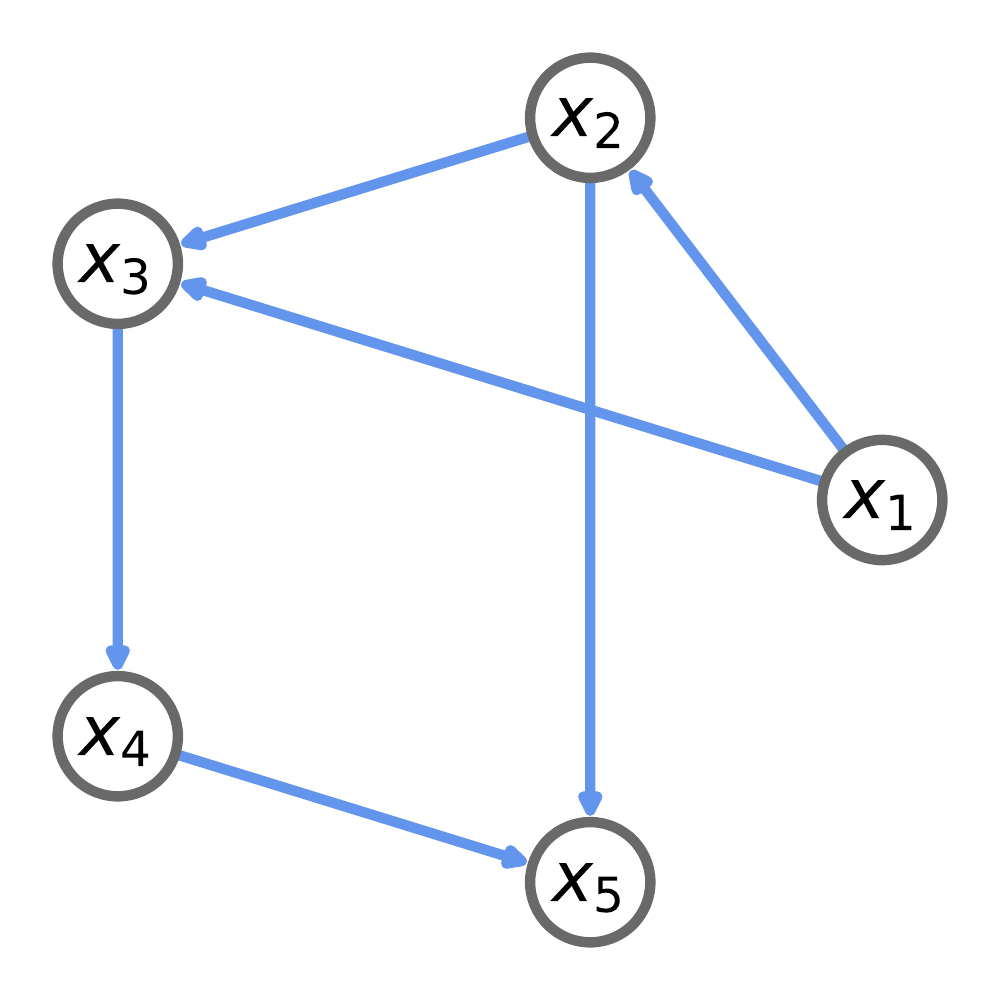}
         \caption{Correct DAG}
    \end{subfigure}
    \begin{subfigure}[b]{0.19\textwidth}
         \centering
         \includegraphics[width=.8\textwidth]{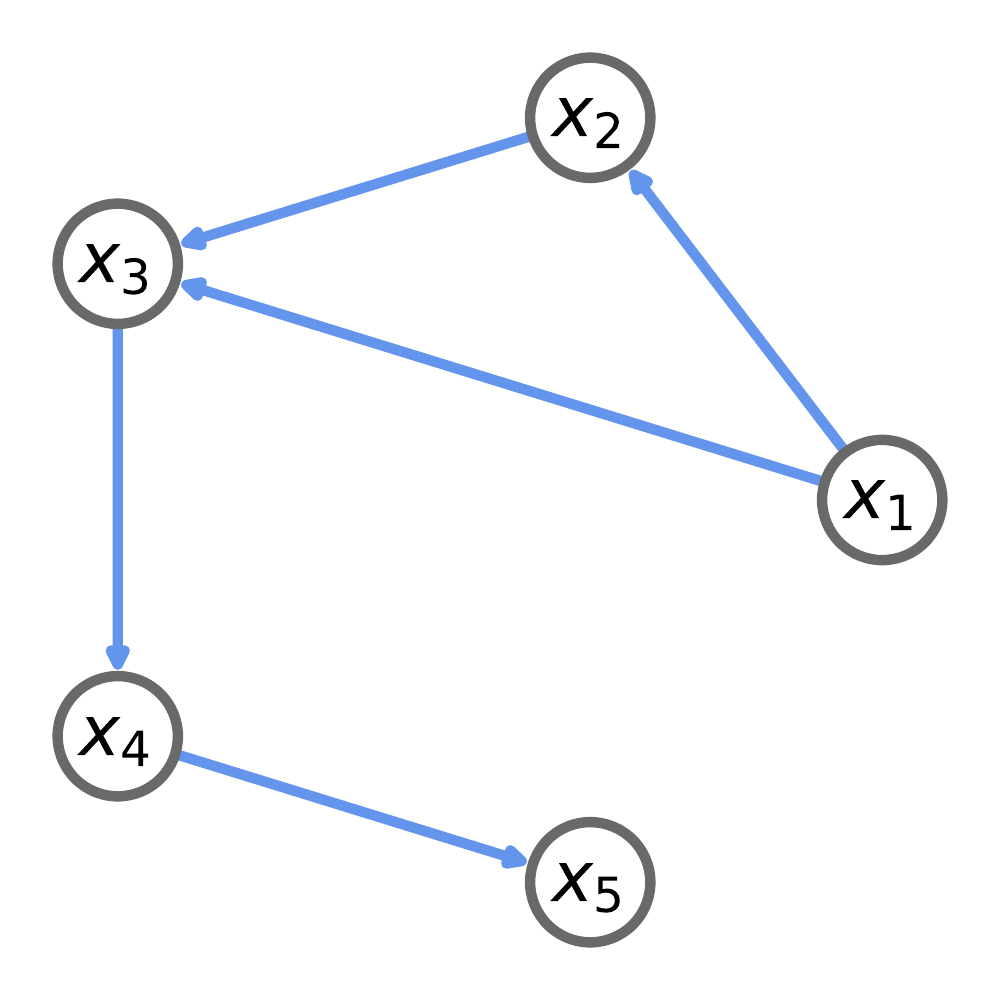}
         \caption{Missing edge}
    \end{subfigure}
    \begin{subfigure}[b]{0.19\textwidth}
         \centering
         \includegraphics[width=.8\textwidth]{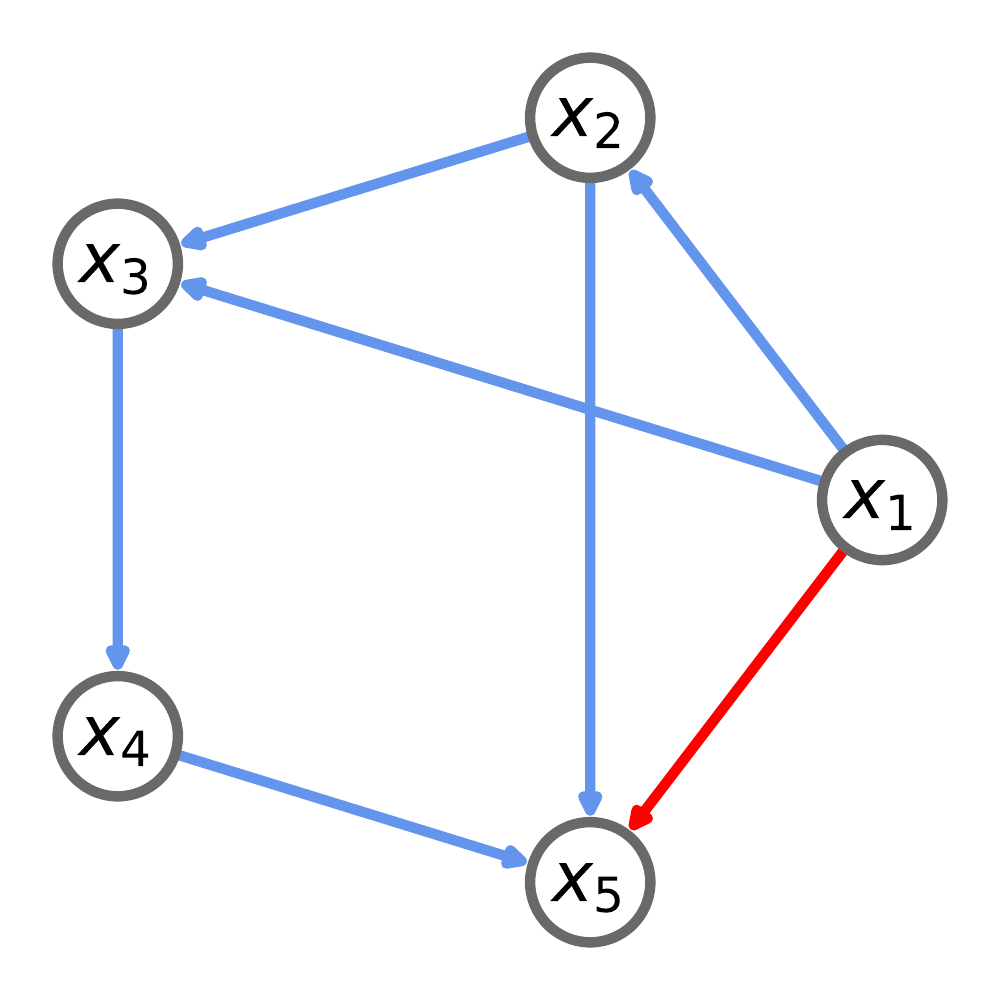}
         \caption{Excessive edge}
    \end{subfigure}
    \begin{subfigure}[b]{0.19\textwidth}
         \centering
         \includegraphics[width=.8\textwidth]{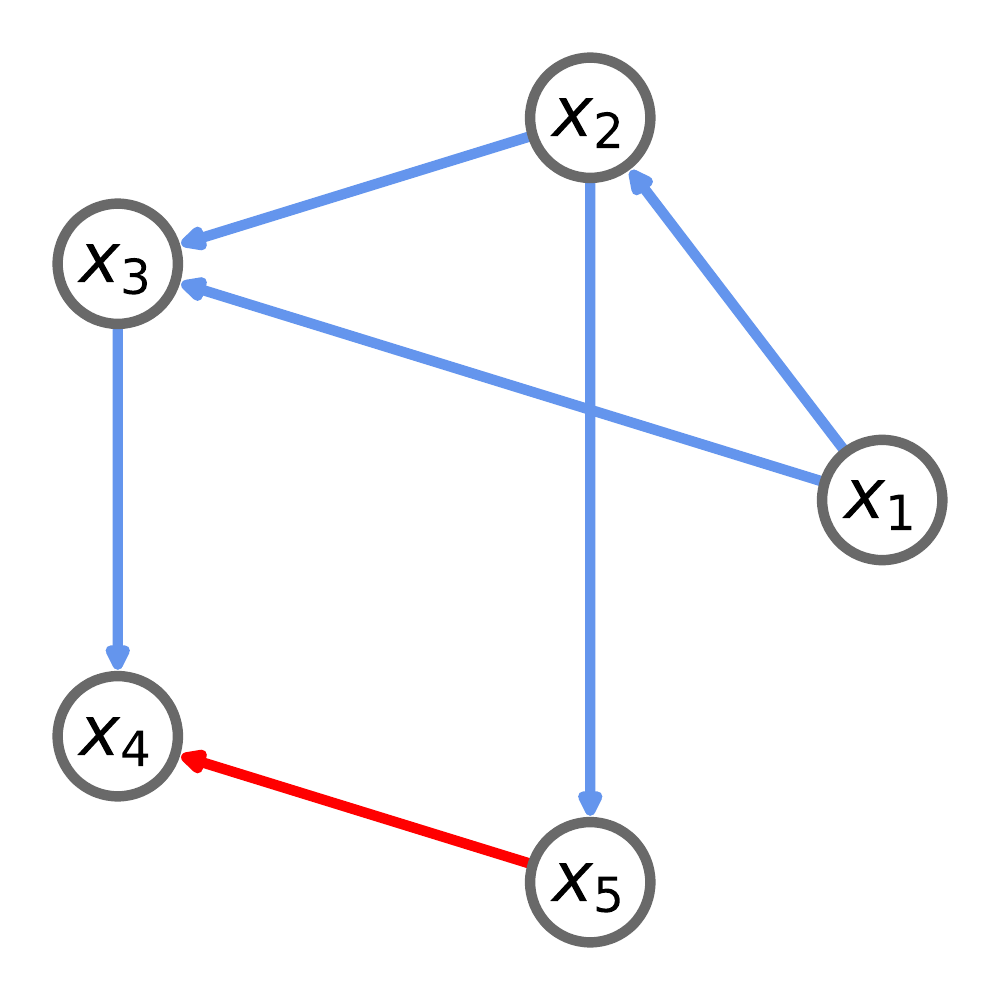}
         \caption{Reversed edge}
    \end{subfigure}
    \caption{{\textbf{Three types of graph misspecification.} The correct DAG is shown in (A). Edge $2\rightarrow 5$ is missing in (B). Excessive edge $1\rightarrow 5$ is added in (C). Edge $4\rightarrow 5$ is reversed in (D).}}
    \label{fig:rebut-figs9}
\end{figure}

\begin{figure}[!ht]
    \centering
    \begin{subfigure}[b]{0.31\textwidth}
         \centering
         \includegraphics[width=\textwidth]{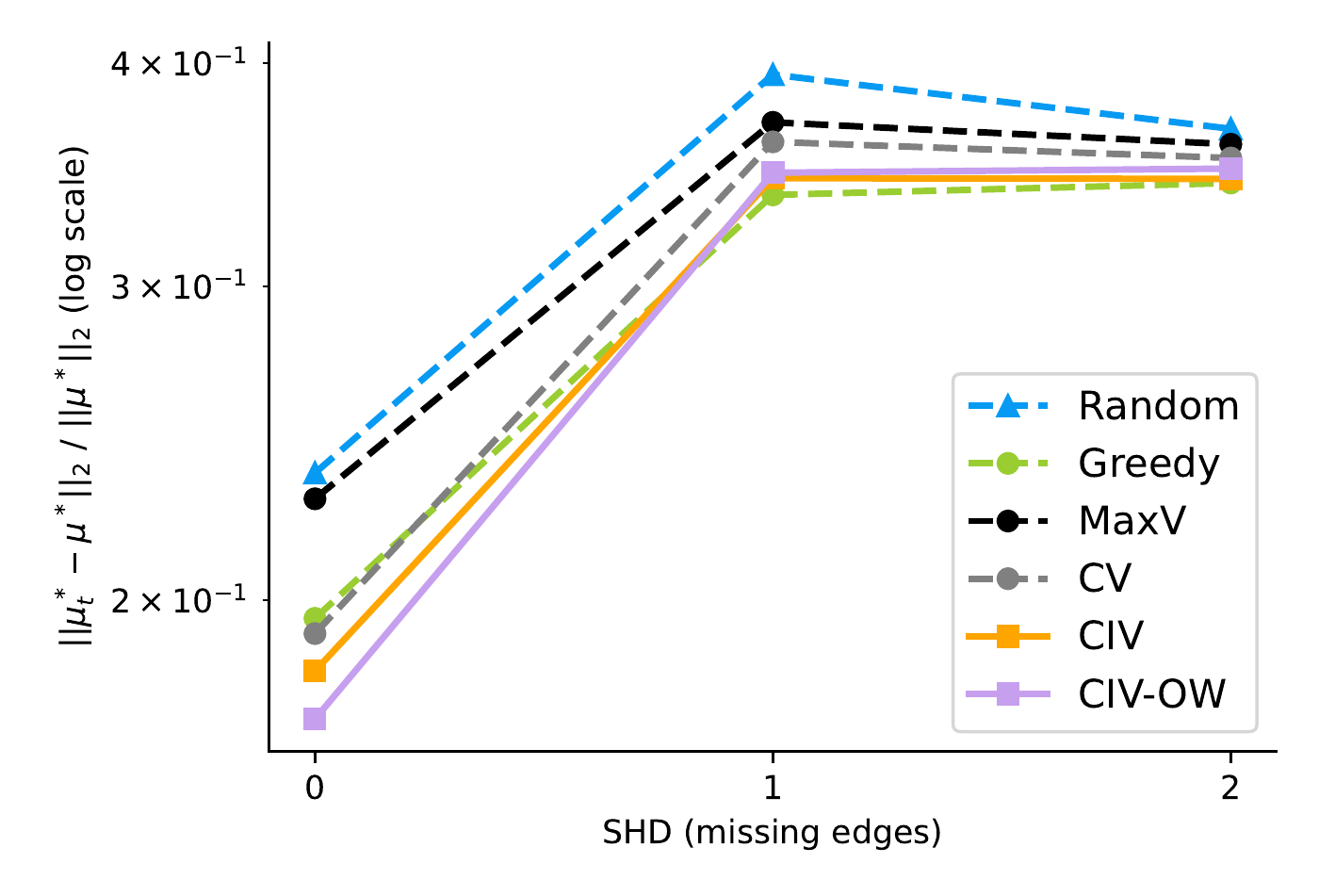}
         \caption{Missing edge}
    \end{subfigure}
    \begin{subfigure}[b]{0.31\textwidth}
         \centering
         \includegraphics[width=\textwidth]{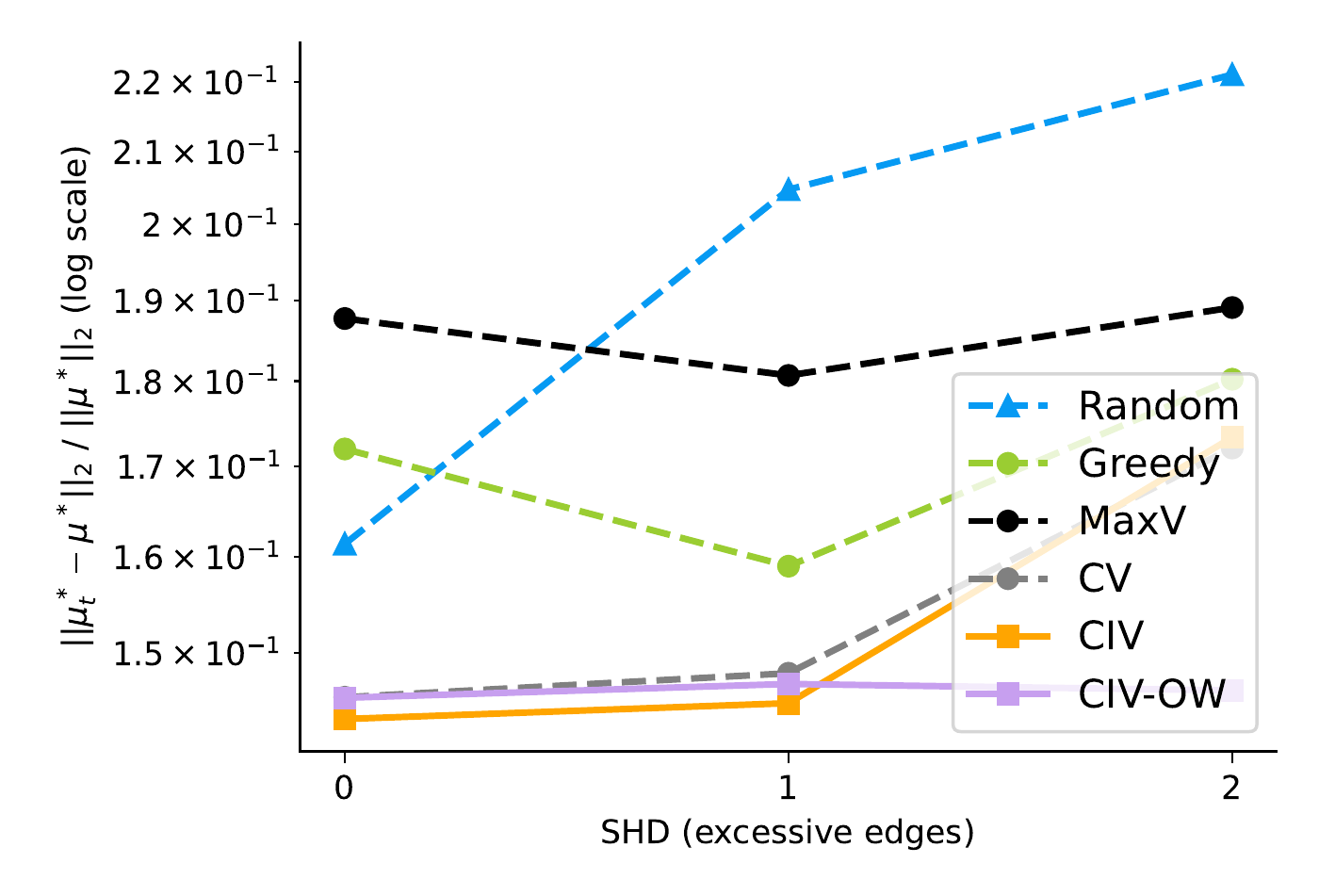}
         \caption{Excessive edge}
    \end{subfigure}
    \begin{subfigure}[b]{0.31\textwidth}
         \centering
         \includegraphics[width=\textwidth]{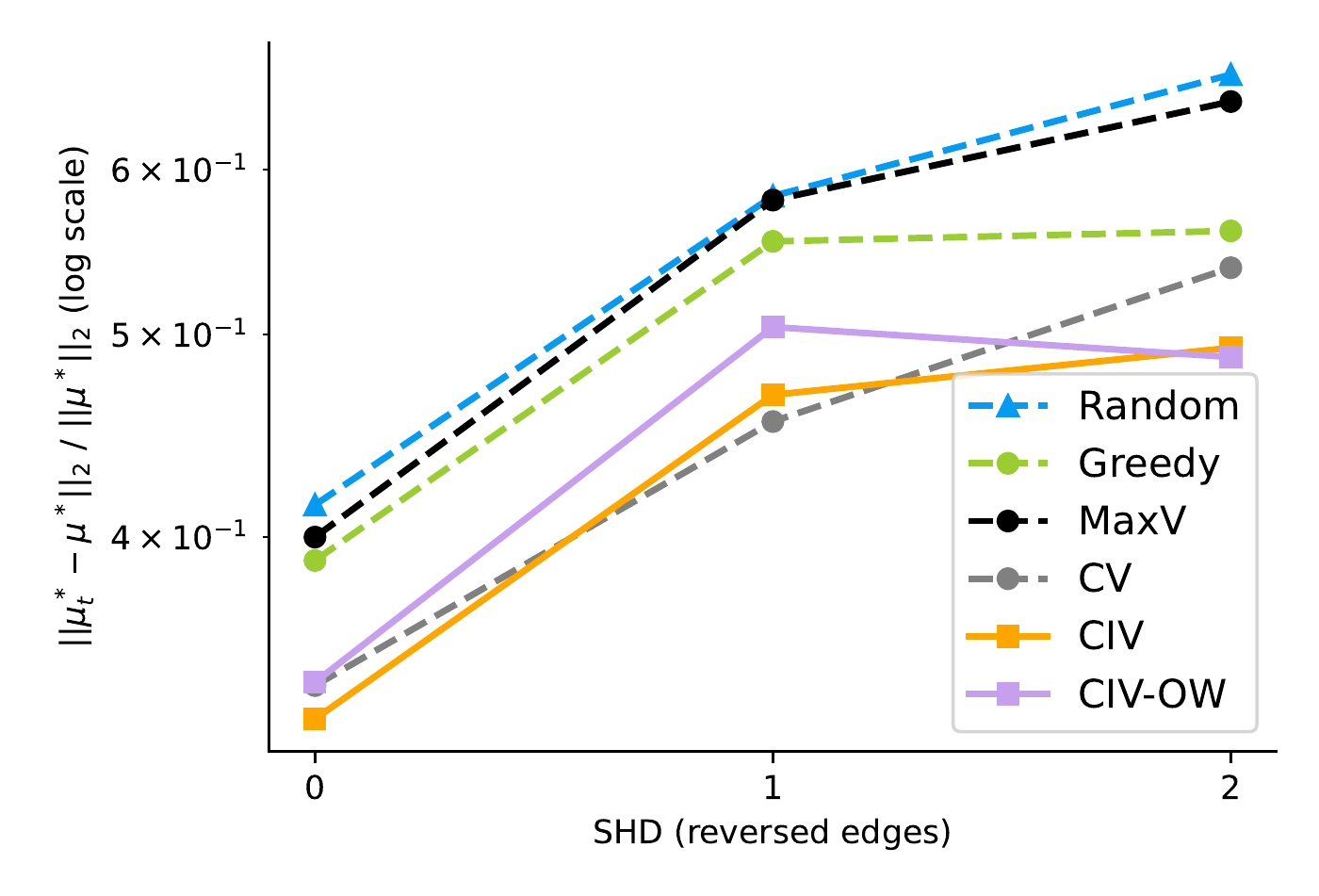}
         \caption{Reversed edge}
    \end{subfigure}
    \caption{
    {\textbf{Performance of the different acquisition functions under three types of DAG misspecifications where the underlying causal graph is a 5-node random Erd\"os-R\'enyi DAG with edge density 0.5 and 3 intervention targets.} Each plot corresponds to an average of the relative distance at time step $10$ across $10$ instances. Each method is run 10 times and averaged. SHD denotes the number of misspecified edges.}
    }
    \label{fig:rebut-figs10}
\end{figure}
}

\section{Extended Experiments on Biological Dataset}\label{sec:h}

In the following set of experiments, we present our analysis of the biological dataset used in the main text and present additional experimental results. 
As in the previous section, the metrics used for the additional experiments are the same as those in the main text.

We retrieved Perturb-CITE-seq data from \cite{frangieh2021multimodal}, consisting of pooled CRISPR screens on patient-derived melanoma cells. 
In a pooled CRISPR screen, a large amount of genetically-encoded perturbations are introduced into a pool of cells.
In the performed experiments, each perturbation targets a specific gene in order to silence its expression \cite{bock2022high}. In these experiments, perturbations were performed on $248$ genes associated with immunotherapy resistance.
Cells may receive no perturbation or a combination of perturbations. The Perturb-CITE-seq technology allows one to identify the perturbations received by each cell and read out the effects of the perturbations on the expression of each gene via single-cell RNA sequencing.

Depending on the culturing conditions, there are three different screens each with its subset of single cells (Supplementary Fig.~\ref{fig:s8}). 
To avoid batch effects and treatment effects from additional substances, we use the control screen that measures melanoma cells maintained in culture medium alone. 
This dataset contains the log-TPM-transformed expression of $23,712$ genes in $57,627$ single cells. 
Among these, $5,039$ cells have no perturbation, which we refer to as \textit{control cells} (or ctrl cells in abbreviation); and $30,486$ cells have interventions over a single target gene (we call these \textit{perturbed cells} or ptb cells). 
The remaining cells received interventions with more than one target gene. We ignore these cells, since each of these multi-target interventions appears in a very small number of cells, with most such interventions having no more than $1$ sample; see Supplementary Fig.~\ref{fig:s9}. 
Since single-cell measurements are highly noisy, estimation and inference based on such a small sample size can be very inaccurate.
Therefore we do not consider multi-target interventions in our experiments.
These data issues can potentially be alleviated using recent technologies \cite{wessels2022efficient} that perform perturbations of combinations more efficiently.

\begin{figure}[!t]
     \centering
     \includegraphics[width=.35\textwidth]{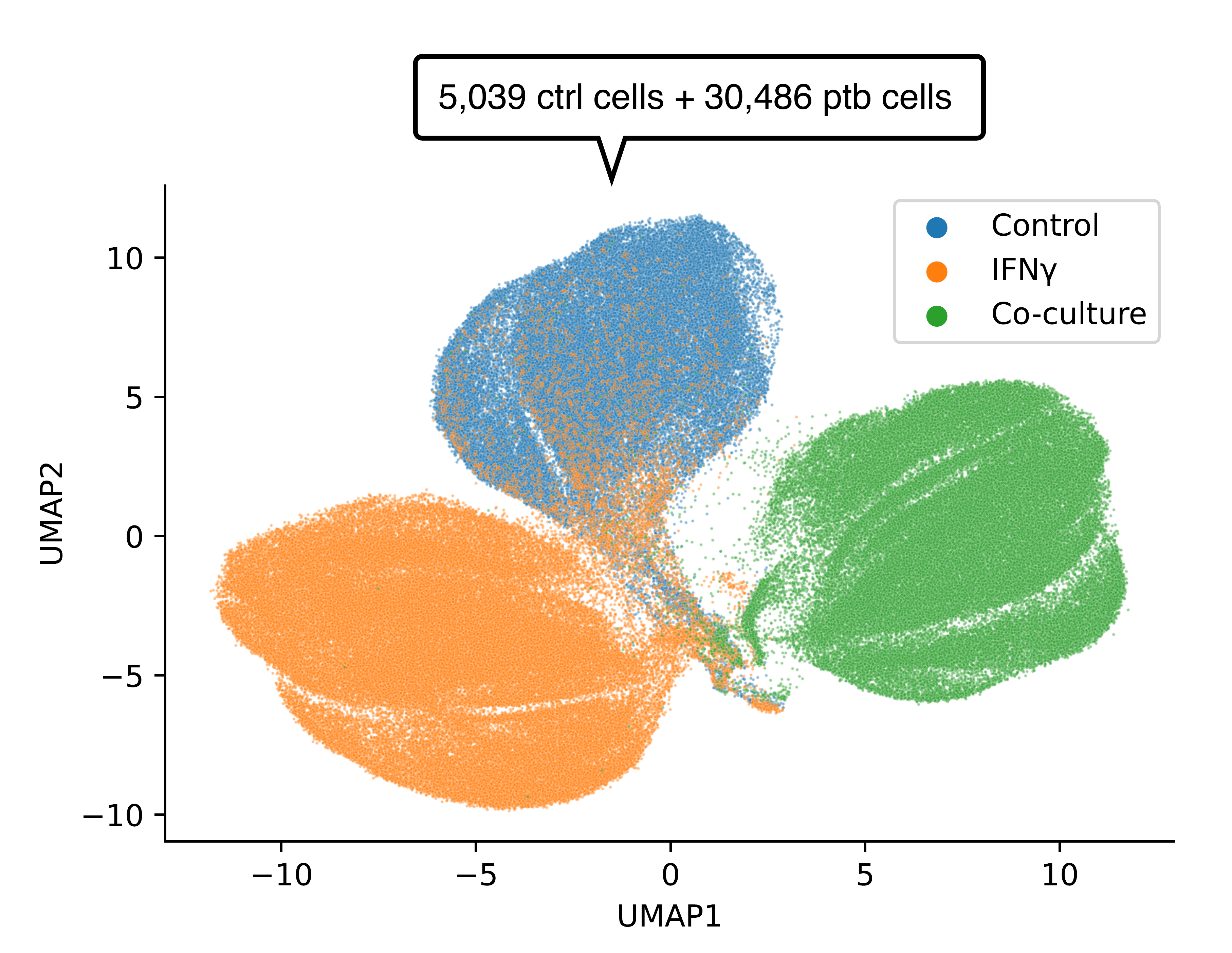}
    \caption{\textbf{UMAP of single-cell gene expression data from~\cite{frangieh2021multimodal} across three culturing conditions.} Control: cells maintained in culture medium alone; IFN$\gamma$: cells treated with IFN-$\gamma$ after 14 days; Co-culture: cells co-cultured with TIL doses after 14 days. We consider only the Control screen in our experiments to avoid batch effects as well as additional treatment effects.}
    \label{fig:s8}
\end{figure}

\begin{figure}[!t]
     \centering
    \begin{subfigure}[b]{0.74\textwidth}
         \centering
         \includegraphics[width=.9\textwidth]{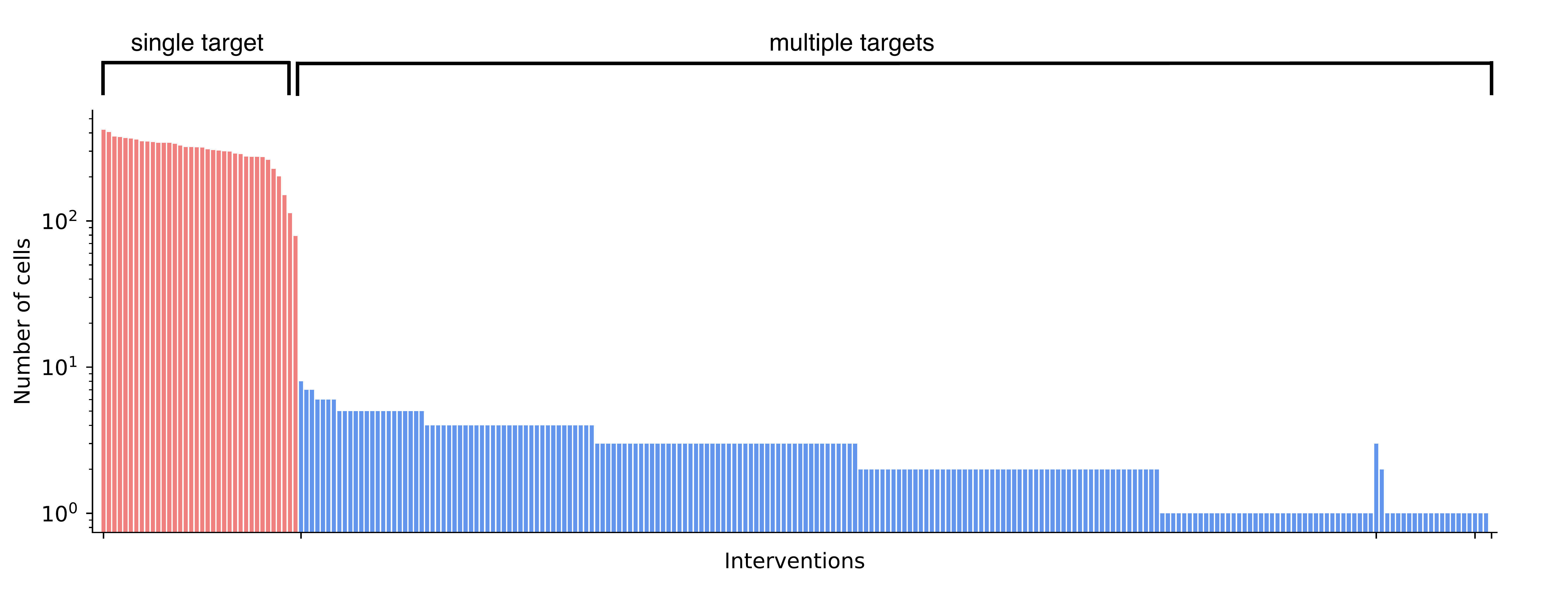}
         \caption{Number of cells for each intervention. Interventions are sorted first by its\\number of targets and then by its number of cells.}
     \end{subfigure}
    \begin{subfigure}[b]{0.24\textwidth}
         \centering
         \includegraphics[width=.9\textwidth]{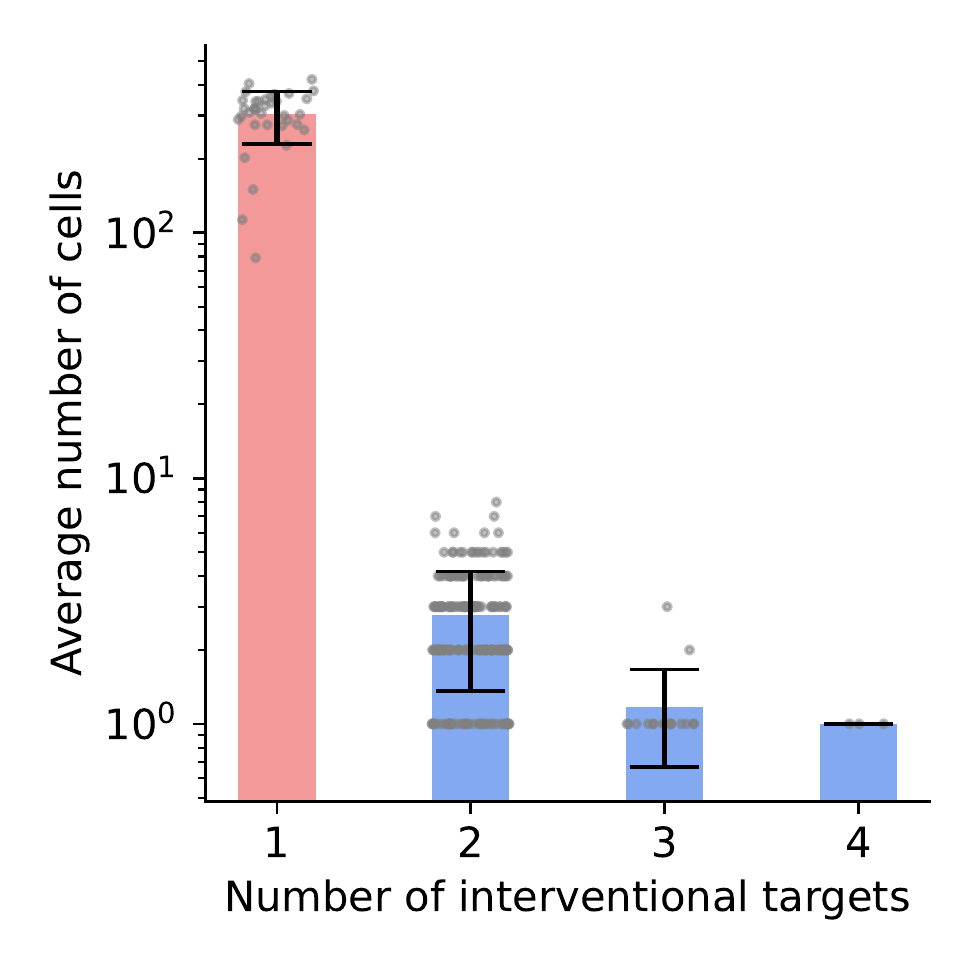}
         \caption{Average number of cells for interventions with the same number of target genes.}
     \end{subfigure}
    \caption{\textbf{Number of samples for the interventions on the $36$ target genes considered in our analysis.} (A) finer-level view of number of cells per intervention. (B) coarser-level view of number of cells per intervention by grouping together and averaging over all interventions with the same number of targets \rec{(36, 196, 18, 3 interventions with 1, 2, 3, 4 targets, respectively)}.}
    \label{fig:s9}
\end{figure}

To reduce the computational burden of considering the entire gene space, we focus on 36 important genes selected by \cite{frangieh2021multimodal} in their Figure~4d that are known to affect immunotherapy resistance in melanoma cell lines, including members of the IFN-$\gamma$ response pathway, surface checkpoints, as well as other genes involved in antigen presentation, inflammation, and cell differentiation. 
These genes were grouped by \cite{frangieh2021multimodal} into several co-functional modules and co-regulatory programs. 
The authors established partial functional relationships between these gene modules and programs (see Supplementary Fig.~\ref{fig:s10}).
These directed relationships can either be used as a prior for learning the underlying DAG or for validating the estimated DAG. Since in many applications, prior knowledge on the underlying causal structure is not available and our goal is to demonstrate the performance of our method in a typical instance, we here estimate the DAG from the data and only use the functional relationships to analyze the performance of the DAG learning algorithm.

\begin{figure}[!t]
    \centering
    \includegraphics[width=.3\textwidth]{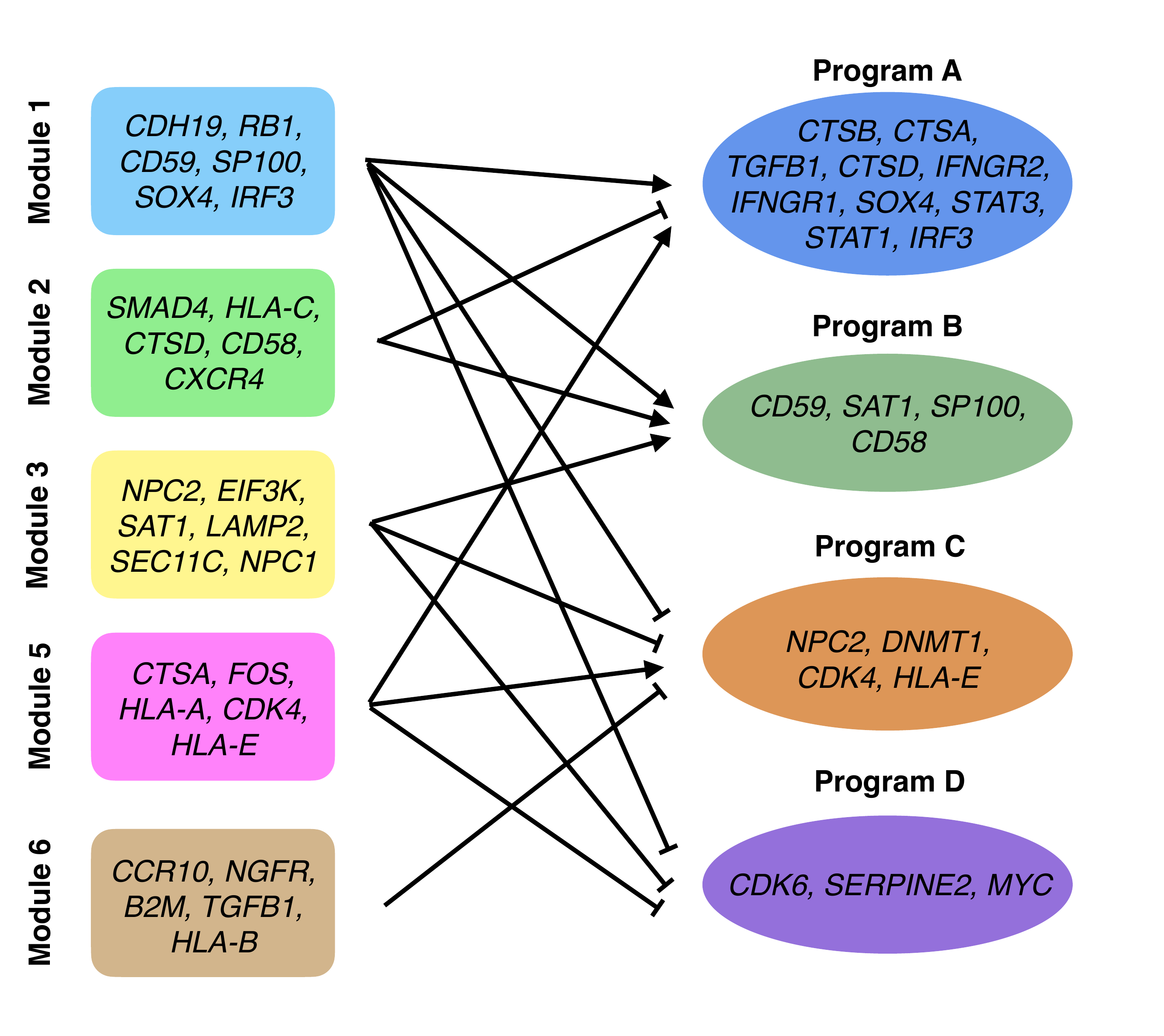}
    \caption{\textbf{Co-functional modules and co-regulatory programs given in \cite{frangieh2021multimodal}.} Intersected with the considered $36$ target genes in our experiment. Arrows $\rightarrow$ represent inducing relationships, whereas arrows $\raisebox{-.2ex}{\rotatebox{90}{$\bot$}}$ represent reducing relationships.}
    \label{fig:s10}
\end{figure}

\begin{figure}[!ht]
     \centering
    \begin{subfigure}[b]{0.85\textwidth}
         \centering
         \includegraphics[width=.95\textwidth]{final-figs/appendix/DAG-gsp.pdf}
         \caption{}
     \end{subfigure}\\
    \begin{subfigure}[b]{0.48\textwidth}
         \centering
         \includegraphics[width=.75\textwidth]{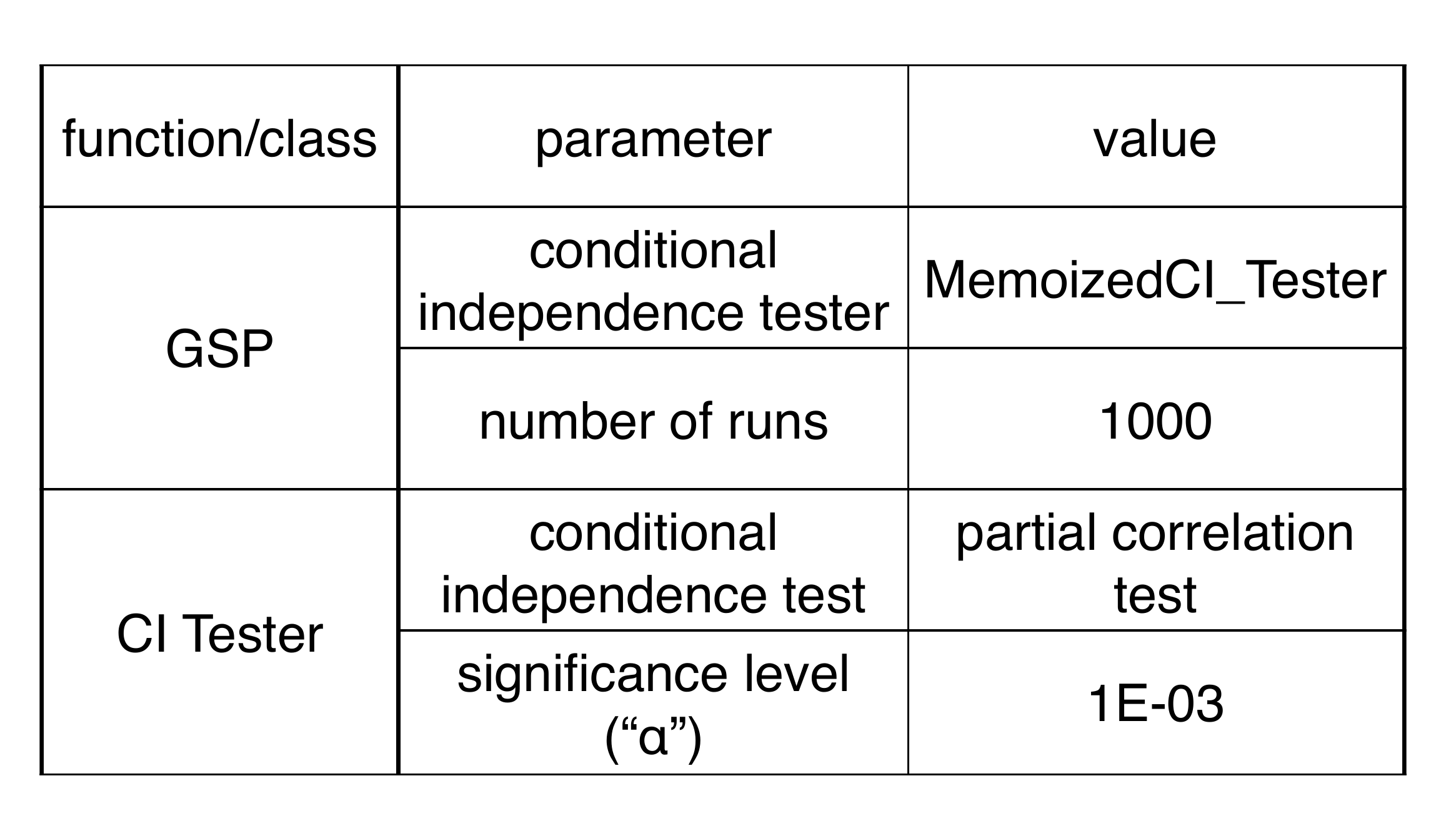}
         \caption{}
     \end{subfigure}
    \begin{subfigure}[b]{0.43\textwidth}
         \centering
         \includegraphics[width=.9\textwidth]{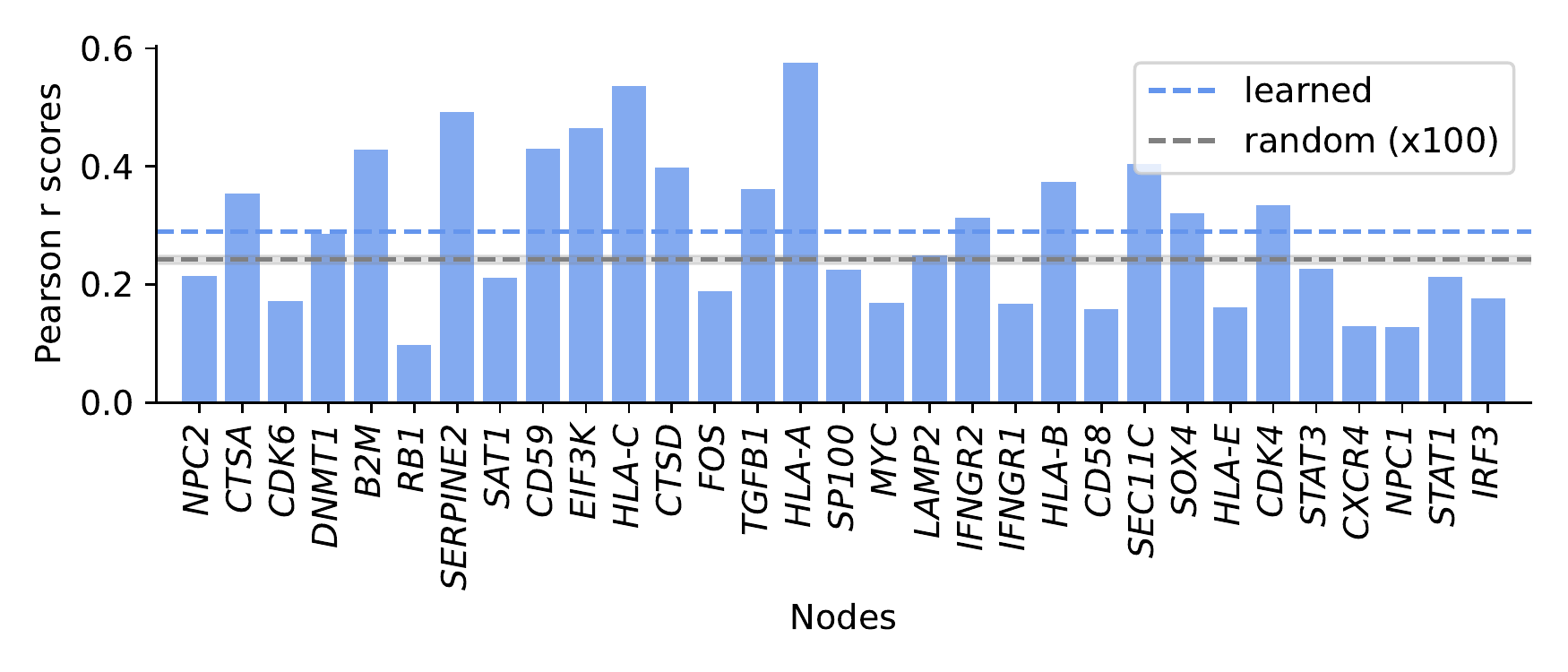}
         \caption{}
     \end{subfigure}
    \caption{\rec{\textbf{Learned linear Gaussian SCM on the $36$ genes of interest based on the control cells.}} (A) Learned DAG on the $36$ considered genes. Nodes are oriented up-down by their topological order and colored by the module/program they belong to in Supplementary Fig.~\ref{fig:s10}. (B) Parameters used in GSP \cite{solus2021consistency} to learn the above DAG. (C) Pearson r scores of regressing each non-source gene against its parents. Blue: average scores in the learned DAG; grey (with errorbars): average of scores in a random graph (100 samples).}
    \label{fig:s11}
\end{figure}

To learn a DAG over the 36 genes, we use the Greedy Sparsest Permutation (GSP) algorithm~\cite{solus2021consistency} based on the control cells (i.e., observational samples with no perturbations).
The parameters used in GSP are given in Extended Data Fig.~\ref{fig:s11}B. 
Extended Data Fig.~\ref{fig:s11}A shows the learned DAG, where we color the nodes based on the module/program they belong to  in Supplementary Fig.~\ref{fig:s10}.
Since each gene can appear in one of the modules, one of the programs, or both sides in Supplementary Fig.~\ref{fig:s10},
we color each node by the module it belongs to on the inside and the program it belongs to on the rim. 
This allows comparing the learned causal relationships with the given partial functional relationships. For example, the edge \emph{NPC2}$\rightarrow$\emph{CDK6} in Extended Data Fig.~\ref{fig:s11}A agrees with Supplementary Fig.~\ref{fig:s10} since Module 3 inhibits Program D. 
However, not all edges agree; in fact, in Supplementary Fig.~\ref{fig:s10} there are bidirectional relationships between genes (e.g., \emph{SOX4} and \emph{IRF3} both belong to Module 1 and Program A), which is not consistent with a DAG structure. Thus there exists no DAG that is in complete agreement with all the directed graph in Supplementary Fig.~\ref{fig:s10}.

When no prior knowledge is available about the DAG structure, the inferred DAG can nevertheless be validated by analyzing how well the expression of a downstream gene can be predicted from the expression of its parents. This is done by regressing the expression of each non-source gene on the expression of its parent genes. 
Extended Data Fig.~\ref{fig:s11}C shows the Pearson correlation coefficients for all non-source genes. 
We compare the average of these correlation coefficients against the average of the correlation coefficients based on a randomly generated DAG (we use an Erdös-Rényi graph~\cite{erdos1960evolution} with the same sparsity as the learned DAG. 
\rec{We achieve this by generating the Erdös-Rényi graph with edge probability equaling $\nicefrac{2 K}{p(p-1)}$, where $K$ is the number of edges in the learned DAG.}).
%
%
Generating $100$ such random DAGs and computing the average correlation coefficients, we observe that the average score in the learned DAG is significantly higher than it is in a random DAG. This shows that the inferred DAG better describes the observed data than a random DAG.  This inferred DAG is used in the sequel to run our optimal perturbation design algorithms.

We consider the following optimal perturbation design problem: let the control cell population (i.e., observational distribution) be the source cell state and let a particular perturbed cell population with a single target gene amongst the considered $36$ genes (i.e., single-node
interventional distribution) be the target cell state. The goal is to identify the targeted gene or a gene target with a similar effect.
In Supplementary Fig.~\ref{fig:s12}A, we show a UMAP \cite{mcinnes2018umap} plot of the control cells and the perturbed cells for $6$ representative perturbations. 
Importantly, this plot shows that the alterations induced by single-target perturbations are subtle and can in general not be captured by conventional dimension reduction techniques. 
To visualize the subtle alterations induced by the different interventions, we use \emph{contrastive PCA} \cite{abid2018exploring}, a method to extract low-dimensional structure enriched in a foreground dataset relative to a background dataset. 
Specifically, we use all the perturbed cells as foreground and all the control cells as background. 
Supplementary Fig.~\ref{fig:s12}B shows the resulting visualization of the $6$ representative interventional distributions, including the ones discussed in the main text. In contrast to the UMAP plot, there are visible differences in the interventional distributions for different targeted genes; e.g., the left two columns are similar and distinct from the right-most column.

\begin{figure}[!h]
    \centering
    \begin{subfigure}[b]{0.41\textwidth}
         \centering
         \includegraphics[width=\textwidth]{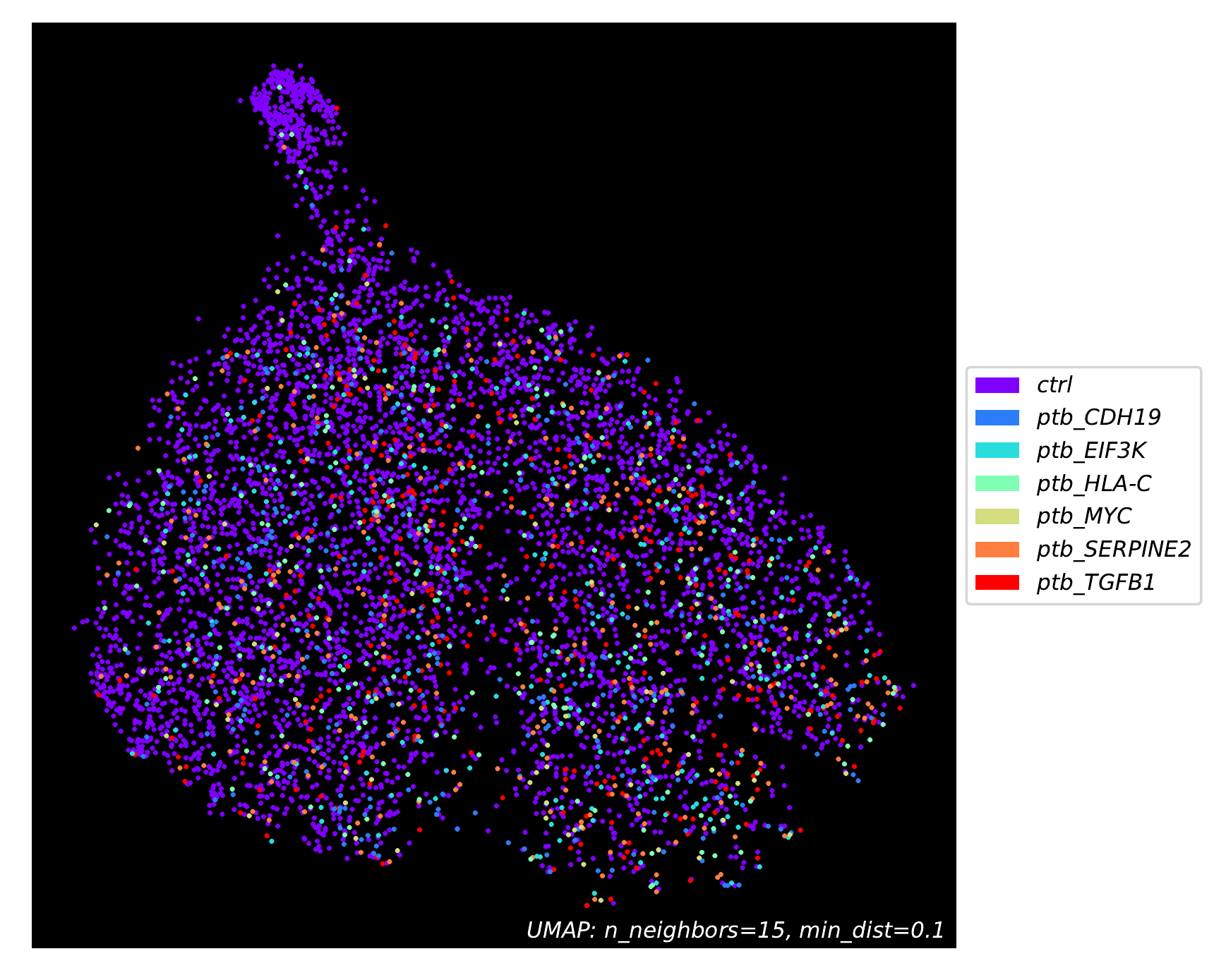}
         \caption{UMAP plot of cells, colored by whether\\the cell belongs to the observational distrib-\\ution or one of the interventional distributio-\\ns.}
    \end{subfigure}
    \begin{subfigure}[b]{0.48\textwidth}
         \centering
         \includegraphics[width=\textwidth]{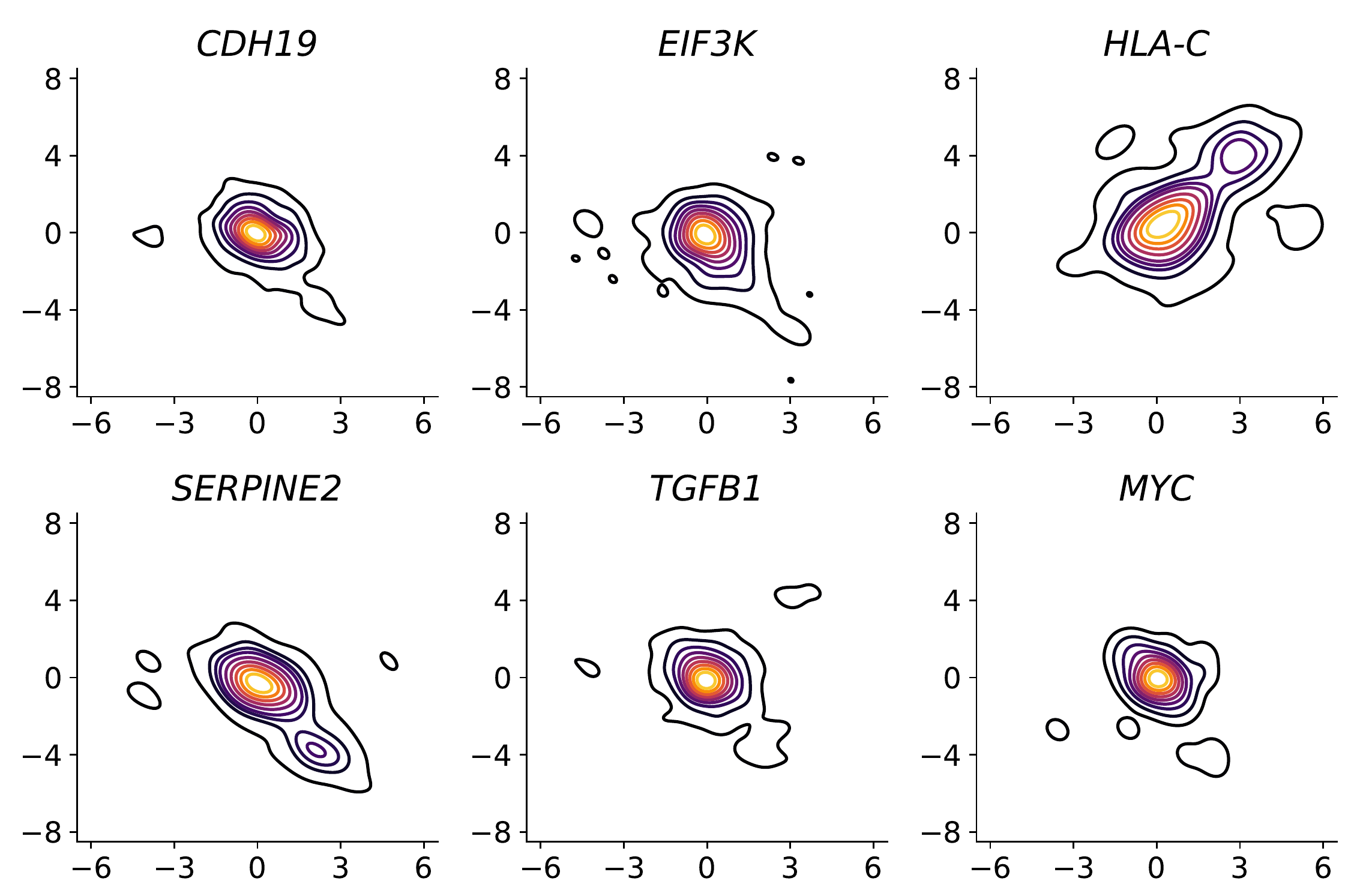}
         \caption{Kernel density estimate (KDE) plot of the 2-d projection of the interventional distribution obtained by targeting $6$ representative genes: \emph{CDH19}, \emph{EIF3K}, \emph{HLA-C}, \emph{SERPINE2}, \emph{TGFB1}, and \emph{MYC}.}
    \end{subfigure}
    \caption{\textbf{Visualization of interventional distributions.} (A) UMAP plot of control cells and perturb cells from $6$ representative interventions. (B) KDE plots using the first two contrastive PCA coordinates on the same $6$ interventional distributions.}
    \label{fig:s12}
\end{figure}

\begin{figure}[!ht]
     \centering
     \begin{subfigure}[b]{0.19\textwidth}
         \centering
         \includegraphics[width=\textwidth]{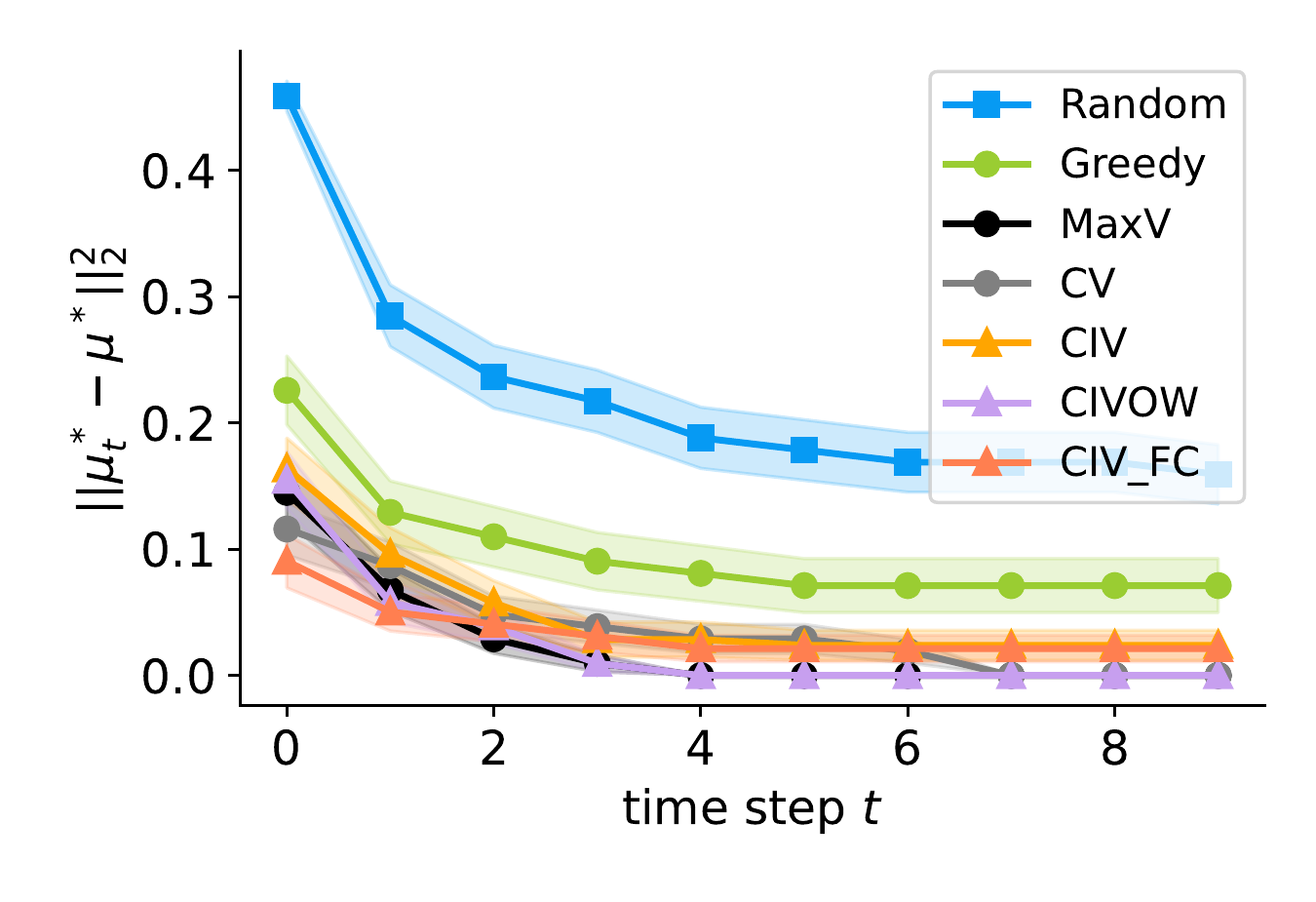}
        \caption{\emph{CDH19}.}
     \end{subfigure}
     \begin{subfigure}[b]{0.19\textwidth}
         \centering
         \includegraphics[width=\textwidth]{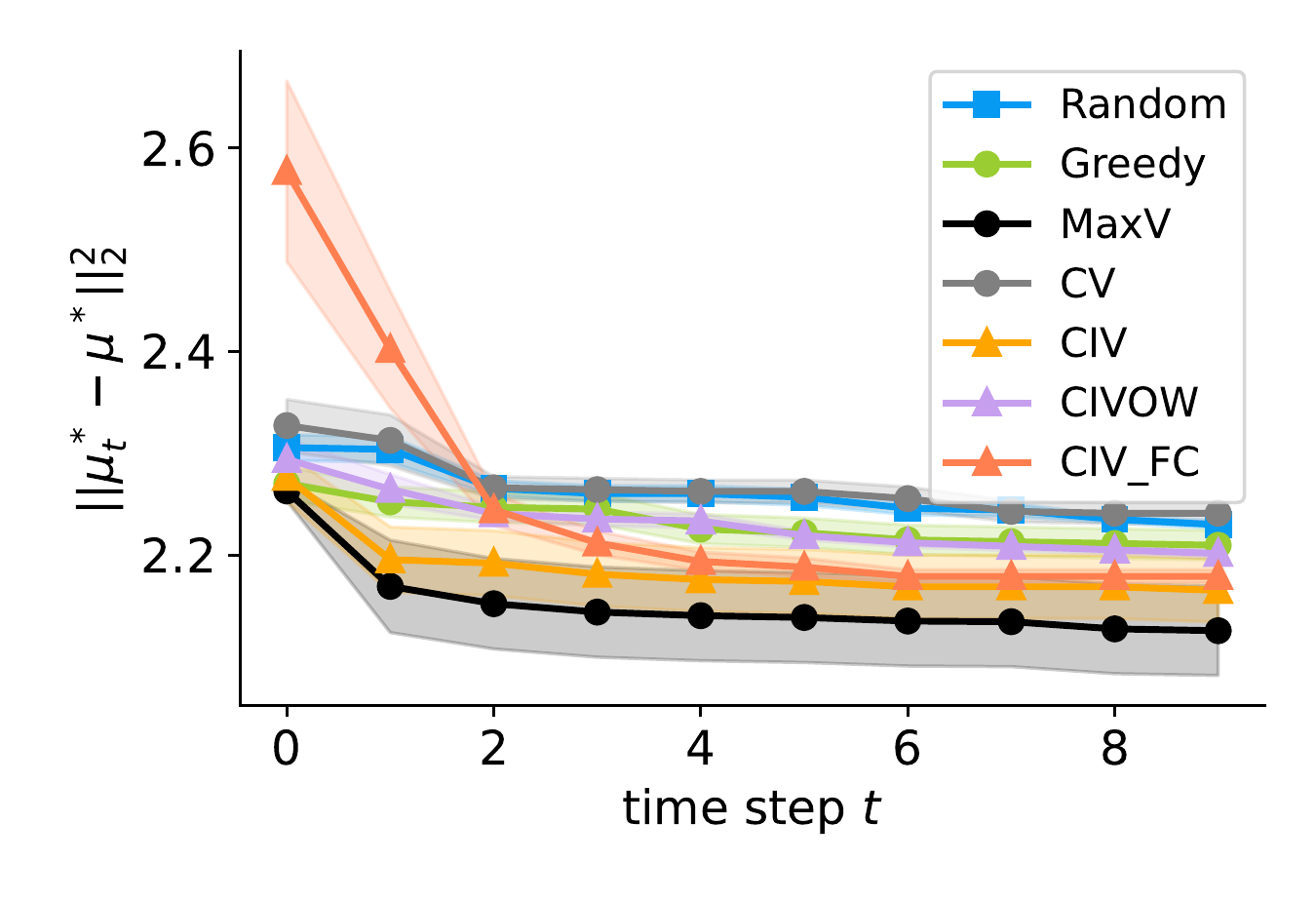}
        \caption{\emph{SERPINE2}.}
     \end{subfigure}
     \begin{subfigure}[b]{0.19\textwidth}
         \centering
         \includegraphics[width=\textwidth]{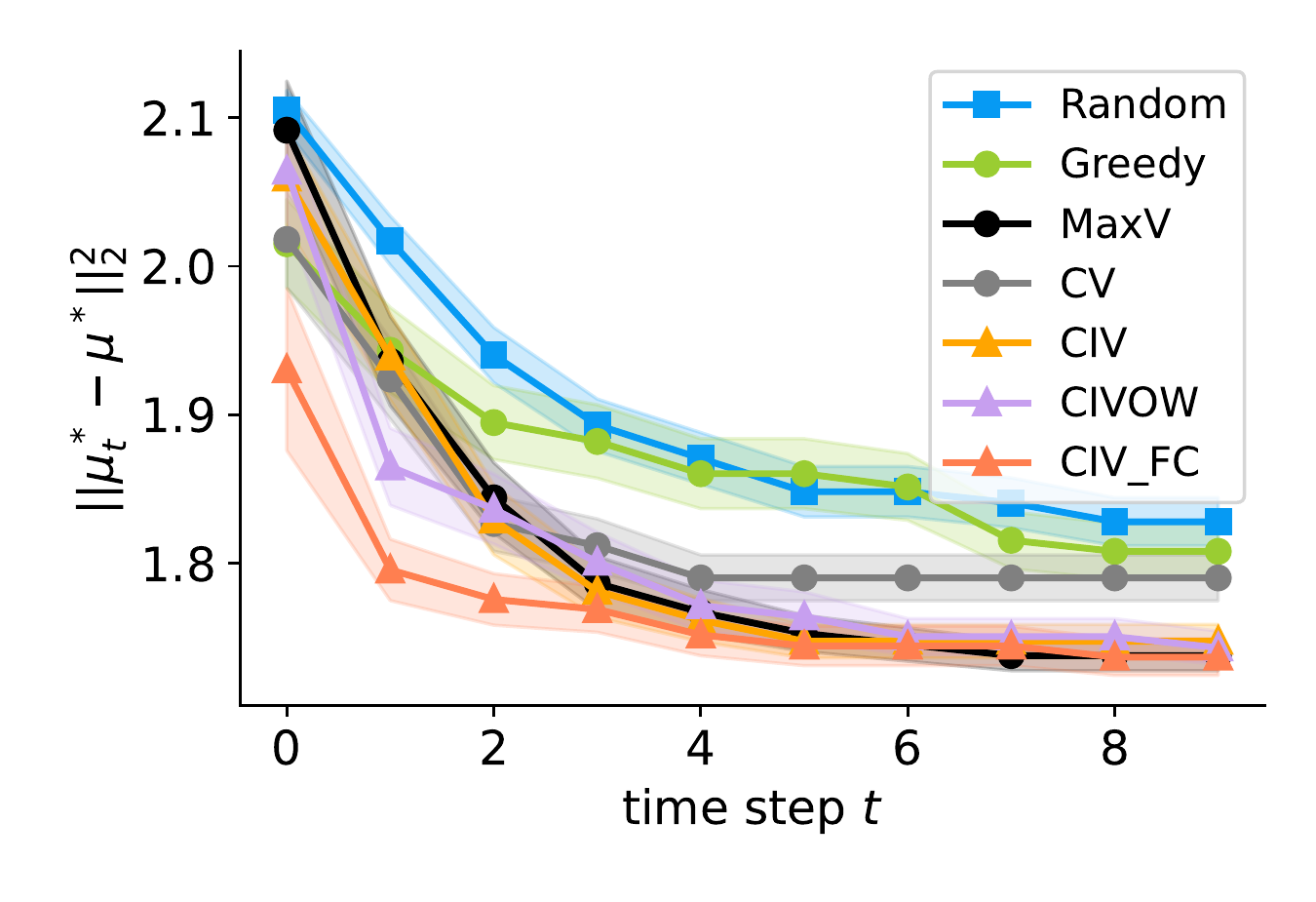}
         \caption{\emph{EIF3K}.}
     \end{subfigure}
     \begin{subfigure}[b]{0.19\textwidth}
         \centering
         \includegraphics[width=\textwidth]{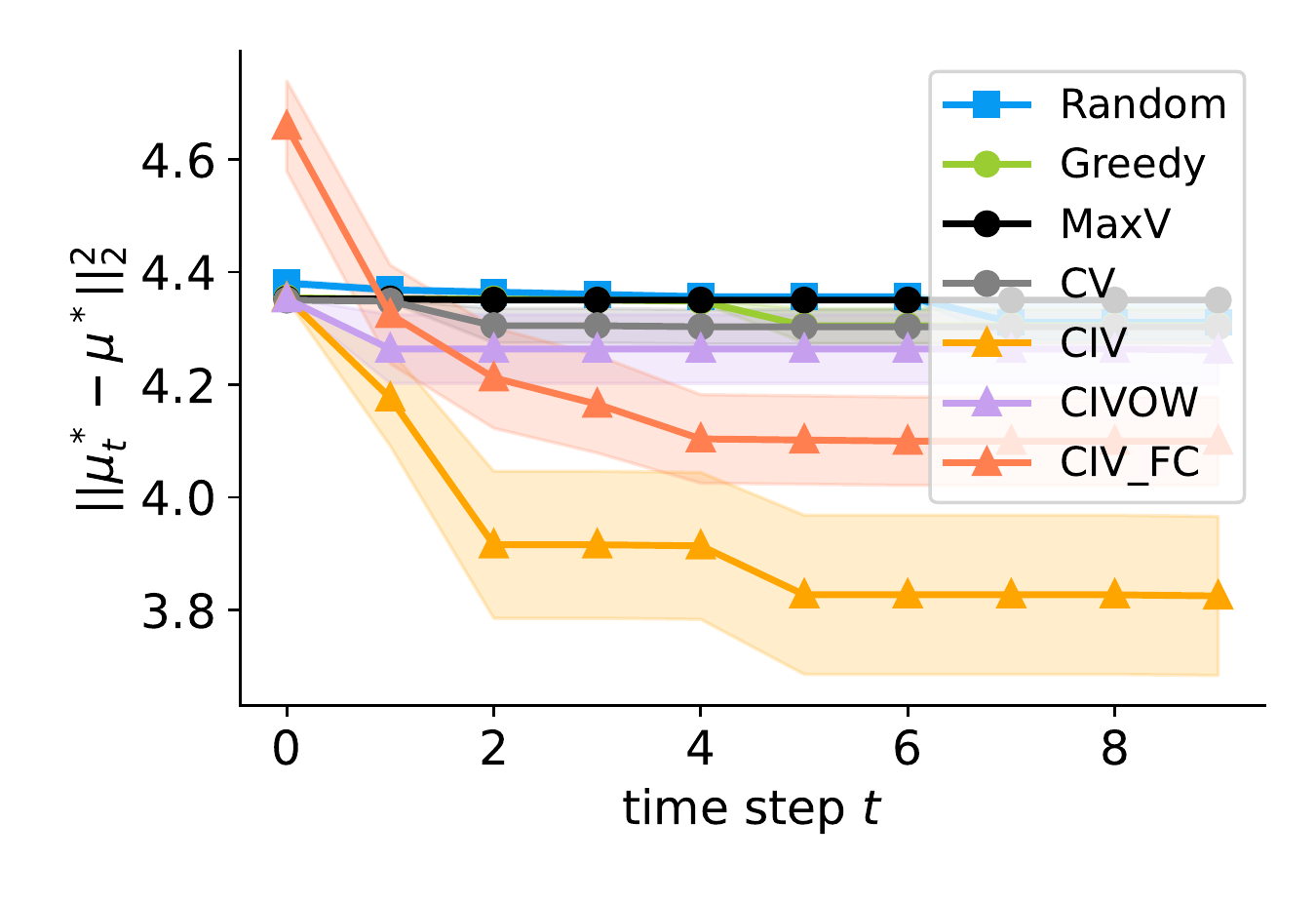}
         \caption{\emph{HLA-C}.}
     \end{subfigure}
    \begin{subfigure}[b]{0.19\textwidth}
         \centering
         \includegraphics[width=\textwidth]{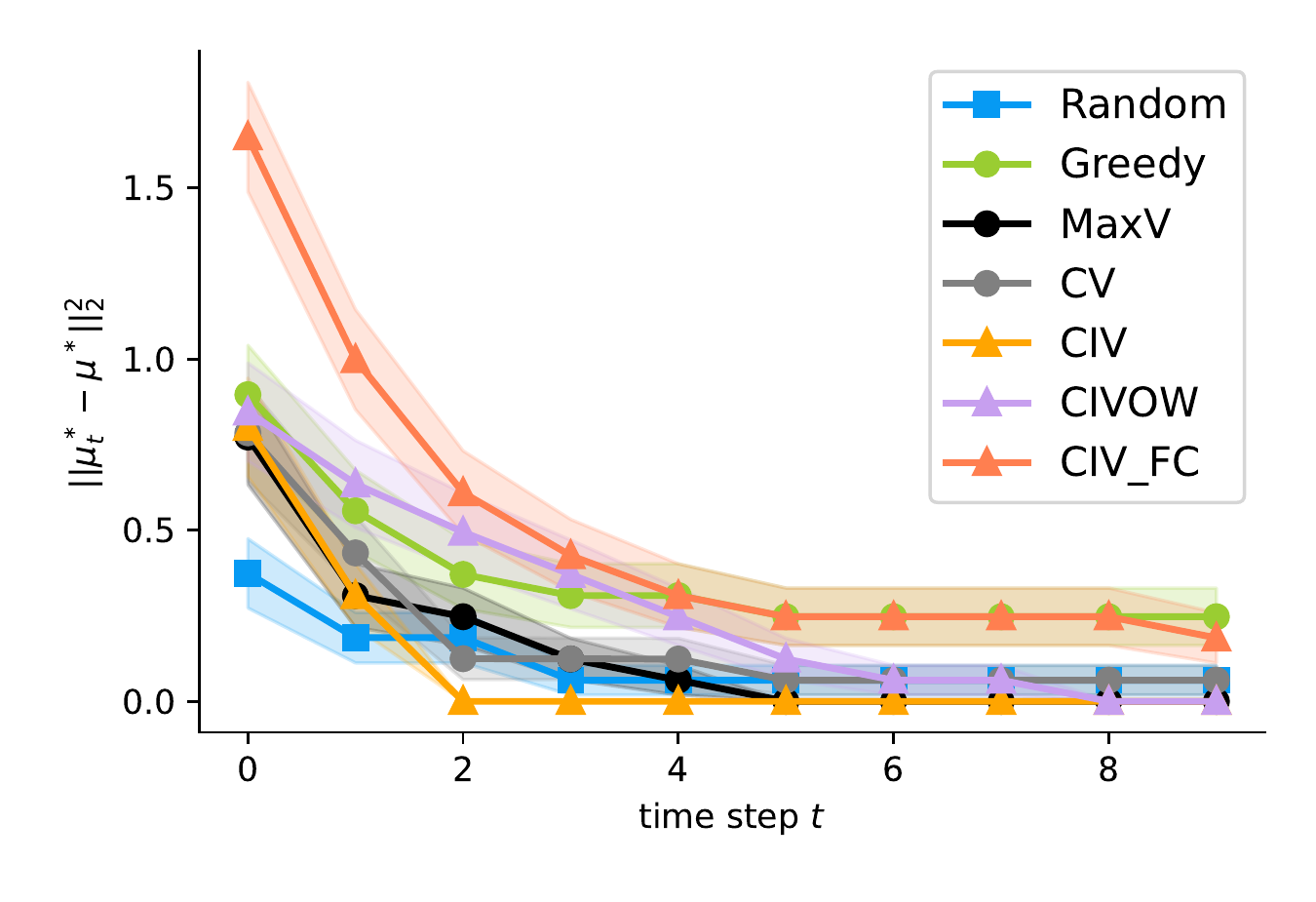}
         \caption{\emph{TGFB1}.}
     \end{subfigure} \\
       \begin{subfigure}[b]{0.19\textwidth}
         \centering
         \includegraphics[width=\textwidth]{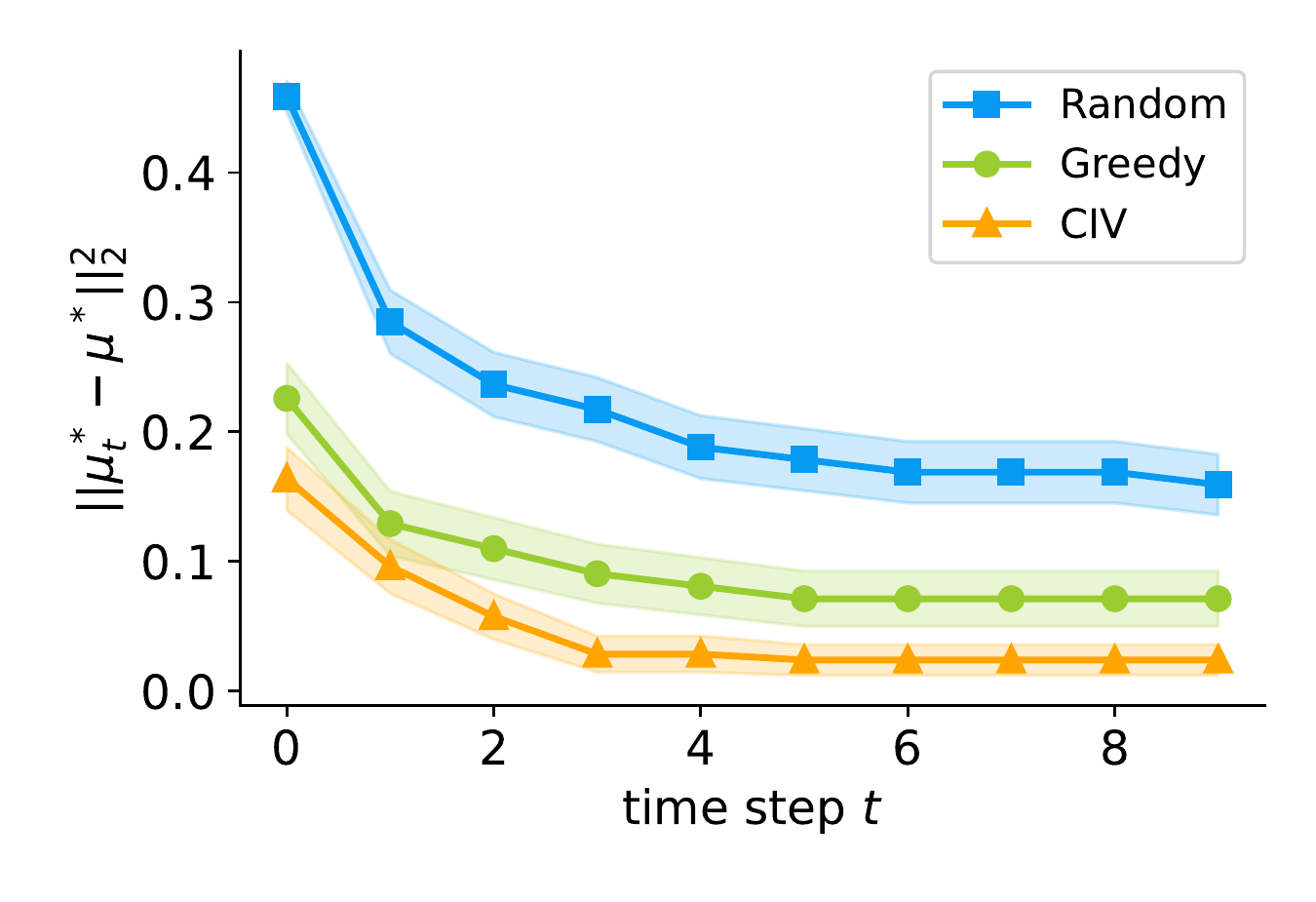}
        \caption{\emph{CDH19}.}
     \end{subfigure}
     \begin{subfigure}[b]{0.19\textwidth}
         \centering
         \includegraphics[width=\textwidth]{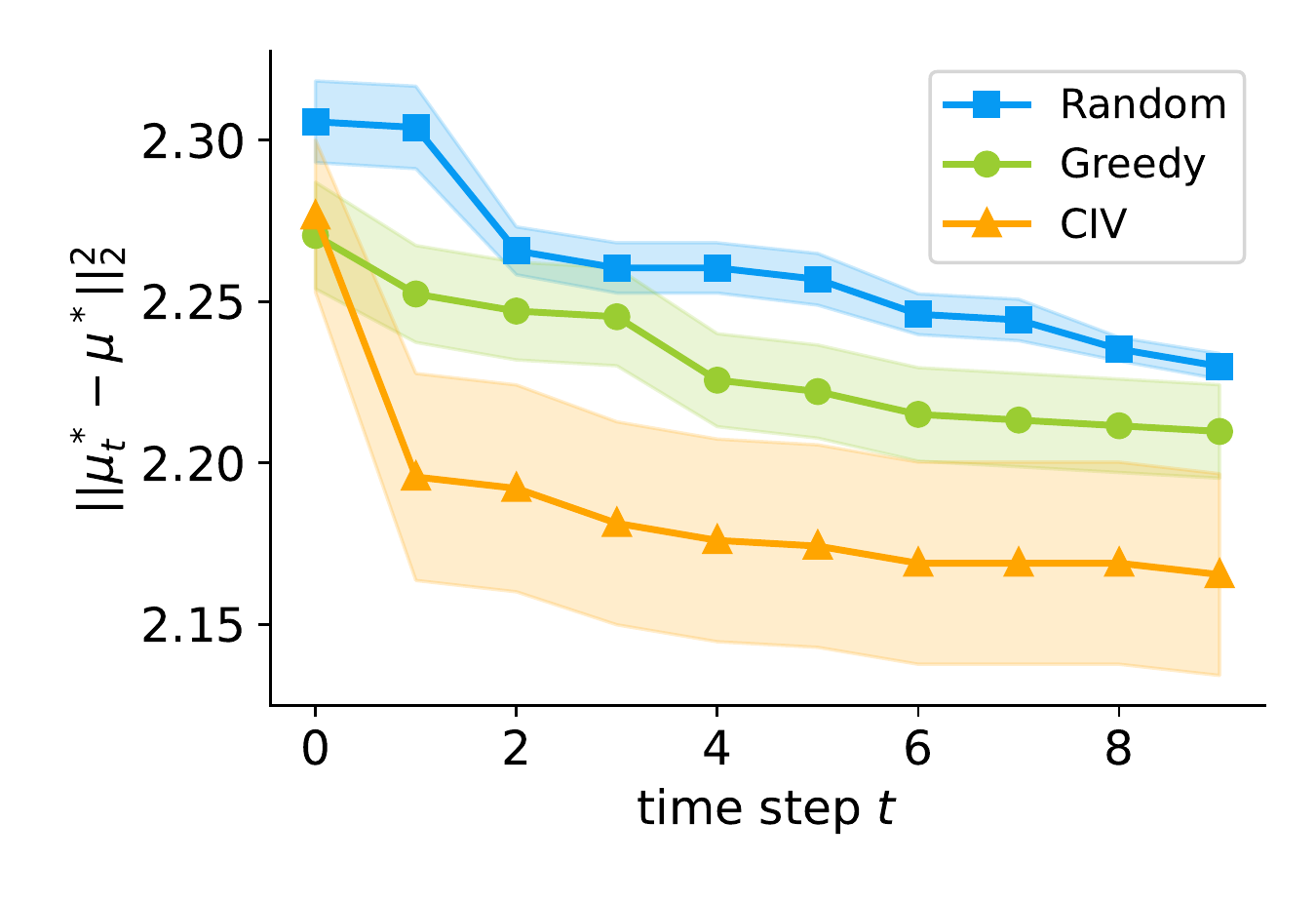}
        \caption{\emph{SERPINE2}.}
     \end{subfigure}
     \begin{subfigure}[b]{0.19\textwidth}
         \centering
         \includegraphics[width=\textwidth]{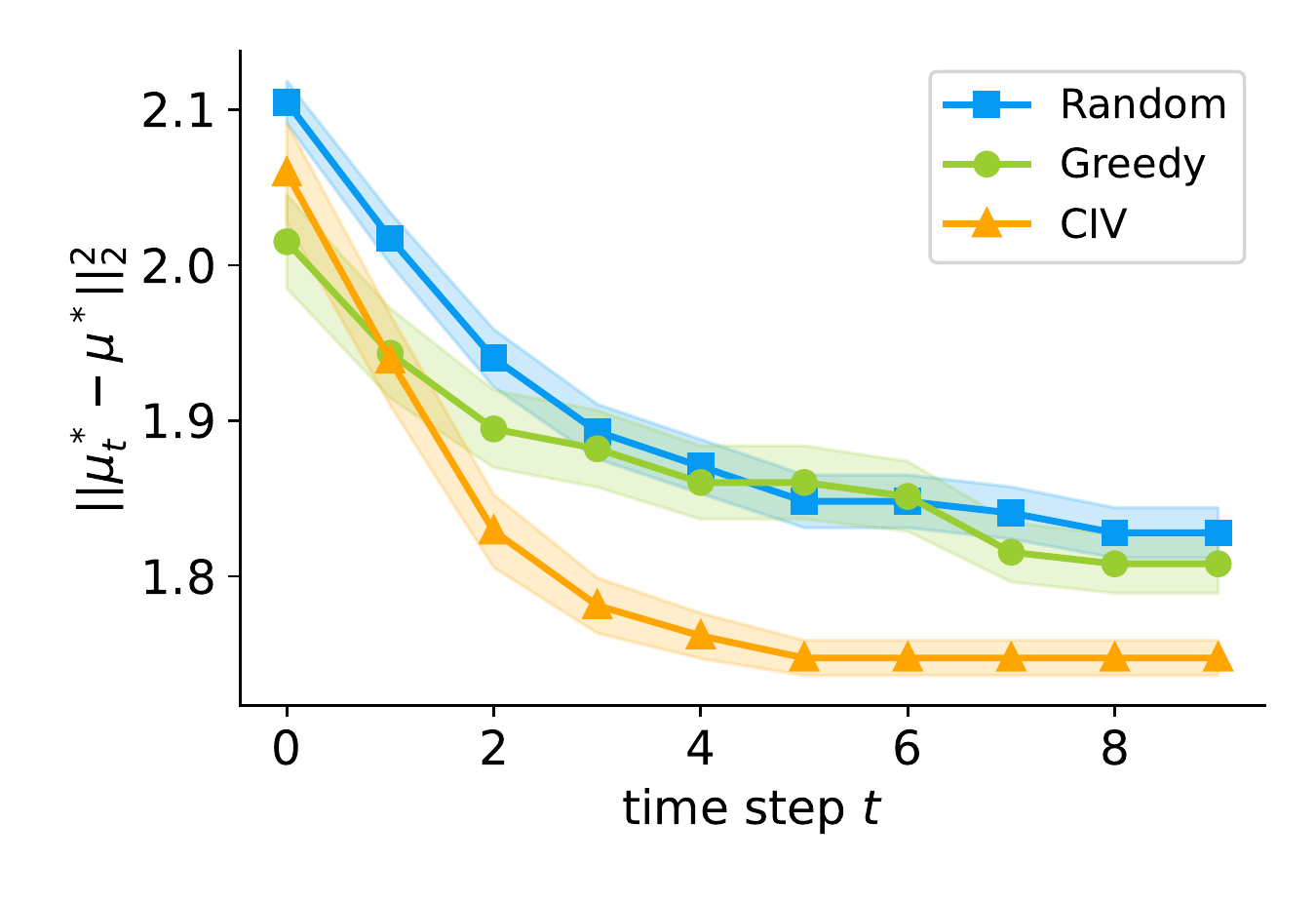}
         \caption{\emph{EIF3K}.}
     \end{subfigure}
     \begin{subfigure}[b]{0.19\textwidth}
         \centering
         \includegraphics[width=\textwidth]{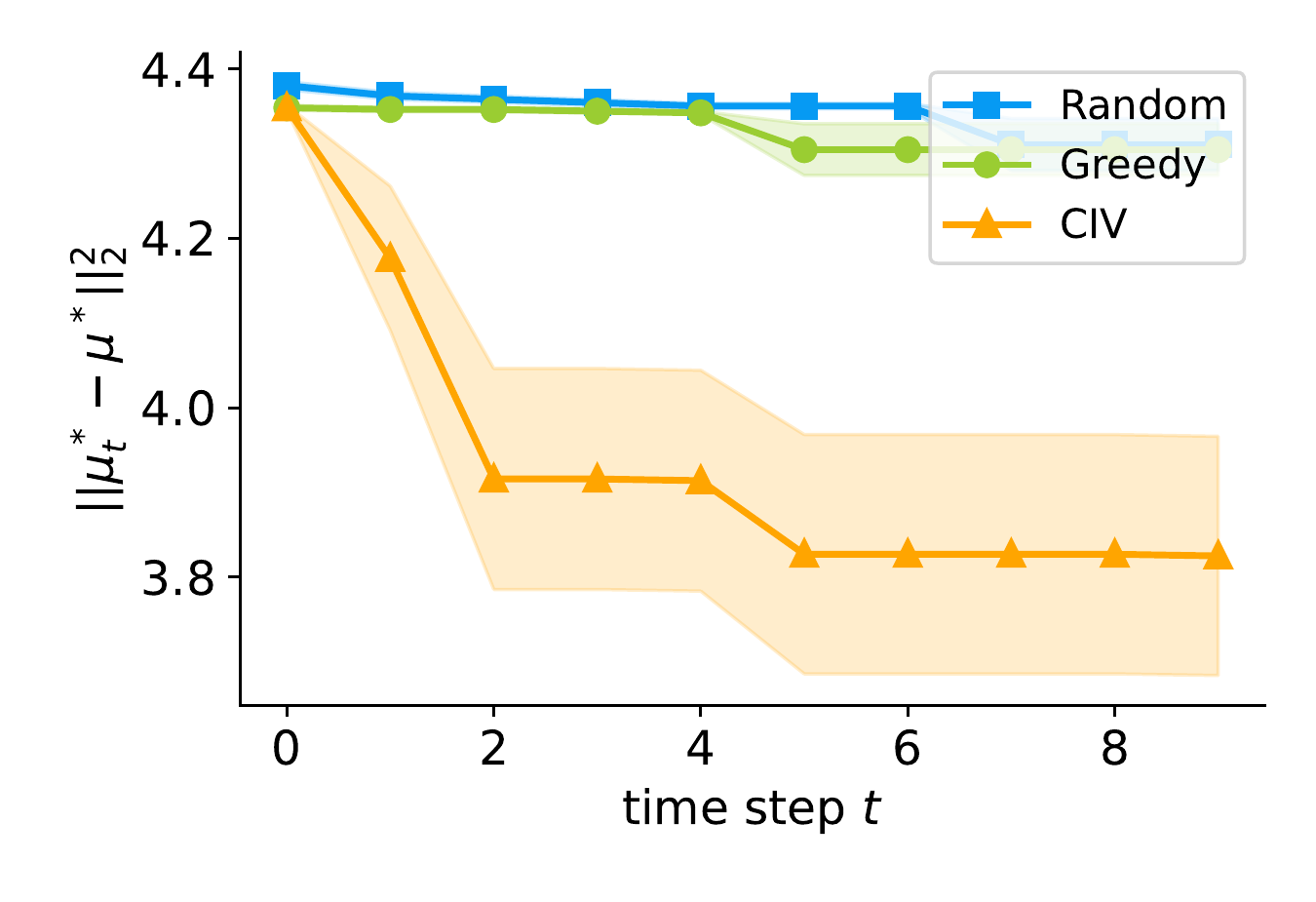}
         \caption{\emph{HLA-C}.}
     \end{subfigure}
    \begin{subfigure}[b]{0.19\textwidth}
         \centering
         \includegraphics[width=\textwidth]{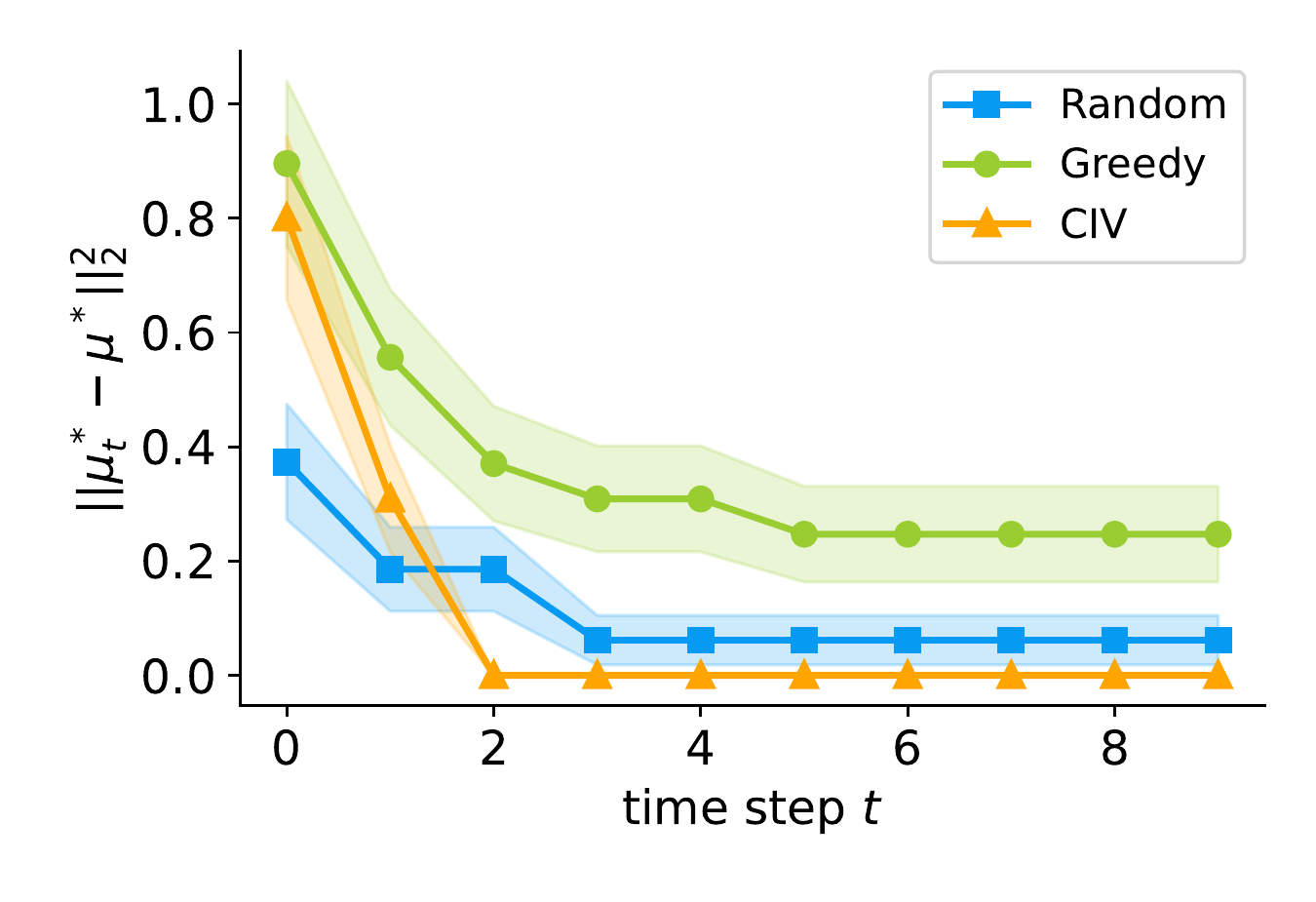}
         \caption{\emph{TGFB1}.}
     \end{subfigure}
    \caption{\rec{\textbf{Comparison of the different acquisition functions for identifying the intervention that matches the target mean for $5$ different ground-truth target genes.}} Square distance \rec{presented as mean value +/- SEM} between the target mean $\bmu^*$ and the best approximation $\bmu^*_t$ across time step $t$ is reported. Top: comparison of {6 acquisition functions}. Bottom: same as top, showing only 3 methods to de-clutter the plots, comparing our CIV acquisition function against the random and greedy baseline. Each plot is captioned with its ground-truth target gene.}
    \label{fig:s13}
\end{figure}

We end by presenting additional experimental results. 
Including the gene \emph{CDH19}, which is discussed in the main text, we selected one gene from each of the $5$ modules in \cite{frangieh2021multimodal} and Supplementary Fig.~\ref{fig:s10}. We selected these genes at random among all genes for which the gene perturbation was effective, i.e., genes that when targeted showed a reduced expression as compared to the control cells. The results for the representative 5 genes \emph{CDH19} (results also discussed in the main text), \emph{SERPINE2}, \emph{EIF3K}, \emph{HLA-C} and \emph{TGFB1} are shown in Extended Data Fig.~\ref{fig:s13}.
{In addition to the $5$ acquistion functions discussed in the main text and also shown in the synthetic experiments, we add another acquisition function to assess the sensitivity of our method to the learned DAG: CIV-FC corresponds to using the CIV acquisition function with a fully-connected version of the learned DAG.}
The results show similar trends as already observed in the main text for the gene \emph{CDH19}: as compared to the passive baseline, the active methods are able to more quickly identify perturbation targets that move the distribution to the desired target distribution given by perturbing one of the genes. In addition, the acquisition functions CIV and CIV-OW are generally among the best performing active methods. 
{While CIV-FC outperforms the passive and greedy baselines, it generally performs worse than CIV, as it requires the estimation of more parameters.}
We note that some of the curves do not go to zero within 50 steps, meaning that in average they did not yet identify the correct intervention target that gives rise to the desired interventional distribution. This is because within $50$ steps some runs only identify interventions with similar effect as the targeted gene.

\begin{figure}[!t]
     \centering
     \begin{subfigure}[b]{0.3\textwidth}
         \centering
    \includegraphics[width=.9\textwidth]{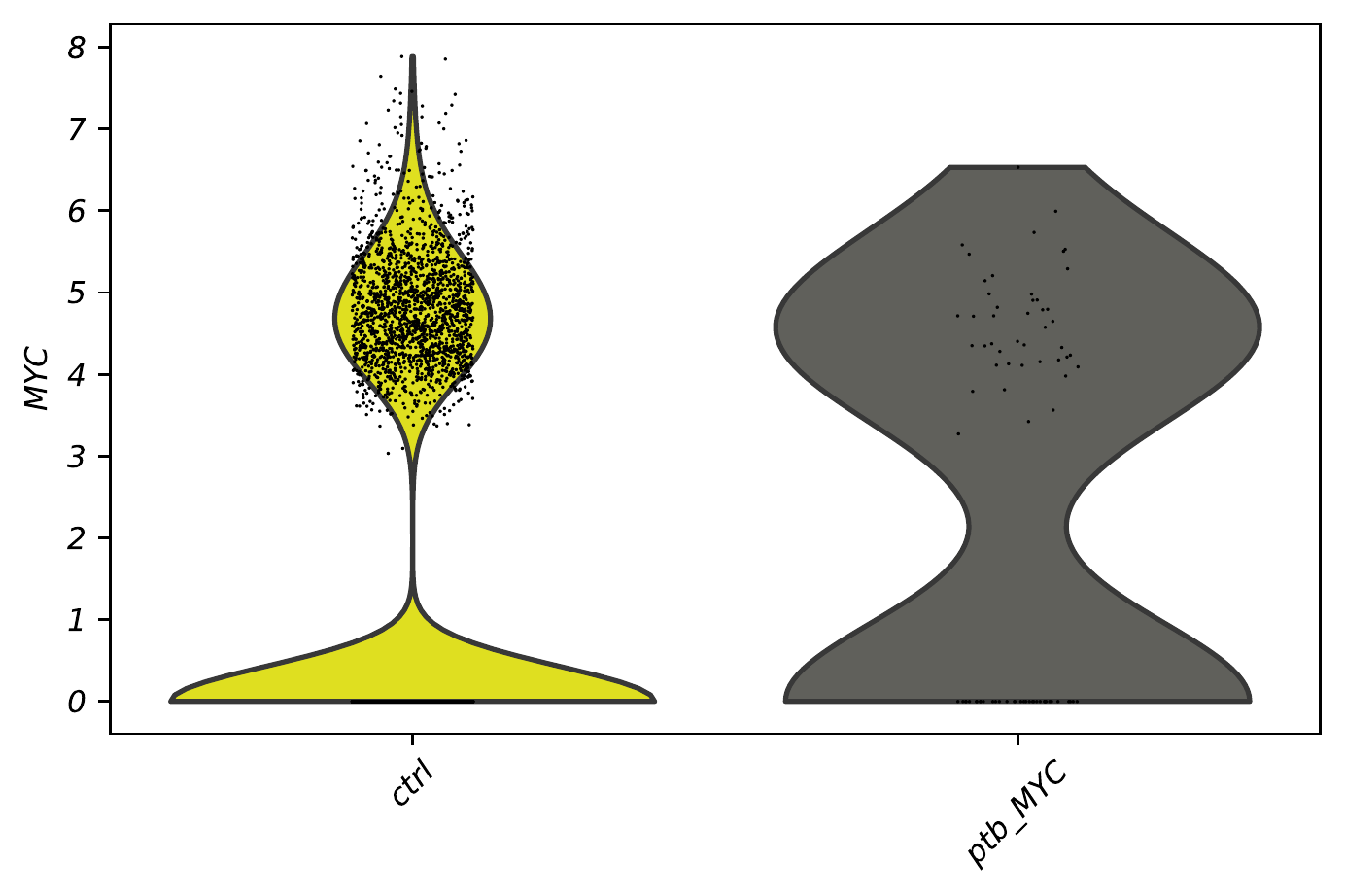}
        \captionsetup{justification=centering}
        \caption{\emph{MYC} expression. Mean:\\1.52~(ctrl)~v.s.~2.58~(ptb\_\emph{MYC}).}
     \end{subfigure}
     \begin{subfigure}[b]{0.3\textwidth}
         \centering
         \includegraphics[width=.9\textwidth]{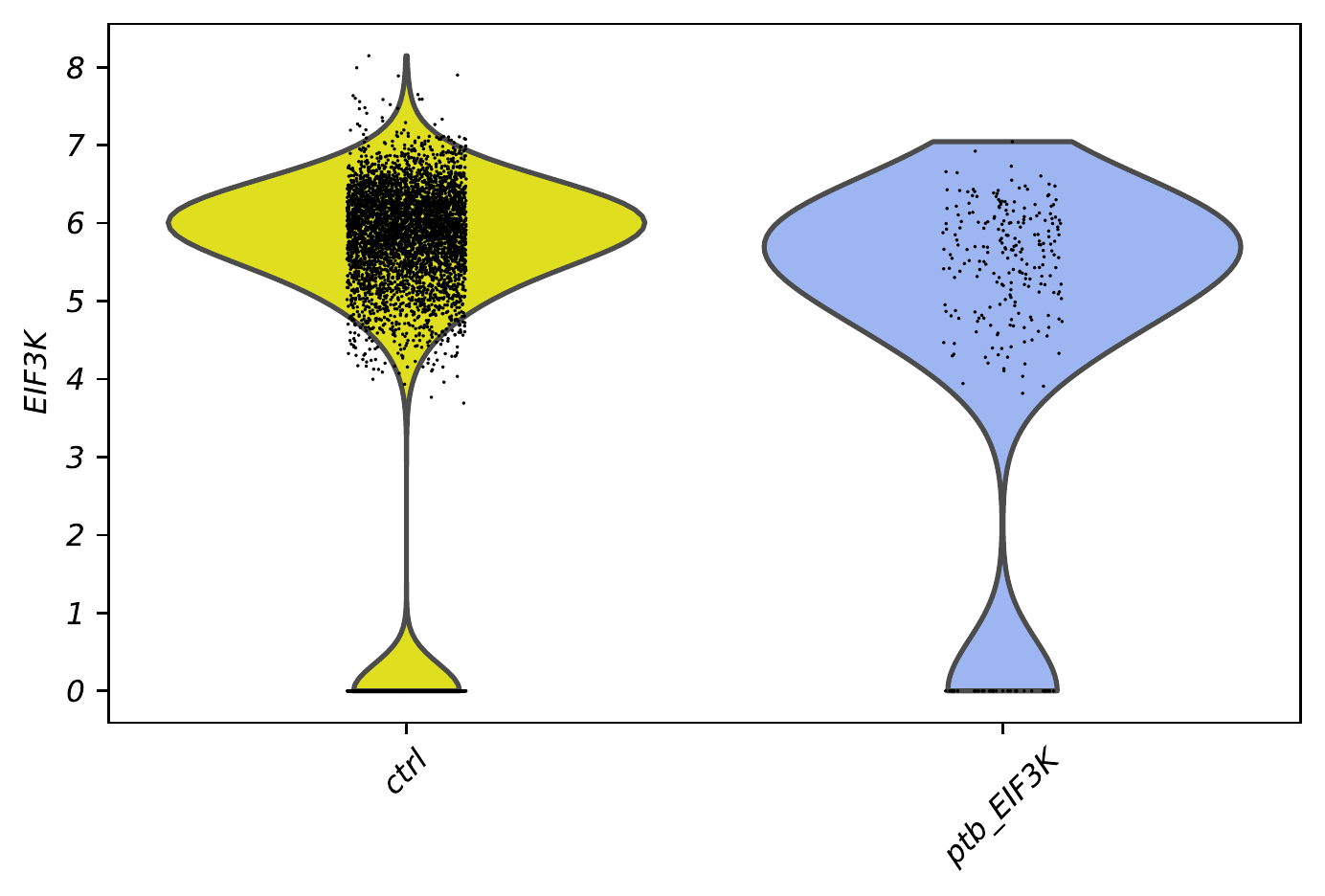}
         \captionsetup{justification=centering}
         \caption{\emph{EIF3K} expression. Mean:\\ 5.19~(ctrl)~v.s.~4.78~(ptb\_\emph{EIF3K}).}
     \end{subfigure}
     \begin{subfigure}[b]{0.3\textwidth}
         \centering
         \includegraphics[width=.9\textwidth]{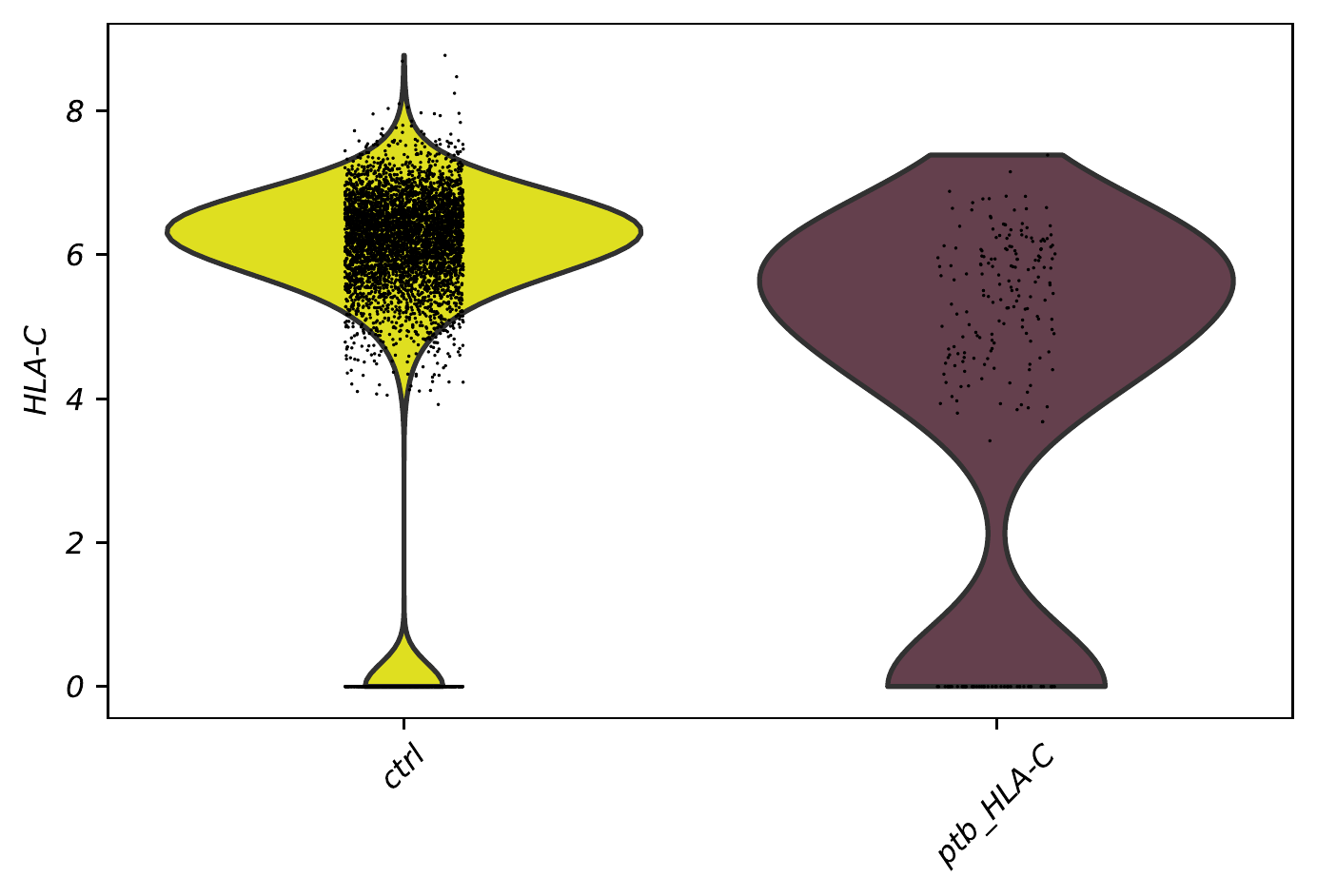}
         \captionsetup{justification=centering}
         \caption{\emph{HLA-C} expression. Mean:\\5.98~(ctrl)~v.s.~4.16~(ptb\_\emph{HLA-C}).}
     \end{subfigure}
     \caption{\rec{\textbf{Gene expression changes for three examples of different knock-out perturbations.} Comparing} target-gene expression in the control cell population and the perturbed cell population of the corresponding knock-out experiment. Included target genes: \emph{MYC}, \emph{EIF3K}, and \emph{HLA-C}. The mean expression of the target gene is given in each subcaption.}
     \label{fig:s14}
\end{figure}

As described above, we were careful to only consider gene targets, where the perturbation was actually effective. To end, we also show an example of a target gene whose expression was not reduced in the corresponding knock-out experiment.
Extended Data Fig.~\ref{fig:s14} shows that knocking out \emph{MYC} does not seem to reduce its expression, while knocking out \emph{EIF3K} or \emph{HLA-C} results in a reduction in the expression of the corresponding gene.
Extended Data Fig.~\ref{fig:s15} shows the results when applying the different methods to identify intervention targets that have a similar effect as knocking out \emph{MYC}.
As expected, the active methods do not outperform the random baseline in this case, since the active methods assume that an intervention on a target gene shifts its expression down. 
As also shown in the above experiments, this limitation is easily alleviated by pre-screening the interventions and sorting out interventions that seem ineffective in that they don't reduce the expression of the targeted gene.

\begin{figure}[!h]
    \centering
    \begin{subfigure}[b]{0.25\textwidth}
         \centering
         \includegraphics[width=\textwidth]{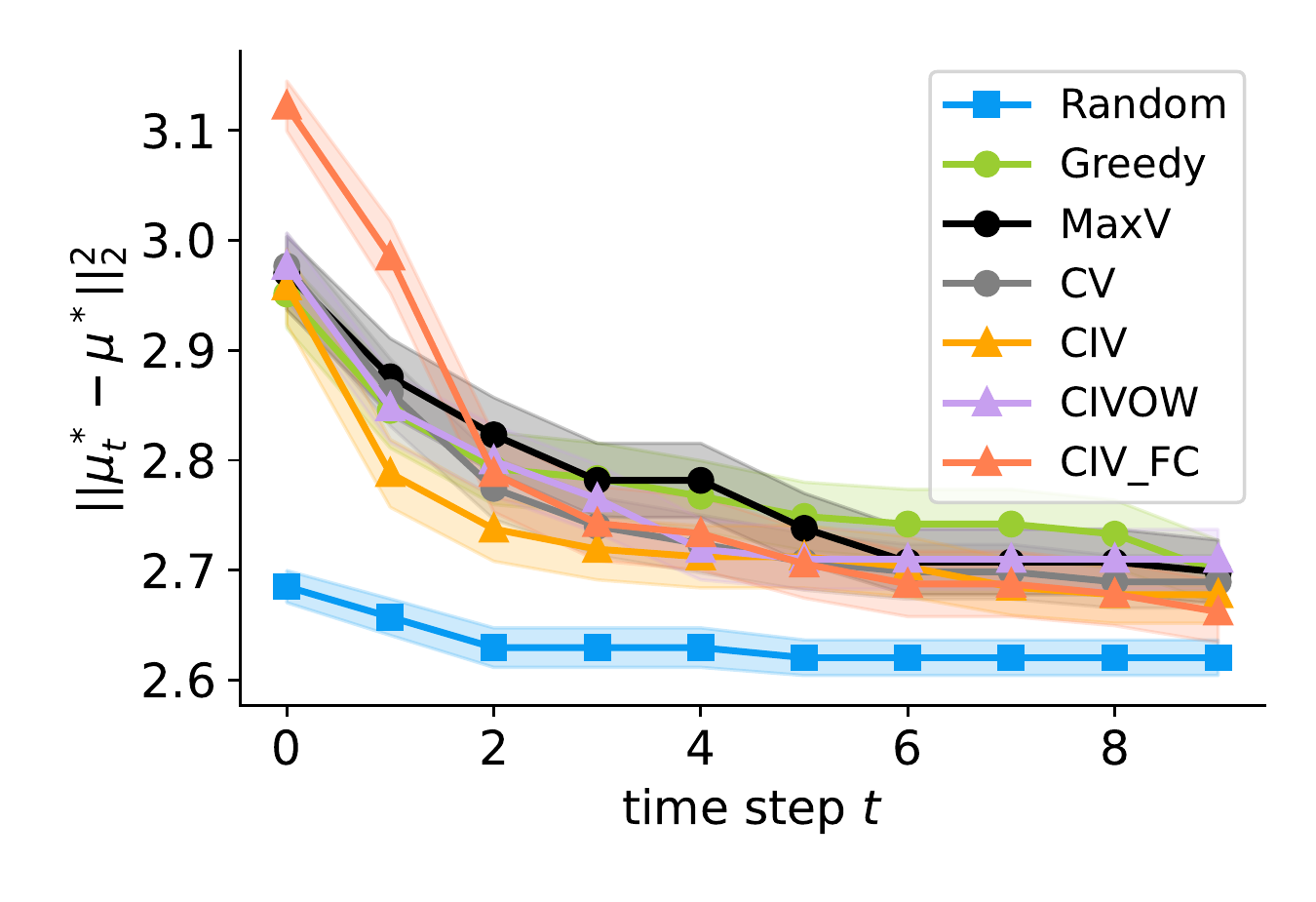}
         \caption{}
     \end{subfigure}
     \begin{subfigure}[b]{0.25\textwidth}
         \centering
         \includegraphics[width=\textwidth]{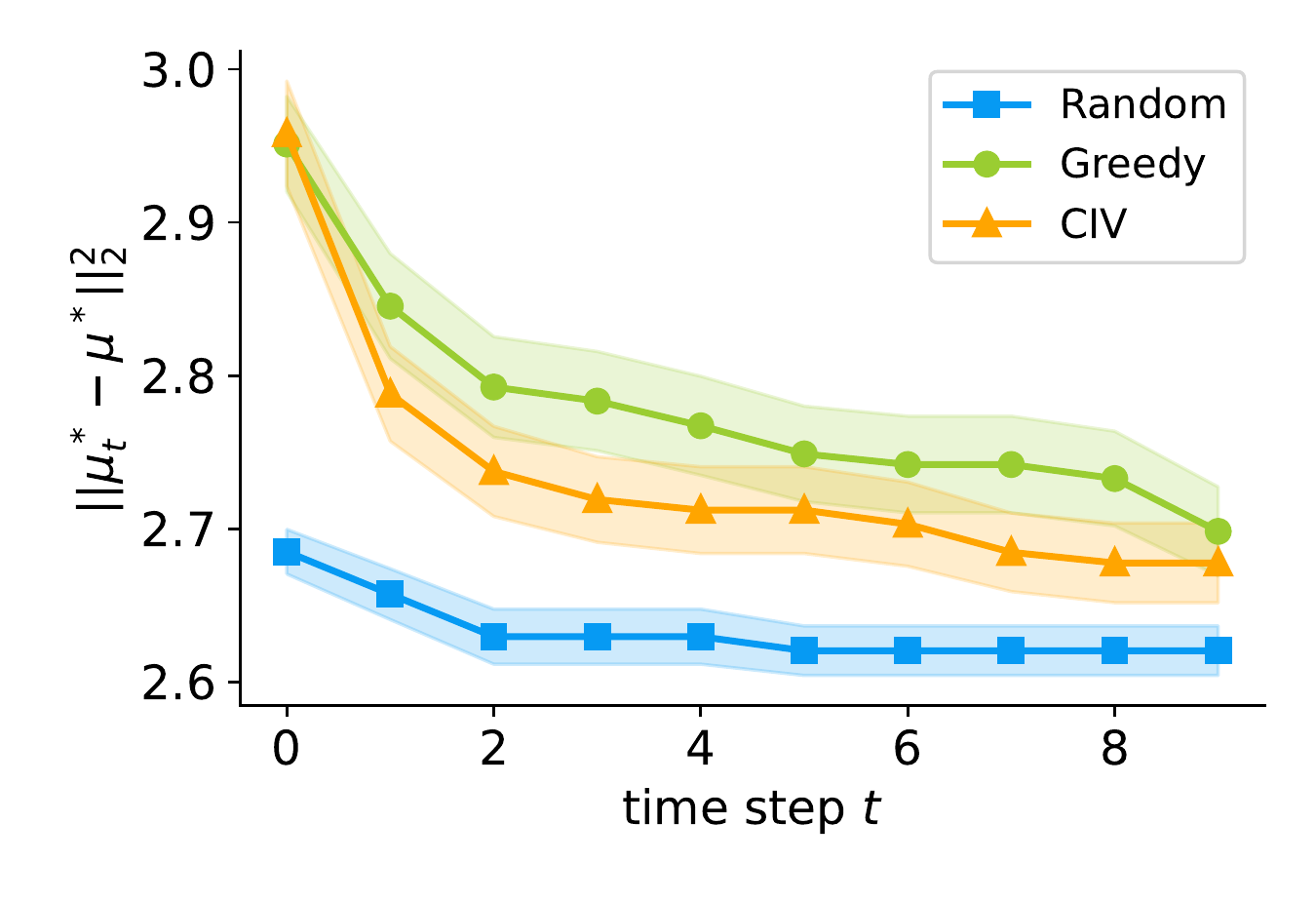}
         \caption{}
     \end{subfigure}
    \caption{\textbf{Comparison of acquisition functions for identifying interventions that match the target mean of perturbing \emph{MYC}.} The reported metric is the \rec{square distance presented as mean value +/- SEM between the target mean $\bmu^*$ and the best approximation $\bmu^*_t$ across all time steps $t$}. (A). All methods. (B) De-cluttered subset of methods.}
    \label{fig:s15}
\end{figure}

{
\section{Related Works}\label{sec:a0}

In this section, we expand on related literature at the intersection of causality and sequential decision making. While \textit{causal experimental design} is concentrated on \emph{estimation}, e.g.~to identify the most effective interventions for estimating a treatment effect, our work is more related to \textit{causal Bayesian optimization}, where the goal is \emph{optimization}, e.g.~to identify an intervention that optimizes treatment effect. In the related causal bandit and reinforcement learning (RL) literature,  the goal is to optimize the \emph{cumulative} reward/regret over multiple time steps rather than at a single endpoint, which is the problem we consider in our work.


\smallskip

\noindent\textbf{Causal Experimental Design.}
Research in this area is centered around  learning the underlying structural causal model or estimating some parameters of the model. The setup typically assumes an unknown directed acyclic graph (DAG) where the goal is to learn this DAG \cite{he2008active} or a specific feature of it (such as the existence of a particular edge) \cite{agrawal2019abcd} by querying interventions. Works have investigated both the noiseless setting, where an infinite number of samples are assumed for each intervention \cite{he2008active, shanmugam2015learning, squires2018causaldag,zhang2021matching}, as well as the noisy finite sample setting \cite{murphy2001active,tong2001active,agrawal2019abcd,gamella2020active,sussex2021near, tigas2022interventions,toth2022active}. A common approach in the finite sample setting is to model the unknown DAG probabilistically using Bayesian methods as in \cite{geiger2002parameter, kuipers2014addendum} and update the posterior distribution based on the obtained samples. Selecting the next intervention often requires marginalizing over all DAGs in the posterior support, which is generally intractable due to a super-exponential number of plausible DAGs. Several approximations have been proposed to overcome this limitation, including using only one DAG that maximizes the posterior likelihood \cite{agrawal2019abcd,sussex2021near}, sampling several DAGs from the posterior either through bootstrapping \cite{friedman2013data} or Markov Chain Monte Carlo \cite{agrawal2018minimal}, and variational approximations \cite{lorch2021dibs} enabled by recent developments in structure learning through continuous optimization \cite{zheng2018dags}. These works will be critical for extending our approach to the unknown DAG setting as discussed in Supplementary Information~\ref{sec:d}. 

\smallskip

\noindent\textbf{Causal Bayesian Optimization.} 
A related problem considered in \cite{aglietti2020causal} is to identify the optimal hard intervention that maximizes the expected value of a target node in a known DAG. Towards this, the authors proposed a three-step approach for acquiring new interventions: (i) an exploration set of intervention targets is identified using \cite{lee2018structural}, (ii) a Gaussian process (GP) model is fitted between the exploration set and the target node, and (iii) the expected improvement of the model output is used as acquisition function. {\textcolor{black}{When the ground-truth causal graph is known, (i) leads to a reduced feasible space that is guaranteed to contain the optimal intervention.} \rec{Similar techniques as (i) in \cite{lee2018structural,lee2020characterizing} could be applied as a pre-processing step to first identify a set of potential intervention targets and then run our method for each element in this set.}} \rev{Note that the causal structure is used to prune intervention targets and incorporated into the predictive model by specifying the GP model as a function of the pruned targets. While this acquisition function can be optimized efficiently over continuous-valued interventions in this model, it could be beneficial to exploit the structural relationship between intervention targets more explicitly.} Concurrent work \cite{sussex2023modelbased} leveraged the causal structure by using multiple GP models, one for each causal mechanism, and proposed an upper confidence bound acquisition function. However, since this acquisition function has no closed form, it leads to scalability issues for continuous-valued interventions in higher dimensions. Similar computational issues arise when using a mutual-information-based acquisition function as proposed in \cite{branchini2022causal} to extend causal Bayesian optimization to the unknown DAG setting. These three acquisition functions were discussed in more detail in Supplementary Information~\ref{add_baselines}, and a comparison to our proposed method was shown in Extended Data~Fig.~\ref{fig:rebut-s7}. 

Bayesian optimization techniques for structured systems have also been explored in contexts beyond causality. For instance, \cite{astudillo2021bayesian,kusakawa2021bayesian} investigated optimizing the output of a network of functions, where each function is a mapping to a node from its parent nodes in a given DAG. While \cite{astudillo2021bayesian} focused on noiseless systems, \cite{kusakawa2021bayesian} considered noisy systems with a fully-connected graph. Additionally, \cite{alabed2022bograph} tackled the same problem with an unknown DAG, proposing a method to first learn the causal structure at each time step. It is worth noting that these works differ from causal Bayesian optimization in that they optimize for system configurations, i.e., the realization of the causal variables, which can be viewed as a special form of hard interventions that target all causal variables.


\smallskip

\noindent\textbf{Causal Bandits, RL, and Control.}
In bandit settings, the goal is to minimize cumulative regret by selecting an arm at each time step \cite{bubeck2012regret}. Prior works have shown that utilizing causal relations between arms and regret can improve the dependence on the total number of arms compared to previous regret bounds \cite{lattimore2016causal, lee2018structural, yabe2018causal}. These works are mainly based on known causal structure (e.g., \cite{lattimore2016causal}) and have recently been extended to unknown causal structure using graphical concepts such separating sets \cite{de2020causal} and directed trees \cite{lu2021causal}. Concurrent work \cite{varici2022causal} considered unknown causal structure and the problem of selecting a set of nodes to intervene on. To avoid estimating a combinatorial number of interventional distributions, the proposed approach estimates all model parameters to compute the reward. In the  applications motivating our work, the simple regret at the last round is a more appropriate metric than cumulative regret, a setting related to the optimal-arm identification problem \cite{audibert2010best}, where the quantity of interest is the ``arm'' with the best performance. In addition, in our problem setting we are concerned with continuous-valued variables and interventions. This results in a harder exploration problem with an infinite number of arms and requires parametric assumptions such as linearity to obtain tractable algorithms. An interesting direction for future work is to test and analyze our uncertainty-based acquisition functions on bandit problems with continuous arms. RL differs from bandits and our setting in that actions accumulate and the system is not reset after each iteration. In systems defined by causal rules, the authors in \cite{zhang2022online} proposed an online algorithm to identify the optimal policy for discrete and finite variables. They considered mixed policy scopes \cite{lee2020characterizing}, which are defined by a target node and a set of context variables. 
Control, on the other hand, typically deals with continuous time settings. A common approach is to first identify the system and then solve for the optimal controller \cite{aastrom2021feedback}. The authors in \cite{weichwald2022learning} designed a challenge at the intersection of these areas, providing an initiative to explore the utilization of causal structures for successful control.
}

\section{Other Applications of Our Framework}\label{sec:j}

While here discussed in the context of the application to cell reprogramming, we envision our framework to be applicable broadly for sequential design problems arising in complex systems.
%
In the following, we provide other applications where we envision that our method could be used to identify optimal interventions, with potentially significantly fewer samples than current methods. 
We focus on contexts where there is an underlying network structure describing the relations among the variables, a setting that is ubiquitous in a wide range of domains.
%
%
In the following, we describe how the structured network emerges in the context of fluid mechanics, dynamic pricing and cancer immunotherapy.
We then provide an example sequential design problem in each domain and discuss how our method can be adapted to make use of the underlying network structure for optimal intervention design.
The goal of this discussion is to showcase how very different applications fit into our proposed framework; once set up, practical considerations can be injected to solve the domain-specific problems.

\smallskip

\noindent\textbf{Fluid mechanics.}
The presence of flow structures and physical interactions among fluid elements, such as a group of vortices, give rise to complex dynamics and turbulence in fluid mechanics \cite{taira2022network}.
These structures and interactions can be modelled via an underlying network structure by discretizing the fluid flow;
for example, the so-called \textit{modal interaction network} describes the interactions among spatial modes of unsteady flows.
This network is often represented through partial differential equations with potentially nonlinear terms and can be obtained through Galerkin projection of the governing equations on the most energetic Fourier modes or principal components \cite{holmes2012turbulence} .

Recent work has explored how to use a directed modal interaction network to control the unsteady wake behind a circular cylinder \cite{nair2018networked}.
In this case, time series from perturbed flow simulations were used to regress and estimate the linear and nonlinear interactions in the network.
Then based on the resulting networked oscillator model, feedback control was implemented to suppress the modal amplitudes that cause wake unsteadiness with drag reduction.
Our framework can be used to identify the optimal interventions by sequentially simulating perturbed flows, updating the parameters of the network model, and optimizing feedback control.
%
In this application, the CIV-class of acquisition functions would be modified for the goal of suppressing wake oscillations, for example, by setting the target mean $\bmu^*$ to a lower bound on the wake oscillations.
The resulting acquisition functions can be computed or approximated similarly as per the derivations presented in this work. 
Given our results, this approach may enable optimal feedback control with a reduced number of perturbed flow simulations.

\smallskip

\noindent\textbf{Dynamic pricing.}
A common assumption in dynamic pricing models is that when an individual adopts a particular product, then customers that are close to this individual (for example with respect to social interactions or direct business relationships) become more likely to adopt the same product. 
The increased benefit with respect to the adoption of a product given by the neighboring customers is known as the \textit{local network benefit} \cite{hill2006network}.
This benefit has been incorporated in studies to devise optimal dynamic development, launch, or post-launch strategies of a product for a network of customers \cite{sunar2019optimal}.
Here the customer network is a connected graph that encodes how the adoption of a product diffuses among customers through time.
Assuming a fully known network model, this work explicitly solves for the optimal development, launch, and post-launch strategies jointly.
An interesting future direction, as mentioned in the discussion section, is to deal with the case where the network model is unknown;
an example being when the firm needs to learn the heterogeneity and parameters of customer-customer interactions from accumulating sales
data~\cite{den2015dynamic}.
Our framework can be applied in this setting to solve for the optimal dynamic pricing strategy in an uncertain environment where the parameters of the customer network need to be estimated.
In this application, an intervention would correspond to a pricing policy and the target mean would be the target revenue.
Using an appropriate posterior model of the unknown parameters in a Bayesian setting, our proposed acquisition functions could provide interesting candidates for designing optimal policies.

\smallskip

\noindent\textbf{Cancer immunotherapy.} 
The development of chimeric antigen receptor (CAR)-T cell therapy is considered a revolution in cancer treatment for its effectiveness in, for example, B cell leukemia and lyphoma \cite{sterner2021car}.
In CAR-T cell therapy, the T cells are extracted from the blood of a patient and engineered in the laboratory so as to induce a more powerful anti-tumor response. Successfully activated cells are then given to the patient by infusion.
Although such treatment has shown great success for certain subsets of blood cancer, it has been challenging to obtain effective responses in solid tumors and hematological malignancies~\cite{sterner2021car}. 

One approach to engineer T cells with desirable properties is through genome editing with lentiviruses \cite{milone2021engineering}.
In this context, the interventions and samples could for example correspond to gene knockouts and single-cell RNA sequencing data.
This is similar to the motivating example of cellular reprogramming considered in this work, where the \textit{gene regulatory network} encodes the relations between individual genes and how a genetic perturbation on one gene affects the expression of other genes.
In CAR-T cell therapy, the desired outcome is better described by a target distribution over cell states (e.g., the enrichment of effector cells and depletion of exhausted cells) rather than the mean expression of a target cell state, which we considered in this work. 
%
%
The calculations of the acquisition function can be adapted accordingly, thereby giving rise to an iterative framework for identifying perturbations that could be prioritized for experimental validation. 


\end{document}